\documentclass[manuscript, screen]{acmart}

\usepackage{tikz} %
\usepackage{graphicx}
\usepackage{xcolor}
\usepackage{colortbl}
\usepackage{stmaryrd}
\usepackage{mathrsfs} %
\usepackage{amsfonts}
\usepackage{cleveref}
\usepackage{upgreek} %
\usepackage{fontawesome5} %
\usepackage{multirow}
\usepackage{pifont}
\usepackage{arydshln} %
\usepackage{subcaption}
\usepackage{rotating}
\usepackage{refcount} %
\usepackage{makecell} %
\usepackage[normalem]{ulem} %
\usepackage{enumerate}

\newtheorem{theorem}{Theorem}[section]
\newtheorem{lemma}[theorem]{Lemma}
\newtheorem{proposition}[theorem]{Proposition}
\newtheorem{fact}[theorem]{Fact}
\newtheorem{corollary}[theorem]{Corollary}
\newtheorem{definition}[theorem]{Definition}
\newtheorem{observation}[theorem]{Observation}

\newtheoremstyle{myexample}%
{3pt}%
{3pt}%
{\rm}%
{}%
{\sc}%
{.}%
{.5em}%
{}%
\theoremstyle{myexample}
\newtheorem{xample}[theorem]{Example}

\usepackage{framed}
\definecolor{shadecolor}{rgb}{0.93,0.93,0.93}
\newenvironment{example}[1][]
{
	\begin{shaded*}	
		\vspace{-1.5ex}
		\begin{xample}[#1]
		}{
		\end{xample}
		\vspace{-1.5ex}
	\end{shaded*}
}

\newcommand{\MC}[1]{\mathcal{#1}}

\newcommand{\Mod}[1]{\llbracket#1\rrbracket}
\newcommand{\I}{\omega}

\newcommand{\K}{\MC{K}}
\newcommand{\G}{\Gamma}

\newcommand{\cupwedge}{\stackrel{\scriptscriptstyle\wedge}{\smash{\cup}}}

\newcommand{\rel}[1]{\prec_{#1}}
\newcommand{\releq}[1]{\preceq_{#1}}
\newcommand{\relc}[1]{\prec^\circ_{#1}}
\newcommand{\releqc}[1]{\preceq^\circ_{#1}}
\newcommand{\relK}{\rel{\K}}
\newcommand{\releqK}{\releq{\K}}
\newcommand{\relcK}{\relc{\K}}
\newcommand{\releqcK}{\releqc{\K}}

\newcommand{\rreleq}[1]{\preceq\hspace{-1.1ex}\preceq_{#1}}

\newcommand{\rreleqc}[1]{\rreleq{#1}^\circ}

\newcommand{\rreleqcK}{\rreleqc{\K}}

\newcommand{\sqreleqc}[1]{\sqsubseteq^\circ_{#1}}

\newcommand{\sqreleqcK}{\sqreleqc{\K}}

\newcommand{\basechange}{base change}

\usepackage{scalerel}
\newcommand{\abst}{{\scaleobj{0.75}{(.)}}}

\newcommand{\criticalloop}{critical loop}
\newcommand{\coverbase}{\G_{\!\triangledown}}
\newcommand{\omegaMinFromCoverbase}{{\I_k^{\star}}}
\newcommand{\omegafromCoverbase}[1]{{\I_{#1}^{\!\triangledown}}}

\newcommand{\definepostulate}[1]{(\textlabel{#1}{pstl:#1})}
\newcommand{\postulate}[1]{(\ref{pstl:#1})}

\newcommand{\setAllDetached}{\ensuremath{\mathfrak{D}^\circ_\K}}

\newcommand{\relationstepdashdash}[1]{\ensuremath{\mathrel{{\releqc{#1}}\!\!''}}}
\newcommand{\relationstepdashK}{\ensuremath{\mathrel{{\releqcK}\!\!'}}}
\newcommand{\relationstepdashdashK}{\ensuremath{\mathrel{{\releqcK}\!\!''}}}
\usetikzlibrary{arrows.meta,decorations.markings}

\makeatletter
\newcommand{\textlabelmarker}[1]{%
	\protected@edef\@currentlabel{#1}%
	\phantomsection%
}
\newcommand{\textlabel}[2]{%
	\textlabelmarker{#1}%
	#1\label{#2}%
}
\makeatother

\newcommand{\sauerwald}[1]{\begingroup\color{black}#1\endgroup}

\def\sqsubsetneq{\mathrel{\sqsubseteq\kern-0.92em\raise-0.15em\hbox{\rotatebox{313}{\scalebox{1.1}[0.75]{\(\shortmid\)}}}\scalebox{0.3}[1]{\ }}}

\definecolor{text-grey}{rgb}{0.45, 0.45, 0.45} %
\definecolor{bg-grey}{rgb}{0.66, 0.65, 0.60} %
\definecolor{fu-blue}{RGB}{0, 51, 102} %
\definecolor{fu-green}{RGB}{153, 204, 0} %
\definecolor{fu-red}{RGB}{204, 0, 0} %
\definecolor{forest-green}{RGB}{34,139,34} %
\definecolor{purple}{RGB}{128,0,128} %

\usetikzlibrary{calc}
\usetikzlibrary{arrows.meta}
\usetikzlibrary{decorations.markings}
\tikzset{balls/.style={style=ball, circle, minimum size=0.8cm, text=black}} 
\tikzset{plainballs/.style={draw,circle, minimum size=0.8cm, text=black}} 
\tikzset{no-outline/.style={circle, minimum size=0.8cm, text=black}} 
\tikzset{
	Stealthnew/.tip={Stealth[width=2mm,length=2mm]},
}

\newcommand{\drawthingiies}[6]{
	\begin{scope}
		\coordinate (A) at (#1);
		\coordinate (B) at (#2);
		\coordinate (C) at (#3);
		\coordinate (D) at (#4);
		
		\coordinate (C1) at ($(A)!.5!(B)$);
		\pgfmathanglebetweenpoints{\pgfpointanchor{A}{center}}{\pgfpointanchor{C1}{center}}
		\let\StartAngle\pgfmathresult
		\pgfmathanglebetweenpoints{\pgfpointanchor{B}{center}}{\pgfpointanchor{C1}{center}}
		\let\EndAngle\pgfmathresult
		
		\coordinate (C2) at ($(C)!.5!(D)$);
		\pgfmathanglebetweenpoints{\pgfpointanchor{C2}{center}}{\pgfpointanchor{C}{center}}
		\let\StartAngleS\pgfmathresult
		\pgfmathanglebetweenpoints{\pgfpointanchor{C2}{center}}{\pgfpointanchor{D}{center}}
		\let\EndAngleS\pgfmathresult

		\draw[#5, #6] (C) let \p1 = ($(C2) - (C)$),
		\n1 = {veclen(\x1, \y1)}
		in
		arc [start angle=\StartAngle, end angle=\EndAngle, radius=\n1]
		-- (A) let \p1 = ($(C1) - (A)$),
		\n1 = {veclen(\x1, \y1)}
		in
		arc [start angle=180+\StartAngleS, end angle=180+\EndAngleS, radius=\n1]
		-- cycle;
	\end{scope}
}

\AtBeginDocument{%
  }

\setcopyright{acmcopyright}
\copyrightyear{20xx}
\acmYear{20xx}
\acmDOI{XXXXXXX.XXXXXXX}

\acmJournal{JACM}
\acmVolume{00} %
\acmNumber{0} %
\acmArticle{000} %
\acmMonth{0} %

\begin{document}

\title{AGM Belief Revision, Semantically}

\author[Falakh]{Faiq Miftakhul Falakh}
\affiliation{%
	\institution{Computational Logic Group, Technische Universität Dresden}
	\city{Dresden}
	\country{Germany}
}
\email{faiq.mif@gmail.com}

\author[Rudolph]{Sebastian Rudolph}
\affiliation{%
  \institution{Computational Logic Group, Technische Universität Dresden}
  \city{Dresden}
  \country{Germany}
}
\email{sebastian.rudolph@tu-dresden.de}

\author[Sauerwald]{Kai Sauerwald}
\affiliation{%
  \institution{Artificial Intelligence Group, FernUniversit\"{a}t in Hagen}
  \city{Hagen}
  \country{Germany}}
\email{kai.sauerwald@fernuni-hagen.de}

\begin{abstract}

We establish a generic, model-theoretic characterization of belief revision operators implementing the paradigm of minimal change according to the seminal work by Alchourr\'{o}n, G\"{a}rdenfors, and Makinson (AGM). 
Our characterization applies to all Tarskian logics, that is, all logics with a classical model-theoretic semantics, and hence a wide variety of formalisms used in knowledge representation and beyond, including many for which a model-theoretic characterization has hitherto been lacking.
Our starting point is the approach by Katsuno and Mendelzon (K\&M), who 
provided %
such a characterization for 
propositional logic over finite signatures.
We generalize K\&M's approach to the setting of AGM-style revision over bases in arbitrary Tarskian logics, where \emph{base} may refer to one of the various ways of representing an agent's beliefs (such as belief sets, arbitrary or finite sets of sentences, or single sentences). 
Our first core result is a representation theorem providing a two-way correspondence between AGM-style revision operators and specific \emph{assignments}: 
functions associating every base to a ``preference'' relation over interpretations, which must be total but is -- in contrast to prior approaches -- not always transitive. 
As our second core contribution, we provide a characterization of all logics for which our result can be strengthened to assignments producing transitive preference relations (as in K\&M's original work).
Alongside these main contributions, we discuss diverse variants of our findings as well as ramifications for other areas of belief revision theory. 

 \end{abstract}



\received{xx Month xxxx}
\received[revised]{xx Month xxxx}
\received[accepted]{xx Month xxxx}

\maketitle

\section{Introduction}\label{sec:introduction}

    The question how a rational agent should change her beliefs in the light of new information is crucial to AI systems.
    The scenario of \emph{belief revision} deals with the case where the new information is considered correct beyond any doubt, whereas the prior beliefs may be erroneous or outdated. 
    Under this assumption, there has been wide agreement that the revision process, which can be conceived as the application of a binary revision operation $\circ$ as per
    \[
    \mathit{prior\ beliefs} \ \circ \ \mathit{new\ information} \ = \  \mathit{resulting\ beliefs}
    \]
    should adhere to the following three foundational principles, ordered by priority:

\begin{itemize}
\item \textbf{Success.} The resulting beliefs should encompass the new information.
\item \textbf{Consistency.} Unless required for success, the resulting beliefs should be free of contradictions. 
\item \textbf{Minimal Change.} Unless in conflict with success and/or consistency, prior beliefs should be preserved.
\end{itemize}  

Another frequent expectation -- although slightly more controversial -- is that of  \textbf{Syntax Independence}: the meaning of the output should only depend on the meaning of its inputs, irrespective of their syntactic representation.

The AGM paradigm of Alchourr\'{o}n, G\"{a}rdenfors, and Makinson \cite{agm_1985} provides a rigid formalization of the notion of belief revision and the above principles in terms of mathematical logic. 
According to this formalization, the new information is expressed by a logical statement, whereas an agent's beliefs are represented by a deductively closed set of such statements, i.e., a set that syntactically contains all the consequences which can be drawn from it. The design decision of using such \emph{belief sets} is appealing from a philosophical and mathematical perspective: Assuming a rational agent, all information inferable from her knowledge should be immediately available to her. Moreover, very conveniently, two belief sets are semantically equivalent if and only if they are syntactically equal. From a computational perspective, however, explicit management of belief sets can be non-trivial: except for extremely simple cases, belief sets will contain infinitely many statements and hence cannot be stored and processed verbatim. Therefore, it has been suggested to represent belief sets implicitly by any of their axiomatizations (i.e., logically equivalent subsets, which may not be deductively closed). Such axiomatizations are typically referred to as \emph{belief bases}. Conversely, it has been argued that confining the new information to single sentences may be too restrictive and should be extended to sets, motivating the definition and investigation of \sauerwald{revision by sets of formulas.}
As one contribution, this paper will introduce the abstract notion of \emph{base logic}, where ``base'' may refer to one of various ways of representing beliefs -- including belief sets, arbitrary or finite sets of sentences, or single sentences (\Cref{sec:bases}). Our results, being established on this abstract level of base logics, will hence readily apply to all these possible instantiations.

\begin{samepage}
For the particular case of propositional logic over a finite signature -- historically the primary application domain and ``testbed'' of belief revision research -- the distinctions discussed above turn out rather immaterial: In such a restricted setting, any (finite or infinite) set of sentences is equivalent to one sentence, thus all belief states and set inputs can be expressed by single sentences without any loss of generality. 
For said setting, Katsuno and Mendelzon {\cite{kat_1991}} -- henceforth abbreviated \emph{K\&M} --  rephrased the AGM conditions into the following concise and elegant set of postulates where $ \varphi,\varphi_1,\varphi_2,\alpha$, and $\beta $ are all propositional sentences, $\models$ denotes logical entailment and $\equiv$ stands for logical equivalence:
\begin{itemize}\setlength{\itemsep}{0pt}\setlength{\itemindent}{2.2ex}
	\item[\definepostulate{KM1}] $ \varphi \circ \alpha \models \alpha$.
	\item[\definepostulate{KM2}] If $\varphi\wedge\alpha$ is consistent, then $\varphi \circ \alpha\equiv \varphi\wedge\alpha$.
	\item[\definepostulate{KM3}] If $\alpha$ is consistent, then $\varphi \circ \alpha$ is consistent.
	\item[\definepostulate{KM4}] If $\varphi_1 \equiv \varphi_2$ and $ \alpha \equiv \beta$, then $ \varphi_1 \circ \alpha \equiv \varphi_2 \circ \beta $.
	\item[\definepostulate{KM5}] $(\varphi \circ \alpha) \wedge \beta \models \varphi \circ (\alpha \wedge \beta)$.
	\item[\definepostulate{KM6}] If $(\varphi \circ \alpha) \wedge \beta$ is consistent, then \mbox{$\varphi \circ (\alpha \wedge \beta) \models (\varphi \circ \alpha) \wedge \beta$.}
\end{itemize}
The postulates \postulate{KM1}--\postulate{KM6} together are equivalent to the \sauerwald{AGM revision postulates \cite{agm_1985}}, under the provision that belief states are represented by single propositional sentences.
Obviously, \postulate{KM1} realizes Success, while \postulate{KM3} enforces Consistency. \postulate{KM2}, \postulate{KM5}, and \postulate{KM6} implement Minimal Change of some sort, whereas \postulate{KM4} ensures Syntax Independence.\footnote{Note that, by virtue of syntax-independence, the semantic content of the revision result is fully determined by the semantic contents of the prior base and the new incoming information; syntactic variations are irrelevant.
This sets the AGM and K\&M line of work apart from other prominent approaches, where revision is performed on a syntactic level and thus the syntactic form of the input may have a semantic effect on the result. A prominent example for such syntactic approaches is base change according to Hansson~\cite{KS_Hansson1999}.}  
\end{samepage}

Given the semantic nature of the AGM postulates, and considering the fact that most logics' semantics is defined by means of model theory,
a natural question is, if AGM revision operators can be naturally characterized in a model-theoretic way. 
A fruitful approach in this regard can be intuitively described as follows: given some prior beliefs $\mathcal{B}$ (associated with a set $\Mod{\mathcal{B}}$ of models through the underlying logic's model-theory), one fixes a ``preference'' relation on all interpretations, telling us for any two interpretations $\I_1$ and $\I_2$, which of the two is ``better suited'' to serve as a model (or if both are equally good). Naturally, the preference relation must be such that the actual models $\I \in \Mod{\mathcal{B}}$ will be considered best, but it may also identify one non-model as more or less suitable than another. With such a relation in place, one could demand that the resulting beliefs' model set can be obtained by picking the ``best suited'' among the models of the new information. This general line of thought, already touched upon in the works of Lewis~\cite{lewis1973} and Grove~\cite{Grove1988}, has been worked out for finite-signature propositional logic in K\&M's seminal paper  {\cite{kat_1991}}, providing a model-theoretic characterization of AGM base revision.\footnote{We would like to point out that, here and in the following, we will use the term ``(knowledge) base revision'' in the sense introduced and used by K\&M, which differs from the usage by Hansson and colleagues.} In this characterization, the idea of ``modelhood preferences'' is formalized by assigning to every base a total preorder over the interpretations. It is then shown that a revision operator satisfies the AGM postulates if and only if such an assignment exists for which the models of the revision result coincide with the preorder-minimal models of the injected information. 

With the advancement of the areas of knowledge representation and intelligent systems, the focus of attention widened to include a variety of logical formalisms, including those commonly known as \emph{ontology languages}, most of which are based on very expressive fragments of first-order predicate logic.
This motivated investigations into tailored characterizations of AGM-style belief change for specific logics, such as Horn logic \cite{KS_DelgrandePeppas2015}, temporal logics \cite{KS_Bonanno2007}, action logics \cite{KS_ShapiroPagnuccoLesperanceLevesque2011}, first-order logic \cite{KS_ZhuangWangWangDelgrande2019}, and description logics \cite{KS_QiLiuBell2006,halaschek-wiener_2006,dong_2017}. What became clear through these investigations is that the immediately plausibile and elegant  characterization by K\&M does not easily generalize beyond its original logical setting. Interestingly, it has been observed that problems arise both when the considered logic is restricted further (e.g., to propositional Horn logic), but also when the expressivity is increased (e.g., to first-order logic).
Examples of attempts toward a more generalized approach include the ones by Ribeiro, Wassermann, and colleagues \cite{KS_RibeiroWassermannFlourisAntoniou2013,KS_Ribeiro2013,KS_RibeiroWassermann2014}, {Delgrande et al.}~\cite{KS_DelgrandePeppasWoltran2018}, {Pardo et al.}~\cite{KS_PardoDellundeGodo2009}, and {Aiguier et al.}~\cite{aiguier_2018}. In the following, we give an overview of the known obstacles toward an immediate generalization of K\&M and sketch strategies how to overcome or avoid them. Alongside we describe the path(s) we take in our paper in pursuit of our goal of producing a most general characterization that applies to all Tarskian logics.

\begin{itemize}
\item \textbf{Lack of Expressibility.} In general, it cannot be taken for granted that every set of models can be exactly singled out by a logical description. While this is the case for finite-signature propositional logic, this pleasant property fails for many other logics. This lack of universal expressivity is almost inevitable for logics with infinitely many interpretations, but it already occurs in much more constrained settings, e.g., for propositional Horn logic. Thus, even if we can, by means of a provided preference relation, determine the set of ``best suited'' models, this model set may be not describable syntactically. In order to avoid this issue, one could either confine one's attention to logics exhibiting universal expressivity \cite{Grove1988,kat_1991} or constrain the admissible preference relations such that they deliver only sets of ``best suited'' models that can be specified by some base of the underlying logic (a condition sometimes referred to as regularity \cite{KS_DelgrandePeppasWoltran2018}). In our approach, aimed at not disregarding any logic, we follow the second strategy, imposing the property of \emph{min-expressibility} on the preference relation (\Cref{sec:abstract_rep}). 
\item \textbf{Non-Existence of Minima.} In case of logics with infinitely many interpretations, we might encounter the peculiar phenomenon that given a set of models, no ``best suited'' model(s) exist(s), because for any model picked, we find one that is ``even better suited''. This kind of preference-unfoundedness needs to be prevented \cite{lewis1973,Grove1988}. Similar to the previous case, this is possible by either restricting one's interest to logics with finitely many interpretations or by imposing a corresponding constraint on the admissible preference relations. Opting for the latter, we require \emph{min-completeness} of the reference relation (\Cref{sec:approachfirst}).
\item \textbf{Transitivity of Preference Relation.} It turns out that, for certain logics, AGM operators exist that cannot be captured by preferences that are total preorders, due to a violation of transitivity. Again, this effect already occurs in the very restricted setting of propositional Horn logic \cite{KS_DelgrandePeppas2015}. While one has to concede that the idea of a  non-transitive preference relation may seem anomalous, such settings are not uncommon, say in game theory or social choice theory. Here, three possible solutions exist: one could confine to logics where said effect does not occur altogether. Alternatively, one could argue that AGM operators giving rise to non-transitive preferences are ``unnatural'' and should be disallowed. This can be achieved by extending the AGM postulates by another postulate (typically denoted \emph{Loop} or \emph{Acyc}) which excludes cyclic preferences \cite{KrausLM90,LehmannMS01,KS_DelgrandePeppas2015,KS_DelgrandePeppasWoltran2018}. A third option would be to again modify the requirements to the preference relation (this time by replacing transitivity by some weaker condition). Indeed, we propose a novel property, called \emph{min-retractivity} (\Cref{sec:approachsecond}), and show that it is the appropriate sought-after condition (\Cref{sec:representation_theorem}). Yet, we will also follow up on the other two options: we will show that the formerly introduced \postulate{Acyc} postulate ensures preference-transitivity also in settings beyond the ones it was originally intended for (\Cref{sec:acycresult}). Finally, we will come up with an exact characterization of those logics, where \postulate{Acyc} holds natively for all AGM operators (\Cref{sec:logics_tpo_representable}). Most notably, we will show that all logics allowing for disjunction of sentences are of that type.
\end{itemize} 

Next to laying the conceptual and technical groundwork as discussed above, we also review diverse variations of the underlying assumtpions of our framework as well as interdependencies between the notions used. These considerations result in further insights into the inner workings of our approach, including the effects of relaxing syntax-independence (\Cref{sec:syntax_independence}), dependency of min-expressibility on the chosen notion of base (\Cref{sec:finite-bases}), characterization of all logics wherein min-friendliness implies transitivity right away (\Cref{sec:disc-minretract}), and the logical strength of a condition called \emph{expressible disjunctive factoring} compared to the postulates (\Cref{sec:disj-factoring}). Moreover, we will specifically look at one core component of any charaterization result: techniques for ``extracting'' a preference relation from a given operator. We will discuss and compare the \sauerwald{encoding} approaches from the literature with ours (\Cref{sec:enc_operators}). We finish by providing a bigger picture on related work, categorized by the imposed assumptions and features of the proposed solution (\Cref{sec:related_works}). 

We point out that a few of the results proven in this paper make use of the axiom of choice, that is, they rely on ZFC set theory rather than ZF. Despite ZFC being accepted by most mathematicians these days, we explicitly indicate whenever the axiom of choice is used. It is needed to show our characterizations in their full generality, while we expect it to be dispensable (or replaceable by a weaker assumption) when showing these results for most concrete logics.

This article provides a full account of the authors' joint research over several years. Parts of the material have been presented in the PhD theses of Kai Sauerwald \cite{Sauerwald22} and Faiq Miftakhul Falakh \cite{Falakh23}, an abridged overview of the results has been published as a conference paper \cite{FalakhRS22}. Next to the full proofs, this article contains a plethora of unpublished material, including a variety of illustrative examples and in-depth discussions on interrelationships between new and established notions. Some of the longer proofs have been moved to the appendix.

\newpage
\section{Preliminaries}\label{sec:prelim}
In this section, we introduce the logical and algebraic notions used in the paper.

\subsection{Tarskian Logics via Classical Model Theory}\label{sec:modeltheory}
We consider 
logics endowed with a classical model-theoretic semantics.
The syntax of such a logic $\mathbb{L}$ is given syntactically  by a (possibly infinite) set $\MC{L}$ of \emph{sentences}, while its model theory is provided by specifying a (potentially infinite) class ${\Omega}$ of \emph{interpretations} (also referred to as \emph{worlds}) and a binary relation $\models$ between ${\Omega}$ and $\MC{L}$ where $\I \models \varphi$ indicates that $\I$ is a model of $\varphi$. Hence, on an abstract level, a logic $\mathbb{L}$ is identified by the triple $(\MC{L},\Omega,\models)$, a representation also known as the logic's \emph{satisfaction system}.
We let $\Mod{\varphi} = \{\I\in {\Omega} \mid \I \models \varphi\}$ denote the set of all models of $\varphi \in \MC{L}$.
Logical entailment is defined as usual via models: for two sentences $\varphi$ and $\psi$ we say $\varphi$ \emph{entails} $\psi$ (written $\varphi \models \psi$) if $\Mod{\varphi} \subseteq \Mod{\psi}$.
Note that this means we overload the symbol ``$\models$'' in the usual way.

Notions of modelhood and entailment can be easily lifted from single sentences to sets.
We obtain the models of a set $\K \subseteq \MC{L}$ of sentences via $\Mod{\K} = \bigcap_{\varphi \in \K} \Mod{\varphi}$. 
For $\K \subseteq \MC{L}$ and $\K' \subseteq \MC{L}$ we say $\K$ \emph{entails} $\K'$ (written $\K \models \K'$) if $\Mod{\K} \subseteq \Mod{\K'}$.
We write $\K \equiv \K'$ to express $\Mod{\K}=\Mod{\K'}$.
A (set of) sentence(s) is called \emph{consistent with} another (set of) sentence(s) if the two have models in common.
Unlike many other belief revision frameworks, we impose no further requirements on $\MC{L}$ (like closure under certain operators or compactness). 

The existence of a classical model-theoretic semantics as above 
is equivalent to the logic being \emph{Tarskian}\footnote{\Cref{sec:app_Tarskian} provides the corresponding formal definition and proofs of this claim.} \cite{KS_Tarski1956,KS_SernadasSernadasCaleiro1997}. This means that all logics considered here satisfy the following conditions:
\begin{align*}
&    \text{If }\varphi \in \MC{K} \text{ then } \MC{K} \models \varphi. \tag{extensivity}\\ 
&	\text{If } \MC{K} \models \varphi  \text{ and } \MC{K} \subseteq \MC{K}'  \text{, then }  \MC{K}' \models \varphi.
	\tag{monotonicity}\\
&	\text{If } \MC{K} \models \MC{K}'  \text{ and } \MC{K}' \models \varphi  \text{, then }  \MC{K} \models \varphi.
\tag{idempotence}
\end{align*}

The notion of Tarskian logic captures many well-known classical logical formalisms and in the following we will provide a few examples.

\subsection{Tarskian Logics: Examples}
We start by providing an example, where sentences and interpretations are finite sets, which allows us to specify them (as well as the $\models$ relation) explicitly.
We note that this is an extension of an example given by Delgrande et al. \cite{KS_DelgrandePeppasWoltran2018}, which will serve as a running example throughout this article.
	
	\begin{example}[based on \cite{KS_DelgrandePeppasWoltran2018}]\label{ex:logicEX}
		Let $ \mathbb{L}_\mathrm{Ex} = (\MC{L}_\mathrm{Ex},\Omega_\mathrm{Ex},\models_\mathrm{Ex}) $ be the logic defined by $ \MC{L}_\mathrm{Ex} = \{\psi_0,\ldots,\psi_5,$ $\varphi_0, \varphi_1,\varphi_2, \chi, \chi' \} $ and $ \Omega_\mathrm{Ex}=\{ \omega_0,\ldots,\omega_5 \} $, with the relation $ \models_\mathrm{Ex} $ implicitly given by:
		$$\begin{array}[h]{r@{\ }ll}
		\Mod{\psi_i} & = \{ \omega_i \}  \\[0.5ex] 
		\Mod{\chi}& = \{ \omega_0, \ldots , \omega_5  \} \\[0.5ex]
		\Mod{\chi'} & = \{ \omega_0, \omega_1, \omega_2, \omega_4, \omega_5 \} \\
		\end{array}\quad\quad\quad\quad
		\begin{array}[h]{r@{\ }ll}
		\Mod{\varphi_0} & = \{ \omega_0,\omega_1 \} \\[0.5ex]
		\Mod{\varphi_1} & = \{ \omega_1,\omega_2 \} \\[0.5ex]
		\Mod{\varphi_2} & = \{ \omega_2,\omega_0 \}  \\
		\end{array}\quad\quad\quad$$
	
		Since $ \mathbb{L}_\mathrm{Ex} $ is defined in the classical model-theoretic way, $ \mathbb{L}_\mathrm{Ex} $ is a Tarskian logic.
		Note that $\mathbb{L}_\mathrm{Ex} $ has no connectives.
		\Cref{fig:logic_ex_1} illustrates the semantics of $ \mathbb{L}_\mathrm{Ex}$.
	\end{example}

\tikzstyle{ball} = [circle,shading=ball, ball color=white!99!white,
minimum size=0.8cm]
\pgfdeclareradialshading{ball}{\pgfpoint{-0.5cm}{0.5cm}}{rgb(0cm)=(1,1,1);rgb(0.6cm)=(0.9,0.9,0.9); rgb(0.82cm)=(0.6,0.6,0.6); rgb(1cm)=(0.3,0.3,0.3); rgb(1.05cm)=(1,1,1)}

\begin{figure}[h!]
	\centering
    \scalebox{1.1}{
	\begin{tikzpicture}[]
	
	\node[plainballs, scale=2] (w3s) at (-3,0.5) {};
	\node[balls, scale=0.75] (w3) at ([shift={(w3s)}] 270:0.1) {};
	\node (psi3) at ([shift={(w3)}] 90:0.55) {\small $\Mod{\psi_3}$};
	\node[] (w3l) at (w3) {\small ${\omega_3}$};
	
	\node[plainballs, scale=2] (w4s) at (-0.5,0.5) {};
	\node[balls, scale=0.75] (w4) at ([shift={(w4s)}] 270:0.1)  {};
	\node[] (w4l) at (w4) {\small ${\omega_4}$};
	
	\node[no-outline, scale=2] (w0s) at (1.5,-1.0) {};
	\node[balls, scale=0.75] (w0) at ([shift={(w0s)}] 270:0.1) {};	
	\node[] (w0l) at (w0) {\small ${\omega_0}$};
	
	\node[no-outline, scale=2] (w0s-150) at ([shift={(w0s)}] 150:0.8) {};
	\node[no-outline, scale=2] (w0s-330) at ([shift={(w0s)}] 330:0.8) {};
	
	\node[no-outline, scale=2] (w1s) at ([shift={(w0s)}] 60:3.5) {};
	\node[balls, scale=0.75] (w1) at ([shift={(w1s)}] 270:0.1) {};
	\node[] (w1l) at (w1) {\small ${\omega_1}$};
	
	\node[no-outline, scale=2] (w1s-150) at ([shift={(w1s)}] 150:0.8) {};
	\node[no-outline, scale=2] (w1s-330) at ([shift={(w1s)}] 330:0.8) {};
	\node[no-outline, scale=2] (w1s-30) at ([shift={(w1s)}] 30:0.8) {};
	\node[no-outline, scale=2] (w1s-210) at ([shift={(w1s)}] 210:0.8) {};
	
	\node[no-outline, scale=2] (w2s) at ([shift={(w1s)}] -60:3.5) {};
	\node[balls, scale=0.75] (w2) at ([shift={(w2s)}] 270:0.1) {};
	\node[] (w2l) at (w2) {\small ${\omega_2}$};
	
	\node[no-outline, scale=2] (w2s-30) at ([shift={(w2s)}] 30:0.8) {};
	\node[no-outline, scale=2] (w2s-210) at ([shift={(w2s)}] 210:0.8) {};
	
	\node[plainballs, scale=2] (w5s) at (7,0.5) {};
	\node[balls, scale=0.75] (w5) at ([shift={(w5s)}] 270:0.1) {};
	\node[] (w5l) at (w5) {\small ${\omega_5}$};

	\drawthingiies{w2s.north}{w2s.south}{w0s.south}{w0s.north}{-}{black};

	\drawthingiies{w1s-150.center}{w1s-330.center}{w0s-330.center}{w0s-150.center}{-}{black};

	\drawthingiies{w2s-30.center}{w2s-210.center}{w1s-210.center}{w1s-30.center}{-}{black};
	
	\node (chi) at (-3.5,2.5) {\small $\Mod{\chi}$};
	\node[draw, rectangle, minimum width=13cm,minimum height=5.8cm, rounded corners=.3cm] (rec2) at (2.25,0.5) {};
	
	\node (phi2) at ([shift={(w0)}] 0:1.75) {\small $\Mod{\varphi_{2}}$};
	\node (phi0) at (2.3,0.5) {\small $\Mod{\varphi_{0}}$};
	\node (phi1) at ([shift={(phi0)}] 0:1.9) {\small $\Mod{\varphi_{1}}$};
	
	\node (psi1) at ([shift={(w1)}] 270:0.55) {\small $\Mod{\psi_1}$};
	\node (psi0) at ([shift={(w0)}] 90:0.55) {\small $\Mod{\psi_0}$};
	\node (psi2) at ([shift={(w2)}] 90:0.55) {\small $\Mod{\psi_2}$};
	
	\node (phi4) at (-1,2.5) {\small $\Mod{\chi'}$};
	\node[draw, rectangle, minimum width=10cm,minimum height=5.2cm, rounded corners=.3cm] (rec1) at (3.25,0.5) {};
	
	\node[] (psi4) at ([shift={(w4)}] 90:0.55) {\small $\Mod{\psi_4}$};
	
	\node (psi5) at ([shift={(w5)}] 90:0.55) {\small $\Mod{\psi_5}$};
	
	\end{tikzpicture}
	}
	\caption{Illustration of the logic $ \mathbb{L}_\mathrm{Ex}$, including the modelhood relations.}
	\label{fig:logic_ex_1}
\end{figure}
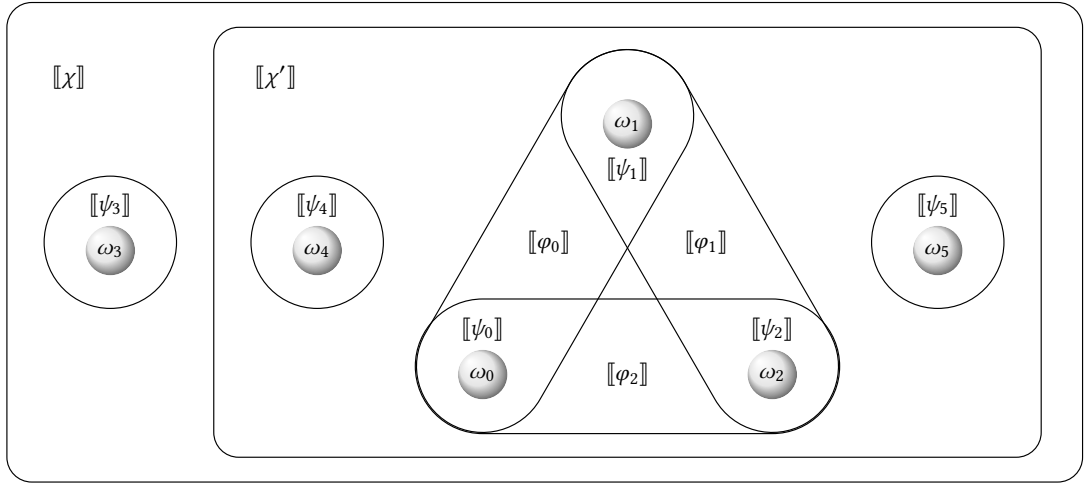

Next we turn to propositional logic, observing that the distinction between the supply of propositional symbols being finite or infinite leads to differences that we will revisit later on.

\begin{example}[$\mathbb{PL}_n$, propositional logic over $n$ propositional atoms]\label{example:PLn}
The logic is defined by $$\mathbb{PL}_n = \left(\MC{L}_{\mathbb{PL}_n},\Omega_{\mathbb{PL}_n},\models_{\mathbb{PL}_n}\right)$$ in the usual way:
Given a finite set $\Sigma = \{p_1,\ldots,p_n\}$ of atomic propositions, we let $\MC{L}_{\mathbb{PL}_n}$ be the set of Boolean expressions built from $\Sigma \cup \{\top,\bot\}$ using the usual set of propositional connectives ($\neg$, $\wedge$, $\vee$, $\to$, and $\leftrightarrow$).
We then let the set $\Omega_{\mathbb{PL}_n}$ of interpretations contain all functions from $\Sigma$ to $\{\textbf{true}, \textbf{false}\}$. The relation $\models_{\mathbb{PL}_n}$ is then defined inductively over the structure of sentences in the usual way.

Notably, finiteness of $\Sigma$ implies finiteness of $\Omega_{\mathbb{PL}_n}$ (more specifically, $|\Omega_{\mathbb{PL}_n}|=2^n$).
This in turn ensures that, despite $\MC{L}_{\mathbb{PL}_n}$ being infinite, there are only finitely many (namely $2^{2^n}$) sentences which are pairwise semantically distinct. Even more so: for every (finite or infinite) set $\MC{K}$ of $\mathbb{PL}_n$ sentences, there exists some sentence $\varphi \in \MC{L}_{\mathbb{PL}_n}$ with $\varphi \equiv \MC{K}$.  
\end{example}

\begin{example}[$\mathbb{H}_n$, propositional Horn logic over $n$ propositional atoms]\label{example:PLHorn}
	This logic, defined by $$\mathbb{H}_n = \left(\MC{L}_{\mathbb{H}_n},\Omega_{\mathbb{H}_n},\models_{\mathbb{H}_n}\right)$$ is obtained from $\mathbb{PL}_n$ by letting $\MC{L}_{\mathbb{H}_n}$ contain only those sentences from $\MC{L}_{\mathbb{PL}_n}$ that are Horn sentences, defined as follows:   	
\begin{enumerate}
	\item $q_1 \wedge q_2 \wedge \ldots \wedge q_{k} \to q_{k+1}$ is a \emph{Horn clause}, where $k \geq 0$ and $\{q_1 \ldots q_{k+1}\} \subseteq \Sigma \cup \{\bot\}$. In case $k=0$, we write  $q_1$ instead of  $\to q_1$ (such particular Horn clauses are also called \emph{facts}). 
	\item If $\varphi$ is a Horn clause, then ($\varphi$) is a \emph{Horn sentence}.
	\item If $\varphi$ and $\psi$ are Horn sentences, then ($\varphi \wedge \psi$) is also a Horn sentence.
\end{enumerate}
    We further let $\Omega_{\mathbb{H}_n} = \Omega_{\mathbb{PL}_n}$ and obtain $\models_{\mathbb{H}_n}$ from $\models_{\mathbb{PL}_n}$ by the appropriate restriction, i.e., ${\models_{\mathbb{H}_n}} =\, {\models_{\mathbb{PL}_n}}\!\! \cap (\Omega_{\mathbb{H}_n} \times \MC{L}_{\mathbb{H}_n})$  
	
	As for $\mathbb{PL}_n$, finiteness of $\Omega_{\mathbb{H}_n}$ ensures that there are only finitely many sentences of $\MC{L}_{\mathbb{H}_n}$ which are pairwise semantically distinct. 
	However, as opposed to $\mathbb{PL}_n$, in $\mathbb{H}_n$, there are sets of world not corresponding to the model set of any sentence or set of sentences. Take for instance $\{\I,\I'\}$ with $\I = \{p_1 \mapsto \textbf{true}, p_2 \mapsto \textbf{false}\}$ and $\I' = \{p_1 \mapsto \textbf{false}, p_2 \mapsto \textbf{true}\}$.
\end{example}

\begin{example}[$\mathbb{PL}_\infty$, propositional logic over infinite signature]\label{example:PLinf}
    The basic definitions for $$ \mathbb{PL}_\infty = \left(\MC{L}_{\mathbb{PL}_\infty},\Omega_{\mathbb{PL}_\infty},\models_{\mathbb{PL}_\infty}\right)$$ are just like for $\mathbb{PL}_n$, with the notable difference of $\Sigma_\mathrm{p} = \{p_1,p_2,\ldots\}$ being countably infinite. 
    This implies immediately that $\Omega_{\mathbb{PL}_\infty}$ is infinite (in fact, even uncountable), implying that there are infinitely many sentences that are pairwise non-equivalent (e.g., all the atomic ones). Also, there exist infinite sets of sentences for which no single equivalent sentence from $\MC{L}_{\mathbb{PL}_\infty}$ exists (e.g., $\{p_2,p_4,p_6,\ldots\}$). 
\end{example}

	In fact, a great variety of formalisms can be conceived as Tarskian logics, even if they're typically not considered to be a logic.
    The following example demonstrates that regular expressions and the corresponding generated formal languages also fit this framework.
\begin{example}[$\mathbb{REG}_\Sigma$, regular expressions over a finite alphabet]\label{example:REGfin}
	Let $ \Sigma $ denote a finite alphabet and let $ \MC{L}_{\mathbb{REG}_\Sigma} $ denote the set of regular expressions\footnote{Recall that the set of regular expressions over $ \Sigma $ is the smallest set $ \MC{L}_{\mathbb{REG}_\Sigma} $ such that $ \Sigma\subseteq \MC{L}_{\mathbb{REG}_\Sigma}  $ and $ \emptyset,\varepsilon\in \MC{L}_{\mathbb{REG}_\Sigma} $, and if $ \alpha,\beta\in \MC{L}_{\mathbb{REG}_\Sigma} $, then is also $ \alpha\beta, \alpha+\beta, \alpha^* \in \MC{L}_{\mathbb{REG}_\Sigma} $. 
		The formal language $ L(\alpha) $ generated by $ \alpha $ is defined inductively over the structure of $ \alpha $: we have $ L(\emptyset)=\emptyset $, we have $ L(\varepsilon)=\{\varepsilon\} $ (empty word), for $ \sigma\in\Sigma $ we have $ L(\sigma)=\{\sigma\} $, we have  $ L(\alpha\beta)= \{wv \mid w\in L(\alpha), v \in  L(\beta)\}$ as well as $ L(\alpha + \beta)=L(\alpha) \cup L(\beta) $, and $ L(\alpha^*)=\{\varepsilon\} \cup L(\alpha) \cup L(\alpha\alpha) \cup \ldots$ (Kleene star).} over $ \Sigma $.
	For a regular expression $ \alpha $ we denote in the following with $ L(\alpha) $ the language of finite words generated by $ \alpha $. The \emph{logic of regular expressions} is then defined as follows:
	\begin{equation*}
		\mathbb{REG}_\Sigma = \left(\MC{L}_{\mathbb{REG}_\Sigma},\Omega_{\mathbb{REG}_\Sigma},\models_{\mathbb{REG}_\Sigma}\right)\ ,
	\end{equation*}
where $ \Omega_{\mathbb{REG}_\Sigma} = \Sigma^*$, that is, the interpretations are all finite words \sauerwald{over $ \Sigma $. For a finite} word $ w\in \Omega_{\mathbb{REG}_\Sigma} $ and a regular expression $ \alpha\in \MC{L}_{\mathbb{REG}_\Sigma} $ we let $ w \models _{\mathbb{REG}_\Sigma} \alpha $ if and only if $ w\in L(\alpha) $. More concisely expressed, we define \( \Mod{\alpha}=L(\alpha) \).
The triple $ \mathbb{REG}_\Sigma $  is indeed a Tarskian logic. 
Interestingly, the ``logic'' $\mathbb{REG}_\Sigma $ is not compact.{\footnotemark} 
 To see this, define for each \( n\in\mathbb{N} \) the regular expression \( \delta_n = \sigma^n\sigma^* = \overbrace{\smash{\sigma}\ldots\smash{\sigma}}^{\raisebox{-0ex}{$\smash{{}\atop{n%
 		}}$}}\sigma^*\), expressing the language of finite words which consist only of the letter \( \sigma \) and contains at least \( n \) letters. Now, observe that the infinite set of formulas \( \Delta=\{ \delta_n \mid n\in \mathbb{N} \} \) has no models, i.e., \( \Mod{\Delta}=\bigcap_{\delta_n\in\Delta} L(\delta_n) \) is the empty language. Yet, for every finite subset \( \Delta' \subseteq \Delta \), the set \( \Mod{\Delta'}=\bigcap_{\delta_n\in\Delta'} L(\delta_n) \) is \sauerwald{a non-empty language}.
\end{example}
\footnotetext{A logic is (strongly) compact if for every set of formulas $ T $ holds: if every finite subset of $ T $ has a model then $ T $ has a model \cite{Barwise1985-BARML-8}.}

Many more logics (and logic-like formalisms) of greatly varying expressivity are captured by the model-theoretic framework assumed by us, e.g. variants and fragments of first-order and second-order predicate logic, modal logics, and description logics. These include highly expressive, yet decidable logics that are popular in various fields such as knowledge representation or verification with a wide range of practical applications.      
Our considerations do, however, \textbf{not} apply to non-monotonic formalisms, such as default logic, circumscription, or logic programming %
frameworks using 
negation as failure.

\subsection{Relation over Interpretations}\label{sec:relations}
For describing belief revision on the semantic level, it is expedient to endow the interpretation space ${\Omega}$ with some structure. 
In particular, we will employ binary relations $\preceq$ over ${\Omega}$ (formally: ${\preceq} \subseteq \Omega \times \Omega$), where the intuitive meaning of $\I_1\preceq\I_2$ is that $\I_1$ is ``equally good or better'' than $\I_2$ when it comes to serving as a model.  
We call $\preceq$ \emph{total} if  $\I_1\preceq\I_2$ or $\I_2\preceq\I_1$ for any $\I_1,\I_2 \in {\Omega}$ holds. 
We write $\I_1\prec\I_2$ as a shorthand, whenever $\I_1 \preceq \I_2$ and $\I_2 \not\preceq \I_1$ (the intuition being that $\I_1$ is ``strictly better'' than $\I_2$).
For a selection $\Omega' \subseteq {\Omega}$ of interpretations, an $\I \in \Omega'$ is called \emph{$\preceq$-minimal in $\Omega'$} if $\I \preceq \I'$ for all $\I'\in\Omega'$.\footnote{If $\preceq$ is total, this definition is equivalent to the \emph{absence} of any $\I'' \in \Omega'$ with $\I'' \prec \I$.}
We let $\min(\Omega',\preceq)$ denote the set of $\preceq$-minimal interpretations in $\Omega'$. 
We call $\preceq$ a \emph{preorder} if it is transitive and reflexive. 
For a relation \( R \subseteq \Omega\times\Omega \), the transitive closure  of \( R \) is the relation \( TC(R) = \bigcup_{i=0}^\infty R^{i} \), where \( R^{0}=R \) and \( R^{i+1} = R^{i} \cup \{ (\omega_1,\omega_3) \mid \exists \omega_2. (\omega_1,\omega_2)\in R^{i} \text{ and } (\omega_2,\omega_3)\in R^{i} \} \).

\subsection{Bases}\label{sec:bases}
This article addresses the revision of and by \emph{bases}. 
In the belief revision community, the term of base commonly denotes an arbitrary (possibly infinite) set of sentences.
However, in certain scenarios, other assumptions might be more appropriate.
Hence, for the sake of generality, we 
decided to define the notion of a base on an abstract level with minimal requirements (just as we introduced our notion of \emph{logic}), allowing for its instantiation in many ways.

\newcommand{\Bases}{\mathfrak{B}}

\begin{definition}
	A \emph{base logic} is a quintuple $\mathbb{B} = (\MC{L},\Omega,\models,\Bases,\Cup),$ where
	\begin{itemize}
		\item
		$(\MC{L},\Omega,\models)$ is a logic,
		\item
		$\Bases \subseteq \mathcal{P}(\MC{L})$ is a family of sets of sentences, called \emph{bases}, and
		\item
		$\Cup: \Bases \times \Bases \to \Bases$ is a binary operator over bases, called the \emph{abstract union}, satisfying $\Mod{\mathcal{B}_1 \Cup \mathcal{B}_2} = \Mod{\mathcal{B}_1} \cap \Mod{\mathcal{B}_2}$.    
	\end{itemize} 
\end{definition}

Next, we will demonstrate how, for some logic $\mathbb{L}=(\MC{L},\Omega,\models)$, a corresponding base logic can be chosen depending on one's preferences regarding what bases should be.

\paragraph{Arbitrary Sets.}
If all (finite and infinite) sets of sentences should qualify as bases, one can simply set $\Bases = \mathcal{P}(\MC{L})$. In that case, $\Cup$ can be instantiated by set union $\cup$, then the claimed behavior follows by definition. Given a logic $\mathbb{L}=(\MC{L},\Omega,\models)$, we denote the corresponding arbitrary-set base logic $(\MC{L},\Omega,\models,\mathcal{P}(\MC{L}),\cup)$ by $\mathbb{L}^\mathrm{arb}$.

\paragraph{Finite Sets.}
In some settings, it is more convenient to assume bases to be finite (e.g. when computational properties or implementations are to be investigated).
In such cases, one can set $\Bases = \mathcal{P}_\mathrm{fin}(\MC{L})$, i.e., all (and only) the finite sets of sentences are bases. Again, $\Cup$ can be instantiated by set union $\cup$ (as a union of two finite sets will still be finite). Given a logic $\mathbb{L}=(\MC{L},\Omega,\models)$, we denote the corresponding finite-set base logic $(\MC{L},\Omega,\models,\mathcal{P}_\mathrm{fin}(\MC{L}),\cup)$ by $\mathbb{L}^\mathrm{fin}$.

\paragraph{Belief Sets.}
This setting is closer to the original framework, where the ``knowledge states'' to be modified were assumed to be deductively closed sets of sentences.
We can capture such situations by accordingly letting $\Bases = \{ \MC{B}\subseteq \MC{L} \mid \forall \varphi \in \mathcal{L}: \MC{B}\models \varphi \Rightarrow \varphi \in \MC{B}\} 
$. In this case, the abstract union operator needs to be defined
via $\mathcal{B}_1 \Cup \mathcal{B}_2 = \{\varphi \in \MC{L} \mid \MC{B}_1 \cup \MC{B}_2 \models \varphi\}$. Given a logic $\mathbb{L}$, we let $\mathbb{L}^\mathrm{bel}$ denote the corresponding belief-set base logic just described.

\paragraph{Single Sentences.}
In this popular setting, one prefers to operate on single sentences only (rather than on proper collections of those).
For this to work properly, an additional assumption needs to be made about the underlying logic $\mathbb{L}=(\MC{L},\Omega,\models)$: it must be possible to express conjunction on a sentence level, either through the explicit presence of the Boolean operator $\wedge$ or by some other means.
Formally, we say that $\mathbb{L}=(\MC{L},\Omega,\models)$ \emph{supports conjunction}, if for any two sentences $\varphi, \psi \in \MC{L}$ there exists some sentence in $\MC{L}$, denoted $\varphi\owedge\psi$,  satisfying $\Mod{\varphi\owedge\psi} = \Mod{\varphi} \cap \Mod{\psi}$.   
For such a logic, we can ``implement'' the single-sentence setting by letting  $\Bases = \binom{\MC{L}}{1} = \{ \{\varphi\} \mid \varphi \in \mathcal{L} \}$ and defining $\{\varphi\} \Cup \{\psi\} = \{ \varphi\owedge\psi \}$. Obviously, if $\wedge$ is available within the logic $\mathbb{L}$ (at least on the sentence level), we would simply define $\varphi\owedge\psi$ to be $\varphi \wedge \psi$. In that case, we let $\mathbb{L}^\mathrm{sng}$ the single-sentence base logic $(\MC{L},\Omega,\models,\binom{\MC{L}}{1},\cupwedge)$, where $\cupwedge$ denotes element-wise conjunction, i.e., $\{\varphi\} \cupwedge \{\psi\} =  \{\varphi \wedge \psi\}$. In the single sentence setting, one would typically refrain from writing bases as singleton sets and instead just assume $\Bases = \MC{L}$.  

\medskip
For any of the four different notions of bases, one can additionally choose to disallow or allow the empty set as a base, while maintaining the required closure under abstract union. 

\begin{example}\label{ex2:regexp}
For a somewhat interesting case consider the ``regular expression logic'' $\mathbb{REG}_\Sigma$ from \Cref{example:REGfin}. If we want to define its single-sentence base logic $\mathbb{REG}_\Sigma^\mathrm{sng}$, we would need to make sure that for any two regular expressions $\alpha$ and $\beta$ over $\Sigma$, there exists a regular expression $\alpha \owedge \beta = \gamma$ satisfying $L(\gamma) = L(\alpha) \cap L(\beta)$. Luckily, this is indeed the case, since the intersection of two regular languages is well known to be regular. However, as the syntax of regular expressions does not allow for explicit intersection, finding the resulting $\gamma$ may be rather costly: in general the time for computing it and its size cannot be avoided to be exponential in the combined sizes of $\alpha$ and $\beta$ \cite{GeladeN08}. So $\mathbb{REG}_\Sigma^\mathrm{sng}$ is a case for a base logic where the abstract union $\Cup$ is far from trivial.
\end{example}

In the following, we will always operate on the abstract level of ``base logics''; our notions, results and proofs will only make use of the few general properties specified for these. This guarantees that our results are generically applicable to any of the four described 
(and any other)
instantiations, and hence, are independent of the question what the right notion of bases ought to be. 
The cognitive overload caused by this abstraction should be minimal; e.g., readers only interested in the case of arbitrary sets can safely assume $\Bases = \mathcal{P}(\MC{L})$ and mentally replace any $\Cup$ by $\cup$.

\subsection{Change Operators for Bases}\label{sec:base_change_operator}
In this article, we use \basechange\ operators to model revision, which is the process of incorporating new beliefs into the present beliefs held by an agent, in a consistent way (whenever that is possible). 
We define change operators over a base logic as follows.
\begin{definition}
	Let $ \mathbb{B} = (\MC{L},\Omega,\models,\Bases,\Cup) $ be a base logic. A \emph{{\basechange} operator for $ \mathbb{B} $} is a function $ \circ: \Bases \times \Bases \to \Bases $.
\end{definition}

Note that this definition differs from the commonly encountered belief change setting where a belief state from $\Bases$ is revised by just a single formula from $\MC{L}$, producing a new belief state from $\Bases$ as result. It should be clear that this simpler setting can be simulated by our more general framework; one just needs to require that $\Bases$ comprises all singleton sets $\{\varphi\}$ for any $\varphi \in \MC{L}$.

We will use \basechange\ operators in the ``standard'' way of the belief change community: the first parameter represents the actual beliefs of an agent, the second parameter contains the new beliefs. The operator then yields the agent's revised beliefs.

So far, the pure notion of \basechange\ operator is unconstrained and can be instantiated by an arbitrary binary function over bases.
Obviously, this does not reflect the requirements or expectations one might have when speaking of a revision operator. 
Hence, in line with the traditional approach, we will consider additional constraints (\sauerwald{called} ``postulates'') for \basechange\ operators, in order to capture the gist of revisions.

\subsection{Postulates for AGM-Style Base Revision}
	
As discussed, we consider ``revision by bases'', \sauerwald{also referred to as \emph{multiple revision} \cite{KS_FermeHansson2018}}. Thereby, we focus \sauerwald{on \emph{package semantics}}, meaning that all given sentences have to be incorporated, that is, given a base $ \K $ and new information $ \G $ (also a base here), we demand success of revision \sauerwald{for all elements from \( \G \)}, i.e., $ \K \circ \G \models \G $.
	
Besides this success condition, the belief change community has brought up and discussed several further requirements for belief change operators to make them \emph{rational}, for summaries see, e.g., \cite{KS_Hansson1999,KS_FermeHansson2018}.
This has led to the now famous AGM approach of revision \cite{agm_1985}, originally proposed through a set of rationality postulates, which correspond to the postulates \postulate{KM1}--\postulate{KM6} by K\&M presented in the introduction.
In our article, we will make use of the K\&M version of the AGM postulates adjusted to our generic notion of a base logic $ \mathbb{B} = (\MC{L},\Omega,\models,\Bases,\Cup) $: %

\begin{itemize}\setlength{\itemsep}{0pt}
	\item[\definepostulate{G1}] $\K \circ \G \models \G$.
	\item[\definepostulate{G2}] If $\Mod{\K\Cup\G} \neq \emptyset$ then $\K \circ \G\equiv \K\Cup \G$.
	\item[\definepostulate{G3}] If $\Mod{\G}\neq\emptyset$ then $\Mod{\K \circ \G}\neq\emptyset$.
	\item[\definepostulate{G4}] If $\K_1 \equiv \K_2$ and $\G_1 \equiv \G_2$ then $\K_1 \circ \G_1 \equiv \K_2 \circ \G_2$.
	\item[\definepostulate{G5}] $(\K \circ \G_1)\Cup \G_2 \models  \K \circ (\G_1\Cup \G_2)$.
	\item[\definepostulate{G6}] If $\Mod{(\K \circ \G_1)\Cup \G_2}\neq\emptyset$ then $ \K \circ (\G_1\Cup \G_2) \models (\K \circ \G_1)\Cup \G_2$.
\end{itemize}
Together, the postulates implement the paradigm of minimal change, stating that a rational agent should change her beliefs as little as possible in the process of belief revision.\pagebreak[3]
We consider the postulates in more detail:
	\postulate{G1} guarantees that the newly added belief must be a logical consequence of the result of the revision.
	\postulate{G2} says that if the expansion of $\K$ by $\G$ is consistent, then the result of the revision is equivalent to the expansion of $\K$ by $\G$.
	\postulate{G3} guarantees the consistency of the revision result if the newly added belief is consistent.
	\postulate{G4} is the principle of the irrelevance of the syntax, stating that the revision operation is independent of the syntactic form of the bases.
	\postulate{G5} and \postulate{G6} ensure more careful handling of (abstract) unions of belief bases. 
		In particular, together, they enforce that \( \K \circ (\G_1\Cup \G_2) \equiv (\K \circ \G_1)\Cup \G_2 \), unless \( \G_2 \) contradicts \( \K \circ \G_1 \).

We can see that, item by item, \postulate{G1}--\postulate{G6} tightly correspond to \postulate{KM1}--\postulate{KM6} presented in the introduction. 
Note also that further formulations similar to \postulate{G1}--\postulate{G6} are given in multiple particular contexts, e.g. in the context of belief base revision specifically for Description Logics  \cite{KS_QiLiuBell2006}, for parallel revision \cite{KS_DelgrandeJin2012} and investigations on multiple revision \cite{KS_Zhang1996,KS_Peppas2004,KS_Kern-IsbernerHuvermann2017}.
An advantage of the specific form of the postulates \postulate{G1}--\postulate{G6} chosen for our presentation is that it does not require $\MC{L}$ to support conjunction (while, of course, conjunction on the \sauerwald{level of bases is still implicitly supported via abstract} union of bases).

\section{AGM-Style Revision of Bases in Propositional Logic}\label{sec:km}
A well-known and by now popular characterization of base revision has been described by Katsuno and Mendelzon {\cite{kat_1991}} for the special case of propositional logic. %
To be more specific and apply our terminology, K\&M's approach applies to the base logic
\begingroup\abovedisplayskip=2.5pt
\belowdisplayskip=2.5pt
\begin{equation*}
    \mathbb{PL}_n^\mathrm{sng} = %
    \big(\MC{L}_{\mathbb{PL}_n},\Omega_{\mathbb{PL}_n},\models_{\mathbb{PL}_n},\{\{\varphi\} \mid \varphi \in \MC{L}_{\mathbb{PL}_n}\},\cupwedge \big)
\end{equation*}
\endgroup
for arbitrary, but fixed $n$ (cf. \Cref{example:PLn}). 
The assumption of the finiteness on the underlying signature of atomic propositions is not overtly explicit in K\&M's paper, but it becomes apparent upon investigating their arguments and proofs -- 
in fact, 
as we 
we will see shortly, their characterization fails as soon as this assumption is dropped.
K\&M's approach also exploits other particularities of this setting: 
As discussed earlier, any propositional belief base $\K$ (even any infinite one) can be equivalently written as a single propositional sentence.  
Consequently, representing belief bases by single sentences comes without a loss of generality and the distinction between $\mathbb{PL}_n^\mathrm{arb}$, $\mathbb{PL}_n^\mathrm{fin}$, and $\mathbb{PL}_n^\mathrm{sng}$ is immaterial in terms of expressivity, which justifies to restrict one's attention to $\mathbb{PL}_n^\mathrm{sng}$. 

The key contribution of K\&M is to provide an alternative characterization of the propositional base revision operators satisfying \postulate{KM1}--\postulate{KM6} by model-theoretic means, i.e. through comparisons between propositional interpretations. 
In the following, we present their results in a formulation that facilitates later generalization. 
One central notion for the characterization is the notion of \emph{faithful assignment}.

\begin{samepage}
\begin{definition}[assignment, faithful]\label{def:faithful}
	Let $ \mathbb{B} = (\MC{L},\Omega,\models,\Bases,\Cup) $ be a base logic.
	An \emph{assignment} for $\mathbb{B}$ is a function $\releq{\abst} : \Bases \to \mathcal{P}({\Omega} \times {\Omega})$ that assigns to each belief base $\K \in \Bases $ a total binary relation $\releqK$ over ${\Omega}$.
	An assignment $\releq{\abst}$ for \( \mathbb{B} \) is called \emph{faithful} if it satisfies the following conditions for all \( \I,\I'\in\Omega \) and all \( \K,\K'\in\Bases \):
	\begin{itemize}\setlength{\itemsep}{0pt}
		\item[\definepostulate{F1}] If $\I,\I' \models \K$, then $\I \relK \I'$ does not hold.
		\item[\definepostulate{F2}] If $\I\models \K$ and $\I'\not\models \K$, then $\I \relK \I'$. 
		\item[\definepostulate{F3}] If $\K\equiv\K'$, then ${\releqK} = \releq{\MC{K'}}$.
	\end{itemize}
	An assignment $\releq{\abst}$ is called a \emph{preorder assignment} if $\releqK$ is a preorder for every $\K \in \Bases$. 	
\end{definition}
\end{samepage}

Note that this definition slightly deviates from that by K\&M, who require assignments to produce preorders right away. In other words, assignments according to K\&M coincide with preorder assignments in our framework. This adaptation of terminology is deliberate in order to facilitate the presentation of our results.  

Intuitively, faithful assignments provide information about which of the two interpretations is ``closer to $\K$-modelhood''. 
Consequently, the actual $\K$-models are $\releqK$-minimal. The next definition captures the idea of an assignment adequately representing the behaviour of a revision operator. 

\begin{definition}[compatible]\label{def:compatible}
	Let $ \mathbb{B} = (\MC{L},\Omega,\models,\Bases,\Cup) $ a base logic.
	A \basechange{} operator $\circ$ for \( \mathbb{B} \) is called \emph{compatible} with some assignment $\releq{\abst}$ for \( \mathbb{B} \) if it satisfies  
	$$
	\Mod{\K\circ{\G}} = \min\big(\Mod{\G},\releqK\big) \label{eq:KM}
	$$
	for all bases $\K$ and ${\G}$ from $\Bases$.
\end{definition}
Compatibility will be treated as a symmetric notion, so saying \( \circ \) is compatible with $\releq{\abst}$ means the same as saying $\releq{\abst}$ is compatible with \( \circ \) or simply saying that $\releq{\abst}$ and \( \circ \) are compatible.

\medskip
With the above notions in place, K\&M's representation result can be 
smoothly expressed in the following way:

\begin{theorem}[Katsuno and Mendelzon \cite{kat_1991}]\label{thm:km1991}
	A \basechange{} operator $\circ$ for \( \mathbb{PL}_n^\mathrm{sng} \) satisfies \postulate{G1}--\postulate{G6} if and only \sauerwald{if \( \circ \) is compatible} with some faithful preorder assignment for \( \mathbb{PL}_n^\mathrm{sng} \).
\end{theorem}

Observe that, despite the ``if and only if'', this theorem rests on the precondition that the \basechange{} operator $\circ$ exists and is given. In and by itself, it does \textbf{not} guarantee that for every faithful preorder assignment, a compatible operator can be found. Therefore, we will refer to this type of theorem as \emph{one-way theorems}. However, as mentioned before, \( \mathbb{PL}_n^\mathrm{sng} \) has the advantageous property that for any set $\Omega'$ of interpretations, there is a base (i.e., a singular propositional sentence) whose model set is exactly $\Omega'$. It is easy to see that under such circumstances, the corresponding stronger \emph{two-way theorem} also follows:      

\pagebreak[3]
\begin{corollary}[Two-Way Representation Theorem]
	The following statements hold:
	\begin{itemize}
		\item Every \basechange{} operator $\circ$ for \( \mathbb{PL}_n^\mathrm{sng} \) that satisfies \postulate{G1}--\postulate{G6} is compatible with a faithful preorder assignment for \( \mathbb{PL}_n^\mathrm{sng} \).
		\item Every faithful preorder assignment for \( \mathbb{PL}_n^\mathrm{sng} \) is compatible with a \basechange{} operator $\circ$ for \( \mathbb{PL}_n^\mathrm{sng} \) that satisfies \postulate{G1}--\postulate{G6}.
	\end{itemize}
\end{corollary}

While the distinction between the one-way and two-way statement is immaterial for \( \mathbb{PL}_n^\mathrm{sng} \), it will turn out very important in our upcoming considerations, as these two readings are far from being guaranteed to coincide in general (cf. \Cref{sec:abstract_rep}).

In the next section, we discuss and provide a generalization of the overall approach to the setting of arbitrary base logics.

\section{Generalization to Arbitrary Tarskian Base Logics}\label{sec:approach}
	In this section, we prepare our main result by revisiting K\&M's concepts for propositional logic and investigating their suitability for our general setting of base logics.
	The result by Katsuno and Mendelzon established an elegant combination of the notions of preorder assignments, faithfulness, and compatibility in order to semantically characterize \sauerwald{base change operators that satisfy \postulate{G1}--\postulate{G6}}.
	However, as we mentioned before and will make more precise in the following,
K\&M's characterization hinges on features of signature-finite propositional logic that do not generally hold for Tarskian logics.
	So far, attempts to find similar formulations for less restrictive logics have made good progress for understanding the nature of AGM revision (cf.~\Cref{sec:related_works}). Here we go further, by extending K\&M's approach by novel notions to the very general setting of base logics.

\subsection{First Problem: Non-Existence of Minima}\label{sec:approachfirst}
The first issue with K\&M's original characterization when generalizing to arbitrary base logics is the possible absence of $\preceq_\K$-minimal elements in $\Mod{\Gamma}$. 
\begin{observation}\label{obs:incomplete}
	For arbitrary base logics, the minimum from \Cref{eq:KM}, required in \Cref{thm:km1991}, might be empty. 
\end{observation}

\newcommand{\plinfty}{\mathbb{PL}_\infty}
One way this might happen is due to infinite descending $\preceq_\K$-chains of interpretations.
To illustrate this problem (and to show that it arises even for propositional logic, if the signature is infinite), consider the base logic 
\begin{equation*}
\plinfty^\mathrm{sng} = \left(\MC{L}_{\plinfty},\Omega_{\plinfty},\models_{\plinfty},\{ \{\varphi\} \mid \varphi \in \MC{L}_{\plinfty}\},\cupwedge\right),
\end{equation*}
i.e., propositional logic with single-sentence bases, but countably infinitely many distinct atomic propositions $\Sigma=\{p_1,p_2,\ldots\}$ (cf. \Cref{example:PLinf}).
We will exhibit a base change operator that is compatible with a faithful preorder assignment, yet does violate one of the postulates, due to the problem mentioned above.

\begin{example}\label{ex:infsigfail}
We define \sauerwald{a base change operator $\circ^{\cupwedge}$ for \( \plinfty^\mathrm{sng} \) by simply letting $\K\circ^{\cupwedge}\G=\K\cupwedge \G$. 
Obviously, $\circ^{\cupwedge}$ violates \postulate{G3}: when picking $\K = \{p_1\}$ and $\G = \{\neg p_1\}$, we obtain \( \Mod{\K \circ^{\cupwedge} \G} = \emptyset \).} 
Nevertheless, for this operator, a compatible assignment exists, as we will show next.
Assume a base $\K$ and two propositional interpretations $\I_1,\I_2: \Sigma \to \{\mathbf{true},\mathbf{false}\}$.
Let $\I_k^\mathbf{true}$ denote $\{p_i \in \Sigma \mid \I_k(p_i)=\mathbf{true}\}$ for $k\in\{1,2\}$, i.e., the set of atomic symbols that $\I_k$ maps to $\mathbf{true}$.
Then we let $\I_1 \releqK^{\cupwedge} \I_2$ if at least one of the following is the case:
\pagebreak[3]
\begin{itemize}\setlength{\itemsep}{0pt}
	\item[(1)] $\I_1 \models \K$
	\item[(2)] $\I_2 \not\models \K$ and $\I_2^\mathbf{true}$ is infinite
	\item[(3)] $\I_1,\I_2 \not\models \K$, both $\I_1^\mathbf{true}$ and $\I_2^\mathbf{true}$ are finite, and $|\I_1^\mathbf{true}| \geq |\I_2^\mathbf{true}|$ 
\end{itemize}
This definition provides a faithful preorder assignment compatible with $\circ^{\cupwedge}$ (see \Cref{prop:appendix_plinfty} in \Cref{sec:app2} for the proof).  
Yet, $\circ^{\cupwedge}$ violates \postulate{G3} despite being compatible with the faithful preorder assignment $\releq{(.)}^{\cupwedge}$.
\end{example}

To remedy the problem exposed above, one needs to impose the requirement that minima exist whenever needed, as specified in the notion of \textit{min-completeness}, defined next.

\begin{definition}[min-complete]  
	Let $ \mathbb{B} = (\MC{L},\Omega,\models,\Bases,\Cup) $ be a base logic.
	A binary relation $\preceq$ over $ {\Omega} $ is called \emph{min-complete} (for $\mathbb{B}$) if $\min(\Mod{\G},\preceq) \not= \emptyset$ holds for every $\G \in \Bases$ with $\Mod{\G} \not= \emptyset$.
\end{definition}

The following example demonstrates that for a binary relation it depends on the base logic whether the relation is min-complete or not.

	\begin{example}
Consider two base logics \( \mathbb{B}_{\mathbb{Z}\leq} \) and \( \mathbb{B}_{\mathbb{Z}\geq} \) with
	\begin{align*}
		\mathbb{B}_{\mathbb{Z}\leq} & = \big(\MC{L}_1,\,\mathbb{Z},\,\models,\, \MC{P}_\mathrm{fin}(\MC{L}_1) {\,\setminus\,} \{\emptyset\},\, \cup\,\big), \text{ and} \\
		\mathbb{B}_{\mathbb{Z}\geq} & = \big(\MC{L}_2,\,\mathbb{Z},\,\models,\, \MC{P}_\mathrm{fin}(\MC{L}_2) {\,\setminus\,} \{\emptyset\},\, \cup\,\big),
	\end{align*}
	where $\MC{L}_1 = \{ [{\leq}n] \mid n\in \mathbb{Z} \}$ and $\MC{L}_2 = \{ [{\geq}n] \mid n\in \mathbb{Z} \}$. Furthermore let $m \models [{\leq}n]$ if $m \leq n$ and $m \models [{\geq}n]$ if $n \leq m$, assuming the usual meaning of $\leq$ for integers.
	In words, these logics talk about the domain of integers by means of comparisons with a fixed integer.
	We now define the relation $\releq{}$ over $\Omega$ by letting $m_1 \releq{} m_2$ if and only if $m_1 \leq m_2$.
	It can be verified that the relation is transitive and for any consistent base $\Gamma\in\MC{P}_\mathrm{fin}(\MC{L}_1)$, respectively for \( \Gamma \in \MC{P}_\mathrm{fin}(\MC{L}_2) \), we have infinitely many models $\Mod{\Gamma}$. 
	
	Note that for each set of sentences of the form \( [{\leq}n] \in \MC{L}_1 \), there are no minimal models $\min(\Mod{\G},\releq{})$, and thus, $\releq{}$ is \emph{not} min-complete for \( \mathbb{B}_{\mathbb{Z}\leq} \).
	However, for \( \mathbb{B}_{\mathbb{Z}\geq} \), the relation $\releq{}$ is min-complete.
\end{example}

In the special case of $\preceq$ being transitive and total, min-completeness trivially holds whenever ${\Omega}$ is finite (as, e.g., in the case of propositional logic over $n$ propositional atoms; cf. \Cref{example:PLn}). In the infinite case, however, it might need to be explicitly imposed, as already noted in earlier works \cite{KS_DelgrandePeppasWoltran2018} (cf. also the notion of \emph{limit assumption} by Lewis \cite{lewis1973}).
Note that min-completeness does not entirely disallow infinite descending chains (as well-foundedness would), it only ensures that minima exist inside all model sets of consistent belief bases.

\subsection{Second Problem: Transitivity of Preorder}\label{sec:approachsecond}
When generalizing from the setting of propositional to arbitrary base logics, the requirement that assignments must produce preorders (and hence transitive relations) turns out to be too restrictive.

\begin{observation}
\sauerwald{For arbitrary Tarskian logics, demanding transitivity of the relation produced by the assignment, as required in~\Cref{thm:km1991}, is too strict for characterizing base change operators that satisfy \postulate{G1}--\postulate{G6}.}
\end{observation}

In fact, it has been observed before that the incompatibility between transitivity and K\&M's approach already arises for propositional Horn logic \cite{KS_DelgrandePeppas2015}. 
The following example builds on \Cref{ex:logicEX} and provides an operator and a belief base for which no compatible transitive assignment exists.

\begin{example}[continuation of \Cref{ex:logicEX}]\label{ex:orderByEx1}
		Consider the base logic $ \mathbb{B_\mathrm{Ex}} \sauerwald{ = \mathbb{L}_\mathrm{Ex}^\mathrm{arb} } = (\MC{L}_\mathrm{Ex},\Omega_\mathrm{Ex},\models_\mathrm{Ex}$, $\MC{P}(\MC{L}_\mathrm{Ex}), \cup)$.
	Let $ \K_\mathrm{Ex} = \{ \psi_3 \} $ and let $ \circ_\mathrm{Ex} $ be the \basechange\ operator defined as follows:
	\begin{equation*}
		\K_\mathrm{Ex}\circ_\mathrm{Ex}\G = \begin{cases}
			\K_\mathrm{Ex} \cup \G            & \hspace{-1ex}\text{if } \Mod{\K_\mathrm{Ex} \cup \G} \neq \emptyset\text{, }                                                                \\
			\G \cup \{\psi_4\}    & \hspace{-1ex}\text{if } \Mod{\K_\mathrm{Ex} \cup \G} = \emptyset \text{ and } \Mod{\{\psi_4\} \cup \G} \neq \emptyset,                                               \\ 
			\G \cup \{ \psi_0 \}  & \hspace{-1ex}\text{if } \Mod{\K_\mathrm{Ex} \cup \G} = \emptyset \text{ and } \Mod{\{\psi_0\} \cup \G} \neq \emptyset \text{ and }\Mod{\{\psi_2\} \cup \G} \,{=}\, \emptyset,  \\
			\G \cup \{ \psi_1 \}  & \hspace{-1ex}\text{if } \Mod{\K_\mathrm{Ex} \cup \G} = \emptyset \text{ and } \Mod{\{ \psi_1 \} \cup \G} \neq \emptyset \text{ and }\Mod{\{\psi_0\} \cup \G} \,{=}\, \emptyset,  \\
			\G \cup \{ \psi_2 \}  & \hspace{-1ex}\text{if } \Mod{\K_\mathrm{Ex} \cup \G} = \emptyset \text{ and } \Mod{\{ \psi_2 \} \cup \G} \neq \emptyset \text{ and }\Mod{\{\psi_1\} \cup \G} \,{=}\, \emptyset,  \\
			\G                    & \hspace{-1ex}\text{if none of the above applies,}
		\end{cases}
	\end{equation*}
	Moreover, for all  $ \K' $ with $ \K'\equiv\K_\mathrm{Ex} $ we define $ \K'\circ_\mathrm{Ex}\G = \K_\mathrm{Ex}\circ_\mathrm{Ex}\G $ and for all $ \K' $ with $ \K'\not\equiv\K_\mathrm{Ex} $ we define
	\begin{equation*}
		\K'\circ_\mathrm{Ex}\G=\begin{cases}
			\K'\cup \G & \text{ if } \K'\cup \G \text{ consistent} \\
			\G & \text{ otherwise.}
		\end{cases}
	\end{equation*}
For all $ \K' $ with $ \K'\not\equiv\K_\mathrm{Ex} $, there is no violation of the postulates \postulate{G1}--\postulate{G6} since we obtain a trivial revision, which satisfies \postulate{G1}--\postulate{G6} (cf. \Cref{ex:full-meet}).
For the case of $ \K'\equiv\K_\mathrm{Ex} $, the satisfaction of \postulate{G1}--\postulate{G6} can be shown case by case or using \Cref{thm:rep1} in \Cref{sec:abstract_rep}. 

	Now assume there were a preorder assignment $\releq{\abst}$ compatible with $\circ_\mathrm{Ex}$.
	This means that for all bases $\K$ and $\G$ from $\MC{P}(\MC{L}_\mathrm{Ex})$, the relation $\releqK$ is a preorder and $\Mod{\K\circ_\mathrm{Ex}\G} = \min(\Mod{\Gamma}, \releq{\K_\mathrm{Ex}})$.
	Now consider $\G_0=\{\varphi_0 \}$, $\G_1=\{\varphi_1\}$, and $\G_2=\{\varphi_2 \}$.
	From the definition of $\circ_\mathrm{Ex}$ and compatibility, we obtain 
	$\Mod{\K_\mathrm{Ex}\circ_\mathrm{Ex}\G_0} = \{\I_0\} = \min(\Mod{\Gamma_0}, \releq{\K_\mathrm{Ex}})$,
	$\Mod{\K_\mathrm{Ex}\circ_\mathrm{Ex}\G_1} = \{\I_1\} = \min(\Mod{\Gamma_1}, \releq{\K_\mathrm{Ex}})$, and
	$\Mod{\K_\mathrm{Ex}\circ_\mathrm{Ex}\G_2} = \{\I_2\} = \min(\Mod{\Gamma_2}, \releq{\K_\mathrm{Ex}})$.
	Recall that $\Mod{\G_0}=\{\I_0,\I_1 \}$, $\Mod{\G_1}=\{\I_1,\I_2 \}$, and $\Mod{\G_2}=\{\I_2,\I_0 \}$. Yet, this implies $\I_0\rel{\K_\mathrm{Ex}}\I_1$, $\I_1\rel{\K_\mathrm{Ex}}\I_2$, and $\I_2\rel{\K_\mathrm{Ex}}\I_0$, contradicting the assumption that $\releq{\K_\mathrm{Ex}}$ is transitive. Hence it cannot be a preorder.
\end{example}

As a consequence, we cannot help but waive transitivity (and hence the property of the assignment providing a preorder) if we want our characterization result to hold for all Tarskian logics.
However, for our result, we need to retain a new, weaker property (which is implied by transitivity) defined next.

\begin{definition}[min-retractive]
	Let $ \mathbb{B} = (\MC{L},\Omega,\models,\Bases,\Cup) $ be a base logic.
	A binary relation $\preceq$ over $ {\Omega} $ is called \emph{min-retractive} (for $\mathbb{B}$) if, for every $\G \in \Bases$ and $\I',\I \in \Mod{\G}$ with $\I'\preceq\I$, $\I\in \min(\Mod{\G},\preceq)$  implies  $\I'\in \min(\Mod{\G},\preceq)$.
\end{definition}

Note that min-retractivity prevents minimal elements from being $\preceq$-equivalent to elements with $\prec$-lower neighbours, for instance elements lying on a ``$\prec$-cycle'' or elements being part of an infinite descending chain.
Consider the following illustrative example.
	
\begin{figure}
	\centering
	\hfill
	\begin{subfigure}[t]{0.49\textwidth}\centering
		\begin{tikzpicture}[]			
			\begin{scope}[xshift=-5cm]
				\node[balls, inner sep=0.1, scale=0.75] (w1) at (4,-0.5) {\Large${\omega_0}$};
				\node[balls, inner sep=0.1, scale=0.75] (w2) at (6,-0.5) {\Large${\omega_2}$};
				\node[balls, inner sep=0.1, scale=0.75] (w3) at (5,0.5) {\Large${\omega_1}$};
				\node[draw, dashed, rectangle, minimum width=3.3cm,minimum height=2.3cm, rounded corners=.2cm] (rec1) at (5,0) {};
			\end{scope}
			
			\node[balls, inner sep=0.1, scale=0.75] (w0) at (-3.5,0) {\Large${\omega_3}$};

			\draw[-Stealthnew] (w0) edge[bend left=10] node[above, scale=1] {$\preceq_1^{\mathrm{mr}} $} ([yshift=.15cm]rec1.west) ;
			\path[-Stealthnew] ([yshift=-.15cm]rec1.west) edge[bend left=10] node[below] {$\preceq_1^{\mathrm{mr}} $} (w0) ;
			
			\path[Stealthnew-] (w3) edge node[left,  near start, scale=0.75] {$\preceq_1^{\mathrm{mr}} $} (w1) ;
			\path[Stealthnew-] (w2) edge node[right,  near end, scale=0.75] {$\preceq_1^{\mathrm{mr}} $} (w3) ;
			\path[Stealthnew-] (w1) edge node[below, scale=0.75] {$\preceq_1^{\mathrm{mr}} $} (w2) ;
		\end{tikzpicture}
		\caption{Not min-retractive relation \( \preceq_1^{\mathrm{mr}} \) for \( \mathbb{B}_\mathrm{mr} \).}\label{fig:min-retractive-circle_1}
	\end{subfigure}
	\hfill
	\begin{subfigure}[t]{0.49\textwidth}\centering
		\begin{tikzpicture}[]
			\begin{scope}[xshift=-5cm]
				\node[balls, inner sep=0.1, scale=0.75] (w1) at (4,-0.5) {\Large${\omega_0}$};
				\node[balls, inner sep=0.1, scale=0.75] (w2) at (6,-0.5) {\Large${\omega_2}$};
				\node[balls, inner sep=0.1, scale=0.75] (w3) at (5,0.5) {\Large${\omega_1}$};
				\node[draw, dashed, rectangle, minimum width=3.3cm,minimum height=2.3cm, rounded corners=.2cm] (rec1) at (5,0) {};
			\end{scope}
			
			\node[balls, inner sep=0.1, scale=0.75] (w0) at (-3.5,0) {\Large${\omega_3}$};

			\path[Stealthnew-] (rec1) edge node[above, scale=1] {$ \preceq_2^{\mathrm{mr}} $} (w0) ;
			
			\path[Stealthnew-] (w3) edge node[left,  near start, scale=0.75] {$\preceq_2^{\mathrm{mr}} $} (w1) ;
			\path[Stealthnew-] (w2) edge node[right,  near end, scale=0.75] {$\preceq_2^{\mathrm{mr}} $} (w3) ;
			\path[Stealthnew-] (w1) edge node[below, scale=0.75] {$\preceq_2^{\mathrm{mr}} $} (w2) ;
		\end{tikzpicture}
		\caption{Min-retractive relation \( \preceq_2^{\mathrm{mr}} \) for \( \mathbb{B}_\mathrm{mr} \).}\label{fig:min-retractive-circle_2}
	\end{subfigure}
	\hfill
	\caption{Illustration of the two relations \( \preceq_1^{\mathrm{mr}} \) and \( \preceq_2^{\mathrm{mr}} \) from  \Cref{ex:min-retractive-circle}.}
	\label{fig:min-retractive-circle}
\end{figure}
	
\begin{example}\label{ex:min-retractive-circle}
	Let $ \mathbb{B}_{\mathrm{mr}} = (\MC{L},\Omega,\models,\Bases,\Cup) $ be a base logic with just one base \( \Bases=\{ \G_{\mathrm{mr}} \}  \)  and four interpretations \( \Omega=\{ \I_0,\I_1,\I_2,\I_3 \} \)
such that \( \Mod{\G_{\mathrm{mr}}}=\Omega \).
Now consider the following total relation \( \preceq_1^{\mathrm{mr}} \) on \( \Omega \) illustrated in \Cref{fig:min-retractive-circle_1} and given by 
\begin{equation*}
	\begin{array}[h]{r@{\ }ll}
		\omega_i & \releq{\K_\mathrm{Ex}}^{\circ_\mathrm{Ex}} \omega_i, & 0\leq i\ \leq 3,\\[1ex] 
		\omega_3 & \rel{\K_\mathrm{Ex}}^{\circ_\mathrm{Ex}} \omega_i, & 0\leq i\ \leq 2,\\[1ex] 
		\omega_i & \rel{\K_\mathrm{Ex}}^{\circ_\mathrm{Ex}} \omega_3, & 0\leq i\ \leq 2,
	\end{array}\quad\quad\quad\quad
	\begin{array}[h]{r@{\ }ll}
		\omega_0 & \rel{\K_\mathrm{Ex}}^{\circ_\mathrm{Ex}} \omega_1, & \\[1ex] 
		\omega_1 & \rel{\K_\mathrm{Ex}}^{\circ_\mathrm{Ex}} \omega_2, & \\[1ex] 
		\omega_2 & \rel{\K_\mathrm{Ex}}^{\circ_\mathrm{Ex}} \omega_0. &  
	\end{array}
\end{equation*}
We show that \( \preceq_1^{\mathrm{mr}} \) is not min-retractive for \( \mathbb{B}_{\mathrm{mr}} \).
The \( \preceq_1^{\mathrm{mr}} \)-minimal models of \( \G_{\mathrm{mr}} \) are given by \( {\min(\Mod{\G_{\mathrm{mr}}},\preceq_1^{\mathrm{mr}})} = {\{ \omega_3 \}} \).
Observe that \( \omega_0 \) is a non-minimal model  of \( \G_{\mathrm{mr}} \) while being \( \preceq_1^{\mathrm{mr}} \)-equivalent to \( \omega_3 \), and in particular \( \omega_0 \preceq_1^{\mathrm{mr}} \omega_3 \).
This is a violation of min-retractivity.

Let \( \preceq_2^{\mathrm{mr}} \) be the same relation as \( \preceq_1^{\mathrm{mr}} \), except that \( \preceq_2^{\mathrm{mr}} \) strictly prefers \( \omega_3 \) over all over interpretations, i.e., \( {\preceq_2^{\mathrm{mr}}} = {\preceq_1^{\mathrm{mr}}} \setminus\, \{ (\omega,\omega_3) \mid \omega \neq \omega_3 \} \). 
An illustration of \( \preceq_2^{\mathrm{mr}} \) is given in \Cref{fig:min-retractive-circle_2}.
Indeed, we have that \( \preceq_2^{\mathrm{mr}} \) is min-retractive for \( \mathbb{B}_{\mathrm{mr}} \).
In particular, observe that the prior counterexample for \( \preceq_1^{\mathrm{mr}} \) does not apply to \( \preceq_2^{\mathrm{mr}} \), as we have \( \omega_0 \not\preceq_2^{\mathrm{mr}} \omega_3 \).
\end{example}

\sauerwald{Coming back to min-completeness,}
let us note that, if $\preceq$ is total but not transitive, min-completeness can be violated even in the setting where $\Omega$ is finite, by means of strict cyclic relationships. 

\newcommand{\Hrock}{\text{\footnotesize\faHandRock[regular]}}
\newcommand{\Hpaper}{\text{\footnotesize\faHandPaper[regular]}}
\newcommand{\Hscissors}{\text{\footnotesize\faHandScissors[regular]}}
\newcommand{\allthree}{\mbox{\rm\textsc{a\hspace{-1pt}l\hspace{-1pt}l\hspace{-1pt}-\hspace{-1pt}t\hspace{-1pt}h\hspace{-1pt}r\hspace{-1pt}e\hspace{-1pt}e}}}

\begin{example}
		Let $ \mathbb{B}_{\mathrm{rps}} = (\MC{L},\Omega,\models, \MC{P}(\MC{L}), \cup) $ be the base logic defined by $ \MC{L}=\{\allthree\} $ and 
		$ \Omega=\{ \Hrock, \Hpaper, \Hscissors \} $, 
		with the models relation $\models$ given by $\Mod{\allthree} =$ $\Omega$.
		We now define the relation $\releq{}^{\mathrm{rps}}$ as the common 
		game ``rock-paper-scissors'': paper beats rock ($\Hpaper \prec^{\mathrm{rps}} \Hrock$), scissors beat paper ($\Hscissors \prec^{\mathrm{rps}} \Hpaper$), and rock beats scissors ($\Hrock \prec^{\mathrm{rps}} \Hscissors$). 
		Clearly, the set of interpretations $\Omega$ is finite and the relation $\preceq^{\mathrm{rps}}$ is total, but not transitive. It is, however vacuously min-retractive. 
		By considering a consistent base $\Gamma$ containing the only sentence $\allthree$, we find that $\min(\Mod{\G},\preceq^{\mathrm{rps}}) = \emptyset$, and hence a violation of min-completeness.
	\end{example}

As a last act in this section, we conveniently unite the two identified properties into one notion. 

\begin{definition}[min-friendly]
	Let $ \mathbb{B} = (\MC{L},\Omega,\models,\Bases,\Cup) $ be a base logic.
	A binary relation $\preceq$ over $ {\Omega} $ is called \emph{min-friendly} (for $\mathbb{B}$) if it is both min-retractive and min-complete. An assignment $\releq{\abst} : \Bases \to \mathcal{P}({\Omega} \times {\Omega})$ is called min-friendly if $\releqK$ is min-friendly for all $\K \in \Bases$.
\end{definition} 

\section{One-Way Representation Theorem}\label{sec:representation_theorem}
We are now ready to generalize K\&M's representation theorem from propositional to \sauerwald{arbitrary base logics}, by employing the notion of compatible min-friendly faithful
	assignments. 
	
\begin{theorem}\label{thm:representation_theorem}
\sauerwald{    Let \( \mathbb{B} \) be a base logic. 
        A base change operator \( \circ \) for 
         \( \mathbb{B} \) satisfies \postulate{G1}--\postulate{G6} if and only if \( \circ \) is compatible with some min-friendly faithful assignment for \( \mathbb{B} \).%
        }
\end{theorem}

We show \Cref{thm:representation_theorem} in three steps. First, we  provide a canonical way of obtaining an assignment for a given revision operator.
Next, we show that our construction indeed yields a min-friendly faithful assignment that is compatible with the revision operator. 
Finally, we show that the notion of min-friendly compatible assignment is adequate to capture the class of {\basechange} operators satisfying \postulate{G1}--\postulate{G6}. 

\subsection{From Postulates to Assignments}\label{sec:postulatestoassignments}

	Very central for the original result by Katsuno and Mendelzon \cite{kat_1991} is a constructive way to obtain the assignment from a revision operator.
	In their proof for \Cref{thm:km1991}, they
	provided the following way of extracting the preference relations from the revision operator:
	\begin{equation}\label{eq:km_encoding}
		\I_1 \leq_\K \I_2 \text{ if } \I_1 \models \K \text{ or } \I_1 \models \K\circ \mathit{form}(\I_1,\I_2)
	\end{equation}
	where $ \mathit{form}(\I_1,\I_2)\in \MC{L} $ denotes a sentence with $ \Mod{\mathit{form}(\I_1,\I_2)}=\{ \I_1,\I_2 \} $. 
	Unfortunately, this method for obtaining a canonical encoding of the revision strategy of $ \circ $ does not generalize to the general setting here.
	This is because a belief base \( \Gamma \) satisfying \( \Mod{\Gamma}= \{ \I_1,\I_2 \} \) may not exist.

 As a recourse, we suggest the following construction, which we consider one of this article's core contributions. 
It realizes the idea that one should (strictly) prefer \( \omega_1 \) over \( \omega_2 \) only if there is a witness belief base \( \G \) that certifies that \( \circ \) prefers \( \omega_1 \) over \( 
 	\omega_2 \). Should no such witness exist, \( \omega_1 \) and \( \omega_2 \) will be deemed equally preferable.

\begin{definition}\label{def:relation_new_first_relation}
    Let $ \mathbb{B} = (\MC{L},\Omega,\models,\Bases,\Cup) $ be a base logic,
    let $\circ$ be a \basechange{} operator for \( \mathbb{B} \)
    and let $\K \in \Bases $ be a belief base. 
                    The relation $ \sqreleqcK $ over  ${\Omega} $ is defined by	
                    \begin{equation*}
                                \I_1 \sqreleqcK \I_2 \text{\ \ if \ }
                                \I_2 \models \K\circ\G \text{ implies } \I_1 \models  \K\circ\G \text{ for all } \G \in \Bases \text{ with } \I_1,\I_2\in \Mod{\G}.
                        \end{equation*}
\end{definition}
 	
\Cref{def:relation_new_first_relation} already yields an adequate encoding strategy for many base logics.
However, to also properly cope with certain ``degenerate'' base logics, we have to hard-code that the prior beliefs of an agent are prioritized in all cases, that is, only models of the prior beliefs are minimal. 
In \Cref{sec:enc_operators} we will analyze this in more detail. 
The following relation builds upon the relation $ \sqreleqcK $ and takes explicit care of handling prior beliefs.
     
\begin{definition}\label{def:relation_new}
	Let $ \mathbb{B} = (\MC{L},\Omega,\models,\Bases,\Cup) $ be a base logic,
	let $\circ$ be a \basechange{} operator for \( \mathbb{B} \)
	and let $\K \in \Bases $ be a belief base. 
        The relation \( \releqcK \)  over  ${\Omega} $ is then defined by
    \begin{equation*}
    \I_1 \releqcK \I_2 \text{\ \ if \ } \I_1 \models \K \text{ or } \left(\ \I_1,\I_2 \not\models \K \text{ and } \I_1 \sqreleqcK \I_2 \ \right).
\end{equation*}
	Let $\releqc{\abst} : \Bases \to \mathcal{P}({\Omega} \times {\Omega})$ denote the mapping $\K \mapsto {\releqcK}$. 
\end{definition}

In the following, we apply the relation encoding given in \Cref{def:relation_new} to our running example and show that the relation is not transitive, yet min-friendly.

	\begin{example}[continuation of \Cref{ex:orderByEx1}]\label{ex:orderByEx2}
		Applying \Cref{def:relation_new} to $ \K_\mathrm{Ex} $ and $ \circ_\mathrm{Ex} $ yields the following relation $ \releq{\K_\mathrm{Ex}}^{\circ_\mathrm{Ex}} $ on $ \Omega_\mathrm{Ex} $ 
		(where $ \omega \rel{\K_\mathrm{Ex}}^{\circ_\mathrm{Ex}} \omega' $ denotes 
		$ \omega\releq{\K_\mathrm{Ex}}^{\circ_\mathrm{Ex}} \omega' $ and 
		$ \omega' \not\releq{\K_\mathrm{Ex}}^{\circ_\mathrm{Ex}} \omega $):%
		$$\begin{array}[h]{r@{\ }ll}
		\omega_i & \releq{\K_\mathrm{Ex}}^{\circ_\mathrm{Ex}} \omega_i, & 0\leq i\ \leq 5\\[0.5ex] 
		\omega_3 & \rel{\K_\mathrm{Ex}}^{\circ_\mathrm{Ex}} \omega_i, & i \in \{0,1,2,4, 5\}\\[0.5ex]
        & & \\
		\end{array}\quad\quad\quad
		\begin{array}[h]{r@{\ }ll}
		\omega_0 & \rel{\K_\mathrm{Ex}}^{\circ_\mathrm{Ex}} \omega_1 & \\[0.5ex]
		\omega_1 & \rel{\K_\mathrm{Ex}}^{\circ_\mathrm{Ex}} \omega_2 & \\[0.5ex]
		\omega_2 & \rel{\K_\mathrm{Ex}}^{\circ_\mathrm{Ex}} \omega_0& \\
		\end{array}\quad\quad
		\begin{array}[h]{r@{\ }ll}
		\omega_4 & \rel{\K_\mathrm{Ex}}^{\circ_\mathrm{Ex}} \omega_i, & i \in \{0,1,2,5\}\\[0.5ex]
		\omega_i & \rel{\K_\mathrm{Ex}}^{\circ_\mathrm{Ex}} \omega_5, & 0\leq i\ < 4\\[0.5ex]
        & & \\		
        \end{array}$$
		
		Observe that $ \releq{\K_\mathrm{Ex}}^{\circ_\mathrm{Ex}} $ is not transitive, since $ \omega_0,\omega_1,\omega_2 $ form a $\rel{\K_\mathrm{Ex}}^{\circ_\mathrm{Ex}}$-cycle (see \Cref{fig:non-trans-rel}). Yet, one can easily verify that $ \releq{\K_\mathrm{Ex}}^{\circ_\mathrm{Ex}} $ is a total and min-friendly relation.
		In particular, as $ \Omega_\mathrm{Ex} $ is finite, min-completeness can be checked by examining minimal model sets of all consistent bases in $\mathbb{L}_\mathrm{Ex}$.
		Moreover, there is no belief base $ \G \in \MC{P}(\MC{L}_\mathrm{Ex}) $ 
		such that there is some $ \omega\notin {\min(\Mod{\G},\releq{\K_\mathrm{Ex}}^{\circ_\mathrm{Ex}})} $ and $ \omega' \in \min(\Mod{\G},\releq{\K_\mathrm{Ex}}^{\circ_\mathrm{Ex}})  $ with $ \omega \releq{\K_\mathrm{Ex}}^{\circ_\mathrm{Ex}} \omega' $. 
		Note that such a situation could appear in $ \releq{\K_\mathrm{Ex}}^{\circ_\mathrm{Ex}} $ if an interpretation $ \omega $ would be $ \releq{\K_\mathrm{Ex}}^{\circ_\mathrm{Ex}} $-equivalent to $ \omega_0$, $ \omega_1$ and $ \omega_2 $ and there would be a belief base $ \G $ satisfied in all these interpretations, e.g., if $ \omega=\omega_5 $ would be equal to $ \omega_0,\omega_1 $ and $ \omega_2 $, and  $ \Mod{\G}=\{ \omega_0,\omega_1,\omega_2,\omega_5 \} $.  However, this is not the case in $ \releq{\K_\mathrm{Ex}}^{\circ_\mathrm{Ex}} $ and such a  belief base $ \G $ does not exist in $ \mathbb{B}_\mathrm{Ex} $.
		Therefore, the relation $ \releq{\K_\mathrm{Ex}}^{\circ_\mathrm{Ex}} $ is min-retractive.
	\end{example}

\tikzstyle{ball} = [circle,shading=ball, ball color=white!90!black,
minimum size=0.8cm]

\begin{figure}[t]
	\centering
	\begin{tikzpicture}[]
	
	\node[balls, inner sep=0.1, scale=0.75] (w0) at (-2.0,0) {\Large ${\omega_3}$};
	\node[balls, inner sep=0.1, scale=0.75] (w4) at (1.0,0) {\Large${\omega_4}$};
	\node[balls, inner sep=0.1, scale=0.75] (w1) at (4,-0.5) {\Large${\omega_0}$};
	\node[balls, inner sep=0.1, scale=0.75] (w2) at (6,-0.5) {\Large${\omega_2}$};
	\node[balls, inner sep=0.1, scale=0.75] (w3) at (5,0.5) {\Large${\omega_1}$};
	\node[balls, inner sep=0.1, scale=0.75] (w5) at (8.5,0) {\Large ${\omega_5}$};
	
	\node[draw, dashed, rectangle, minimum width=3.3cm,minimum height=2.3cm, rounded corners=.2cm] (rec1) at (5,0) {};
	
	\path[-Stealthnew] (w0) edge node[above, scale=1] {$\rel{\K_\mathrm{Ex}}^{\circ_\mathrm{Ex}} $} (w4) ;
	\path[-Stealthnew] (w4) edge node[above, scale=1] {$\rel{\K_\mathrm{Ex}}^{\circ_\mathrm{Ex}} $} (rec1) ;
	\path[-Stealthnew] (rec1) edge node[above, scale=1] {$\rel{\K_\mathrm{Ex}}^{\circ_\mathrm{Ex}} $} (w5) ;
	
	\path[-Stealthnew] (w0) edge[bend left=25] node[above, near end, scale=1] {$\rel{\K_\mathrm{Ex}}^{\circ_\mathrm{Ex}} $} (rec1) ;
	\path[-Stealthnew] (w0) edge[bend left=45] node[above, scale=1] {$\rel{\K_\mathrm{Ex}}^{\circ_\mathrm{Ex}} $} (w5) ;
	\path[-Stealthnew] (w4) edge[bend right=45] node[below, near end, scale=1] {$\rel{\K_\mathrm{Ex}}^{\circ_\mathrm{Ex}} $} (w5) ;
	
	\path[-Stealthnew] (w1) edge node[left,  near end, scale=0.75] {$\rel{\K_\mathrm{Ex}}^{\circ_\mathrm{Ex}}\ $} (w3) ;
	\path[-Stealthnew] (w3) edge node[right,  near start, scale=0.75] {$\rel{\K_\mathrm{Ex}}^{\circ_\mathrm{Ex}} $} (w2) ;
	\path[-Stealthnew] (w2) edge node[below, scale=0.75, yscale=1] {$\rel{\K_\mathrm{Ex}}^{\circ_\mathrm{Ex}} $} (w1) ;

	\end{tikzpicture}
	\caption{The structure of relation $ \releqK^{\circ_\mathrm{Ex}} $ on $ \Omega_\mathrm{Ex} $, where a solid arrow represents $ \omega \relK^{\circ_\mathrm{Ex}} \omega' $ for any $\omega, \omega' \in  \Omega_\mathrm{Ex} $.}
	\label{fig:non-trans-rel}
\end{figure}

As a first insight, we obtain that
the construction in \Cref{def:relation_new} is strong enough for always obtaining a relation that is total and reflexive. 

\begin{lemma}[totality]\label{lem:totality}
	If $ \circ $ satisfies \postulate{G5} and \postulate{G6}, the relations $ \releqcK $ and $ \sqreleqcK $ are total (and hence reflexive) for every $\K \in \Bases $. 
\end{lemma}
\begin{proof}
    Note that by construction, totality of \( \releqcK \) is an immediate consequence of totality of $ \sqreleqcK $.  We show the latter by contradiction: Assume the contrary, i.e. there are $ \sqreleqcK $-incomparable $ \I_1 $ and $ \I_2 $.
	Due to \Cref{def:relation_new_first_relation}, there must exist $ \G_1,\G_2 \in \Bases $ with $ \I_1,\I_2\models\G_1 $  and $ \I_1,\I_2\models\G_2 $, such that $ \I_1 \models \K\circ\G_1 $ and $ \I_2 \not\models \K\circ\G_1 $ whereas $ \I_1 \not\models \K\circ\G_2 $ and $ \I_2 \models \K\circ\G_2 $. %
	Since $\I_1 \in \Mod{\K\circ\G_1} \cap \Mod{\G_2} = \Mod{(\K\circ\G_1)\Cup\G_2}$ and thus $\Mod{(\K\circ\G_1)\Cup\G_2}\neq\emptyset$, \postulate{G5} and \postulate{G6} jointly entail 
	$\Mod{(\K\circ\G_1)\Cup\G_2} = \Mod{\K\circ(\G_1\Cup\G_2)}$.
	From commutativity of $\Cup$, $\Mod{\K\circ(\G_1\Cup\G_2)} = \Mod{\K\circ(\G_2\Cup\G_1)}$ follows.
	Now again, since $\I_2 \in \Mod{\K\circ\G_2} \cap \Mod{\G_1} = \Mod{(\K\circ\G_2)\Cup\G_1}$ and hence $\Mod{(\K\circ\G_2)\Cup\G_1}\neq\emptyset$, 
	\postulate{G5} and \postulate{G6} together entail $\Mod{\K\circ(\G_2\Cup\G_1)} = \Mod{(\K\circ\G_2)\Cup\G_1}$.
    So, together, we obtain $\I_1 \in \Mod{(\K\circ\G_2)\Cup\G_1} = \Mod{\K\circ\G_2} \cap \Mod{\G_1}$ which directly contradicts our assumption $\I_1\not\in\Mod{\K\circ\G_2}$.
    
	Reflexivity follows immediately from totality.
\end{proof}

We proceed with an auxiliary lemma about belief bases and $ \releqcK $. 
\begin{lemma}\label{lem:help}
	Let $ \circ $ satisfy \postulate{G2}, \postulate{G5} and \postulate{G6} and let $ \K \in \Bases$. Then the following hold: 
	\begin{enumerate}\setlength{\itemsep}{0pt}
		\item[(a)] If $ \I_1 \not\releqcK \I_2 $ and \( \I_2 \not\models \K \), then 
		there exists some $ \G $ with $ \I_1,\I_2\models\G $ as well as $ \I_2 \models \K\circ\G $ and $ \I_1 \not\models \K\circ\G $.
		\item[(b)] If there is a $ \G $ with $ \I_1,\I_2\models\G $ such that $ \I_1\models \K\circ\G $, then $ \I_1 \releqcK \I_2 $.
		\item[(c)] If there is a $ \G $ with $ \I_1,\I_2\models\G $ such that $ \I_1\models \K\circ\G $ and $ \I_2\not\models \K\circ\G $, then $ \I_1 \relcK \I_2 $.
	\end{enumerate}
\end{lemma}
\begin{proof} For the proofs of all statements, recall that by \Cref{lem:totality}, the relation $ \releqcK  $ is total.
\begin{enumerate}\setlength{\itemsep}{0pt}
\item[(a)]
    By totality of $\releqcK$, guarateed by \Cref{lem:totality} , we obtain $ \I_2 \releqcK \I_1 $.
    By definition of $\releqcK$,  this together with \( \I_2 \not\models \K \) entails \( \I_1 \not\models \K \). Therefore, again by definition, we obtain $\I_1 \not\sqreleqcK \I_2$.
    Consequently, in view of \Cref{def:relation_new_first_relation}, there must exist some  $ \G \in \Bases$ with $ \I_1,\I_2\models\G $ such that $\I_2 \models \K\circ\G$ does not imply $\I_1 \models  \K\circ\G$. Yet this can only be the case if $ \I_2 \models \K\circ\G $ and $ \I_1 \not\models \K\circ\G $, as claimed.
\item[(b)]	
Let $ \G $ and $ \I_1,\I_2 $ be as assumed. We proceed by case distinction:
\begin{itemize}
\item[$\I_2$]$\!\models\K$. Then $\I_2 \in \Mod{\K} \cap \Mod{\G} = \Mod{\K\Cup\G}$ and thus $\Mod{\K\Cup\G}\neq\emptyset$.
	Therefore, by \postulate{G2}, we obtain \( \Mod{\K\circ\G} = \Mod{\K\Cup\G} =  \Mod{\K} \cap \Mod{\G}\) and consequently \( \I_1 \models \K \).
    By \Cref{def:relation_new}, we conclude $ \I_1 \releqcK \I_2 $.
\item[$\I_2$]$\!\not\models\K$. Toward a contradiction, suppose $ \I_1 \not\releqcK \I_2 $. Then, by part (a) above, there is a $ \G' $ with $ \I_1,\I_2\models\G' $, $ \I_1\not\models\K\circ\G' $ and $ \I_2\models\K\circ\G' $.
	Thus $\I_1$ and $\I_2$ witness non-emptiness of $\Mod{(\K\circ\G)\Cup\G'}$ and $\Mod{(\K\circ\G')\Cup\G}$, respectively. Then, using \postulate{G5} and \postulate{G6} twice, we obtain $(\K\circ\G')\Cup\G \equiv \K\circ(\G\Cup\G')\equiv(\K\circ\G)\Cup\G' $. 
	But this allows to conclude $\I_1 \in \Mod{\K\circ\G} \cap \Mod{\G'} = \Mod{(\K\circ\G)\Cup\G'} = \Mod{(\K\circ\G')\Cup\G} = \Mod{\K\circ\G'} \cap \Mod{\G} \subseteq \Mod{\K\circ\G'}$, and thus $\I_1 \models \K\circ\G'$, which contradicts $\I_1 \not\models \K\circ\G'$ above.  
\end{itemize}   
\item[(c)]		
Let $ \G $ and $ \I_1,\I_2 $ be as assumed. We already know $ \I_1 \releqcK \I_2 $ due to part~(b). It remains to show $ \I_2 \not\releqcK \I_1 $.
We proceed by case distinction:
\begin{itemize}
	\item[$\I_1$]$\!\models\K$. Then $\I_1 \in \Mod{\K} \cap \Mod{\G} = \Mod{\K\Cup\G}$ and thus $\Mod{\K\Cup\G}\neq\emptyset$. Therefore, by \postulate{G2}, we obtain $\Mod{\K\circ\G} = \Mod{\K\Cup\G} =  \Mod{\K} \cap \Mod{\G}$. Since $ \I_2\not\models \K\circ\G $ but $\I_2\models\G$ we can infer $\I_2\not\models\K$. Consequently, by \Cref{def:relation_new}, we obtain \( \I_2 \not\releqcK \I_1 \).
	\item[$\I_1$]$\!\not\models\K$. Since we already established $\I_1 \releqcK \I_2$, \Cref{def:relation_new} ensures $\I_2\not\models\K$. Yet, by \Cref{def:relation_new_first_relation}, the existence of $\G$ implies $ \I_2 \not\sqreleqcK \I_1 $, and thus \Cref{def:relation_new} yields $\I_2 \not\releqcK \I_1$.\qedhere
\end{itemize}   
\end{enumerate}
\end{proof}

We show that our construction indeed yields a compatible assignment.

\begin{lemma}[compatibility]\label{lem:compatibility}
	If $ \circ $ satisfies \postulate{G1}--\postulate{G3}, \postulate{G5}, and \postulate{G6}, then it is compatible with $\releqc{\abst}$.
\end{lemma}
\begin{proof}
	We have to show that $\Mod{\K\circ\G} = \min(\Mod{\G},\releqcK)$.
	In the following, we 
	show %
	inclusion in both directions. 
\begin{enumerate}\setlength{\itemsep}{0pt}\setlength{\partopsep}{-10pt}
\item[($\subseteq$)] 
Let $\I\in \Mod{\K\circ\G}$.
	By \postulate{G1}, we obtain $\I\in \Mod{\G}$.
	But then, using \Cref{lem:help}(b), we can conclude $\I\releqcK\I'$ for any $\I'\in \Mod{\G}$, hence $\I\in \min(\Mod{\G},\releqcK)$.
	
\item[($\supseteq$)] Let $\I\in \min(\Mod{\G},\releqcK)$. 
	Due to $\Mod{\G}\neq\emptyset$ and \postulate{G3}, there exists an  $\I' \in  \Mod{\K\circ\G}$.
	From the ($\subseteq$)-proof follows $\I' \in \min(\Mod{\G},\releqcK)$. 
	Then, by \postulate{G1} and \Cref{lem:help}(b), we obtain $\I' \releqcK \I$ from $\I \in \Mod{\G}$ and $\I' \in \Mod{\G}$ and $\I' \in \Mod{\K\circ\G}$. 
	From $\I\in \min(\Mod{\G},\releqcK)$ and $\I' \in \Mod{\G}$ follows $\I \releqcK \I'$.
We proceed by case distinction:
\begin{itemize}
	\item[$\I$]$\!\models\K$. Then $\I \in \Mod{\K} \cap \Mod{\G} = \Mod{\K\Cup\G}$ and thus $\Mod{\K\Cup\G}\neq\emptyset$. Therefore, by \postulate{G2}, we obtain $\Mod{\K\circ\G} = \Mod{\K\Cup\G} =  \Mod{\K} \cap \Mod{\G}$ and hence $\I \in \Mod{\K\circ\G}$. 
	\item[$\I$]$\!\not\models\K$. Then by \Cref{def:relation_new}, $\I \releqcK \I'$ requires $\I'\not\models\K$ and therefore $\I \sqreleqcK \I'$ must hold.        
	Consequently, by \Cref{def:relation_new_first_relation}, $\I,\I'\in \Mod{\G}$ and $\I' \in \Mod{\K\circ\G}$ imply $\I\in \Mod{\K\circ\G}$. 
	\qedhere
\end{itemize}   
\end{enumerate}
\end{proof}

For min-friendliness, we have to show that each $ \releqcK $ is min-complete and min-retractive.

\begin{lemma}[min-friendliness]\label{lem:minwell}
	If $ \circ $ satisfies \postulate{G1}--\postulate{G3}, \postulate{G5}, and \postulate{G6}, then $ \releqcK $ is min-friendly for every $\K \in \Bases $. 
\end{lemma}
{\begin{proof}
		Observe that min-completeness is a consequence of \postulate{G3} and the compatibility of $ \releqc{\abst} $ with $ \circ $ from \Cref{lem:compatibility}.
		
		For min-retractivity, suppose towards a contradiction that it does not hold. 
		That means there is a belief base $\G$ and interpretations $\I',\I \models \G$ with $\I'\releqcK\I$ and $\I\in {\min(\Mod{\G},\releqcK)}$ but $\I'\not\in {\min(\Mod{\G},\releqcK)}$.
		From \Cref{lem:compatibility} we obtain $ \I \models \K\circ\G $ and $ \I'\not\models \K\circ\G $.
		Now, applying \Cref{lem:help}(c)  yields $ \I\relcK \I' $, contradicting $\I'\releqcK\I$.
\end{proof}

We show that $ \releqc{\abst} $ yields faithful relations for every belief base.

\begin{lemma}[faithfulness]\label{lem:faithfulness}
	If $ \circ $ satisfies \postulate{G2}, \postulate{G4}, \postulate{G5}, and \postulate{G6}, the assignment $\releqc{\abst}$ is faithful. 
\end{lemma}
\begin{proof}
	We show satisfaction of the three conditions of faithfulness, \postulate{F1}--\postulate{F3}.
\begin{itemize}\setlength{\itemsep}{0pt}
	\item[\postulate{F1}]	
Let $\I,\I'\in \Mod{\K}$. 
	Then $\I' \releqcK \I$ is an immediate consequence of \Cref{def:relation_new}.
	This implies $\I \not\relcK \I'$.
\item[\postulate{F2}]		
Let $\I\in \Mod{\K}$ and $\I'\not\in \Mod{\K}$. By \Cref{def:relation_new} we obtain $\I \releqcK \I'$ and $\I' \not\releqcK \I$.
\item[\postulate{F3}]	
Let $\K \equiv \K'$ (i.e. $\Mod{\K} = \Mod{\K'}$).
	From \Cref{def:relation_new} and \postulate{G4} follows $ {\releqcK} = {\releqc{\K'}} $, i.e., $ \I_1 \releqcK \I_2 $ if and only if $ \I_1 \releqc{\K'} \I_2 $.\qedhere 
\end{itemize}
\end{proof}

The previous lemmas can finally be used to show that the construction of $\releqc{\abst}$ according to \Cref{def:relation_new} yields an assignment with the desired properties.

\begin{proposition}\label{lfa}
	If 	$\circ$ satisfies \postulate{G1}--\postulate{G6}, then $\releqc{\abst} $ is a min-friendly faithful assignment compatible with $\circ$.
\end{proposition}
\begin{proof}
	Assume \postulate{G1}--\postulate{G6} are satisfied by $ \circ $. Then $\releqc{\abst} $ is an assignment since every $\releqK$ is total by \Cref{lem:totality};
	it is min-friendly by \Cref{lem:minwell};
	it is faithful by \Cref{lem:faithfulness}; and
	it is compatible with $\circ$ by \Cref{lem:compatibility}.
\end{proof}

\subsection{From Assignments to Postulates}

Now, it remains to show the ``if" direction of \Cref{thm:representation_theorem}.

\begin{proposition}\label{lem:lfa2_new}
	If there exists a min-friendly faithful assignment $\releq{\abst}$ compatible with $\circ$, then $\circ$ satisfies \postulate{G1}--\postulate{G6}.
\end{proposition}

\begin{proof}
	Let $\releq{\abst} : \K \mapsto \releqK $ be as described. We now show that $\circ$ satisfies all of \postulate{G1}--\postulate{G6}.
\begin{itemize}\setlength{\itemsep}{0pt}

\item[\postulate{G1}]	
 Let $\I\in \Mod{\K\circ\G}$. Since $\Mod{\K\circ\G} = \min(\Mod{\G},\releqK)$, we have that $\I\in \min(\Mod{\G},\releqK)$. Then, we also have that $\I\in \Mod{\G}$. Thus, we have that $\Mod{\K\circ\G}\subseteq \Mod{\G}$ as desired.

\item[\postulate{G2}]
Assume $\Mod{\K\Cup\G} \neq \emptyset$. 
We prove $\Mod{\K\circ\G} = \Mod{\K\Cup\G}$ by showing
inclusion in both directions. 
\begin{enumerate}\setlength{\itemsep}{0pt}\setlength{\partopsep}{-10pt}
	\item[($\subseteq$)] 
	Let $\I\in\Mod{\K\circ\G}$. 
	By compatibility, we obtain $\I\in\min(\Mod{\G},\releqK)$ and thus trivially also $\I\in\Mod{\G}$.
	Since $\Mod{\K\Cup\G} \neq \emptyset$, there exists some other $\I'\in\Mod{\K\Cup\G}=\Mod{\K}\cap\Mod{\G}$, which implies $\I'\in\Mod{\K}$ and $\I'\in\Mod{\G}$.
	Therefore, $\I\in\min(\Mod{\G},\releqK)$ implies $\I\releqK\I'$, which means that $\I'\relK\I$ cannot hold and therefore, by contraposition, \postulate{F2} ensures $\I\in\Mod{\K}$.   
    Yet then $\I\in\Mod{\K}\cap\Mod{\G}=\Mod{\K\Cup\G}$ as desired.
	\item[($\supseteq$)] 
	Let $\I\in\Mod{\K\Cup\G}=\Mod{\K}\cap\Mod{\G}$, i.e. $\I\in\Mod{\K}$ and $\I\in\Mod{\G}$.
    Since $\I\in\Mod{\K}$, we obtain from \postulate{F1} and \postulate{F2} that $\I\releqK\I'$ must hold for all $\I'\in\Mod{\G}$.
    Hence, $\I\in\min(\Mod{\G},\releqK)$, and by compatibility $\I\in\Mod{\K\circ\G}$.

\end{enumerate}
	
\item[\postulate{G3}]	
Assume \mbox{$\Mod{\G} \neq \emptyset$}. By min-completeness, we have $\min(\Mod{\G},\releqK)\neq\emptyset$. 
	Since $\Mod{\K\circ\G} = \min(\Mod{\G},\releqK)$ by compatibility, we obtain $\Mod{\K\circ\G} \neq \emptyset$.
	
\item[\postulate{G4}]	
Suppose there exist $\K_1,\K_2,\G_1,\G_2 \in \Bases$ with $\K_1 \equiv \K_2$ and $\G_1\equiv\G_2 $. Then, $\Mod{\K_1} = \Mod{\K_2}$ and $\Mod{\G_1} = \Mod{\G_2}$. From \postulate{F3}, we conclude ${\releq{\K_1}} = {\releq{\K_2}}$.
	Now assume some $\I\in \Mod{\K_1\circ{\G_1}}$, then by compatibility $\I\in \min(\Mod{\G_1},\releq{\K_1}) = \min(\Mod{\G_2},\releq{\K_2})$.  Therefore, again by compatibility, $\I\in \Mod{\K_2\circ\G_2}$. Thus, $\Mod{\K_1\circ\G_1} \subseteq \Mod{\K_2\circ\G_2}$ holds. Inclusion in the other direction follows by symmetry. Therefore, we have $\K_1\circ\G_1 \equiv \K_2\circ{\G_2}$.
	
\item[\postulate{G5}]	
Let $\I\in \Mod{(\K\circ\G_1) \Cup \G_2} = \Mod{\K\circ\G_1} \cap \Mod{\G_2}$.
This means that $\I\in \Mod{\G_2}$ but -- since $\Mod{\K\circ\G_1} = \mbox{$\min(\Mod{\G_1},\releqK)$}$ by compatibility -- we also obtain $\I\in \min(\Mod{\G_1},\releqK)$, meaning that $\I \releqK \I'$ holds for all $\I' \in \Mod{\G_1}$.
Yet then $\I \releqK \I'$ holds particularly for all $\I' \in \Mod{\G_1} \cap \Mod{\G_2}$ and hence $\I\in \min(\Mod{\G_1}\cap \Mod{\G_2},\releqK) = \min(\Mod{\G_1\Cup \G_2},\releqK)$.    
By compatibility follows $\I \in \Mod{\K\circ(\G_1 \Cup\G_2)}$. Thus $\Mod{(\K\circ\G_1) \Cup \G_2} \subseteq \Mod{\K\circ(\G_1 \Cup\G_2)}$ as desired. 
	
\item[\postulate{G6}]	
Let $ (\K\circ\G_1)\Cup\G_2 \not= \emptyset$, thus $ \I' \in \Mod{(\K\circ\G_1)\Cup\G_2} = \Mod{\K\circ\G_1} \cap \Mod{\G_2}$ for some $\I'$.
	By compatibility, we then obtain $  \I' \in \min( \Mod{\G_1 }, \releqK)$.
	Now consider an arbitrary $\I$ with $\I\in \Mod{\K\circ(\G_1 \Cup\G_2)}$. 	
	By compatibility we obtain $ \I \in\min( \Mod{\G_1 \Cup\G_2 } ,\releqK )$ and therefore, since $\I' \in \Mod{\G_1} \cap \Mod{\G_2} = \Mod{\G_1 \Cup\G_2}$, we can conclude $\I \releqK \I'$. This and $  \I' \in \min( \Mod{\G_1 }, \releqK)$ imply $\I \in \min( \Mod{\G_1 }, \releqK)$ by min-retractivity. Hence every $\I\in \Mod{\K\circ(\G_1 \Cup\G_2)}$ satisfies $\I \in \min( \Mod{\G_1 }, \releqK) = \Mod{\K \circ \G_1}$ but also $\I \in \Mod{\G_2}$, whence $\Mod{\K\circ(\G_1 \Cup\G_2)} \subseteq \Mod{\K \circ \G_1} \cap \Mod{\G_2} = \Mod{(\K \circ \G_1) \Cup\G_2}$ as desired.\qedhere 
\end{itemize}
\end{proof}

The proof of \Cref{thm:representation_theorem} follows from \Cref{lfa} and~\ref{lem:lfa2_new}.

\section{Two-Way Representation Theorem}\label{sec:abstract_rep}
\Cref{thm:representation_theorem} establishes the correspondence between operators and assignments under the assumption that $\circ$ is given and therefore known to exist. What remains unsettled is the question if generally \textbf{every} min-friendly faithful assignment is compatible with some base change operator that satisfies \postulate{G1}--\postulate{G6}. It is not hard to see that this is not the case.

\begin{example}
Consider the base logic 
$\mathbb{B}_{\mathrm{nb}}  = (\MC{L},{\Omega},\models, \MC{P}(\MC{L}), \cup)$
with $\MC{L}=\{none,both\}$ and $\Omega=\{\I_1,\I_2\}$ satisfying $\Mod{none}=\emptyset$ and $\Mod{both}=\{\I_1,\I_2\}=\Omega$. There are four bases in this logic, satisfying $\{none\}\equiv\{none,both\}$ and $\emptyset\equiv\{both\}$. Let the assignment $\releq{\abst}^{\mathrm{nb}}$ be such that
 ${\releq{\{none\}}^{\mathrm{nb}}}={\releq{\{none,both\}}^{\mathrm{nb}}}= \{(\I_1,\I_1),(\I_1,\I_2),(\I_2,\I_2)\}$ and ${\releq{\emptyset}^{\mathrm{nb}}}={\releq{\{both\}}^{\mathrm{nb}}}= \Omega \times \Omega$. 
It is straightforward to check that $\releq{\abst}^{\mathrm{nb}}$ is a min-friendly faithful assignment. Note that any $\circ$ compatible with $\releq{\abst}^{\mathrm{nb}}$ would have to satisfy $\Mod{\{none\}\circ\{both\}}=\min(\Mod{\{both\}},{\releq{\{none\}}^{\mathrm{nb}})} = \{\I_1\}$, yet, as we have seen, no base with this model set exists, therefore such a $\circ$ is impossible.  
\end{example}

\begin{example}
	For a \sauerwald{less contrived} negative example, recall the ``regular expression base logic'' $\mathbb{REG}_\Sigma^\mathrm{sng}$ from \Cref{ex2:regexp}. For a given regular expression $\alpha$, we let $f_\alpha:\Sigma^* \to \{0,1,2\}$ be defined as follows: $f_\alpha(w) = 0$ iff $w \in L(\alpha)$, $f_\alpha(w)=1$ iff $w \not\in L(\alpha)$ but $w$ has the same numbers of occurrences of each letter as some $w' \in L(\alpha)$ (i.e., the Parikh image of $w$ and $w'$ coincide), and $f_\alpha(w) = 2$ iff none of these is the case. Now let $\preceq^\mathrm{r1}_{(.)}$ be such that $w \preceq^\mathrm{r1}_{\alpha} w'$ exactly if $f_\alpha(w) \leq f_\alpha(w')$. Then it is not hard to verify that $\preceq^\mathrm{r1}_{(.)}$ is a min-friendly faithful assignment (in fact, even a preorder assignment). Yet, there is no compatible base change operator $\circ$, because it would have to satisfy $$ \Mod{(ba)^* \circ aa^*bb^*} = \min\Big(\Mod{aa^*bb^*},\preceq^\mathrm{r1}_{\scriptscriptstyle(ba)^*}\!\!\Big) = \min\Big(\{a^nb^m \mid n,m > 0\},\preceq^\mathrm{r1}_{\scriptscriptstyle(ba)^*}\!\!\Big) = \{a^nb^n \mid n>0\},$$ which is well known not to be a regular language and hence not expressible by any regular expression.
\end{example}

Therefore, toward a full, two-way correspondence, we have to provide an additional condition on assignments, capturing operator existence.

As indicated by the example, for the existence of an operator, it will turn out to be essential that any minimal model set of a belief base obtained from an assignment corresponds to some belief base, a property which is formalized by the following notion.
	\begin{definition}[min-expressible]\label{def:minfinite} 
Let $ \mathbb{B} = (\MC{L},\Omega,\models,\Bases,\Cup) $ be a base logic.
	A binary relation $\preceq$ over $ {\Omega} $ is called \emph{min-expressible} if for each $\G\in \Bases$ there exists a belief base $ \MC{B}_{\G,\preceq} \in \Bases$ such that $ \Mod{\MC{B}_{\G,\preceq}} \,{=}\, \min(\Mod{\G},\preceq)$.
	An assignment $\releq{\abst}$ will be called min-expressible, if for each $\K \in \Bases$, the relation $\releqK$ is min-expressible. 
	Given a min-expressible assignment $\releq{\abst}$, let $\circ_{\releq{\abst}}$  denote the \basechange{} operator defined by
	$\K \circ_{\releq{\abst}}\! \G \,{=}\, \MC{B}_{\G,\releqK}$.
\end{definition}

It should be noted that min-expressibility is a straightforward generalization of the notion of \emph{regularity} by {Delgrande et al.} \cite{KS_DelgrandePeppasWoltran2018} to base logics.
By virtue of this extra notion, we now find the following bidirectional relationship between assignments and operators, amounting to a full characterization.
\begin{theorem}\label{thm:rep1}
Let $\mathbb{B}$ be a base logic. Then the following hold:
\begin{itemize}
\item Every base change operator for $\mathbb{B}$ satisfying \postulate{G1}--\postulate{G6} is compatible with some min-expressible min-friendly faithful assignment \sauerwald{for \( \mathbb{B} \)}. 
\item Every min-expressible min-friendly faithful assignment for $\mathbb{B}$ is compatible with some 
base change operator satisfying \postulate{G1}--\postulate{G6}.
\end{itemize}
\end{theorem}

\begin{proof}
For the first item, let $\circ$ be the corresponding \basechange{} operator. Then, by Proposition~\ref{lfa}, the assignment $\releq{\abst}^\circ$ as given in Definition~\ref{def:relation_new} is min-friendly, faithful, and compatible with $\circ$. As for min-expressibility, recall that, by compatibility, $\Mod{\K\circ\G} =$ $\min(\Mod{\G},\releqcK)$ for every $\G$. As $ \K\circ\G $ is a belief base, min-expressibility follows immediately.

For the second item, let $\releq{\abst}$ be the corresponding min-expressible assignment and $\circ_{\releq{\abst}}$ as provided in Definition~\ref{def:minfinite}. By construction, $\circ_{\releq{\abst}}$ is compatible with $\releq{\abst}$. Proposition~\ref{lem:lfa2_new} implies that $\circ_{\releq{\abst}}$ satisfies \postulate{G1}--\postulate{G6}. \qedhere
\end{proof}

As an aside, note that the above theorem also implies that every min-expressible min-friendly faithful assignment is compatible \textbf{only} with \sauerwald{base change operators that satisfy \postulate{G1}--\postulate{G6}}. This is due to the fact that, on the one hand, any such assignment fully determines the corresponding compatible base change operator model-theoretically and, on the other hand, \postulate{G1}--\postulate{G6} are purely model-theoretic conditions.

Continuing our running example, we observe that $ \releq{\K_\mathrm{Ex}}^{\circ_\mathrm{Ex}} $ is also a min-expressible relation.

\begin{example}[continuation of Example \ref{ex:orderByEx2}]\label{ex:orderByEx3}
	Consider again $ \releq{\abst}^{\circ_\mathrm{Ex}} $, and observe that $ \releq{\abst}^{\circ_\mathrm{Ex}} $ is compatible with $ \circ_\mathrm{Ex} $, e.g. $ \Mod{\K_\mathrm{Ex}\circ_\mathrm{Ex}\G}=\min(\Mod{\G},\releq{\K_\mathrm{Ex}}^{\circ_\mathrm{Ex}}) $.
	Thus, for every belief base $ \G \in \MC{P}(\MC{L}_\mathrm{Ex}) $, the minimum $ \min(\G,\releq{\K_\mathrm{Ex}}^{\circ_\mathrm{Ex}}) $ yields a set expressible by a belief base.
	Theorem~\ref{thm:rep1} guarantees us that $ \circ_\mathrm{Ex} $ satisfies \postulate{G1}--\postulate{G6},  as $ \releq{\abst}^{\circ_\mathrm{Ex}} $ is a faithful min-expressible and min-friendly assignment.
\end{example}

\begin{example}
	\newcommand{\tuple}[1]{ \langle #1 \rangle }
	For a positive example for our ``regular expression base logic'' $\mathbb{REG}_\Sigma^\mathrm{sng}$ (\Cref{ex2:regexp}), recall that the Levenshtein distance \cite{levenshtein1966}, also known as edit distance, is defined as the function $\mathrm{lev}:\Sigma^* \times \Sigma^* \to \mathbb{N}$ which, given two words, returns the smallest number of insertions, deletions and replacements of letters necessary to transform one word into the other. This metric can be naturally lifted to measure the edit distance between a language and a word by letting $\mathrm{lev}(L,w) := \min_{u \in L}\mathrm{lev}(u,w)$.  
	It is now rather natural to define a preference relation $\preceq^\mathrm{r2}_{(.)}$ by letting $w \preceq^\mathrm{r2}_{\alpha} w'$ whenever 
	$\mathrm{lev}(\Mod{\alpha},w) \leq \mathrm{lev}(\Mod{\alpha},w')$. We note that $\preceq^\mathrm{r2}_{(.)}$ is a min-friendly faithful assignment (and again, even a preorder assignment). To prove min-expressibility, we have to show that for any two regular expressions $\alpha$ and $\beta$, there is a regular expression $\gamma$ satisfying $\Mod{\gamma} = \min(\Mod{\beta},\preceq^\mathrm{r2}_{\alpha})$. 
	Let $\Mod{\alpha}^{\leq k}$ denote the language $\{ w \in \Sigma^* \mid \mathrm{lev}(\Mod{\alpha},w)\leq k \}$. Then we observe that $\min(\Mod{\beta},\preceq^\mathrm{r2}_{\alpha}) = \Mod{\beta} \cap \Mod{\alpha}^{\leq k}$ where $k$ is the minimal natural number for which the intersection is non-empty. To show that 
	this language is regular (and hence representable by a regular expression), we observe that $\Mod{\beta}$ is regular, thus, due to the fact that regular languages are closed under intersection, it remains to show that $\Mod{\alpha}^{\leq k}$ is a regular as well. 
	We do so by constructing a nondeterministic finite automaton (NFA) recognizing it. Let $\mathcal{A} = (\Sigma,Q,I,F,\Delta)$ be an NFA recognizing $\Mod{\alpha}$. Let $\Delta^{[0]} = \{(q,q) \mid q \in Q\}$ and, for any $j>0$, let $\Delta^{[j]} = \{ (q,q'') \mid (q,q') \in \Delta^{[j-1]},\ (q',\sigma,q'') \in \Delta\}$.  
	Now let $\mathcal{A}^{\leq k} = (\Sigma,Q',I',F',\Delta')$ where
	\begin{itemize}
	\item 
	$Q' = Q \times \{0,\ldots,k\}$, 
	\item 
	$I' = \bigcup_{i=0}^k \big(\{q' \mid q \in I,\ (q,q') \in \Delta^{[i]}\} \times \{i\}\big)$
	\item 
	$F' = F \times \{0,\ldots,k\}$, and
	\item 
	$\Delta' = \{ (\tuple{q,i}, \sigma, \tuple{q''\!,i{+}j}) \mid (q,\sigma,q') \in \Delta,\  (q',q'') \in \Delta^{[j]},\ 0\leq j\leq k-i\}$ \\
	$\left.\right.\qquad \cup \{ (\tuple{q,i}, \sigma, \tuple{q'\!,i{+}1}) \mid (q,\sigma',q') \in \Delta, \sigma \in \Sigma \} \cup \{ (\tuple{q,i}, \sigma, \tuple{q,i{+}1}) \mid \sigma \in \Sigma \}.$ 
	\end{itemize}
A standard induction now shows that $\mathcal{A}^{\leq k}$ recognizes $\Mod{\alpha}^{\leq k}$ which concludes our argument that $\preceq^\mathrm{r2}_{(.)}$ is min-expressible and hence a corresponding revision operator $\circ^\mathrm{r2}$ indeed exists. Note that the our automaton construction also paves the way toward an algorithm realizing $\circ^\mathrm{r2}$. 
\end{example}

As a last step of this section, we will apply the theory developed here to demonstrate that the standard operator of trivial revision\footnote{Note, trivial revision is known to coincide with full meet revision in many logical settings.} \cite{KS_Hansson1999,KS_FermeHansson2018} indeed satisfies \postulate{G1}--\postulate{G6} in the general setting of base logics.
\begin{example}\label{ex:full-meet}
	Let $\mathbb{B}  = (\MC{L},{\Omega},\models, \MC{P}(\MC{L}), \Cup)$ be an arbitrary base logic.
	We define the \emph{trivial revision operator \( \circ^\mathrm{fm} \)} for \( \mathbb{B} \) by
	\begin{equation*}
		\K \circ^\mathrm{fm} \G = \begin{cases}
			\K \Cup \G & \text{if } \Mod{\K \Cup \G} \text{is consistent}\\
			\G & \text{otherwise}
		\end{cases}
	\end{equation*}
	To show satisfaction of \postulate{G1}--\postulate{G6} we construct a min-expressible min-friendly faithful assignment $\releq{\abst}^\mathrm{fm}$ compatible with \( \circ^\mathrm{fm} \).
	For each \( \K\in\Bases \) let  \( \I_1 \releqK^\mathrm{fm} \I_2 \) if \( \I_1\models \K \) or \( \I_2 \not\models \K \). Obviously, the relation \( \releqK^\mathrm{fm} \) is a total preorder where \( \I_1 , \I_2 \) are \( \releqK^\mathrm{fm} \)-equivalent, if either \( \I_1,\I_2\models\K \) or \( \I_1,\I_2 \not\models\K \) holds. Moreover, it is not hard to see that the relation \( \releqK^\mathrm{fm} \) is min-complete and min-retractive. 
	By construction of \( \releq{\abst}^\mathrm{fm} \) we obtain that \( \min(\Mod{\G},\releqK^\mathrm{fm})=\Mod{\G} \) if \( \K \Cup \G \) is inconsistent. If \( \K \Cup \G \) is consistent, we obtain \( \min(\Mod{\G},\releqK^\mathrm{fm})=\Mod{\K}\cap\Mod{\G} = \Mod{\K \Cup\G} \).
	In summary, the assignment \( \releq{\abst}^\mathrm{fm} \) is min-expressible and min-friendly, and the \basechange\ operator \( \circ^\mathrm{fm} \) is compatible with it.
\end{example}

\section{Interim Conclusion}
In Section \ref{sec:approach} to Section \ref{sec:abstract_rep}, we discussed how K\&M's result about semantically characterizing AGM belief revision in finite-signature propositional logic can be generalized to arbitrary base logics. 
Thereby, we cover all Tarskian logics and support any notion of bases that are closed under ``abstract union''.
We demonstrated certain central aspects by our running example (see Example~\ref{ex:logicEX}, Example~\ref{ex:orderByEx1}, Example~\ref{ex:orderByEx2}%
, Example~\ref{ex:orderByEx3}%
), which can be summarized as follows.
\begin{fact}
	The operator \( \circ_\mathrm{Ex} \) for the base logic \( \mathbb{B_\mathrm{Ex}} \) satisfies \postulate{G1}--\postulate{G6} and is compatible with the faithful min-friendly and min-expressible assignment $ \preceq^{\circ_\mathrm{Ex}}_{(.)}$.
	That is, for any base \( \K \) of \( \mathbb{B_\mathrm{Ex}} \), the relation $\preceq^{\circ_\mathrm{Ex}}_{\K}$ is min-friendly and min-expressible. However there is a base $\K_\mathrm{Ex}$, such that $\releqc{\K_\mathrm{Ex}}$ is not transitive. 
    In fact, no transitive faithful min-friendly and min-expressible assignment compatible with \( \circ_\mathrm{Ex} \) exists, whatsoever.
\end{fact}

By now, our rationale has been to cover the most general setting of base logics possible, while sticking to  the complete set of the AGM postulates and without adding further conditions.}%

However, one might remark that the AGM postulates were specifically designed for describing the change of belief sets, i.e., deductively closed theories, which naturally include all syntactic variants. 
As opposed to this, approaches to describing the change of (not necessarily deductively closed) bases might take the syntax into account \cite{KS_Hansson1999}. 
Under such circumstances, the syntax-independence expressed by \postulate{G4} might be called into question.

Another aspect is that, for the sake of generality, we had to replace the stronger requirement of transitivity by the weaker notion of min-retractivity inside the assignments. 
Waiving transitivity (and hence preorders) might be considered unconventional, as a transitive preference relation is often deemed to be the actual motivation behind the postulates \postulate{G5} and \postulate{G6}. 
This raises the question of how to obtain a compatible \textbf{preorder} assignment for any \sauerwald{{\basechange} operator that satisfies the AGM revision postulates}, whenever possible, and what are the conditions that guarantee universal preorder representability.

In the following sections we will discuss these aspects as variations of the approach we presented in the preceding sections, showing that exact characterizations exist for these cases as well.
Moreover, we will discuss some aspects of the notion of base logic, and the role of disjunctions in decomposability. 

\section{Base Changes and Syntax-Independence}\label{sec:syntax_independence}
Up to this point, we have been considering \basechange{} operators fulfilling the full set of postulates \postulate{G1}--\postulate{G6}.
The research on base changes deals with syntax-dependent changes, and in our approach the postulate \postulate{G4} implies that a \basechange{} operator yields semantically  the same result on all semantically equivalent bases.
As consequence, one might conclude that the \basechange{} operators considered here have only limited freedom when it comes to taking the syntactic structure into account when changing.

However, note that neither the postulates \postulate{G1}--\postulate{G6} nor our representation results make assumptions about the specific syntactic structure of a base obtained by a \basechange{} operator.
Thus, for syntactically different bases $\G_1$ and $\G_2$ that are semantic equivalent, we might obtain syntactically different results after revision, which are semantic equivalent.

\begin{example}
	Consider the logic \( \mathbb{PL}_2 \) (cf. \Cref{example:PLn}), e.g. propositional logic over the signature $\{p,q\}$ as follows.
	Given $\K_{1} = \{p,q\}, \K_{2} = \{p\wedge q\}, \G_{1}=\{p, p\to \neg q\}$, and $\G_{2}=\{p\wedge\neg q\}$.
	We have $\K_1$ and $\K_{2}$, as well as $\G_1$ and $\G_2$, which are two semantic equivalent bases with different syntax.
	By applying the trivial revision operation \( \circ^\mathrm{fm} \) (cf. \Cref{ex:full-meet}) to $\K_1$ by $\G_1$ and to $\K_2$ by $\G_2$, we obtain $\K_1\circ\G_1 = \{p, p\to \neg q\}$ and  $\K_2\circ\G_2 = \{p\wedge\neg q\}$.
	The two revision results are different syntactically, yet semantically equivalent (i.e. $\Mod{\K_1\circ\G_1} = \Mod{\K_2\circ\G_2 } = \{\I:p\mapsto\textbf{true}, q\mapsto\textbf{false}\}$).
\end{example}

Moreover, the  semantic viewpoint developed here in this article is flexible and is eligible for further liberation regarding syntax-dependence of a \basechange{} operator.
In particular, our approach allows us to drop \postulate{G4}. 
As an alternative to  \postulate{G4}, consider the following weaker version \cite{KS_Hansson1999}:
\begin{itemize}\setlength{\itemsep}{0pt}
	\item[]\definepostulate{G4w}~~If $\G_1 \equiv \G_2$, then $\K \circ \G_1 \equiv \K \circ \G_2$.
\end{itemize}
The main difference between \postulate{G4w} and \postulate{G4} is that by \postulate{G4w} a \basechange{} operator is not restricted to treat semantically equivalent prior belief bases equivalently.
When considering the extended AGM postulates \postulate{G5} and \postulate{G6} it  turns out that postulate  \postulate{G4w} is a baseline of syntax-independence, as \postulate{G1}, \postulate{G5} and \postulate{G6} together  already imply \postulate{G4w}, which is a generalization of a result by Aiguier et al. \cite[Prop. 3]{aiguier_2018}.

\begin{proposition}
\sauerwald{If a \basechange{} operator \( \circ \) for some base logic \( \mathbb{B} \) satisfies \postulate{G1}, \postulate{G5} and \postulate{G6}, then \( \circ \) satisfies~\postulate{G4w}.}
\end{proposition}
\begin{proof}
	Let \sauerwald{\( \mathbb{B} = (\MC{L},\Omega,\models,\Bases,\Cup) \) be a base logic and let} \( \K,\G_1,\G_2\in\Bases \) be belief bases such that \( \G_1 \equiv \G_2 \).
    By \postulate{G1}, the postulate \postulate{G4w} holds if \( \G_1 \) is inconsistent. For the remaining parts of the proof, we assume consistency of \( \G_1 \).
	First observe that \( (\K \circ \G_1) \Cup \G_2 \equiv \K \circ \G_1 \) by \postulate{G1} and analogously \( (\K \circ \G_2) \Cup \G_1 \equiv \K \circ \G_2 \).
	By \postulate{G5} we obtain \( (\K \circ \G_1) \Cup \G_2 \models \K \circ (\G_1 \Cup \G_2) \). Moreover, because \( (\K \circ \G_1) \Cup \G_2 \) is consistent, we obtain \( \K \circ (\G_1 \Cup \G_2) \models (\K \circ \G_1) \Cup \G_2 \) by \postulate{G6}.
	In summary we obtain \( (\K \circ \G_1) \Cup \G_2 \equiv \K \circ (\G_1 \Cup \G_2) \). By an analogous line of arguments we obtain \( (\K \circ \G_1) \Cup \G_2 \equiv \K \circ (\G_1 \Cup \G_2) \equiv (\K \circ \G_2) \Cup \G_1 \).
	Using our prior observations this expands to \( \K \circ \G_1 \equiv  \K \circ \G_2  \).
\end{proof}

To obtain a representation theorem for \basechange{} operators without \postulate{G4}, relaxing the constraint on the syntactic side requires the adjustment of the conditions on the semantic side. 
When dropping \postulate{G4}, we need to weaken the notion of faithfulness to the notion of quasi-faithfulness.
\begin{definition}[quasi-faithful]\label{def:quasi-faithful}
	An assignment $\releq{\abst}$ is called \emph{quasi-faithful} if it satisfies the following conditions:
	\begin{itemize}\setlength{\itemsep}{0pt}
		\item[(F1)] If $\I,\I' \models \K$, then $\I \relK \I'$ does not hold.
		\item[(F2)] If $\I\models \K$ and $\I'\not\models \K$, then $\I \relK \I'$. 
	\end{itemize}	
\end{definition}
Note that quasi-faithful assignments might assign to every belief base a different order, independent from whether they are semantically equivalent or not.
Thus, this enables a \basechange{} operator to treat bases differently depending on their syntactic structure.

Luckily, our canonical assignment \( \releqc{\abst} \) (cf. \Cref{def:relation_new}) carries over to the setting where \postulate{G4} is not satisfied.
    The following lemma attests that \( \releqc{\abst} \)  yields a quasi-faithful assignment for this cases.
\begin{lemma}\label{lem:quasi-faithfulness}
    If $ \circ $ satisfies \postulate{G2}, \postulate{G5}, and \postulate{G6}, then the assignment $\releqc{\abst}$ is quasi-faithful. 
\end{lemma}
\begin{proof}
    The proof of the two conditions of quasi-faithfulness, \postulate{F1} and \postulate{F2}, is identical to the proof of \postulate{F1} and \postulate{F2} in \Cref{lem:faithfulness}.
\end{proof}

Using the notion of quasi-faithfulness and \( \releqc{\abst} \) (cf. \Cref{def:relation_new}) we obtain the following characterization result, which is similar to a result already provided by Aiguier et al. \cite[Thm. 2]{aiguier_2018}.
\begin{proposition}\label{prop:agm_withoutg4}
	Let $\circ$ be a \basechange{} operator.
	The operator $\circ$ 
	satisfies \postulate{G1}--\postulate{G3}, \postulate{G5}, and \postulate{G6} if and only if it is compatible with some min-friendly quasi-faithful assignment.
\end{proposition}
\begin{proof}[Proof (Sketch)]
	The proof is nearly the same as for \Cref{thm:representation_theorem}.
	Note that the proof of \Cref{thm:representation_theorem}, which shows correspondence between \postulate{G1}--\postulate{G6} and compatible min-friendly faithful assignments uses \postulate{G4} and \postulate{F3} only in special situations.
	In particular, observe that condition \postulate{F3} is only used to show satisfaction of \postulate{G4} in the proof of \Cref{lem:lfa2_new}. Moreover, note that \( \releqcK \) from \Cref{def:relation_new} is a total min-friendly relation due to \Cref{lem:totality} and \Cref{lem:minwell} for each \( \K\in\Bases \); compatibility of \( \releqc{\abst}  \) with \( \circ \) is ensured by \Cref{lem:compatibility} while satisfaction of quasi-faithfulness is ensured by \Cref{lem:quasi-faithfulness}.
\end{proof}

In view of this, we can now present the syntax-dependent version of our two-way representation theorem.

\begin{theorem}\label{thm:rep2}
	Let $\mathbb{B}$ be a base logic. Then the following hold:
	\begin{itemize}
		\item Every base change operator for $\mathbb{B}$ satisfying \postulate{G1}--\postulate{G3}, \postulate{G5}, and \postulate{G6} is compatible with some min-expressible min-friendly quasi-faithful assignment. 
		\item Every min-expressible min-friendly quasi-faithful assignment for $\mathbb{B}$ is compatible with some 
		base change operator satisfying \postulate{G1}--\postulate{G3}, \postulate{G5}, and \postulate{G6}.
	\end{itemize}
\end{theorem}

In research on  base revision, various special postulates for the changing of bases have been considered, e.g. in the seminal research on belief revision by Hansson, special postulates for base changes are proposed, e.g., see \cite{KS_Hansson1999}. 
Of course, an interesting and open question is, which of them could be characterized or reconstructed by the approach of this article.
\sauerwald{In the next section we discuss the encoding of operators by quasi-faithful and faithful assignments.}

\section{Encoding of Base Change Operators \protect{by Assignments}}\label{sec:enc_operators}
Recall that for K\&M's encoding, presented in Equation \eqref{eq:km_encoding}, the existence of a sentence $ \mathit{form}(\I_1,\I_2)$ which satisfies $ {\Mod{\mathit{form}(\I_1,\I_2)}=\{ \I_1,\I_2 \}} $ is required for any interpretations $\I_1,\I_2$ in the considered logic.
The problem in a general Tarskian logical setting is that there might not be such a sentence or base.
Generalizing
K\&M's encoding
to our setting (using just the bases that do exist) bears the danger that  the relation between certain pairs of interpretations is left undetermined: depending on the shape of the logic (and its model theory) as well as the operator, there might be no preference between certain elements (because there is no revision which provides information on the preference). 
The following notion formally defines such pairs of interpretations.
\begin{definition}\label{def:detached}
    Let $ \circ $ be a \basechange\ operator  for  $\mathbb{B}$ and \( \K \) a base of \( \mathbb{B} \).
    A pair $ (\I,\I')\in\Omega\times \Omega $ is called \emph{detached from $ \circ $ in $ \K $}, if $ \I,\I'\not \models\K\circ\G$ for all $ \G\in\Bases $ with $ \I,\I'\models\G $. 
    With $\setAllDetached$ we denote the set of all pairs \( (\I,\I')  \) which are detached from $ \circ $ in $ \K $  and satisfy  \( \I\neq\I' \).
\end{definition}

It is immediate that the behaviour of base change operators is determined by all non-detached pairs, as elaborated in the following corollary.
\begin{corollary}
    Let $ \mathbb{B}=(\MC{L},{\Omega},\models,\Bases,\Cup)$ be a base logic, let \( \K\in\Bases \) and let \( \circ \) be a base change operator for \( \mathbb{B} \).\newline
    For each quasi-faithful assignment \( {\releq{\abst}} \) that is compatible with \( \circ \) the following statements hold:
    \begin{itemize}
        \item \( \min(\Mod{\G},\releqK) = \min(\Mod{\G},\releqK\setminus\setAllDetached) \).
        \item For each pair \( (\I,\I')\in (\Omega\times\Omega) \setminus \setAllDetached \) with \( \I\neq\I' \) there exists a base \( \G\in\Bases \) such that  \( \min(\Mod{\G},\releqK) \neq \min(\Mod{\G},\releqK\setminus\{(\I,\I')\}) \) holds.
    \end{itemize}
\end{corollary}
Consequently, the only difference between all assignments that are compatible with one operator is how they treat detached pairs.
\sauerwald{This observation motivates the following definition:} two assignments \( {\releq{\abst}} \)  and \( {\releq{\abst}'} \) for some base logic \( \mathbb{B} \) are \sauerwald{called \emph{similar}, written \( {\releq{\abst}} \simeq {\releq{\abst}'} \),} if they yield the same minimal sets of models, i.e., 
\[
    {\releq{\abst}} \simeq {\releq{\abst}'} \text{\ \ if for all \( \K\in\Bases \) holds } \min(\Mod{\G},\releqK) = \min(\Mod{\G},\releqK') \ .
    \]
\sauerwald{Obviously,} all similar assignments encode the same {\basechange} operator.
\begin{corollary}\label{prop:quasiff_similar}
    Let $\mathbb{B}$ be a base logic and let \( \circ \) be a base change operator for \( \mathbb{B} \).
    \sauerwald{It holds \( {\releq{\abst}} \simeq {\releq{\abst}'} \) for all assignments \( {\releq{\abst}} \) and \( {\releq{\abst}'} \) that are compatible with \( \circ \).}
\end{corollary}
\begin{proof}
    \sauerwald{This follows immediately from the definition of \( \simeq \) and the definition of compatibility (\Cref{def:compatible}).}
\end{proof}

\sauerwald{In this section, we will discuss how these aspects influence how to encode base change operators by assignments.}

\subsection{\sauerwald{On Assignments that Encode Syntax-Independent Operators}}

We discuss the relation between quasi-faithful and faithful assignments for syntax-independent operators, i.e., operators that satisfy \postulate{G4}.
\sauerwald{
    First, we will see that, even when \( \circ \) satisfies \postulate{G4}, there might be many quasi-faithful assignments that are compatible with \( \circ \) but are not inevitably faithful.
The following example demonstrates this observation.
\begin{example}
Consider the simplistic base logic $\mathbb{B}_{\mathbf{01}} {=} (\MC{L},{\Omega},\models, \Bases, \Cup)$ 
where  $\MC{L}$ contains exactly the sentences \( \textsc{zero} \), \( \textsc{one} \), \( \bot \) and \( \pm \), i.e., \( \MC{L}=\{ \textsc{zero}, \textsc{one}, \bot, \pm \} \).
We let \( \Omega=\{0,1\} \) contain the numbers \( 0,1 \) and define \( \models \) by letting \( \Mod{\textsc{zero}}=\{ 0 \} \), and \( \Mod{\textsc{one}}=\{ 1 \} \), and \( \Mod{\bot}=\Mod{\pm}=\emptyset \).
The bases are the singleton set of sentences, i.e., \( \Bases = \{ \{ \textsc{zero} \},\  \{ \textsc{one} \},\  \{ \bot \},\ \{\pm\} \} \). 
In the following we ease notation, and write \( \textsc{zero} \) instead of \( \{\textsc{zero}\} \), and so forth, when there is no ambiguity. 
The abstract union \( \Cup \) is the commutative operation with \( \textsc{zero} \Cup \textsc{one} =  \G \Cup  \bot  = \G \Cup \pm = \bot   \) for all \( \G\in\Bases \) and \( \G \Cup \G = \G \) for all \( \G\in\{ \{\textsc{zero}\} ,\ \{\textsc{one}\} \} \).
The trivial revision operator  \( \circ:  (\K, \Gamma) \mapsto \Gamma \) from \Cref{ex:full-meet} is the only base change operator for \( \mathbb{B}_{\mathbf{01}} \) that satisfies \postulate{G1}--\postulate{G6}. 
We consider all assignments that are compatible with \( \circ \).
There are exactly three total relations over \( \Omega \); the relations \( \preceq^{\mathbf{0}<\mathbf{1}} \), \( \preceq^{\mathbf{1}<\mathbf{0}} \), and  \( \preceq^{\mathbf{1}=\mathbf{0}} \) given by:
\begin{align*}
  0  &     \preceq^{\mathbf{0}<\mathbf{1}} 1 &  1  & \not\preceq^{\mathbf{0}<\mathbf{1}} 0 &
  0  & \not\preceq^{\mathbf{1}<\mathbf{0}} 1 &  1  &     \preceq^{\mathbf{1}<\mathbf{0}} 0 & 
  0  & \preceq^{\mathbf{1}=\mathbf{0}} 1    &      1  & \preceq^{\mathbf{1}=\mathbf{0}} 0 
\end{align*}
Note that every faithful assignment \( \releq{\abst} \) for \( \mathbb{B}_{\mathbf{01}} \) satisfies \( {\releq{\textsc{zero}}} = {\preceq^{\mathbf{0}<\mathbf{1}}} \) and \( {\releq{\textsc{one}}} = {\preceq^{\mathbf{1}<\mathbf{0}}} \).
Furthermore,  assignments that are compatible with \( \circ \) differ only in what is assigned to the bases \( \{ \bot \} \) and \( \{ \pm \} \).
This is because the pair \( (0,1) \) is detached from \( \circ \) in \( \{ \bot \} \) (and \( \{ \pm \} \)), i.e., there is no base \( \G\in\Bases \) with \( \Mod{\G}=\{0,1\} \) and \( 0\models \bot\circ\G \) or \( 1\models \bot\circ\G \) (analogously for \( \pm \)).
The following listing contains all faithful assignments for \( \mathbb{B}_{\mathbf{01}} \) that are compatible with \( \circ \):
\begin{align*}
 {\preceq^1_{\bot}} & = {\preceq^{\mathbf{0}<\mathbf{1}}} &   {\preceq^2_{\bot}} & = {\preceq^{\mathbf{1}<\mathbf{0}}} &  {\preceq^3_{\bot}} & = {\preceq^{\mathbf{0}=\mathbf{1}}} \\
 {\preceq^1_{\pm}}  & = {\preceq^{\mathbf{0}<\mathbf{1}}} &   {\preceq^2_{\pm}}  & = {\preceq^{\mathbf{1}<\mathbf{0}}} &  {\preceq^3_{\pm}}  & = {\preceq^{\mathbf{0}=\mathbf{1}}} 
\end{align*}
One can check easily that the assignments \( \releq{\abst}^1,\releq{\abst}^2,\releq{\abst}^3 \) are also min-expressible and min-friendly.
Next, we describe all quasi-faithful assignments for \( \mathbb{B}_{\mathbf{01}} \) that are  compatible with \( \circ \), yet are not faithful:
\begin{align*}
    {\preceq^4_{\bot}} & = {\preceq^{\mathbf{0}<\mathbf{1}}} & {\preceq^6_{\bot}} & = {\preceq^{\mathbf{0}<\mathbf{1}}} & {\preceq^8_{\bot}} & = {\preceq^{\mathbf{1}<\mathbf{0}}} & {\preceq^{10}_{\bot}} & = {\preceq^{\mathbf{1}<\mathbf{0}}}  & {\preceq^{12}_{\bot}} & = {\preceq^{\mathbf{0}=\mathbf{1}}} & {\preceq^{14}_{\bot}} & = {\preceq^{\mathbf{0}=\mathbf{1}}}  \\
    {\preceq^4_{\pm}}  & = {\preceq^{\mathbf{1}<\mathbf{0}}} & {\preceq^6_{\pm}}  & = {\preceq^{\mathbf{0}=\mathbf{1}}} & {\preceq^8_{\pm}}  & = {\preceq^{\mathbf{0}<\mathbf{1}}} & {\preceq^{10}_{\pm}}  & = {\preceq^{\mathbf{0}=\mathbf{1}}}  & {\preceq^{12}_{\pm}}  & = {\preceq^{\mathbf{0}<\mathbf{1}}} & {\preceq^{14}_{\pm}}  & = {\preceq^{\mathbf{1}<\mathbf{0}}}  \\
    {\preceq^5_{\bot}} & = {\preceq^{\mathbf{1}<\mathbf{0}}} & {\preceq^7_{\bot}} & = {\preceq^{\mathbf{0}=\mathbf{1}}} & {\preceq^9_{\bot}} & = {\preceq^{\mathbf{0}<\mathbf{1}}} & {\preceq^{11}_{\bot}} & = {\preceq^{\mathbf{0}=\mathbf{1}}}  & {\preceq^{13}_{\bot}} & = {\preceq^{\mathbf{0}<\mathbf{1}}} & {\preceq^{15}_{\bot}} & = {\preceq^{\mathbf{1}<\mathbf{0}}}  \\
    {\preceq^5_{\pm}}  & = {\preceq^{\mathbf{0}<\mathbf{1}}} & {\preceq^7_{\pm}}  & = {\preceq^{\mathbf{0}<\mathbf{1}}} & {\preceq^9_{\pm}}  & = {\preceq^{\mathbf{1}<\mathbf{0}}} & {\preceq^{11}_{\pm}}  & = {\preceq^{\mathbf{1}<\mathbf{0}}}  & {\preceq^{13}_{\pm}}  & = {\preceq^{\mathbf{0}=\mathbf{1}}} & {\preceq^{15}_{\pm}}  & = {\preceq^{\mathbf{0}=\mathbf{1}}}  
\end{align*}
Due to \Cref{prop:quasiff_similar} it holds for all \( i,j\in\{1,\ldots,15\} \) that \( {\releq{\abst}^i} \simeq {\releq{\abst}^j} \).  
As the notions of min-expressibility and min-friendliness only rely on minima in \( \releqK \) with respect to bases, the min-expressibility and min-friendliness from, e.g., \( \releq{\abst}^3 \), carries over to all assignments \( {\releq{\abst}^1},\ldots,{\releq{\abst}^{15}} \).
\end{example}
The existence of non-faithful, yet quasi-faithful, assignments compatible with {\basechange} operators does not contradict \Cref{thm:representation_theorem};
the important matter is that compatible quasi-faithful assignments behave like faithful assignments regarding the minimal models of bases.
}
We will now see that we can obtain a faithful assignment from a \sauerwald{compatible quasi-faithful assignment} whenever \postulate{G4} is satisfied.

\begin{proposition}\label{col:quaisVSnonquasiG4}
    Let $\mathbb{B}$ be a base logic and let \( \circ \) be a base change operator for \( \mathbb{B} \) that satisfies \postulate{G1}--\postulate{G6}.
    For each min-friendly quasi-faithful assignment \( {\releq{\abst}} \) that is compatible with \( \circ \), there exists a min-friendly faithful assignment \( {\preceq_{\abst}^{\mathrm{ff}}} \) that is compatible with \( \circ \) and \( {\preceq_{\abst}^{\mathrm{ff}}} \simeq {\releq{\abst}} \).
\end{proposition}
\begin{proof}
We define $\preceq_{\abst}^{\mathrm{ff}}$ as $\K \mapsto {\preceq_{\sigma([\K]_\equiv)}}$ where $\sigma$ is a selection function mapping every $\equiv$-equivalence class of $\Bases$ to one of its elements (i.e., $\sigma([\K]_\equiv) \in [\K]_\equiv$). 
Then, the property of being a min-complete quasi-faithful preorder assignment compatible with~$ \circ $ carries over pointwise from $\preceq_{\abst}$ to $\preceq_{\abst}^{\mathrm{ff}}$, while the construction ensures that $\preceq_{\abst}^{\mathrm{ff}}$ also satisfies \postulate{F3} from \Cref{def:faithful} and hence is faithful. 
\end{proof}

\sauerwald{Equipped with the notion of similarity between assignments and our observations, we can, for once more, generalize our main representation result even further.
     We relax the notion of faithfulness. 
     
     \pagebreak
    \begin{definition}[semi-faithful]\label{def:semi-faithful}
        An assignment $\releq{\abst}$ is called \emph{semi-faithful} if it satisfies the following conditions:
        \begin{itemize}\setlength{\itemsep}{0pt}
            \item[(F1)] If $\I,\I' \models \K$, then $\I \relK \I'$ does not hold.
            \item[(F2)] If $\I\models \K$ and $\I'\not\models \K$, then $\I \relK \I'$. 
            \item[(\textlabel{F3\( \simeq \)}{pstl:F3sim})] If $\K\equiv\K'$, then ${\releqK} \simeq {\releq{\MC{K'}}}$.
        \end{itemize}	
    \end{definition}

    The following theorem shows that semi-faithfulness precisely captures the class of all assignments that are compatible with a base change operator that satisfies \postulate{G1}--\postulate{G6}.
    
    \begin{theorem}\label{thm:rep_semifaithful}
        Let $\mathbb{B}$ be a base logic. Then the following hold:
        \begin{itemize}
            \item Every base change operator for $\mathbb{B}$ satisfying \postulate{G1}--\postulate{G6} is compatible with some min-expressible min-friendly semi-faithful assignment for \( \mathbb{B} \). 
            \item Every min-expressible min-friendly semi-faithful assignment for $\mathbb{B}$ is compatible with some 
            base change operator satisfying \postulate{G1}--\postulate{G6}.
            \item Every assignment for $\mathbb{B}$ that is compatible with some base change operator that satisfies \postulate{G1}--\postulate{G6} is min-expressible, min-friendly and semi-faithful.
        \end{itemize}
    \end{theorem}
    \begin{proof}
        We consider all statements from \Cref{thm:rep_semifaithful} independently:
           The first item follows directly from \Cref{thm:rep1}, because every min-expressible min-friendly faithful assignment for \( \mathbb{B} \) is also semi-faithful.
           
           The proof of the second item is nearly the same as in \Cref{lem:lfa2_new} used for showing \Cref{thm:representation_theorem}.
           In particular, observe that condition \postulate{F3} is only used to show satisfaction of \postulate{G4} in the proof of \Cref{lem:lfa2_new}. We show that \postulate{F3sim} is also sufficient for the satisfaction of \postulate{G4}:
           \begin{itemize}\setlength{\itemsep}{0pt}
               \item[\postulate{G4}]	
               Suppose there exist $\K_1,\K_2,\G_1,\G_2 \in \Bases$ with $\K_1 \equiv \K_2$ and $\G_1\equiv\G_2 $. 
               Then, $\Mod{\K_1} = \Mod{\K_2}$ and $\Mod{\G_1} = \Mod{\G_2}$. 
               Now assume some $\I\in \Mod{\K_1\circ{\G_1}}$, then by compatibility $\I\in \min(\Mod{\G_1},\releq{\K_1})$.  
               From \postulate{F3sim}, we conclude ${\releq{\K_1}} \simeq {\releq{\K_2}}$, which implies \( \min(\Mod{\G_1},\releq{\K_1}) = \min(\Mod{\G_2},\releq{\K_2}) \).
               Therefore, again by compatibility, $\I\in \Mod{\K_2\circ\G_2}$. Thus, $\Mod{\K_1\circ\G_1} \subseteq \Mod{\K_2\circ\G_2}$ holds. Inclusion in the other direction follows by symmetry. Therefore, we have $\K_1\circ\G_1 \equiv \K_2\circ{\G_2}$.
           \end{itemize}        
           Moreover, note that \( \releqcK \) from \Cref{def:relation_new} is a total min-friendly relation due to \Cref{lem:totality} and \Cref{lem:minwell} for each \( \K\in\Bases \); compatibility of \( \releqc{\abst}  \) with \( \circ \) is ensured by \Cref{lem:compatibility} while satisfaction of quasi-faithfulness is ensured by \Cref{lem:quasi-faithfulness}.
           
           We show the third item. Let \( \releq{\abst} \) be an assignment for \( \mathbb{B} \) that is compatible with some base change operator \( \circ \) that satisfies \postulate{G1}--\postulate{G6}.           
           From \Cref{lfa} we obtain that the assignment \( \releqc{\abst} \) from \Cref{def:relation_new} is a min-expressible min-friendly faithful assignment compatible with $\circ$.
           From \Cref{prop:quasiff_similar} we obtain that \( {\releq{\abst}} \simeq {\releqc{\abst}} \) holds.
           This means that \( \Mod{\K\circ\G} = \min(\Mod{\G},\releq{\K}) = \min(\Mod{\G},\releqc{\K}) \) holds for all belief bases \( \K,\G \in\Bases \).
           Consequently, \( \releq{\abst} \) inherits min-expressibility and min-friendliness from \( \releqc{\abst} \).
           
           Towards a contradiction, assume that \( \releq{\abst} \) is not semi-faithful.
           Because \( \releq{\abst} \) is not semi-faithful, there are some belief bases \( \K_1,\K_2,\G \in\Bases \) such that \( \K_1 \equiv \K_2 \) and \( \min(\Mod{\G},\releq{\K_1}) \neq \min(\Mod{\G},\releq{\K_2}) \).
           Because \( \circ \) satisfies \postulate{G4}, we obtain that \(  \Mod{\K\circ\G}= \Mod{\K\circ\G} \) holds from \( \K_1 \equiv \K_2 \).
           As \( \releq{\abst} \) is compatible with \( \circ \), we have \( \Mod{\K_1\circ\G}=\min(\Mod{\G},\releq{\K_1}) \) and \( \Mod{\K_2\circ\G} = \min(\Mod{\G},\releq{\K_2})   \).
           Combining these observations yields the contradiction  \( \Mod{\K_1\circ\G}\neq\Mod{\K_2\circ\G} \).           
\end{proof}}

\sauerwald{From now on, we will mainly focus on propositions and theorems concerning operators that satisfy \postulate{G1}--\postulate{G3}, \postulate{G5} and \postulate{G6}, and their corresponding quasi-faithful assignments.
However, as shown in \Cref{col:quaisVSnonquasiG4},} this is not an important restriction, as many of those theorems carry over to operators that satisfy additionally \postulate{G4}.

\subsection{Specifics of our Encodings}

When one wants to encode an operator by assignments in the sense of \Cref{def:faithful}, which use total relations, 
the detached pairs
have to be ordered in a certain way, and the appropriate selection of a \textit{``preference''} between these two interpretations is a \textit{``non-local''} choice (as it may have ramifications for other ``ordering choices'').

As a solution for the arrangement of detached pairs, we came up with the encodings  \Cref{def:relation_new_first_relation}, providing an encoding different from the approach by K\&M.
The refinement of \Cref{def:relation_new_first_relation} to \Cref{def:relation_new}, providing $\releqc{(.)}$, turns out to be the general suitable encoding for revision in base logics, as witnessed by \Cref{thm:representation_theorem}.
    \sauerwald{\Cref{def:relation_new_first_relation} and \Cref{def:relation_new} solve} the problem how to arrange the preferences of the detached pairs by treating them as equally preferable.
    This encoding  is unique 
in the following sense:
     $\releqc{(.)}$ turns out to be the (set-inclusion-)maximal 
     representation for the preferences of an operator -- a property the encoding approaches given by Equation~\eqref{eq:km_encoding} \sauerwald{by K\&M} do not have.
\begin{proposition}\label{lem:maxrelation}
	Let $ \circ $ be a {\basechange} operator 
    that satisfies \postulate{G1}--\postulate{G3}, \postulate{G5}, and \postulate{G6}.
    If  $\releq{\abst}$ is a min-friendly faithful
	assignment compatible with $ \circ $, then  $ \I_1\releqK\I_2  $ implies $ \I_1\releqcK\I_2 $ for every $ \I_1,\I_2\in\Omega $ and every belief base $ \MC{K}\in\Bases$.
\end{proposition}
\begin{proof}
	Toward a contradiction, assume there were $ \I_1,\I_2\in\Omega $ such that $ \I_1\releqK\I_2  $ but $ \I_1\not\releqcK\I_2 $ (hence, by totality $ \I_2\relcK\I_1 $).  
	
	Let us first consider the case $\I_2 \models \K$. Then $ \I_2\relcK\I_1 $ and faithfulness of $ \relc{\abst}$ imply $\I_1 \not\models \K$. But this contradicts $ \I_1\releqK\I_2  $, as  $\releq{\abst}$ is also faithful by assumption. 
	
	It remains to consider the case $\I_2 \not\models \K$.
	Then, by \Cref{lem:help}(a), there is a belief base $ \G $ with $ \I_1,\I_2\models\G $ such that $ \I_2\models\MC{K}\circ\G $ and $ \I_1\not\models\MC{K}\circ\G $.
	Therefore, by compatibility, $ \I_2 \in \min(\Mod{\G},\releq{\K}) = \Mod{\MC{K}\circ\G}$ and $ \I_1 \notin \min(\Mod{\G},\releq{\K}) = \Mod{\MC{K}\circ\G} $,
	a contradiction to $ \I_1\releqK\I_2  $ due to min-retractivity.
\end{proof}

Next, we would like to point out that the more smoothly and economically defined relation $\sqreleqc{\abst}$ (\Cref{def:relation_new_first_relation}) is very close to already serving the purpose of the somewhat more ``tinkered'' $\releqc{\abst}$ (\Cref{def:relation_new}). In fact, the very natural assumption of the existence of an ``non-constraining'' base that covers all interpretations makes the two relation encodings coincide. In most logics, such a base is trivially available (for instance, the empty base).

\begin{proposition}\label{prop:relation_equivalence}
    Let $ \mathbb{B}=(\MC{L},{\Omega},\models,\Bases,\Cup)$ be a base logic and $ \circ $ be a {\basechange} operator for $\mathbb{B}$ 
    that satisfies \postulate{G1}--\postulate{G3}, \postulate{G5}, and \postulate{G6}.
    If there exists a base $\G_\Omega\in\Bases$ such that $\Mod{\G_\Omega}=\Omega$,
    then ${\sqreleqcK}={\releqcK}$ for any $\K\in\Bases$, i.e. $\I_1\sqreleqcK\I_2$ if and only if $\I_1\releqcK\I_2$ for any $\I_1,\I_2\in\Omega$.
\end{proposition}
\begin{proof}
    Let $\I_1,\I_2$ be two interpretations and assume there exists a base $\G_\Omega\in\Bases$ such that $\Mod{\G_\Omega}=\Omega$.
    Then, for any $\K\in\Bases$, we have $\Mod{\K\Cup\G_\Omega} = \Mod{\K}\cap\Omega = \Mod{\K}$.
    We show the equivalence of $\sqreleqcK$ and $\releqcK$ in two directions:
    
    \begin{itemize}
        \item[``$\Rightarrow$'']
        Let $\I_1\sqreleqcK\I_2$.
        Assume for a contradiction that $\I_1\not\releqcK\I_2$.
        From \Cref{def:relation_new}, we have $\I_1\not\models\K$ and three cases: 
        $\I_1\models\K$, $\I_2\models\K$ or $\I_1\not\sqreleqcK\I_2$.
        The case $\I_1\models\K$ immediately contradicts $\I_1\not\models\K$ and the third case $\I_1\not\sqreleqcK\I_2$ contradicts our assumption $\I_1\sqreleqcK\I_2$.
        For the remaining case $\I_2\models\K$, since $\I_2\in\Mod{\K}=\Mod{\K\Cup\G_\Omega}$, from postulate \postulate{G2} we obtain $\Mod{\K\circ\G_\Omega}=\Mod{\K\Cup\G_\Omega}$.
        Now consider from \Cref{def:relation_new_first_relation} we have two subcases:
        \begin{itemize}
            \item[$\I_1$]$\models\K\circ\G_\Omega$. Since $\Mod{\K\circ\G_\Omega}=\Mod{\K\Cup\G_\Omega}$, we have $\I_1\in\Mod{\K\Cup\G_\Omega}=\Mod{\K}$, which contradicts $\I_1\not\models\K$.
            \item[$\I_2$]$\not\models\K\circ\G_\Omega$. Since $\Mod{\K\circ\G_\Omega}=\Mod{\K\Cup\G_\Omega}$, we have $\I_2\not\in\Mod{\K\Cup\G_\Omega}$, and hence $\I_2\not\in\Mod{\K}$, which contradicts our case assumption $\I_2\models\K$.
        \end{itemize} 
        
        \item[``$\Leftarrow$'']
        Let $\I_1\releqcK\I_2$. 
        In view of \Cref{def:relation_new}, we consider two cases: $\I_1\models\K$ or ($\I_1,\I_2\not\models\K$ and $\I_1\sqreleqcK\I_2$).
        The second case immediately yields the desired $\I_1\sqreleqcK\I_2$.
        For the former case, $\I_1\models\K$, assume for a contradiction $\I_1\not\sqreleqcK\I_2$.
        Then, there exists $\G\in\Bases$ with $\I_1,\I_2\in\Mod{\G}$ such that  $\I_1\not\models\K\circ\G$ and $\I_2\models\K\circ\G$.
        Since $\I_1\in\Mod{\K}\cap\Mod{\G}=\Mod{\K\Cup\G}$, from postulate \postulate{G2} we have $\Mod{\K\circ\G}=\Mod{\K\Cup\G}$. This implies $\I_1\in\Mod{\K\circ\G}$, which contradicts $\I_1\not\models\K\circ\G$.\qedhere
    \end{itemize}
\end{proof}

\subsection{Preorder Assignments and \Cref{def:relation_new}}\label{sec:transformation2tpo}

The uniform treatment of all detached pairs in \Cref{def:relation_new} may produce a non-preorder assignment even in cases where  an encoding by means of a preorder assignment were actually possible as demonstrated next.
\begin{example}\label{ex:nopreoderreleqck}
    Let $ \mathbb{B}=(\MC{L},{\Omega},\models,\Bases,\Cup)$ with $ \MC{L}=\{\bot,\varphi,\psi,\gamma_1,\ldots,\gamma_4 \} $ and $ \Omega=\{ \I_1,\ldots,\I_4 \} $, such that:
        \begin{equation*}
            \Mod{\bot} = \emptyset \qquad
            \Mod{\varphi} =\{ \I_1,\I_2,\I_4\}\qquad
            \Mod{\psi} =\{ \I_1,\I_3\}\qquad
            \Mod{\gamma_i} =\{ \I_i\}
        \end{equation*}
        Moreover let $\Bases = \{ \{\chi\} \mid \chi \in \mathcal{L}\}$ and let $\Cup$ be the idempotent, commutative binary function over $\Bases$ satisfying $\{\varphi\}\Cup\{\psi\} = \{\varphi\}\Cup\{\gamma_1\} = \{\psi\}\Cup\{\gamma_1\} = \{\gamma_1\}$ and producing $\{\bot\}$ in all other cases. 
        Let $\circ$ be as defined in the following operator table:
        \begin{equation*}
            \begin{array}{|c||c|c|c|c|c|c|c|}
                \hline
                \circ & \{\bot\} & \{\varphi\} & \{\psi\} & \{\gamma_1\} & \{\gamma_2\} & \{\gamma_3\} & \{\gamma_4\} \\
                \hline\hline	
                \{\bot\} & \{\bot\} & \{\varphi\} & \{\psi\} & \{\gamma_1\} & \{\gamma_2\} & \{\gamma_3\} & \{\gamma_4\} \\
                \hline
                \{\varphi\} & \{\bot\} & \{\varphi\} & \{\gamma_1\} & \{\gamma_1\} & \{\gamma_2\} & \{\gamma_3\} & \{\gamma_4\} \\
                \hline
                \{\psi\} & \{\bot\} & \{\gamma_1\} & \{\psi\} & \{\gamma_1\} & \{\gamma_2\} & \{\gamma_3\} & \{\gamma_4\} \\
                \hline
                \{\gamma_1\} & \{\bot\} & \{\gamma_1\} & \{\gamma_1\} & \{\gamma_1\} & \{\gamma_2\} & \{\gamma_3\} & \{\gamma_4\} \\
                \hline
                \{\gamma_2\} & \{\bot\} & \{\gamma_2\} & \{\psi\} & \{\gamma_1\} & \{\gamma_2\} & \{\gamma_3\} & \{\gamma_4\} \\
                \hline
                \{\gamma_3\} & \{\bot\} & \{\varphi\} & \{\gamma_3\} & \{\gamma_1\} & \{\gamma_2\} & \{\gamma_3\} & \{\gamma_4\} \\
                \hline
                \{\gamma_4\} & \{\bot\} & \{\gamma_4\} & \{\gamma_3\} & \{\gamma_1\} & \{\gamma_2\} & \{\gamma_3\} & \{\gamma_4\} \\
                \hline
            \end{array}
        \end{equation*}
    In particular, for $ \MC{K}=\{\gamma_4\} $, we thus obtain 
    \begin{align*}
        \K\circ\{ \varphi \}  & =\{ \gamma_4  \} & \Mod{\K\circ\{ \varphi \}} &=\{ \I_4 \}  \\
        \K\circ\{ \psi \}     & =\{ \gamma_3 \}  & \Mod{\K\circ\{ \psi \}} & = \{\I_3 \} \\
        \K\circ\{ \gamma_i \} & =\{ \gamma_i \}  & \Mod{\K\circ\{ \gamma_i \}} & =\{\I_i\}
    \end{align*}
    \Cref{fig:no-po} shows that the assignment $\releqc{(.)}$ derived from $\circ$ is not a preorder assignment, while \Cref{fig:yes-po} demonstrates that such an assignment for $\circ$ indeed exists.

\end{example}

\begin{figure}[t]
    \centering
    \begin{subfigure}[t]{0.99\textwidth}
        \centering
        \begin{tikzpicture}[]
            
            \node[balls, inner sep=0.1, scale=0.75] (w0) at (-2.0,0) {\Large${\omega_4}$};
            \node[balls, inner sep=0.1, scale=0.75] (w4) at (1.0,0) {\Large${\omega_3}$};
            \node[balls, inner sep=0.1, scale=0.75] (w2) at (4,0) {\Large${\omega_2}$};
            \node[balls, scale=0.75] (w5) at (7,0) {\Large${\omega_1}$};

            \path[-Stealthnew] (w0) edge node[below, scale=1] {$\relc{\K} $} (w4) ;
            \path[Stealthnew-Stealthnew] (w4) edge node[below, scale=1] {$\releqc{\K} $} (w2) ;
            \path[Stealthnew-Stealthnew] (w2) edge node[below, scale=1] {$\releqc{\K} $} (w5) ;
            
            \path[-Stealthnew] (w0) edge[bend left=15] node[above, scale=1] {$\relc{\K} $} (w2.north west) ;
            \path[-Stealthnew] (w0) edge[bend left=28] node[above, scale=1] {$\relc{\K} $} ([yshift=.08cm,xshift=0.17cm]w5.north west) ;
            \path[-Stealthnew] (w4) edge[bend left=15] node[above, scale=1] {$\relc{\K} $} (w5.north west) ;

        \end{tikzpicture}
        \caption{Relation $ \releqcK$ for $\K=\{\gamma_4\}$. No preorder as 
            $\I_1 \releqcK \I_2$ and $\I_2 \releqcK \I_3$, yet $\I_1 \not\releqcK \I_3$.}
        \label{fig:no-po}
    \end{subfigure}
    \begin{subfigure}[t]{0.99\textwidth}
        \centering
        \begin{tikzpicture}[]
            
            \node[balls, inner sep=0.1, scale=0.75] (w0) at (-2.0,0) {\Large${\omega_4}$};
            \node[balls, inner sep=0.1, scale=0.75] (w4) at (1.0,0) {\Large${\omega_3}$};
            \node[balls, inner sep=0.1, scale=0.75] (w2) at (4,0) {\Large${\omega_2}$};
            \node[balls, inner sep=0.1, scale=0.75] (w5) at (7,0) {\Large${\omega_1}$};

            \path[-Stealthnew] (w0) edge node[below, scale=1] {$\rel{\K} $} (w4) ;
            \path[-Stealthnew] (w4) edge node[below, scale=1] {$\rel{\K} $} (w2) ;
            \path[Stealthnew-Stealthnew] (w2) edge node[below, scale=1] {$\releq{\K} $} (w5) ;
            
            \path[-Stealthnew] (w0) edge[bend left=15] node[above, scale=1] {$\rel{\K} $} (w2.north west) ;
            \path[-Stealthnew] (w0) edge[bend left=28] node[above, scale=1] {$\rel{\K} $} ([yshift=.08cm,xshift=0.17cm]w5.north west) ;
            \path[-Stealthnew] (w4) edge[bend left=15] node[above, scale=1] {$\rel{\K} $} (w5.north west) ;

        \end{tikzpicture}
        \caption{Appropriate preorder encoding $\releqK$ of preference relation with respect to $\K=\{\gamma_4\}$  for $\circ$.}
        \label{fig:yes-po}
    \end{subfigure}
    \caption{Illustrations of the relations used in \Cref{ex:nopreoderreleqck}.}
\end{figure}
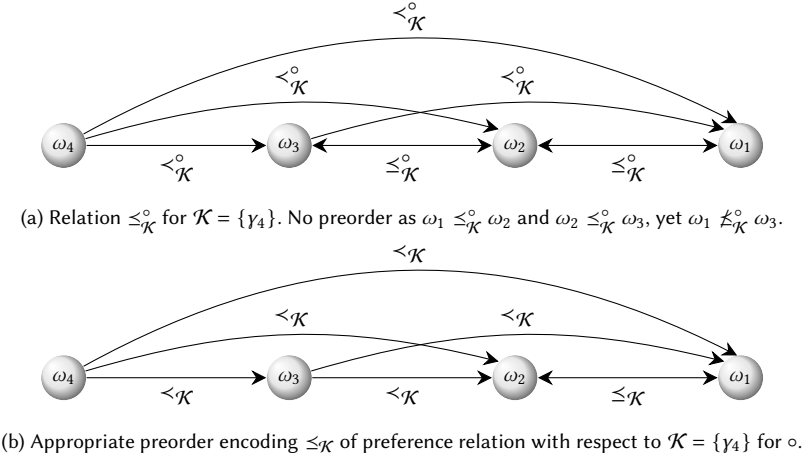

   While \Cref{def:relation_new} sometimes fails to yield preorder assignments despite their existence,
    we will now see how to obtain a preorder assignment \( \rreleqc{\abst} \) from $\releqc{(.)}$, whenever \sauerwald{an compatible preorder assignment exists}. 
The following result by Bengt Hansson, stating that every preorder \( {\leq} \) can be extended to a total preorder \( {\leq^\mathrm{lin}} \) by ``linearizing``, is an essential part of this procedure. 

\vspace{-0.5em}
    \begin{theorem}[{{Hansson} \cite[Lemma 3]{KS_Hansson1968}}]\label{thm:tpo_extension}
        Assume the \emph{axiom of choice}.
        For every preorder \( {\leq} \) on a set \( X \) there exists a unique total preorder \( {\leq^\mathrm{lin}} \) on \( X \) such that
        \begin{itemize}
            \item if \( x \leq y  \), then \( x \leq^\mathrm{lin} y \), and
            \item if \( x \leq y  \) and \( y \not\leq x  \), then \( x \leq^\mathrm{lin} y  \) and \( y \not\leq^\mathrm{lin} x  \).
        \end{itemize}        
    \end{theorem}
The following definition captures the transformation of $\releqcK$ into a suitable total preorder \( \rreleqcK \).
\begin{definition}\label{def:relation_tpo_reduction}
	Let $ \mathbb{B} = (\MC{L},\Omega,\models,\Bases,\Cup) $ be a base logic,
	let $\circ$ be a \basechange{} operator for \( \mathbb{B} \)
	and let $\K \in \Bases $ be a belief base. 
	The relation $ \rreleqcK $ over  ${\Omega} $ is defined by	
        \begin{equation*}
            {\rreleqcK}  =  (TC({\releqcK}\setminus\setAllDetached))^\mathrm{lin} \ .
        \end{equation*}
\end{definition}
That is, the first step of transforming $\releqcK$ into \( {\rreleqcK} \) consists in drastically removing all (non-reflexive) detached pairs $\setAllDetached$ from \( \releqcK \), resulting in \( \relationstepdashK \). 
The relation \( \relationstepdashK \) will be a non-transitive and non-total relation, but minima of models of bases will be preserved. 
We will then extend \( \relationstepdashK \) to a transitive relation \( \relationstepdashdashK \) in the second step, by taking the transitive closure. 
One can show that by this step only elements from $\setAllDetached$ can be added back by the transitive closure, which guarantees that, again, minima of models of bases are preserved.
In a last step, we obtain the final result $\rreleqc{(.)}$ by ``linearizing`` \( \relationstepdashdashK \) to a total preorder in a way that minima of models of bases are again preserved.
In summary, the steps are the following:
\begin{align*}
    {\relationstepdashK} & = {\releqcK\setminus\setAllDetached} \tag{Step I: remove detached pairs}\\
    {\relationstepdashdashK} & = TC({\relationstepdashK}) =TC({{\releqcK}\setminus\setAllDetached}) \tag{Step II: transitive closure} \\
    {\rreleqcK} & =   (\relationstepdashdashK)^\mathrm{lin}   =   (TC({\relationstepdashK}))^\mathrm{lin} =  (TC({\releqcK}\setminus\setAllDetached))^\mathrm{lin} \tag{Step III: linearizing}
\end{align*}
The following theorem shows that these transformations yield indeed a preorder assignment, whenever possible.
\begin{theorem}\label{prop:encodedown}
    Let $ \mathbb{B}=(\MC{L},{\Omega},\models,\Bases,\Cup)$ be a base logic and let \( \circ \) be a change operator for \( \mathbb{B} \) that satisfies \postulate{G1}-\postulate{G3}, \postulate{G5} and \postulate{G6}. 
    If \( \circ \) is compatible with some min-friendly faithful preorder assignment, then \( {\rreleqc{\abst}} \) is a min-expressible min-friendly faithful preorder assignment compatible with~\( \circ \).
\end{theorem}
\begin{proof}[Proof (sketch)]
    Let \( {\propto_{\abst}} \) be a 
    min-expressible min-complete quasi-faithful preorder assignment compatible with~\( \circ \). %
    The proof of \Cref{prop:agm_withoutg4} yields that  $\releqc{\abst}$ from \Cref{def:relation_new} is a min-friendly quasi-faithful assignment compatible with~$ \circ $.
Now let \( \K\in\Bases \) be an arbitrary base of \( \mathbb{B} \).
We make some observations for the start. Both, \( \propto_{\K} \) and \( \releqcK \), are total relations that are min-friendly, i.e., min-complete and min-retractive.
We consider each  of the transformation steps from \( {\releqcK} \) to \( {\rreleqcK} \).
\begin{itemize}
    \item[]\emph{Step I: Removing detached pairs.}
    From the definition of \( \setAllDetached \) we obtain \( \min(\Mod{\G},\relationstepdashK)=\min(\Mod{\G},\releqcK) \) for all \( \Gamma\in\Bases \).
    This is because for every \( \I,\I' \in \Mod{\G} \) with \( \I \in \min(\Mod{\G},\releqcK) \) we have \( \I \models \K\circ\G \) by compatibility of \( \releqc{\abst} \) with \( \circ \). 
    Consequently, the pair \( (\I,\I') \) is not detached \sauerwald{from \( \circ \) in \( \K \)} and thus $\min(\Mod{\G},\relationstepdashK)=\min(\Mod{\G},\releqcK)$.
    Note also that by totality of \( \releqcK \) and the definition of \( \setAllDetached \), we obtain that \( \relationstepdashK \) is a reflexive relation.
    
    \smallskip
    \item[]\emph{Step II: Taking the transitive closure.}
    First, observe that because \( \relationstepdashK \) is a reflexive relation, we have that \( \relationstepdashdashK \) is a preorder.
    \sauerwald{Second, we show that \( \min(\Mod{\G},\relationstepdashdashK)=\min(\Mod{\G},\releqcK) \) holds for all \( \Gamma\in\Bases \). 
        Observe that removing \( \setAllDetached \) from \( \releqcK \) in Step I renders \( \relationstepdashK \) the smallest relation with \( \min(\Mod{\G},\relationstepdashK)=\min(\Mod{\G},\releqcK) \).
Because \( {\propto_{\abst}} \) and $\releqc{\abst}$ are compatible with \( \circ \), we obtain  \( \min(\Mod{\G},\propto_{\K})=\min(\Mod{\G},\releqcK) \) \sauerwald{from \Cref{prop:quasiff_similar}}. 
        By using the last two observations, we conclude \( {\relationstepdashK} \subseteq  {\propto_{\K}} \) from \( \min(\Mod{\G},\propto_{\K})=\min(\Mod{\G},\releqcK) \).}
    Recall that \( TC \) preserves set-inclusion \sauerwald{and that \( TC \) is idempotent}. Consequently, \( {\relationstepdashK} \subseteq  {\propto_{\K}} \) implies \( TC({\relationstepdashK}) = {\relationstepdashdashK} \subseteq  {\propto_{\K}} = {TC(\propto_{\K})} \).
    \sauerwald{Finally, from} \( {\relationstepdashdashK} \subseteq  {\propto_{\K}} \) we obtain that \( \min(\Mod{\G},\relationstepdashdashK)=\min(\Mod{\G},\releqcK) \) holds for all \( \Gamma\in\Bases \).
    
    \smallskip
    \item[]\emph{Step III: Linearizing.}
    Inspecting \Cref{thm:tpo_extension} yields directly that the transformation from \( \relationstepdashdashK \) to \( {\rreleqcK} \) preserves minimal elements with respect to bases too, i.e., we have
        \( \min(\Mod{\G},\rreleqcK)
          = \min(\Mod{\G},\relationstepdashdashK)  \)
 for all \( \Gamma\in\Bases \).
\end{itemize}
Considering the whole transformation from \( \releqcK \) to \( \rreleqcK \) reveals that \( \rreleqcK \) is a total preorder over \( \Omega \), while minima are preserved in each step, i.e., we have \( \min(\Mod{\G},\rreleqcK)
= \min(\Mod{\G},\releqcK)  \).
This shows that \( \rreleqc{\abst} \) is a faithful preorder assignment that is compatible with \( \circ \).
Because every total preorder is also min-retractive and min-completeness is inherited from \( {\releqcK} \) by preservation of minimal elements, we have that \( \rreleqc{\abst} \)  is also min-friendly. 
Min-expressibility is given by compatibility with \( \circ \).
\end{proof}
\sauerwald{Note that the specific pre-requirement of ``the existence of a compatible preorder assignment'' in \Cref{prop:encodedown} is only used in Step II of the proof. 
    This will be very helpful in the next section, where we will heavily use the transformation from \( \releqcK \) to \( {\rreleqcK} \) for characterizing conditions that guarantee the representability of revision operators by preorder assignments.}

\section{Characterizations of Total-Preorder-Representability}\label{sec:logic_capture_hidden_cycles}
As we showed, 
not every \sauerwald{base change operator that satisfies satisfy the AGM revision postulates in every base logic can be described by a preorder assignment}. Yet, we also saw that, for some logics (like $\mathbb{PL}_n)$, this correspondence does indeed hold. 
In \Cref{sec:enc_operators}, we considered a general approach for obtaining a min-complete (quasi-)faithful preorder assignment for \sauerwald{{\basechange} operators that satisfy the AGM revision postulates} in such base logics.
This section is dedicated to  
precisely characterize conditions
under which \sauerwald{base change operators that satisfy the AGM revision postulates} are representable by a compatible min-complete (quasi-)faithful preorder assignment. 
The following definition captures the notion of operators that are well-behaved in that sense.

\begin{definition}[total-preorder-representable]
	A \basechange\ operator $ \circ $ for some base logic is called \emph{total-preorder-representable} if there is a  min-complete quasi-faithful preorder assignment compatible with~$ \circ $. 
\end{definition}

Recall that transitivity implies min-retractivity, and thus, every min-complete preorder is automatically min-friendly.
Moreover, in view of \Cref{sec:syntax_independence}, our definition uses the more lenient notion of quasi-faithfulness to accommodate the syntax-dependent setting. However, as the following lemma shows, the same definition of total-preorder-representability is adequate in the syntax-independent setting.
\begin{lemma}\label{lem:doesntmatter}
	For any \basechange\ operator $ \circ $ that satisfies \postulate{G4}, total-preorder-repre\-sen\-ta\-bi\-li\-ty coincides with the existence of a  min-complete \emph{faithful} preorder assignment compatible with~$ \circ $. 	
\end{lemma}
\begin{proof}
	Any compatible min-complete \emph{faithful} preorder assignment is also \emph{quasi-faithful} and hence the existence of such an assignment implies total-preorder-representability.
	For the other direction, let $\preceq_{\abst}$ be a min-complete quasi-faithful preorder assignment compatible with~$\circ$. 
    \Cref{col:quaisVSnonquasiG4} yields existence of a min-complete \emph{faithful} preorder assignment compatible with~$ \circ $.
\end{proof}
As observed by Delgrande et. al. \cite{KS_DelgrandePeppasWoltran2018}, \emph{non}-total-preorder-repre\-sen\-ta\-blity is connected with the existence of cyclic interrelations of some kind. 
    This can be made concrete by considering the representation of change operators via assignments.
    We say a total relation \( {\preccurlyeq} \subseteq \Omega\times\Omega \) contains a \emph{proper cycle} if there is a finite sequence \( \I_0,\ldots,\I_n \in \Omega \) of at least three interpretations (,i.e., \( n\geq 2 \)), such that \( \I_{1} \not\preccurlyeq \I_{0} \) and \( \I_{i} \preccurlyeq \I_{i\oplus 1} \) for each  $i$ with \( 0\leq i \leq n \), where \( \oplus \) denotes addition  \( \mathrm{mod}\,  n+1 \).    
The following proposition is a consequence of the fact that every total non-transitve relation contains a proper cycle.
    \begin{proposition}\label{prop:implycircle}
        Let $\mathbb{B} = (\MC{L},\Omega,\models,\Bases,\Cup)$ be a base logic and let \( \circ \) be a base change operator that is compatible with some min-friendly quasi-faithful assignment \( \releq{\abst} \).
        If \( \circ \) is not total-preorder-repre\-sen\-ta\-ble, then there exists some \( \K\in\Bases \) such that  \( \releqK \) contains a proper cycle.
    \end{proposition}
\begin{proof}
Let \( \K \) be such that \( \releqK \) is not transitive. Pick arbitrary \( \I_{0},\I_{1},\I_{2} \in\Omega \) witnessing the violation of transitivity, i.e., \( \I_{0} \releqK \I_{1} \) and \( \I_{1} \releqK \I_{2} \) and \( \I_{0} \not \releqK \I_{2} \). As \( \releqK \) is total, we immediately obtain that \( \I_{2} \releqK \I_{0} \) holds. Consequently, \( \I_0,\ldots,\I_2 \in \Omega \) is a proper cycle.
\end{proof}
\sauerwald{\Cref{prop:implycircle} confirms that non-total-preorder-repre\-sen\-ta\-bility implies cyclicity.
Yet, 
    it is less obvious whether the converse also holds, i.e., whether} the absence of a ``certain internal cyclicity'' guarantees total-preorder-representability of {\basechange} operators. 
    \sauerwald{In the following sections, we characterize the specific type of acyclicity that guarantees total-preorder-representable {\basechange} operators and we characterize the specific type of acyclicity for base logics that guarantee universal total-preorder-representability}:
	\begin{itemize}
        \item \emph{Acyclic base change operators.}  %
        For particular logics with a finite set of interpretations, Delgrande et. al. \cite{KS_DelgrandePeppasWoltran2018} show that total-preorder-representability coincides with the satisfaction of one extra postulate for operators, called \postulate{Acyc}.
        We adapt this postulate to arbitrary base logics obtaining an appropriate notion of \emph{acyclic} base change operators.
        Lifting the result by Delgrande et. al. \cite{KS_DelgrandePeppasWoltran2018}, we show that a base change operator satisfying \postulate{G1}--\postulate{G6} is total-preorder-representable if and only if it is acyclic.        
\item \emph{Assignments with forged proper cycles.} The notion of forged proper cycles describes proper cycles where the cyclic structure of interpretations is justified by the existence of bases that ``form'' this proper cycle. We show that for each assignment that is compatible with a {\basechange} operator contains a forged proper cycle if and only if the operator is not total-preorder-representable.
		\item \emph{Loop-free base logics.} We precisely identify those base logics exhibiting universal total-preorder-repre\-sen\-ta\-bi\-li\-ty, i.e. for which every base change operator satisfying \postulate{G1}--\postulate{G6} is total-preorder-representable. 
        To this end, we introduce the notion of a \emph{\criticalloop} for a base logic \( \mathbb{B} \), which, roughly speaking, consists of several bases from  \( \mathbb{B} \) which induce a circular preference arrangement, which is not ``short-circuited'' by any other base that would be consistent with more than two bases of the sequence.     
        It will turn out that {\criticalloop}s are a forbidden substructure for base logics regarding total-preorder-repre\-sen\-ta\-bi\-li\-ty, as absence of a {\criticalloop} is the necessary and sufficient condition for universal total-preorder-repre\-sen\-ta\-bi\-li\-ty.
         We will call a base logic \( \mathbb{B} \) loop-free if there is no {\criticalloop} for~\( \mathbb{B} \).
	\end{itemize}
We start by considering acyclic base change operators and show coincidence with total-preorder-representability in the next section.
At the end of the same section, we introduce and consider forged proper cycles.
In the subsequent section,  we  show that loop-free base logics are exactly those base logics where every \sauerwald{{\basechange} operator that satisfies \postulate{G1}--\postulate{G6}} is total-preorder-representable. 
Then, we will encounter further consequences of the prior two characterization results and will consider two examples in the last part of this section.

\subsection{Characterizing Total-Preorder-Representable Base Change Operators}\label{sec:acycresult}

In this section, we capture total-preorder-representability on the operator level. We consider a necessary and sufficient criterion for \sauerwald{base change operators} for \( \mathbb{B} \) such that total-preorder-representability is guaranteed\sauerwald{, when presuming that \postulate{G1}--\postulate{G3}, \postulate{G5}, and \postulate{G6} are satisfied}.
Delgrande et. al. \cite{KS_DelgrandePeppasWoltran2018} add an extra postulate to \postulate{G1}--\postulate{G6} to rule out the ``unnatural'' change operators that do \textbf{not} correspond to some preorder assignment. The postulate can be expressed in the framework of base logics as follows:
\begin{definition}\label{def:acyclic}
	Let $\mathbb{B} = (\MC{L},\Omega,\models,\Bases,\Cup)$ be a base logic.
	A change operator $ \circ $ for $ \mathbb{B} $ is \sauerwald{called \emph{acyclic} if it satisfies}:
	\begin{quote}
		\definepostulate{Acyc} Let $\K\in \Bases$ be any base and $\G_0,\ldots,\G_n \in \Bases$ any sequences of bases with $\Mod{\G_i \Cup (\K\circ \G_{i\oplus 1})}\neq\emptyset$ for each $0\leq i \leq n$, where \( \oplus \) denotes addition \( \mathrm{mod}\, {(n+1)} \). Then $\Mod{\G_0 \Cup (\K\circ \G_n)}\neq\emptyset$ holds.
	\end{quote} 
\end{definition}
With these ingredients in place, {Delgrande et al.}~\cite{KS_DelgrandePeppasWoltran2018}~establish that, for the logics they consider (see also \Cref{sec:related_works}), there is a two-way correspondence between those AGM revision operators satisfying \postulate{Acyc} and min-expressible faithful preorder assignments. In the following, we recall the original result by {Delgrande et al.}~\cite{KS_DelgrandePeppasWoltran2018} translated into our notation.
\begin{proposition}[{\cite{KS_DelgrandePeppasWoltran2018}}]\label{prop:delgrande_acyc}
	Let $\mathbb{B} = (\MC{L},\Omega,\models,\Bases,\Cup)$ be a base logic where \( \Omega \) is finite, \( \Bases=\MC{P}(\MC{L}) \), and for every two interpretations \( \I_1,\I_2\in\Omega \) there exists some sentence \( \varphi\in\MC{L} \) with \( \I_1\models\varphi \) and \( \I_2\not\models\varphi \). Let $ \circ $ be a change operator for $ \mathbb{B} $ that satisfies \postulate{G1}--\postulate{G6}.
	Then $ \circ $ is acyclic if and only if $ \circ $ is total-preorder-representable.
\end{proposition}

We generalize \Cref{prop:delgrande_acyc} by {Delgrande et al.}~\cite{KS_DelgrandePeppasWoltran2018} to the full setting of base logics.
\begin{theorem}\label{thm:acyctpo}
	Let $\mathbb{B} = (\MC{L},\Omega,\models,\Bases,\Cup)$ \sauerwald{be a base logic} and let $ \circ $ be a change operator for $ \mathbb{B} $ that satisfies \postulate{G1}--\postulate{G3}, \postulate{G5}, and \postulate{G6}.
	Then $ \circ $ is acyclic if and only if $ \circ $ is total-preorder-representable.
\end{theorem}
In the following we provide an outline of the proof.
A full proof of \Cref{thm:acyctpo} is given in \Cref{sec:app_acyc}. %
\begin{proof}[Proof (outline)]
	We consider both directions of the proof independently.
	\begin{itemize}
		\item[]\hspace{-1.4em} \emph{From total-preorder-representability to acyclicity.} \Cref{thm:rep2} yields that there exists a min-complete quasi-faithful preorder assignment $\releq{\abst}$ compatible with~$ \circ $. Let \( \G_0, \ldots, \G_n \in \Bases \) be arbitrary bases such that the premises of \postulate{Acyc} are satisfied. Compatibility of  $\releq{\abst}$ with~$ \circ $ yields for each \( \K \) a relation between the minimal models of the bases \( \G_0, \ldots, \G_n \) with respect to $\releqK$. By employing transitivity of $\releqK$, one obtains the conclusion of \postulate{Acyc}.
		\item[]\hspace{-1.4em}\emph{From acyclicity to total-preorder-representability.}
        We show that  \( \rreleqc{\abst} \) from \Cref{sec:transformation2tpo} is a min-complete quasi-faithful preorder-assignment compatible with~$ \circ $. 
            The overall structure of the proof is the same as for \Cref{prop:encodedown}.             
        The critical difference  to the proof of \Cref{prop:encodedown} is the argumentation required for showing that taking the transitive closure of \( \relationstepdashK \) in ``Step II'' only adds detached pairs.
        To this end, we show that acyclicity of \( \circ \) ensures that taking the transitive closure only does add detached pairs, i.e., \( {\relationstepdashdashK} \subseteq {\releqc{\K}} \) holds.
        The first observation is that the acyclicity of  \( \circ \) implies that every proper cycle in \( \releqcK \) contains a detached pair. 
        One can show that if the transitive closure added a non-detached pair \( (\I,\I') \) to \( \relationstepdashK \), this would imply that  \( (\I,\I') \) is part of a proper cycle in \( \releqcK \) without any detached pairs. This yields a contradiction to the observation that every proper cycle in \( \releqcK \) contains a detached pair. \qedhere
	\end{itemize}
\end{proof}

	We employ now \Cref{thm:acyctpo} to extend \Cref{prop:implycircle} such that we will capture the characteristics of those proper cycles that are responsible for the non-total-preorder-representablity of an base change operator. Consider the following specific kind of proper cycles, where the cycle has a counterpart on the leve of bases of the underlying base logic.	
	\begin{definition}\label{def:forgedcycle}
        Let $\mathbb{B} = (\MC{L},\Omega,\models,\Bases,\Cup)$ be a base logic and let $ {\preccurlyeq} \subseteq \Omega\times\Omega $ be a total relation.
        A proper cycle \( \I_{0},\ldots,\I_{n} \in \Omega \) of $ {\preccurlyeq} $ is called \emph{forged by $\mathbb{B}$} if there exist bases \( \G_{{i},{i\oplus 1}} \) for \( 0\leq i\leq n \) such that
        \begin{itemize}
            \item \( \I_{i},\I_{i\oplus 1} \models  \G_{{i},{i\oplus 1}}   \)  for each \( 0\leq i\leq n \),
            \item \( \I_{i} \in \min( \Mod{\G_{{i},{i\oplus 1}}}, {\preccurlyeq} ) \)  for each \( 0\leq i\leq n \), and
            \item \(  \Mod{ \G_{{1},{2}} } \cap \min( \Mod{\G_{{0},{1}}}, {\preccurlyeq} ) = \emptyset \),
        \end{itemize}
    where \( \oplus \) denotes addition \( \mathrm{mod}\, n+1 \).
    \end{definition}
The following theorem shows that forged proper cycles characterizes non-total-preorder-representablity from a relational perspective, i.e. a base change operator \( \circ \) that satisfies \sauerwald{\postulate{G1}--\postulate{G3}, \postulate{G5}, and \postulate{G6}} is total-preorder-representable if and only if every assignment compatible with \( \circ \) does not contain a forged proper cycle.
\begin{theorem}\label{thm:forgedtpo}
    Let $\mathbb{B} = (\MC{L},\Omega,\models,\Bases,\Cup)$ be a base logic, let $ \circ $ be a change operator for $ \mathbb{B} $ that satisfies \postulate{G1}--\postulate{G3}, \postulate{G5}, and \postulate{G6} and let \( \releq{\abst} \) be a min-friendly quasi-faithful assignment compatible with \( \circ \).
    Then \( \circ \) is total-preorder-representable if and only if for each \( \K\in\Bases \), the corresponding \( \releqK \) does not contain a proper cycle \( \I_{0},\ldots,\I_{n} \) forged by $\mathbb{B}$.
\end{theorem}
\pagebreak[3]
\begin{proof}[Proof (sketch)]
    We show both directions independently.
    \begin{itemize}
            \item[``$\Rightarrow$''] We show this direction by contraposition. Assume that \( \releqK \) has a proper cycle \( \I_{0},\ldots,\I_{n} \) that is forged by $\mathbb{B}$.
            By definition, there are bases  \( \G_{{i},{i\oplus 1}} \) as given in \Cref{def:forgedcycle}.  By employing compatibility of \( \releq{\abst} \) with $\circ$, we obtain that \( \G_{{i},{i\oplus 1}} \) witness that \postulate{Acyc} is violated.
            Consequently \( \circ \) is acyclic, and thus, by \Cref{thm:acyctpo}, $\circ$, is not total-preorder-representable.
            \item[``$\Leftarrow$''] We show this direction by contraposition. Assume that $ \circ $ is not total-preorder-representable. By employing \Cref{thm:acyctpo}, obtain that $\circ$ violates \postulate{Acyc}, i.e., there exist $\K\in\Bases$ and $\G_0,\ldots,\G_n \in \Bases$ with $\Mod{\G_i \Cup (\K\circ \G_{i\oplus 1})}\neq\emptyset$ for each $i$ with $0\leq i \leq n$ such that $ \Mod{\G_0 \Cup (\K\circ \G_n)}=\emptyset $, where \( \oplus \) denotes addition \( \mathrm{mod}\, {(n+1)} \). 
            From $\Mod{\G_i \Cup (\K\circ \G_{i\oplus 1})}\neq\emptyset$ we obtain that there exist two interpretations $ \I_{i\oplus 1},\I_{i\oplus 2} $ with $ \I_{i\oplus 1},\I_{i\oplus 2} \models \G_{i} $, and \( \I_{i\oplus 1} \models \K\circ \G_{i} \), and \( \I_{i\oplus 2} \models \G_i \Cup (\K\circ \G_{i\oplus 1}) \) for each $i$ with $0\leq i \leq n$. Moreover, $\Mod{\G_0 \Cup (\K \circ \G_n)}=\emptyset $ implies that $ \I_{1} \not\models \K \circ \G_n $ holds. Recall that \( \releq{\abst} \) is compatible with $\circ$ , and so we have $ \Mod{\K\circ\G} = \min(\Mod{\G},\releqK) $ for each $\G\in\Bases$. By using this and setting $ \G_{{i\oplus 1},{i\oplus 2}}=\G_{i} $ for each $i$ with $0\leq i \leq n$ we observe that \( \I_{0},\ldots,\I_{n} \in \Omega \) is forged by $\mathbb{B}$.            
            It remains to show that \( \I_{0},\ldots,\I_{n} \in \Omega \) is a proper cycle.
            From the compatibility of \( \releq{\abst} \) with $\circ$, and $ \I_{i},\I_{i\oplus 1} \models \G_{i,\oplus 1} $, and \( \I_{i} \models \K\circ \G_{i,\oplus 1} \) we obtain \( \I_{i} \releqK \I_{i\oplus 1} \).
            From the compatibility of \( \releq{\abst} \) with $\circ$, and $\I_{1}\models\G_{0,1}$, and $\Mod{\G_{1,2} \Cup (\K \circ \G_{0,1})}=\emptyset $ we obtain $\I_{1} \not\releqK \I_{0}$.            
            \qedhere
        \end{itemize}
\end{proof}
In the next section, we consider total-preorder-representability from a global perspective and provide a characterizing criterion for a base logic where every \sauerwald{base change operator that satisfies \postulate{G1}--\postulate{G3}, \postulate{G5}, and \postulate{G6}} is total-preorder-representable.

\subsection{Characterizing Base Logics with Universal Total-Preorder-Representability}\label{sec:logics_tpo_representable}

The following definition describes the occurrence of a certain relationship between several bases. Such an occurrence will turn out to be the one and only reason to prevent total-preorder-representability.

\begin{definition}[critical loop, loop-free]\label{def:critical_loop}
    Let $\mathbb{B}=(\MC{L},{\Omega},\models,\Bases,\Cup)$ be a base logic.\\
    Three or more bases $ \G_{0,1},\allowbreak\G_{1,2},\ldots,\G_{n,0} \in\Bases$ are said to form a \emph{critical loop of length \( (n+1) \) for $ \mathbb{B}$} if
    there exists a base \( \K\in\Bases \) and consistent bases $\G_0,\ldots,\G_n \in \Bases$  such that 
    \begin{itemize}\setlength{\itemsep}{0pt}
        \item[]\hspace{-4.4ex}(1)~~$ \Mod{\K\Cup \G_{i,i \oplus 1}}  = \emptyset$ for every \( i\in\{0,\ldots,n\} \),  where  \( \oplus \) is addition \( {\mathrm{mod}\, (n+1)} \),
        \item[]\hspace{-4.4ex}(2)~~$ \Mod{\G_{i}} \cup \Mod{\G_{i \oplus 1}} \subseteq \Mod{\G_{i,i\oplus 1}} $ and  \( \Mod{\G_{j}\Cup\G_{i}} = \emptyset \) for each \( i,j\in\{0,\ldots,n\} \) with \( i \neq j \),  and
        
        \item[]\hspace{-4.4ex}(3)~~for each \( \coverbase \in \Bases \) that is consistent with at least three bases from \( \G_{0}, \ldots, \G_{n} \),
        there exists some \( \coverbase' \in\Bases \) such that \( \Mod{\coverbase'}\neq\emptyset \) and \( \Mod{\coverbase'} \subseteq \Mod{\coverbase} \setminus \left( \Mod{\G_{0,1}}\cup\ldots\cup\Mod{\G_{n,0}} \right) \).
    \end{itemize}
    A base logic $ \mathbb{B} $ is called \emph{loop-free} if there is no critical loop for $ \mathbb{B} $.
\end{definition}

The three conditions in \Cref{def:critical_loop} describe the canonic situation brought about by some bases \( \G_{0,1},\ldots,\G_{n,0} \) allowing for the construction of a revision operator that unavoidably gives rise to a circular compatible relation. Note that due to Condition (3), every three of \( \G_{0,1},\G_{1,2},\ldots,\G_{n,0} \) together are inconsistent, but each two of them which have an index in common are consistent, i.e. \( \G_{i,i\oplus{1}}\Cup \G_{i\oplus{1},i\oplus{2}} \) is consistent for each \( i\in\{0,\ldots,n\} \).

\begin{figure}[t]
    \centering
    \begin{subfigure}[t]{0.48\textwidth}
        \centering
        \begin{tikzpicture}
            \node [inner sep=0.1em,plainballs, scale=1.6] (q) at (-1.7,2.2) {};
            \node [inner sep=0.23em,no-outline] (p) at (q) {\small $\Mod{\K}$};
            
            \node [inner sep=0.26em,plainballs] (u) at (0,1.25) {\small $\Mod{\G_{\!0}}$};
            \node [inner sep=0.1em,no-outline, scale=1.5] (a) at (u) {};
            
            \node [inner sep=0.26em,plainballs] (l) at ([shift={(a)}] 90:2) {\small $\Mod{\G_{\!1}}$};
            \node [inner sep=0.1em,no-outline, scale=1.5] (b) at (l) {};
            
            \draw[-] (a.west) -- (b.west); 
            \draw[-] (a.east) -- (b.east); 
            \draw[-] (a.west) arc (180:360:0.61cm);
            \draw[-] (a.north east) arc (45:225:0.61cm);
            \draw[-] (b.west) arc (180:0:0.61cm);
            
            \node [inner sep=0.26em,plainballs] (l+1) at ([shift={(b)}] 45:2.1) {\small $\Mod{\G_{\!2}}$};
            \node [inner sep=0.1em,no-outline, scale=1.5] (c) at (l+1) {};
            
            \drawthingiies{c.north west}{c.south east}{b.south east}{b.north west}{-}{black};
            
            \node [inner sep=0.1em,no-outline, scale=1.5] (d) at ([shift={(c)}] 0:1.5) {};
            
            \drawthingiies{d.north}{d.south}{c.south}{c.north}{dotted}{black};

            \node [inner sep=0.26em,plainballs] (u-1) at ([shift={(a)}] -45:2.1) {\small $\Mod{\G_{\!n}}$};
            \node [inner sep=0.1em,no-outline, scale=1.5] (f) at (u-1) {};
            
            \node [inner sep=0.1em,no-outline, scale=1.5] (e) at ([shift={(f)}] 0:1.5) {};
            
            \drawthingiies{e.north}{e.south}{f.south}{f.north}{dotted}{black};
            
            \path[-, dash pattern=on \pgflinewidth off 10pt, ultra thick, lightgray] ([shift={(d.east)}] 0:-0.8) edge[bend left=90,looseness=1.5] ([shift={(e.east)}] 0:-0.8);
            
            \draw[-] (f.north east) arc (45:-135:0.61cm);
            
            \draw[-] (a.north east) -- (f.north east); 
            \draw[-] (a.south west) -- (f.south west); 
            
            \node (g01) at (0,2.2) {\small $\Mod{\G_{\!0,1}}$};
            \node (g12) at (0.66,4) {\small $\Mod{\G_{\!1,2}}$};
            \node (gn0) at (0.66,0.46) {\small $\Mod{\G_{\!n,0}}$};
            \node[no-outline] (xx) at (-0.25,-1.25) {};
            
        \end{tikzpicture}
        \caption{By Condition (2), the models of each base \( \G_{{i},{i\oplus 1}} \) encompass the models of \( \G_{i} \) and of \( \G_{i\oplus 1} \), while by Condition (1), all these model sets are disjoint from the models of \( \K \).}
        \label{fig:critical_loop_condition_1-2}
    \end{subfigure}
    \hfill
    \begin{subfigure}[t]{0.48\textwidth}
        \hspace{-2mm}\begin{tikzpicture}
            \node [inner sep=0.1em,plainballs, scale=1.25] (a) at (0,0) {};
            \node [inner sep=0.1em,no-outline, scale=2] (a0) at (a) {};
            \node [inner sep=0.1em,no-outline] (al) at ([shift={(a)}] 270:0) {\small $\Mod{\G_{\!i}}$};
            
            \node [inner sep=0.1em,plainballs, scale=1.25] (b) at ([shift={(a)}] 60:4.5) {};
            \node [inner sep=0.1em,no-outline, scale=2] (b0) at (b) {};
            \node [inner sep=0.1em,no-outline] (bl) at (2.25,3.9) {\small $\,\Mod{\G_{\!j}}$};
            
            \node [inner sep=0.1em,plainballs,scale=1.25] (c) at ([shift={(a)}] 120:4.5) {};
            \node [inner sep=0.1em,no-outline, scale=2] (c0) at (c) {};
            \node [inner sep=0.1em,no-outline] (cl) at (-2.25,3.9) {\small $\Mod{\G_{\!k}}\,$};

            \draw [rounded corners=7.0mm] (a.center)--(b.center)--(c.center)--cycle;
            
            \node[no-outline] (cb0) at ([shift={(a)}] 90:1.5){};
            \node[no-outline] (cb1) at ([shift={(cb0)}] 60:1.5){};
            \node[no-outline] (cb2) at ([shift={(cb0)}] 120:1.5){};
            
            \draw [rounded corners=3.0mm] (cb0.center)--(cb1.center)--(cb2.center)--cycle;
            \node [inner sep=0.1em,no-outline] (cbl) at ([shift={(cb0)}] 90:0.85) {\small $\Mod{\coverbase'}$};
            \node [inner sep=0.1em,no-outline] (e) at ([shift={(cbl)}] 90:1) {\small $\Mod{\coverbase}$};
            
            \node [inner sep=0.1em,no-outline, scale=2] (a1) at ([shift={(a)}] 160:1.5) {};
            \node [inner sep=0.1em,no-outline, scale=2] (a2) at ([shift={(a)}] 20:1.5) {};
            
            \node [inner sep=0.1em,no-outline,scale=2] (a0-120) at ([shift={(a0)}] 120:0.61) {};
            \node [inner sep=0.1em,no-outline,scale=2] (a0-300) at ([shift={(a0)}] 300:0.61) {};
            \node [inner sep=0.1em,no-outline,scale=2] (a0-60) at ([shift={(a0)}] 60:0.61) {};
            \node [inner sep=0.1em,no-outline,scale=2] (a0-240) at ([shift={(a0)}] 240:0.61) {};
            
            \node [inner sep=0.1em,no-outline,scale=2] (a1-60) at ([shift={(a1)}] 60:0.61) {};
            \node [inner sep=0.1em,no-outline,scale=2] (a1-240) at ([shift={(a1)}] 240:0.61) {};
            \node [inner sep=0.1em,no-outline,scale=2] (a1-150) at ([shift={(a1)}] 150:0.61) {};
            
            \node [inner sep=0.1em,no-outline,scale=2] (a2-30) at ([shift={(a2)}] 30:0.61) {};
            \node [inner sep=0.1em,no-outline,scale=2] (a2-120) at ([shift={(a2)}] 120:0.61) {};
            \node [inner sep=0.1em,no-outline,scale=2] (a2-300) at ([shift={(a2)}] 300:0.61) {};
            
            \drawthingiies{a2-120.center}{a2-300.center}{a0-300.center}{a0-120.center}{dotted}{black};
            \drawthingiies{a0-60.center}{a0-240.center}{a1-240.center}{a1-60.center}{dotted}{black};
            
            \node [inner sep=0.1em,no-outline, scale=1.5] (c1) at ([shift={(c)}] 260:1.5) {};
            \node [inner sep=0.1em,no-outline, scale=1.5] (c2) at ([shift={(c)}] 40:1.5) {};
            
            \node [inner sep=0.1em,no-outline,scale=2] (c1-340) at ([shift={(c1)}] 340:0.61) {};
            \node [inner sep=0.1em,no-outline,scale=2] (c1-160) at ([shift={(c1)}] 160:0.61) {};
            
            \node [inner sep=0.1em,no-outline,scale=2] (c0-340) at ([shift={(c0)}] 340:0.61) {};
            \node [inner sep=0.1em,no-outline,scale=2] (c0-160) at ([shift={(c0)}] 160:0.61) {};
            
            \node [inner sep=0.1em,no-outline,scale=2] (c2-300) at ([shift={(c2)}] 300:0.61) {};
            \node [inner sep=0.1em,no-outline,scale=2] (c2-120) at ([shift={(c2)}] 120:0.61) {};
            
            \node [inner sep=0.1em,no-outline,scale=2] (c0-300) at ([shift={(c0)}] 300:0.61) {};
            \node [inner sep=0.1em,no-outline,scale=2] (c0-120) at ([shift={(c0)}] 120:0.61) {};
            
            \drawthingiies{c0-160.center}{c0-340.center}{c1-340.center}{c1-160.center}{dotted}{black};
            \drawthingiies{c2-120.center}{c2-300.center}{c0-300.center}{c0-120.center}{dotted}{black};
            
            \node [inner sep=0.1em,no-outline,scale=1.5] (b1) at ([shift={(b)}] -80:1.5) {};
            \node [inner sep=0.1em,no-outline, scale=1.5] (b2) at ([shift={(b)}] 140:1.5) {};
            
            \node [inner sep=0.1em,no-outline,scale=2] (b0-60) at ([shift={(b0)}] 60:0.61) {};
            \node [inner sep=0.1em,no-outline,scale=2] (b0-240) at ([shift={(b0)}] 240:0.61) {};
            
            \node [inner sep=0.1em,no-outline,scale=2] (b1-20) at ([shift={(b1)}] 20:0.61) {};
            \node [inner sep=0.1em,no-outline,scale=2] (b1-200) at ([shift={(b1)}] 200:0.61) {};
            
            \node [inner sep=0.1em,no-outline,scale=2] (b0-20) at ([shift={(b0)}] 20:0.61) {};
            \node [inner sep=0.1em,no-outline,scale=2] (b0-200) at ([shift={(b0)}] 200:0.61) {};
            
            \node [inner sep=0.1em,no-outline,scale=2] (b0-30) at ([shift={(b0)}] 30:0.61) {};
            \node [inner sep=0.1em,no-outline,scale=2] (b0-210) at ([shift={(b0)}] -150:0.61) {};
            
            \node [inner sep=0.1em,no-outline, scale=2] (b2-60) at ([shift={(b2)}] 60:0.61) {};
            \node [inner sep=0.1em,no-outline, scale=2] (b2-240) at ([shift={(b2)}] 240:0.61) {};
            
            \drawthingiies{b1-200.center}{b1-20.center}{b0-20.center}{b0-200.center}{dotted}{black};
            \drawthingiies{b0-60.center}{b0-240.center}{b2-240.center}{b2-60.center}{dotted}{black};

            \path[-, dash pattern=on \pgflinewidth off 4pt, ultra thick, lightgray] (a1-150.center) edge[bend left] (c1.south);
            \path[-, dash pattern=on \pgflinewidth off 4pt, ultra thick, lightgray] (c2.east) edge[bend left] (b2.west);
            \path[-, dash pattern=on \pgflinewidth off 4pt, ultra thick, lightgray] (b1.south) edge[bend left] (a2-30.center);
            
        \end{tikzpicture}
        \caption{By Condition (3), for each \( \coverbase \) that is consistent with at least three distinct elements of the cycle (e.g. $\G_{\!i}, \G_{\!j}, \G_{\!k} \in \{\G_{\!0},\ldots,\G_{\!n} \}$), there exists a base \( \coverbase' \) that is subsumed by \( \coverbase \) but inconsistent with all $\G_{\!0,1}$, $\ldots$ , $\G_{\!n-1,n}$, $\G_{\!n,0}$.}
        \label{fig:critical_loop_condition_3}
    \end{subfigure}
    \caption{Illustrations of the Conditions (1)--(3) of a critical loop given in \Cref{def:critical_loop}.}
    \label{fig:critical_loop_condition}
\end{figure}

In the following, we provide some intuition for the notion of \criticalloop.
The bases \( \G_0,\ldots,\G_n \) provide model sets that are pairwise disjoint (cf.~the second part of Condition~(2)) and can be thought of as arranged in a cycle, while the bases \( \G_{0,1},\ldots,\G_{n,0} \) overlap any two adjacent model sets as indicated by their indices (cf.~the first part of Condition (2)). 
Exploiting this situation, we now want to define the result of revising \( \K \) such that the circular arrangement governs the choice of the ``$\K$-preferred'' models as follows: the models of \( \K\circ\G_{i,i\oplus 1} \), obtained by revising $\K$ with $\G_{i,i\oplus 1}$, encompass all models of $\G_{i}$, but no model of~\( \G_{i \oplus 1} \).
Consequently, for any $i$, the revision \( \K\circ \G_{i,i\oplus 1} \) provides a preference of \( \G_i \) over~\( \G_{i \oplus 1} \). Thus, a relation compatible to \( \circ \) has to contain a ``preference-loop'' of interpretations.
In order to guarantee that this arrangement technique is applicable, Condition~(1) and Condition~(3) from \Cref{def:critical_loop} are ruling out all cases, where other bases of \( \mathbb{B} \) together with \postulate{G1}--\postulate{G6} prevent our intended construction from working:
\begin{itemize}\setlength{\itemsep}{0pt}
    
    \item[]\hspace{-3ex} Condition (1) ensures that none of the bases \( \G_{0,1},\ldots,\G_{n,0} \) has models in common with the current belief base \( \K \) (c.f.  \Cref{fig:critical_loop_condition_1-2}).
    If one base \( \G_{i,i\oplus 1} \) would have a model in common with \( \K \), then the postulate \postulate{G2} would prevent a circular situation.
    Thus, this condition is necessary for admitting circular situations.
    
    \item[]\hspace{-3ex} Condition (3) comes into play if a belief base \( \coverbase \) ``covers'' three or more elements of the cycle, meaning that three or more interpretations of a cycle are models of this base \( \coverbase \).
    For any such $\coverbase$,  there is a consistent belief base \( \coverbase' \) which shares all of its models with \( \coverbase  \) but no model with any of the \( \G_{i,i\oplus 1} \) (c.f.  \Cref{fig:critical_loop_condition_3}). 
    This is crucial for the presence of cycles: if no such \( \coverbase' \) would exist, the operator would (by min-completeness and min-expressibility) choose models of the cycle, e.g., the bases \( \G_i \Cup \coverbase \), as the result of the revision by \( \coverbase \). In the end, this would give one base \( \G_i \) preference over \( \G_{i\oplus 1}, \ldots, \G_{i\oplus n} \) and thus, would prevent creation of a cycle. Therefore, Condition (3) rules out the cases where min-completeness and min-expressibility and non-existence of such a \( \coverbase' \) together would prevent formation of a cycle.
\end{itemize}

\Cref{def:critical_loop} is inspired by our running example. 
Before explicating this link, we continue with the presentation of the general results.

        The next theorem is the central result of this section, stating that the notion of {\criticalloop} captures exactly those base logics for which some operator exists that is not total-preorder-representable. 
        By contraposition, this just means that for all base logics $\mathbb{B}$, the absence of critical loops from $\mathbb{B}$ is a necessary and sufficient criterion for universal total-preorder-representability and hence for the existence of a characterization result for $\mathbb{B}$ that is based on total preorders.
        More concisely expressed, loop-free base logics are exactly those base logics where every base change operator that satisfies \postulate{G1}--\postulate{G6} is total-preorder-representable.
        This characterization result will not only hold for {\basechange} operators that satisfy \postulate{G1}--\postulate{G6}, but also for operators that does not satisfy \postulate{G4}, but the remaining postulates \postulate{G1}--\postulate{G3}, \postulate{G5}, and \postulate{G6}.
        To provide a result applicable for both groups of postulates, we will show for the necessary and sufficient direction the respectively stronger result, i.e., if our base logic exhibits a {\criticalloop} we provide a construction for a non-total-preorder-representable {\basechange} operator that satisfies \postulate{G1}--\postulate{G6}, and for the other direction, we show that in the absence of {\criticalloop s} every operator that satisfies \postulate{G1}--\postulate{G3}, \postulate{G5}, and \postulate{G6} is total-preorder-representable.
        
        \pagebreak[3]
        \begin{theorem}\label{thm:when_tranisitive} 
            For all base logics $\mathbb{B}$, the following statements hold:
            \begin{itemize}
                \item[\rm{(I)}] If \,$\mathbb{B}$ is loop-free, then every {\basechange} operator for $\mathbb{B}$ that satisfies \postulate{G1}--\postulate{G3}, \postulate{G5}, and \postulate{G6} is total-preorder-representable. 
                \item[\rm{(II)}] If \,$\mathbb{B}$ is not loop-free, then there exists a {\basechange} operator for $\mathbb{B}$ that satisfies \postulate{G1}--\postulate{G6} and is not total-preorder-representable.
            \end{itemize}
        \end{theorem}
        We will only present an outline of the proof of \Cref{thm:when_tranisitive} here. The full, detailed proof is given in \Cref{sec:app_loops_new}.
            \begin{proof}[Proof (outline)]\newcommand{\B}{\mathcal{C}}
                \newcommand{\A}{\mathcal{A}}
                \newcommand{\critloop}{\mathfrak{C}}
                We consider both statements independently.
                \begin{itemize}
                    \item[(I)]
                    The proof outline for this direction is the same as in \Cref{prop:encodedown}.             
                    The critical difference  to the proof of \Cref{prop:encodedown} is the argumentation required for showing that taking the transitive closure of \( \relationstepdashK \) in ``Step II'' only adds detached pairs.
                    To this end, we show that loop-freeness of \( \mathbb{B} \) ensures that taking the transitive closure only does add detached pairs, i.e., \( {\relationstepdashdashK} \subseteq {\releqc{\K}} \) holds.
                    The first observation is that the loop-freeness of \( \mathbb{B} \) implies that every proper cycle of length \( 2 \) in \( \releqcK \) contains a detached pair. 
                    \sauerwald{This can be shown by contradiction: if a proper cycle of length \( 2 \) in \( \releqcK \) would contain no detached pair, then there is a {\criticalloop} in \( \mathbb{B} \) consisting of bases that cover the proper cycle.
                    Suppose \( \I_0,\I_1,\I_2 \) is a proper cycle with \( \I_{0} \relcK \I_{1} \). 
                    Because each \( ( \I_{i} ,\I_{i\oplus 1} ) \) is not detached, there must exist a base \( \G_{i,i\oplus 1} \in\Bases \) such that \( \I_{i} \models \K\circ \G_{i,i\oplus 1} \) and \( \I_{i} \models \K\circ \G_{i,i\oplus 1} \) holds.
                    Employing the information \( \I_{0} \relcK \I_{1} \), yields that \( \I_{1} \not\models \K\circ \G_{0,1} \) holds.
                    Let \( \G_i = (\K\circ\G_{i,{i \oplus 1}})\Cup\G_{{i\oplus 2},{i}} \) for each \( i\in\{0,1,2\} \).
                    The bases \( \G_{0,1}, \G_{1,2},\G_{2,0} \) satisfy Condition (1) and Condition (2) of \Cref{def:critical_loop}.
                    Because, \( \mathbb{B} \) is loop-free, i.e., contains no {\criticalloop}, there is some base \( \coverbase \) that violates Condition (3) of \Cref{def:critical_loop}.
                    A detailed analysis by case distinction reveals that the existence of the base \( \coverbase \) contradicts \( \I_{0} \relcK \I_{1} \).
                    }
                    
                    As next step, observe that loop-freeness  of \( \mathbb{B} \)  guarantees that every proper cycle of length \( n \)  in \( \releqcK \) contains a proper cycle of length \( 2 \) in \( \releqcK \).
                    One can show that if the transitive closure added a non-detached pair \( (\I,\I') \) to \( \relationstepdashK \), this would imply that  \( (\I,\I') \) is part of a proper cycle of length \( n \) in \( \releqcK \) without any detached pairs.
                    Consequently, there is proper cycle of length \( 2 \) in \( \releqcK \) without any detached pairs.
                    A contradiction to the observation that every proper cycle of length \( 2 \) in \( \releqcK \) contains a detached pair.
                    
					\item[(II)]
                    Let $\mathbb{B}  = (\MC{L},{\Omega},\models, \Bases, \Cup)$ be a base logic with a {\criticalloop} $\critloop = (\G_{0,1},\G_{1,2},\allowbreak\ldots,\G_{n,0})$ and let $ \G_0,\ldots,\G_{n}\in \Bases$ and $ \K $ 
                    be as in \Cref{def:critical_loop}.
                    We construct an operator \( \circ_\critloop \) that is not total-preorder-representable and satisfies \postulate{G1}--\postulate{G3}, \postulate{G5}, and \postulate{G6}. 
                    For any base \( \K'\in\Bases \) with \( \K'\not\equiv\K \), the operator \( \circ_\critloop \) behaves as the trivial revision operator (see \Cref{ex:full-meet}).
                    In the case \( \K'\equiv\K \), the result \( \K \circ_\critloop \G  \) is given as follows. 
                    If \( \G \) is consistent with \( \K \), then \( \K \circ_\critloop \G  = \G  \Cup   \K' \).
                    Let $ \B' $ denote the set of all $ \coverbase' $ guaranteed by Condition~(3) from \Cref{def:critical_loop} that are inconsistent with \( \K \) .
                    If \( \G \) is consistent with some element from \( \B' \), then \( \K\circ_\critloop \G  = \G  \Cup   \G_{\min}^{\B'} \), where \( \G_{\min}^{\B'} \) is the minimal element in \( \B' \) that is consistent with \( \G \).
                    If \( \G \) is consistent with some \( \G_i \), then \( \K \circ_\critloop \G = \G  \Cup \G_i \).
                    If none of the former cases applies, then \( \K \circ_\critloop \G = \G  \).
                    Indeed, one can show that \( \circ_\critloop \) satisfies \postulate{G1}--\postulate{G3}, \postulate{G5}, and \postulate{G6}.
                    As last step, we show that \( \circ_\critloop \) is not total-preorder-representable by showing that the revisions \( \K\circ_\critloop\G_{0,1}, \ldots, \K\circ_\critloop\G_{n,0} \) impose a violation of transitivity in  \( \releq{\K'} \) for every min-friendly faithful assignment \( \releq{\abst} \) that is compatible with \( \circ \) in the sense of \Cref{thm:representation_theorem}.
\qedhere
                \end{itemize}
        \end{proof}

            \Cref{thm:when_tranisitive} shows that loop-freeness is the necessary and sufficient criterion for base logics such that every \sauerwald{{\basechange} operator that satisfies \postulate{G1}--\postulate{G3}, \postulate{G5}, and \postulate{G6}} is total-preorder-representable.
            \Cref{thm:when_tranisitive} (I) provides that loop-freeness of \( \mathbb{B} \) guarantees universal total-preorder-representability for \sauerwald{{\basechange} operators for \( \mathbb{B} \) that satisfy \postulate{G1}--\postulate{G3}, \postulate{G5}, and \postulate{G6}}, and \Cref{thm:when_tranisitive} (II) provides that universal total-preorder-representability of \sauerwald{{\basechange} operators for \( \mathbb{B} \) that satisfy \postulate{G1}--\postulate{G6}} is only given in loop-free base logics.
            
            \sauerwald{A direct consequence of \Cref{thm:when_tranisitive} is that loop-freeness of base logics is naturally connected with the notion of acyclicity of change operators, as summarized in the following.
            \begin{corollary}\label{col:acyc_criticalloop}
                Let $\mathbb{B}$ be a base logic. The following holds:
                \begin{itemize}
                    \item If\, \( \mathbb{B} \) is loop-free, then every base change operator for $\mathbb{B}$ that satisfies \postulate{G1}--\postulate{G3}, \postulate{G5}, and \postulate{G6} also satisfies \postulate{Acyc}.
                    \item If every base change operator for $\mathbb{B}$ that satisfies \postulate{G1}--\postulate{G3}, \postulate{G5}, and \postulate{G6} also satisfies \postulate{Acyc}, then \( \mathbb{B} \) is loop-free.
                \end{itemize}
            \end{corollary}}         
In the next section, we draw further conclusions of our characterizations of total-preorder-representability.

\pagebreak[3]
\subsection{Extended Theorems}\label{sec:extendedtheorems}
We established that loop-free base logics are exactly those base logics where every base change operator that satisfies the postulates \postulate{G1}--\postulate{G6} is total-preorder-representable, i.e.,  the absence of critical loop coincides with universal total-preorder-representability. 
Moreover, we showed that the \postulate{Acyc} postulate by {Delgrande et al.}~\cite{KS_DelgrandePeppasWoltran2018} serves in our more general setting the same purpose of ruling out operators that are not total-preorder-representable.

\Cref{thm:when_tranisitive} and \Cref{thm:acyctpo} allows us to provide more versions of the two-way representation theorem.
The first two representation theorems are ones where the \basechange\ operator satisfies \postulate{G4}, thus abstracting from the syntactic form of the belief bases.
Note that transitivity implies min-retractivity, and thus a transitive min-complete relation is automatically min-friendly.
\begin{theorem}\label{thm:rep3}
	Let $\mathbb{B}$ be a \sauerwald{loop-free base logic}. Then the following hold:
	\begin{itemize}
		\item Every base change operator for $\mathbb{B}$ satisfying \postulate{G1}--\postulate{G6} is compatible with some min-expressible min-complete faithful preorder assignment. 
		\item Every min-expressible min-complete faithful preorder assignment for $\mathbb{B}$ is compatible with some 
		base change operator satisfying \postulate{G1}--\postulate{G6}.
	\end{itemize}
\end{theorem}
\begin{proof}
	The first statement is a consequence of statement (II) of \Cref{thm:when_tranisitive} together with \Cref{lem:doesntmatter}.
	The second statement is an immediate consequence of the second statement of~\Cref{thm:rep1}.
\end{proof}
\begin{theorem}\label{thm:rep3acyc}
    Let $\mathbb{B}$ be a base logic. Then the following hold:
    \begin{itemize}
        \item Every base change operator for $\mathbb{B}$ satisfying \postulate{G1}--\postulate{G6} and \postulate{Acyc} is compatible with some min-expressible min-complete faithful preorder assignment. 
        \item Every min-expressible min-complete faithful preorder assignment for $\mathbb{B}$ is compatible with some 
        base change operator satisfying \postulate{G1}--\postulate{G6} and \postulate{Acyc}.
    \end{itemize}
\end{theorem}
\begin{proof}
    The first statement is a consequence of \Cref{thm:acyctpo} together with \Cref{lem:doesntmatter}.
    The second statement is an immediate consequence of \Cref{thm:acyctpo} and the second statement of~\Cref{thm:rep1}.
\end{proof}
The second group of representation theorems is for \basechange\ operators which do not necessarily satisfy \postulate{G4}, and thus might be sensitive to the syntax of the prior belief base.
\begin{theorem}\label{thm:rep4}
	Let $\mathbb{B}$ be a \sauerwald{loop-free base logic}. Then the following hold:
	\begin{itemize}
		\item Every base change operator for $\mathbb{B}$ satisfying \postulate{G1}--\postulate{G3}, \postulate{G5}, and \postulate{G6} is compatible with some min-expressible min-complete quasi-faithful preorder assignment. 
		\item Every min-expressible min-complete quasi-faithful preorder assignment for $\mathbb{B}$ is compatible with some 
		base change operator satisfying \postulate{G1}--\postulate{G3}, \postulate{G5}, and \postulate{G6}.
	\end{itemize}
\end{theorem}
\begin{proof}
	The first statement is a consequence of statement (II) of \Cref{thm:when_tranisitive}.
	The second statement is an immediate consequence of the second statement of~\Cref{thm:rep2}.
\end{proof}
\newpage
\begin{theorem}\label{thm:rep4acyc}
Let $\mathbb{B}$ be a base logic. Then the following hold:
\begin{itemize}
    \item Every base change operator for $\mathbb{B}$ satisfying \postulate{G1}--\postulate{G3}, \postulate{G5}, \postulate{G6} and \postulate{Acyc} is compatible with some min-expressible min-complete quasi-faithful preorder assignment. 
    \item Every min-expressible min-complete quasi-faithful preorder assignment for $\mathbb{B}$ is compatible with some 
    base change operator satisfying \postulate{G1}--\postulate{G3}, \postulate{G5}, \postulate{G6} and \postulate{Acyc}.
\end{itemize}
\end{theorem}
\begin{proof}
The first statement is a consequence of \Cref{thm:acyctpo}.
The second statement is an immediate consequence of \Cref{thm:acyctpo} and the second statement of~\Cref{thm:rep2}.
\end{proof}
We have shown in \Cref{thm:forgedtpo} that absence of forged proper cycles in assignments and total-preorder-representability are immediately linked. 
Using this observation, we want to note that preorder assignments are not the only assignments that are compatible with total-preorder-representable {\basechange} operators, e.g., our canonical encoding from \Cref{def:relation_new} does always yield a compatible assignment, but not necessary a preorder assignment (see also \Cref{sec:enc_operators}).
Specifically, we will show now that  assignments without forged proper cycle are forming the class of (min-expressible min-complete faithful) assignments that are compatible with  \sauerwald{{\basechange} operators that satisfy the AGM revision postulates}.
\begin{proposition}
	Let $\mathbb{B} = (\MC{L},{\Omega},\models$, $\Bases, \Cup) $ be a base logic.\\
     If  \( \releq{\abst} \) is an arbitrary min-expressible min-complete quasi-faithful assignment, then the following statements are equivalent:
    \begin{itemize}
        \item  \( \releq{\abst} \) is compatible with a change operator for \( \mathbb{B} \) that satisfies \postulate{G1}--\postulate{G3}, \postulate{G5}, \postulate{G6} and \postulate{Acyc}.
        \item There is no \( \K \in \Bases \) such that \( \releqcK \) contains a proper cycle that is forged by~\( \mathbb{B} \).
    \end{itemize}
If  \( \releq{\abst} \) is an arbitrary min-expressible min-complete faithful assignment, then the following statements are equivalent:
\begin{itemize}
    \item  \( \releq{\abst} \) is compatible with a change operator for \( \mathbb{B} \)  that satisfies \postulate{G1}--\postulate{G6} and \postulate{Acyc}.
    \item There is no \( \K \in \Bases \) such that \( \releqcK \) contains a proper cycle that is forged by~\( \mathbb{B} \).
\end{itemize}
\end{proposition}
\begin{proof}
	The statements are a consequence of \Cref{thm:acyctpo} and \Cref{thm:forgedtpo}.
\end{proof}

We close this section with an implication of \Cref{thm:when_tranisitive}.
A base logic 
$\mathbb{B} \,{=}\, (\MC{L},{\Omega},\models$, $\Bases, \Cup)$
is called \emph{disjunctive}, if for every two bases $\G_1,\G_2 \in \Bases$ there is a base $ \G_1 \ovee \G_2 \in \Bases $ such that $ \Mod{\G_1 \ovee \G_2}=\Mod{\G_1}\cup\Mod{\G_2}$.
This includes the case of any (base) logic allowing disjunction to be expressed on the sentence level, i.e., when for every $\gamma,\delta \in \MC{L}$ there exists some $\gamma \ovee \delta \in \MC{L}$ with $\Mod{\gamma \ovee \delta} = \Mod{\gamma} \cup \Mod{\delta}$, such that  
$\G_1 \ovee \G_2$ can be obtained as $\{\gamma \ovee \delta \mid \gamma \in \G_1, \delta \in \G_2\}$.

\begin{corollary}\label{cor:disjunctive_tpo}
	Let \( \mathbb{B} \) be a disjunctive base logic. The following statements hold:
    \begin{itemize}
        \item Every base change operator for \( \mathbb{B} \) satisfying \postulate{G1}--\postulate{G6} is total-preorder-representable.
        \item If a base change operator \( \circ \) for \( \mathbb{B} \) satisfies \postulate{G1}--\postulate{G6}, then \( \circ \) satisfies also \postulate{Acyc}.
    \end{itemize}
\end{corollary}
\begin{proof}
For the first statement, observe that a disjunctive base logic never exhibits a critical loop; Condition (3) would be violated by picking $\G = \left((\G_0 \ovee \G_1) \ldots \right) \ovee \G_n$. 
For the second statement employ the first statement and \Cref{col:acyc_criticalloop}.
\end{proof}

As a consequence, for a vast amount of well-known logics, including all classical logics such as first-order and second order predicate logic, one directly obtains total-preorder-representablility of every \sauerwald{{\basechange} operator that satisfies \postulate{G1}--\postulate{G6}} by \Cref{cor:disjunctive_tpo} .

\subsection{Examples}

We will now demonstrate the notion of acyclic base change operators and the novel notion of a loop-free base logic, specifically the notion of a {\criticalloop} (cf. \Cref{def:critical_loop}), by extending two previously given examples.

\begin{example}[continuation of \Cref{ex:orderByEx2}]\label{ex:orderByEx4}
	We will now see that 
	the base logic $\mathbb{B}_\mathrm{Ex} = \mathbb{L}_\mathrm{Ex}^\mathrm{arb}  = (\MC{L}_\mathrm{Ex} ,{\Omega}_\mathrm{Ex} ,\models_\mathrm{Ex} ,\MC{P}(\MC{L}_\mathrm{Ex} ),\cup)$ from \Cref{ex:orderByEx1} constructed from $ \mathbb{L}_\mathrm{Ex} $
    is not loop-free, i.e. \( \mathbb{B_\mathrm{Ex}} \)
	exhibits a critical loop.        
    Moreover, we will \sauerwald{see that the base change operator \( \circ_\mathrm{Ex} \) for \( \mathbb{B}_\mathrm{Ex} \) from Example \ref{ex:orderByEx1}} that is not acyclic, i.e. that violates \postulate{Acyc}.
	
	For this, choose $ \G_{i,i\oplus1}=\{ \varphi_{i} \} $, and $ \K=\K_\mathrm{Ex} = \{ \psi_3 \} $ (as in \Cref{ex:orderByEx1}) and $ \G_i=\{\psi_{i}\} $  for $ i\in\{0,1,2\} $, where $ \oplus $ denotes addition $ \mathrm{mod}\ 3 $.
	We consider each of the three conditions of \Cref{def:critical_loop} as a separate case:        
	\begin{itemize}\setlength{\itemsep}{0pt}
		\item[]\hspace{-4.4ex}\emph{Condition (1).}~Observe that $ \K_\mathrm{Ex} $  is inconsistent with $ \G_{0,1} $, $ \G_{1,2} $ and $ \G_{2,0} $. Thus, Condition (1) is satisfied.
		
		\item[]\hspace{-4.4ex}\emph{Condition (2).}~For each $ i\in\{0,1,2\} $, the models of bases $ \G_{i} $ and $ \G_{i\oplus 1} $ are contained in $\Mod{\G_{i,i\oplus1}}$, but $ \G_i $ is inconsistent with $ \G_{j} $ with $i\neq j$, e.g. $ \Mod{\{\psi_0\}}\cup \Mod{\{\psi_1\}} \subseteq \Mod{\{\varphi_0\}} $ and $\{\psi_0\}$ is not consistent with neither $\{\psi_1\}$ nor $\{\psi_2\}$.
		
		\item[]\hspace{-4.4ex}\emph{Condition (3).}~The belief base $ \coverbase=\{ \chi' \} $ is the only belief base consistent with $ \G_0  $, $ \G_1 $, and $ \G_2 $. 
		For the satisfaction of Condition (3) observe that $ \coverbase'=\{\psi_4\} $ fulfils the required condition $ \emptyset\neq \Mod{\coverbase'} \subseteq \Mod{\coverbase}\setminus(\Mod{\G_{0,1}}\cup \Mod{\G_{1,2}} \cup \Mod{\G_{2,0}}) $.
	\end{itemize}%
	In summary, $ \G_{0,1}$, $\G_{1,2}$, and $\G_{2,0}$ form a critical loop for $ \mathbb{B}_\mathrm{Ex} $ (see \Cref{fig:critical_loop_ex}). 
	Indeed, in line with \Cref{thm:when_tranisitive} (II), every min-complete faithful assignment compatible with $\circ_\mathrm{Ex} $ maps certain bases to a non-transitive relation. 
    To illustrate this, we use here the assignment \( \releq{\abst}^{\circ_\mathrm{Ex}} \) defined in \Cref{ex:orderByEx2} and sketched in \Cref{fig:non-trans-rel} for \( \K_\mathrm{Ex} \) (see also \Cref{fig:critical_loop_ex}).
	Consider the revisions $\K_\mathrm{Ex}\circ_\mathrm{Ex} \G_{0,1}$, $\K_\mathrm{Ex}\circ_\mathrm{Ex} \G_{1,2}$, and $\K_\mathrm{Ex}\circ_\mathrm{Ex} \G_{2,0}$.
From the construction of $\circ_\mathrm{Ex}$ \sauerwald{given in \Cref{ex:orderByEx1}} and compatibility of $\releq{\abst}^{\circ_\mathrm{Ex}}$ with \( \circ_\mathrm{Ex} \), we have 
	\begingroup\abovedisplayskip=3pt
	\belowdisplayskip=3pt
	\begin{align*}
		\I_0\in{\min(\Mod{\G_{0,1}},\releq{\K_\mathrm{Ex}}^{\circ_\mathrm{Ex}})} & \text{, but } \I_1\not\in{\min(\Mod{\G_{0,1}},\releq{\K_\mathrm{Ex}}^{\circ_\mathrm{Ex}})} \ , \\
		\I_1\in{\min(\Mod{\G_{1,2}},\releq{\K_\mathrm{Ex}}^{\circ_\mathrm{Ex}})}  & \text{, but } \I_2\not\in{\min(\Mod{\G_{1,2}},\releq{\K_\mathrm{Ex}}^{\circ_\mathrm{Ex}})} \text{ , and }\\ \I_2\in{\min(\Mod{\G_{2,0}},\releq{\K_\mathrm{Ex}}^{\circ_\mathrm{Ex}})} & \text{, but }  \I_0\not\in{\min(\Mod{\G_{2,0}},\releq{\K_\mathrm{Ex}}^{\circ_\mathrm{Ex}})} \ , 
	\end{align*}    \endgroup
	showing $\I_0 \rel{\K_\mathrm{Ex}}^{\circ_\mathrm{Ex}} \I_1 \rel{\K_\mathrm{Ex}}^{\circ_\mathrm{Ex}} \I_2 \rel{\K_\mathrm{Ex}}^{\circ_\mathrm{Ex}}   \I_0$, which is	 impossible for a transitive relation.
\sauerwald{Moreover, observe that the construction of $ \circ_\mathrm{Ex} $ given in \Cref{ex:orderByEx1} matches the construction outlined in the proof of \Cref{thm:when_tranisitive}~(II).}
	In particular, for the example presented here one would obtain $ \mathcal{C}'
    =\{ \{\psi_4\} \} $ when following the outline of the construction.
\end{example}

\begin{example}
	We will now see that the base logic ${\mathbb{H}_n^\mathrm{arb}} = (\MC{L}_{\mathbb{H}_n},\Omega_{\mathbb{H}_n},\models_{\mathbb{H}_n}, \mathcal{P}(\MC{L}_{\mathbb{H}_n}), \cup)$ constructed from the Horn logic $\mathbb{H}_n$ \sauerwald{is not loop-free, i.e., $\mathbb{H}_n$} exhibits a critical loop. 
	
	In this Horn logic example, we will shorten the interpretation notations. For example for $\Sigma = \{p_1, p_2\}$, an interpretation $\I$ where $\I(p_1) = \textbf{true}$ and $\I(p_2) = \textbf{false}$ will be written as $p_1\bar{p_2}$.
	
	Following an example in the work by Delgrande and Peppas \cite{KS_DelgrandePeppas2015}, we will consider $\Sigma = \{p_1,p_2,p_3\}$ and show the existence of a critical loop of length 3.
	To show the critical loop, we let $\K_\mathbb{H} = \{(p_1\wedge p_2\wedge p_3)\}$ and choose $\G_{0,1} = \{(p_1), (p_2\wedge p_3 \to \bot)\}$, $\G_{1,2} = \{(p_3), (p_1\wedge p_2\to\bot)\}$, and $\G_{2,0} = \{(p_2), (p_1\wedge p_3\to \bot)\}$.
	We then choose $\G_0 = \{(p_1 \wedge p_2), (p_3\to\bot)\}$, $\G_1 = \{(p_1 \wedge p_3), (p_2\to\bot)\}$, and $\G_2 = \{(p_2\wedge p_3), (p_1\to\bot)\}$.
	We evaluate the satisfaction of the three conditions of \Cref{def:critical_loop} case by case:
	\begin{itemize}\setlength{\itemsep}{0pt}
		\item[]\hspace{-4.4ex}\emph{Condition (1).}~Observe that $ \K_\mathbb{H} $  is inconsistent with $ \G_{0,1} $, $ \G_{1,2} $ and $ \G_{2,0} $ as $ \K_\mathbb{H} $ has an only model where all three atoms must be true. Thus, Condition (1) is satisfied.
		
		\item[]\hspace{-4.4ex}\emph{Condition (2).}~We have $\Mod{\G_0} = \{p_{1}p_{2}\bar{p_3}\}$, $\Mod{\G_1} = \{p_1\bar{p_2}{p_3}\}$, and $\Mod{\G_2} = \{\bar{p_1}p_2{p_3}\}$. We also have $\Mod{\G_{0,1}} = \{p_{1}p_2\bar{p_3},$ $ p_1\bar{p_2}{p_3},$ $p_1\bar{p_2}\bar{p_3}\}$, $\Mod{\G_{1,2}} = \{p_1\bar{p_2}{p_3}, \bar{p_1}p_{2}p_3, \bar{p_1}\bar{p_2}p_3\}$, and 
		$\Mod{\G_{2,0}} = \{p_{1}p_2\bar{p_3}, \bar{p_1}p_{2}p_3, \bar{p_1}p_2\bar{p_3}\}$. Thus, for each $i\in\{0,1,2\}$, the models of bases $\G_i$ and $\G_{i\oplus 1}$ are contained in $\Mod{\G_{i,i\oplus 1}}$, but $\G_i$ is inconsistent with $\G_j$ with $i\neq j$. (The symbol $\oplus$ denotes addition \textsf{mod} 3)
		
		\item[]\hspace{-4.4ex}\emph{Condition (3).}~The empty base ($\coverbase = \emptyset$) is the only base consistent with $\G_0$, $\G_1$, and $\G_2$ (and with all other bases). For this, we choose $\coverbase' = \K_\mathbb{H}$ as it fulfills the required condition $ \emptyset\neq \Mod{\coverbase'} \subseteq \Mod{\coverbase}\setminus(\Mod{\G_{0,1}}\cup \Mod{\G_{1,2}} \cup \Mod{\G_{2,0}}) $.
	\end{itemize}%
	
	In summary, $\G_{0,1}$, $\G_{1,2}$ and $\G_{2,0}$ construct a critical loop for ${\mathbb{H}_n^\mathrm{arb}}$. It then follows from \Cref{thm:when_tranisitive} (II) that there exists a base change operator for Horn logic which satisfies \postulate{G1}--\postulate{G6} and is not total-preorder-representable. Moreover, from \Cref{cor:disjunctive_tpo} (by contraposition), we have that Horn logic is \emph{not} a disjunctive logic. For example, given $\Sigma = \{p_1, p_2\}$ and Horn bases $\G_1 = \{(p_1), (p_2\to\bot)\}$ and $\G_2 = \{(p_1\to\bot), (p_2)\}$, a base $\G_1\ovee\G_2$ such that $\Mod{\G_1\ovee\G_2} = \Mod{\G_1} \cup \Mod{\G_2} = \{p_1\bar{p_2}, \bar{q_1}q_2\}$ is \emph{not} expressible by any Horn sentence or base (as the intersection of atoms assigned to \textbf{true} in the elements of $\Mod{\G_1\ovee\G_2}$ is empty  \cite{KS_DelgrandePeppas2015}). We note that this critical loop example also works for $\mathbb{H}_n^\mathrm{fin}$, $\mathbb{H}_n^\mathrm{bel}$, and $\mathbb{H}_n^\mathrm{sng}$.
    \sauerwald{Consequently, also $\mathbb{H}_n^\mathrm{fin}$, $\mathbb{H}_n^\mathrm{bel}$, and $\mathbb{H}_n^\mathrm{sng}$ are not loop-free.}
\end{example}

\begin{figure}[t]
	\centering
	\begin{tikzpicture}[]
		
		\node[plainballs, scale=2] (w3s) at (-3.6,0.5) {};
		\node[balls, inner sep=0.0em, scale=0.65] (w3) at (-3.65, 0.4) {};
		\node[rotate=00] (psi3) at ([shift={(w3s)}] 100:0.44) {\small $\Mod{\K}$};
		\node[] (w3l) at (w3) {\small ${\omega_3}$};
		
		\node[plainballs, scale=2] (w4s) at (-1.2,0.5) {};
		\node[balls, inner sep=0.0em, scale=0.65] (w4) at ([shift={(w4s)}] 270:0.1)  {};
		\node[] (w4l) at (w4) {\small ${\omega_4}$};
		
		\node[no-outline, scale=2] (w0s) at (1.5,-0.91) {};
		\node[balls, inner sep=0.0em, scale=0.65] (w0) at ([shift={(1.53,-1.13)}] 60:0.43) {};	
		\node[] (w0l) at (w0) {\small ${\omega_0}$};
		\node[plainballs, scale=1.7] (w0ss) at (w0s) {};
		
		\node[no-outline, scale=2] (w0s-150) at ([shift={(w0s)}] 150:0.8) {};
		\node[no-outline, scale=2] (w0s-330) at ([shift={(w0s)}] 330:0.8) {};
		
		\node[no-outline, scale=2] (w1s) at ([shift={(w0s)}] 60:3.5) {};
		\node[balls, inner sep=0.0em, scale=0.65] (w1) at ([shift={(w0)}] 60:3.00) {};
		\node[] (w1l) at (w1) {\small ${\omega_1}$};
		\node[plainballs, scale=1.7] (w1ss) at (w1s) {};
		
		\node[no-outline, scale=2] (w1s-150) at ([shift={(w1s)}] 150:0.8) {};
		\node[no-outline, scale=2] (w1s-330) at ([shift={(w1s)}] 330:0.8) {};
		\node[no-outline, scale=2] (w1s-30) at ([shift={(w1s)}] 30:0.8) {};
		\node[no-outline, scale=2] (w1s-210) at ([shift={(w1s)}] 210:0.8) {};
		
		\node[no-outline, scale=2] (w2s) at ([shift={(w0s)}] 0:3.6) {};
		
		\node[balls, inner sep=0.0em, scale=0.65] (w2) at ([shift={(w0)}] 0:3.1) {};
		\node[] (w2l) at (w2) {\small ${\omega_2}$};
		\node[plainballs, scale=1.7] (w2ss) at (w2s) {};
		
		\node[no-outline, scale=2] (w2s-30) at ([shift={(w2s)}] 30:0.8) {};
		\node[no-outline, scale=2] (w2s-210) at ([shift={(w2s)}] 210:0.8) {};
		
		\node[plainballs, scale=2] (w5s) at (8.0,0.5) {};
		\node[balls, inner sep=0.0em, scale=0.65] (w5) at (8.4,0.4) {};
		\node[] (w5l) at (w5) {\small ${\omega_5}$};
		\node[rotate=-0] (psi5) at (7.8,0.8) {\small $\Mod{\psi_5}$};

		\drawthingiies{w2s.north}{w2s.south}{w0s.south}{w0s.north}{-}{black};

		\drawthingiies{w1s-150.center}{w1s-330.center}{w0s-330.center}{w0s-150.center}{-}{black};

		\drawthingiies{w2s-30.center}{w2s-210.center}{w1s-210.center}{w1s-30.center}{-}{black};
		
		\node (chi) at (-4.3,4.0) {\small $\Mod{\chi}$};
		\node[draw, rectangle, minimum width=14.3cm,minimum height=7.8cm, rounded corners=.3cm] (rec3) at (2.4,0.55) {};
		
		\node (phi2) at (3.25,-1.3) {\small $\Mod{\G_{2,0}}$};
		\node[rotate=58] (phi0) at (1.9,0.6) {\small$\Mod{\G_{0,1}}$};
		\node[rotate=-58] (phi1) at ([shift={(phi0)}] 0:2.7) {\small $\Mod{\G_{1,2}}$};
		
		\node (psi1) at ([shift={(w1)}] 90:0.48) {\small $\Mod{\G_1}$};
		\node[rotate=0] (psi0) at ([shift={(w0)}] -115:0.58) {\small $\Mod{\G_0}$};
		\node[rotate=0] (psi2) at ([shift={(w2)}] -63:0.58) {\small $\Mod{\G_2}$};
		
		\node (phi4) at (-1.4,3.75) {\small $\Mod{\coverbase}=\Mod{\chi'}$};
		\node[draw, rectangle, minimum width=11.6cm,minimum height=7.2cm, rounded corners=.3cm] (rec1) at (3.4,0.55) {};
		\node[draw, dashed, rectangle, minimum width=5.6cm,minimum height=5.1cm, rounded corners=.3cm,white] (rec2) at (3.3,0.4) {};
		\node[draw, dashed, rectangle, minimum width=0.6cm,minimum height=1cm, rounded corners=.3cm,white] (rec2ghost2) at (0.8,2) {};
		\node[draw, dashed, rectangle, thick, minimum width=5.6cm,minimum height=5.1cm, rounded corners=.3cm] (rec2ghost) at (3.3,0.57) {};
		
		\node[] (psi4) at ([shift={(w4)}] 90:0.55) {\small $\Mod{\coverbase'}$};

		\path[-Stealthnew] (w3) edge[thick] node[above, scale=1] {} (w4) ;
		\path[-Stealthnew] (w4) edge[thick] node[above, scale=1] {} (rec2.west) ;
		\path[-Stealthnew] (rec2.east) edge[thick] node[above, scale=1] {} (w5) ;
		\path[-Stealthnew] (w0) edge[thick] node[left,  near end, scale=0.75] {} (w1) ;
		\path[-Stealthnew] (w1) edge[thick] node[right,  near start, scale=0.75] {} (w2) ;
		\path[-Stealthnew] (w2) edge[thick] node[below, scale=0.75, xscale=-1, yscale=1] {} (w0) ;
		
		\path[-Stealthnew] (w3) edge[thick, bend left, in=160, out=15] node[above, near end, scale=1] {} (rec2ghost2.west) ;
		\path[-Stealthnew] (w4) edge[thick, bend right=90, in=255, out=270] node[below, scale=1] {} (w5) ;
		\path[-Stealthnew] (w3) edge[thick, bend left=60, in=105, out=53] node[above, scale=1] {} (w5) ;

	\end{tikzpicture}
	\caption{Critical loop situation in 
		$ \mathbb{B}_\mathrm{Ex}$ presented in \Cref{ex:orderByEx4}.
		The solid borders represent the sets of models and each arrow depicts the relation $\rel{\K_\mathrm{Ex}}^{\circ_\mathrm{Ex}} $ between models.}
	\label{fig:critical_loop_ex}
\end{figure}
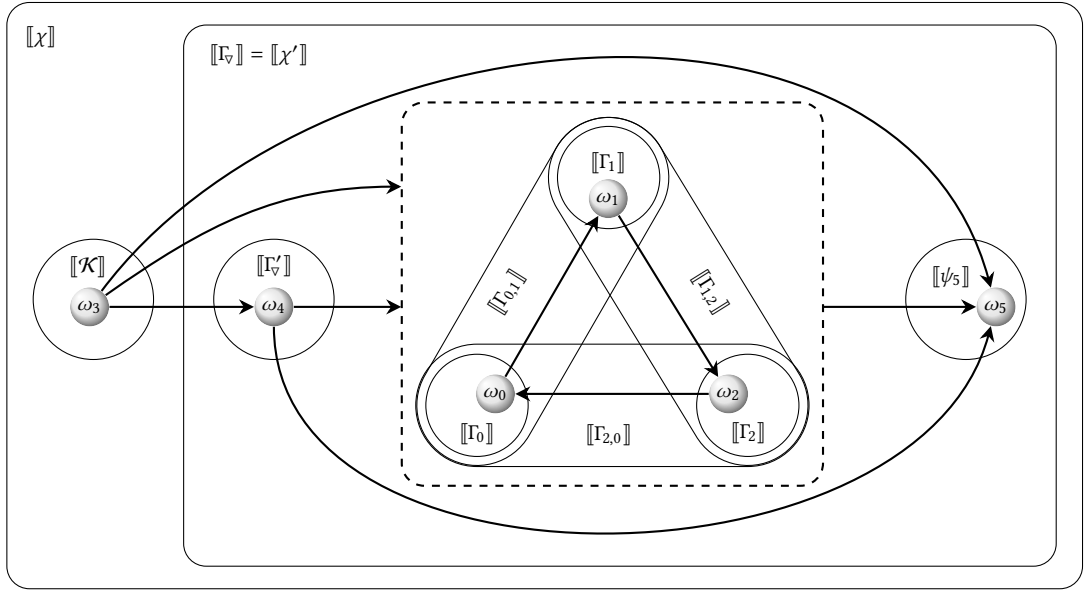

\newpage
\section{Further Discussion}
In this section, we discuss some more specific aspects and noteworthy implications of our approach. 
First, we will discuss the notion of base logics and demonstrate that the decision how to define bases affects the applicability of certain notions. 
Next, we explore the novel notion of min-retractivity (in comparison to transitivity) and discuss its relationship to decomposability of disjunctions.
 
\subsection{Dependency on the Notion of Base}\label{sec:finite-bases}

In order to capture different conceptualisations of the notion of a ``base'' in a generic way, we introduced \emph{base logics}.
In this section, we would like to highlight that the question if our characterization applies depends on the particular notion of base applied, even if the underlying logic is the same.
In particular, we would like to make the point that the case of finite bases cannot just be seen as a ``special case" of arbitrary bases, as demonstrated in the following example.

\begin{example}
Consider the simplistic base logic 
$\mathbb{B}_{\mathrm{f}1} \,{=}\, (\MC{L},{\Omega},\models, \Bases, \cup)$ 
where  $\MC{L}$ contains two types of sentences: $[{\geq} q]$, and  $[{=} q]$, for each $q \in \mathbb{Q}^+$ (i.e., $q$ ranges over the nonnegative rational numbers). Moreover, let $\Omega = \mathbb{R}^+$,
i.e., the interpretations are just nonnegative real numbers. Then we let $r \models [{\geq} q]$ if
$r\geq q$ and we let $r \models [{=}q]$ if $r=q$ (we assume the usual meaning of $\leq$, $\geq$, and $=$ for real numbers).

For $\Bases = \MC{P}_\mathrm{fin}(\MC{L})$,
we now define a finite base revision operator $\circ^{\mathrm{f}1}$ as follows:
$$
\K \circ^{\mathrm{f}1} \G = 
\begin{cases}
\K \cup \Gamma & \text{ if } \Mod{\K \cup \G}\not= \emptyset, \\	
\G & \text{ if } \Mod{\G}= \emptyset,\\
\{[{=}q_\G]\} & \text{ otherwise, where }\{q_\G\} = min (\Mod{\G},\leq).\\
\end{cases}
$$
Note that the usage of finite bases ensures that the minimum in the third case exists and is indeed a rational number.
It turns out that $\circ^{\mathrm{f}1}$ satisfies \postulate{G1}--\postulate{G6}. This could be proven directly without much effort, but we can also use our result just established, exploiting the fact that we can easily come up with a corresponding assignment: Let $\releq{\abst}^{\mathrm{f}1}$ be defined by letting $r_1 \releqK^{\mathrm{f}1} r_2$ exactly if either $r_1 \in \Mod{\K}$, or both $\{r_1,r_2\} \cap \Mod{\K} = \emptyset$ and $r_1 \leq r_2$.
Then we find that $\releq{\abst}^{\mathrm{f}1}$ is a min-expressible min-friendly faithful assignment compatible with $\circ^{\mathrm{f}1}$ \textbf{for finite bases}.

The attempt to generalize this operator to arbitrary bases (i.e. $\Bases = \MC{P}(\MC{L})$) fails, despite the fact that our previous definition of $\releq{\abst}^{\mathrm{f}1}$ as such seamlessly carries over to arbitrary bases and we obtain a min-friendly faithful assignment. However, this assignment is no longer min-expressible as soon as infinite bases are permitted. To see this, consider $\K=\{[{=}0]\}$ and the infinite base $\G = \{[{\geq} q] \mid q\in \mathbb{Q}^+, q^2 \leq 2\}$. We obtain $\Mod{\G} = \{r \mid r\in \mathbb{R}^+, r\geq \sqrt{2}\}$ and hence $\min(\Mod{\G},\releqK^{\mathrm{f}1}) =\{\sqrt{2}\}$. As $\sqrt{2}$ is an irrational number, this model set cannot be characterized by any finite or infinite base of our logic (it can be readily checked, that the only characterizable model sets are intervals of the form $[r,\infty)$ for any $r\in \mathbb{R}^+$ or $[q,q]$ for any $q \in \mathbb{Q}^+$).
\end{example}

In view of the above example, one might now surmise that the case of finite bases is the more restrictive one and any assignment that works for arbitrary bases will work for finite ones as well. The next example shows that this is not the case.

\begin{example}
	Consider the simplistic base logic 
	$\mathbb{B}_{\mathrm{f}2} \,{=}\, (\MC{L},{\Omega},\models, \Bases, \cup)$
	where  $\MC{L}$ contains two types of sentences: $[{\geq} q]$ and $[{\leq} q]$  for each $q \in \mathbb{Q}^+$, and additionally the sentence $[{\geq} \sqrt{2}]$. Again, let $\Omega = \mathbb{R}^+$.
	The relation $\models$ is defined in exactly the same way as before.	

	For $\Bases = \MC{P}(\MC{L})$,
	we now define a finite base revision operator $\circ^{\mathrm{f}2}$ as follows:
	$$
	\K \circ^{\mathrm{f}2} \G = 
	\begin{cases}
		\K \cup \Gamma & \text{ if } \Mod{\K \cup \G}\not= \emptyset, \\	
		\G & \text{ if } \Mod{\G}= \emptyset,\\
		\{[{\geq} q^-] \mid q^- \leq r\} \cup \{[{\leq} q^+] \mid q^+ \geq r\} & \text{ otherwise, where }\{r\} = min (\Mod{\G},\leq).\\
	\end{cases}
	$$
	Note that the usage of infinite bases ensures that, in the third case, $\Mod{\K\circ^{\mathrm{f}2}\G}= \{r\}$ even if $r$ is irrational.
	Note that $\circ^{\mathrm{f}2}$ satisfies \postulate{G1}--\postulate{G6}. In fact we can define the assignment in exactly the same way as in the previous example to show that.
	Thanks to the change of the logic, we now find that $\releq{\abst}^{\mathrm{f}2}$ is a min-expressible, min-friendly faithful assignment compatible with $\circ^{\mathrm{f}2}$ \textbf{for arbitrary bases}.
	
	However, now the case of finite bases (i.e. $\Bases = \MC{P}_\mathrm{fin}(\MC{L})$) fails to work, although $\releq{\abst}^{\mathrm{f}2}$ remains a min-friendly faithful assignment. However, this assignment is no longer min-expressible as soon as finiteness of bases is imposed. To see this, consider $\K=\{[\geq 0],[\leq 0]\}$ and $\G = \{[{\geq} \sqrt{2}]\}$. We obtain $\Mod{\G} = \{r \mid r\in \mathbb{R}^+, r\geq \sqrt{2}\}$ and hence $\min(\Mod{\G},\releqK^{\mathrm{f}2}) =\{\sqrt{2}\}$. As $\sqrt{2}$ is an irrational number, this model set cannot be characterized by any finite base of our logic (it can be readily checked, that the only characterizable model sets are intervals of the form $[q_1,q_2]$ for any for any $q_1 \in \mathbb{Q}^+\cup \{\sqrt{2}\}$ and $q_2 \in \mathbb{Q}^+$).
\end{example}

These two examples highlight that notions that we define relative to a base logic (like the postulates or the properties of assignments) may crucially depend on $\Bases$ and cease to apply when the notion of base is changed, even if the underlying logic and the considered assignment stay the same.

\subsection{On the Notion of Min-Retractivity}\label{sec:disc-minretract}
As a primary ingredient to our results, the novel notion of min-retractivity has been introduced and motivated in Section~\ref{sec:approachsecond} and proven to serve its purpose later on. As noted earlier, it is immediate that any preorder over $\Omega$ is min-retractive, irrespective of the choice of the other components of the underlying base logic. On the other hand, we have exposed examples of min-retractive relations that are not transitive for certain base logics. This raises the question if there are conditions, under which the two notions \textbf{do} coincide, at least when presuming min-completeness (a condition already known and generally accepted). We start by formally defining this notion of coincidence. 

\begin{definition}
A base logic $ \mathbb{B} = (\MC{L},\Omega,\models,\Bases,\Cup) $ is called \emph{preorder-enforcing}, if every binary relation over $\Omega$ that is total and min-friendly for $\mathbb{B}$ is also transitive (and hence a total preorder).  	
\end{definition}

As an aside, we note that being preorder-enforcing implies the absence of critical loops.

\begin{proposition}
Every preorder-enforcing base logic is loop-free.
\end{proposition}
\begin{proof}
Let $ \mathbb{B} = (\MC{L},\Omega,\models,\Bases,\Cup) $ be a preorder-enforcing base logic.
By \Cref{thm:when_tranisitive}, absence of critical loops would follow from the fact that every base change operator for $\mathbb{B}$ satisfying \postulate{G1}--\postulate{G6} is total-preorder-representable.
To show the latter, consider an arbitrary base change operator $\circ$ of that kind. By \Cref{lfa}, $\releqc{\abst}$ is a min-friendly faithful assignment compatible with $\circ$. In particular, for every $\K \in \Bases$, the corresponding $\releqcK$ is total and min-friendly for $\mathbb{B}$.
Yet then, by assumption, any such $\releqcK$ is also a total preorder, and therefore $\releqc{\abst}$ is even a preorder assignment. Hence, $\circ$ is total-preorder-representable.    
\end{proof}

The question remains, which base logics are actually preorder-enforcing. Note that this is a stronger condition than warranting universal total-preorder-representability (that is, being loop-free), since loop-free base logics may allow for the existence of  non-transitive total min-friendly relations, as long as any \sauerwald{base change operator} has a compatible preorder-assignment (possibly besides other non-transitive ones).   
We will next present a simple criterion and then show that it is indeed necessary and sufficient for being preorder-enforcing. 

\begin{definition}
	A base logic $ \mathbb{B} = (\MC{L},\Omega,\models,\Bases,\Cup) $ is called \emph{trio-expressible}, if for any three interpretations $\I_1,\I_2,\I_3 \in \Omega$ there is a base $\G_{\I_1 \I_2 \I_3}$ satisfying $\Mod{\G_{\I_1 \I_2 \I_3}} = \{\I_1,\I_2,\I_3\}$.  	
\end{definition}

\begin{theorem}
A base logic is preorder-enforcing if and only if it is trio-expressible.
\end{theorem}

\begin{proof}
We show the ``if'' direction followed by the ``only if'' one.  
\begin{itemize}
\item[``$\Leftarrow$'']
Let $\preceq$ be a min-friendly total relation over $\Omega$. 
Toward a contradiction, assume $\preceq$ is not transitive, i.e., there exist
interpretations $\I,\I',\I'' \in \Omega$ such that $\I \preceq \I'$ and $\I' \preceq \I''$ but $\I \not\preceq \I''$. By totality, the latter implies $\I'' \prec \I$. Now consider $\G_{\I \I' \I''}$ and note that $\min(\Mod{\G_{\I \I' \I''}},\preceq) \neq \emptyset$ thanks to min-completeness, but then, by min-retractivity, $\min(\Mod{\G_{\I \I' \I''}},\preceq)=\{\I,\I',\I''\}$ follows. However, $\I \in \min(\Mod{\G_{\I \I' \I''}},\preceq)$ contradicts $\I \not\preceq \I''$.    
\item[``$\Rightarrow$'']
We actually show the contraposition: starting from a base logic $ \mathbb{B} = (\MC{L},\Omega,\models,\Bases,\Cup) $ that is not trio-expressible, we show violation of being preorder-enforcing by exhibiting a total, min-friendly relation over $\Omega$ that is not transitive. 
From non-trio-expressibility, the existence of $\I_1,\I_2,\I_3 \in \Omega$ with $\Mod{\G} \neq \{\I_1,\I_2,\I_3\}$ follows for every $\G \in \Bases$. Let now $\preceq^-$ be an arbitrary well-order\footnote{As discussed earlier, existence of such a $\preceq^-$ is assured by the well-ordering theorem, depending on the axiom of choice.} over $\Omega^- = \Omega \setminus \{\I_1,\I_2,\I_3\}$, i.e., it is total, transitive (hence min-retractive) and min-complete, therefore also min-friendly. We now define $${\preceq} = {\preceq^-} \cup (\Omega^- \times \{\I_1,\I_2,\I_3\}) \cup \{(\I_1,\I_2),(\I_1,\I_3),(\I_3,\I_1),(\I_2,\I_3)\}.$$
It is easy to see that $\preceq$ is still a total relation. It is min-complete (for $\mathbb{B}$) by case distinction: on one hand, if $\Mod{\G} \not\subseteq \{\I_1,\I_2,\I_3\}$, then $\min(\Mod{\G},\preceq) \neq \emptyset$ follows from min-completeness of ${\preceq^-}$, on the other hand, for any two- or one-element subset of $\{\I_1,\I_2,\I_3\}$ also a minimum clearly exists (note that by assumption $\Mod{\G} = \{\I_1,\I_2,\I_3\}$ cannot occur). We proceed to show min-retractivity of $\preceq$, again by case-distinction: If $\Mod{\G} \not\subseteq \{\I_1,\I_2,\I_3\}$ then, due to min-completeness and antisymmetry of ${\preceq^-}$, the set $\min(\Mod{\G},\preceq)$ contains exactly one element and is strictly smaller than any other element from $\Mod{\G}$, thus min-retractivity is vacuously satisfied. If $\Mod{\G} \subset \{\I_1,\I_2,\I_3\}$ min-retractivity is easily verified case by case. We finish our argument by showing that $\preceq$ is not transitive: we have $\I_2 \preceq \I_3$ as well as $\I_3 \preceq \I_1$, but $\I_2 \preceq \I_1$ fails to hold. \qedhere  
\end{itemize}
\end{proof}

The preceding theorem provides yet another argument why preorders can be used as preference relations for finite-signature propositional logic (as in fact, they are the \textbf{only} preference relations arising in that setting). However, note that the result also applies to more complex logics such as first-order logic under the finite model semantics.\footnote{Strictly speaking, this requires a slightly non-standard (but semantically equivalent) model theory which abstracts from the domains used and considers isomorphic models as equal.}

\subsection{Decomposability of Disjunctions}\label{sec:disj-factoring}
In the belief revision literature, the postulates \postulate{G5} and \postulate{G6} are strongly associated with a decomposability of revisions for disjunctive beliefs.
More specifically, under the AGM assumptions \cite{KS_RibeiroWassermannFlourisAntoniou2013}, which includes closure under disjunction, it is known that \postulate{G5} and \postulate{G6} are together equivalent to \emph{disjunctive factoring}\footnote{Here given in a semantic reformulation for (disjunctive) base logics.} \cite{KS_Gaerdenfors1988}
 in the presence of \postulate{G1}--\postulate{G4}:
\begin{equation*}
	\Mod{\K \circ (\G_1 \ovee \G_2)}  = \begin{cases}
		\Mod{\K \circ \G_1} & \text{ or }\\
		\Mod{\K \circ \G_2} & \text{ or }\\
		\Mod{\K \circ \G_1} \cup  \Mod{\K \circ \G_2} & 
	\end{cases}
	\end{equation*}
Factoring conditions like disjunctive factoring are useful and particularly important, as they are commonly used in many generalizations of AGM revision, e.g. \cite{KS_FermeHansson1999,KS_HanssonFermeCantwellFalappa2001}.

We will see that the revision operators considered in this article satisfy a similar decomposability property.
Of course, in the unrestricted setting of arbitrary base logics there might be no way to express disjunctions. 
We propose the following alternative postulates, which we call \emph{expressible disjunctive factoring}:
\begin{equation*}\text{For all } \G, \G_1, \G_2\in\Bases \text{ with } \Mod{\G}=\Mod{\G_1}\cup\Mod{\G_2} \text{ holds }
	\Mod{\K \circ \G}  = \begin{cases}
		\Mod{\K \circ \G_1} & \text{ or }\\
		\Mod{\K \circ \G_2} & \text{ or }\\
		\Mod{\K \circ \G_1} \cup  \Mod{\K \circ \G_2}. & 
	\end{cases}
\end{equation*}
Clearly, for disjunctive base logics, expressible disjunctive factoring and disjunctive factoring are the same.
Now consider the following observation on min-retractivity for the unrestricted setting of arbitrary base logics.
\begin{proposition}\label{prop:retractivity_trichotonomoty}
	Let $ \mathbb{B}=(\MC{L},{\Omega},\models,\Bases,\Cup)$ 	 be a base logic, let $ \G_1,\ldots,\G_n,\G \in \Bases $ be belief bases with $ \Mod{\G}=\Mod{\G_1}\cup \ldots \cup \Mod{\G_n} $ and let $ \preceq $ be a relation over $\Omega$ which is min-retractive for \( \mathbb{B} \). 
	Then there exists some set $  I\subseteq \{1,\ldots,n\} $ such that $ {\min(\Mod{\G},\preceq)} = {\bigcup_{i\in I}\min(\Mod{\G_i},\preceq)} $.
\end{proposition}
\begin{proof}
    
	Let $ I $ be a minimal set such that $ {\min(\Mod{\G},\preceq)}\subseteq \bigcup_{i\in I} \Mod{\G_i} $ holds, i.e., there is no \( j\in I \) such that $ {\min(\Mod{\G},\preceq)}\subseteq \bigcup_{i\in I\setminus\{j\}} \Mod{\G_i} $ holds.    
    Note that the assumption $ \Mod{\G}=\Mod{\G_1}\cup \ldots \cup \Mod{\G_n} $ guarantees the existence of \( I \).    
    Because of $ \Mod{\G_i} \subseteq \Mod{\G}$, we have that $ \I\in{\min(\Mod{\G},\preceq)} \cap \Mod{\G_i}  $ implies $ \I \in {\min(\Mod{\G_i},\preceq)} $. 
    This shows that \(  \min(\Mod{\G},\preceq) \subseteq \bigcup_{i\in I}\min(\Mod{\G_i},\preceq) \) holds.
    
    We show that  \(   \bigcup_{i\in I}\min(\Mod{\G_i},\preceq) \subseteq \min(\Mod{\G},\preceq) \) holds.
    Because \( I \) is a minimally chosen, for every \( i\in I \) the set \( \Mod{\G_i} \cap \min(\Mod{\G},\preceq)  \) is non-empty.    
    If $ \I\in \min(\Mod{\G},\preceq)  $ and $ \I\in \Mod{\G_i} $, then we have for each $ \I'\in {\min(\Mod{\G_i},\preceq)}  $ that $ \I' \preceq \I $ holds. 
    Since $ \preceq $ is min-retractive, we conclude that $ \I' \in  \min(\Mod{\G},\preceq) $. 
    Consequently, we obtain $ {\min(\Mod{\G_i},\preceq)} \subseteq {\min(\Mod{\G},\preceq)} $.\qedhere
    
\end{proof}

By Proposition~\ref{prop:retractivity_trichotonomoty}, min-retractivity of a relation $ \preceq $ guarantees a decomposition property on the relation with respect to minima and bases. 
Thus, if a \basechange\ operator $ \circ $ is compatible with an assignment $ \releq{\abst} $, which yields always min-retractive relations, then the decomposition property of $ \releq{\abst} $ carries over to the operator.
In the following, we use Proposition~\ref{prop:retractivity_trichotonomoty} to obtain for arbitrary base logics that \postulate{G1}--\postulate{G6} imply expressible disjunctive factoring.

\begin{proposition}
	Let $ \mathbb{B}=(\MC{L},{\Omega},\models,\Bases,\Cup)$ be a base logic. If \( \circ \) is a \basechange\ operator for \( \mathbb{B} \) that satisfies \postulate{G1}--\postulate{G6}, then \( \circ \) satisfies \emph{expressible disjunctive factoring}.
\end{proposition}
\begin{proof}
	If \( \circ \) satisfies \postulate{G1} -- \postulate{G6}, then by \Cref{prop:retractivity_trichotonomoty} there is a min-friendly assignment \( \releq{\abst} \) compatible with \( \circ \).
	Now let \( \G_1,\G_2,\G_3\in\Bases \) such that \( \Mod{\G_3}=\Mod{\G_1}\cup\Mod{\G_2} \).
	By compatibility and Proposition   \ref{prop:retractivity_trichotonomoty} we obtain that expressible disjunctive factoring is satisfied.
\end{proof}

It turns out that, in the general setting of base logics, \postulate{G5} and \postulate{G6} are not implied by expressible disjunctive factoring. 
In particular, the following proposition shows the even stronger statement that \postulate{G5} and \postulate{G6} are not equivalent to disjunctive factoring even when restricting to disjunctive base logics. 

\begin{proposition}\label{prop:violation_disjunctive_factoring}
	There is a disjunctive base logic $ \mathbb{B} $ and a \basechange\ operator $ \circ $ for $ \mathbb{B} $ which satisfies \postulate{G1}--\postulate{G4} and disjunctive factoring, yet both \postulate{G5} and \postulate{G6} are violated.
\end{proposition}
\begin{proof}
\newcommand{\threename}{\ensuremath{\mathrm{four}}}
In the following, we denote with \( \leq \) the usual relation on the natural numbers.
Let \( \mathbb{B}_{\threename} = (\MC{L},\Omega,\models,\Bases,\Cup) \) be the base logic such that \( \MC{L}=\{[{\geq} i] \mid 0 \leq i \leq 4 \} \) and \( \Omega=\{ 0,1,2,3 \} \). %
The bases of \( \mathbb{B}_{\threename} \) are all sets with exactly one sentence, i.e., we have \( \Bases=\{ \{[{\geq} i]\} \mid [{\geq} i] \in \MC{L}\}\).
For every  \( [{\geq} i] \in\MC{L} \) and every interpretation \( \I \in \Omega \),  we let \( \I \models [{\geq} i] \) if \( i \leq \I \). %
We will use \( \min \) and \( \max \) with respect to the usual ordering on natural numbers, but slightly abuse notions, and define \( \min(\G)=\min(\{ i \mid [{\geq} i]\in\G \}) \) and \( \max(\G)=\max(\{ i \mid [{\geq} i]\in\G \}) \) for \( \G\subseteq \MC{L} \).
Using these notions, we define the operations  \( \Cup \) and \( \ovee \) on bases as follows:
\begin{align*}
	\G_1 \Cup \G_2  & = \textstyle\max(\G_1\cup\G_2) & \G_1 \ovee \G_2 & = \textstyle\min(\max(\G_1) \cup \max(\G_2))
\end{align*}
We obtain that \( \mathbb{B}_{\threename} \) is a disjunctive base logic, as for every two bases \( \G_1,\G_2\in \Bases \) we have \( \Mod{\G_1\Cup\G_2}=\Mod{\G_1}\cap\Mod{\G_2} \) and we have \( \G_1 \ovee \G_2 = \Mod{\G_1}\cup\Mod{\G_2} \).
Note that we have \( \Mod{[{\geq} 4]}=\emptyset \) and that \( \Mod{\G_1\Cup\G_2}\neq\emptyset \) for every two bases \( \G_1, \G_2\in\Bases \)  whenever \( [{\geq} 4] \notin \G_1\cup\G_2 \).

Now let \( \circ^{\threename} \) be the \basechange\ operator for \( \mathbb{B}_\threename \) such that:
\begin{equation*}
	\K \circ^{\threename} \G = \begin{cases}
		\K \Cup \G & \text{if } \Mod{\K \Cup \G}\neq\emptyset,\\
		\{ [{\geq} 3] \} & \text{if } \Mod{\K \Cup \G}=\emptyset \text{ and } \G = \{[{\geq} 1]\}, \\
		\G & \text{otherwise}.
	\end{cases}
\end{equation*}

Satisfaction of \postulate{G1}--\postulate{G4} is immediate due to the construction of \( \circ^{\threename} \).
We show that \( \circ^{\threename} \) satisfies (expressible) disjunctive factoring. 
Let \( \K,\G_1,\G_2\in\Bases \) be three belief bases. 
We consider two cases:

\( \Mod{\K \Cup (\G_1 \ovee \G_2)}\neq\emptyset \). This implies that \( [{\geq} 4]\notin\K\cup (\G_1 \cap \G_2) \). 
Therefore, we have \( \Mod{\K \Cup \G_1} \neq\emptyset \) or \( \Mod{\K \Cup \G_2} \neq\emptyset \).
In particular, we have either \( \max(\K \Cup \G_1) = \max(\K \Cup (\G_1 \ovee \G_2)) \) or \( \max(\K \Cup \G_2) = \max(\K \Cup (\G_1 \ovee \G_2)) \).
In the first case, we have \( \Mod{\K \Cup (\G_1 \ovee \G_2)} = \Mod{\K \Cup \G_1} \); in the latter we have  \( \Mod{\K \Cup (\G_1 \ovee \G_2)} = \Mod{\K \Cup \G_2} \).

\( \Mod{\K \cup (\G_1 \ovee \G_2)}=\emptyset \). We have either \( \max(\G_1 \ovee \G_2) = \max(\G_1) \) or \( \max(\G_1 \ovee \G_2)=\max(\G_2) \).
In the first case, we obtain \( \Mod{\K \Cup (\G_1 \ovee \G_2)} = \max(\K \circ^{\threename}\G_1) \); in the latter case, we obtain \( \Mod{\K \cup (\G_1 \ovee \G_2)} = \max(\K \circ^{\threename}\G_2) \).

For a violation of \postulate{G6} observe that by \postulate{G4}, satisfaction of \postulate{G6} implies satisfaction of:
\begin{equation}\label{eq:g6_implication}
	\text{if } \G_1 \models \G_2 \text{ and } \Mod{(\K\circ^{\threename}\G_2)\Cup \G_1}\neq\emptyset \text{, then } \K\circ^{\threename}\G_1 \models (\K\circ^{\threename}\G_2)\Cup \G_1
\end{equation}	
Now choose \( \G_1 = \{ [{\geq} 2] \} \) and \( \G_2=\{[{\geq} 1] \} \), and  let \( \K = \{ [{\geq} 4] \} \).
Observe that \( \G_1 \models \G_2 \), as \( \Mod{\G_1}=\{ 2,3 \} \) and \( \Mod{\G_2}=\{ 1,2,3 \} \).
Moreover, \( \Mod{\K\circ^{\threename}\G_1}=\{2,3\} \) and \( \Mod{\K\circ^{\threename}\G_2}=\{3\} \).
Thus, we have \( \Mod{(\K\circ^{\threename}\G_2)\Cup \G_1}=\{3\}\neq\emptyset \).
We obtain a contradiction to Equation \eqref{eq:g6_implication}, because \( \Mod{\K\circ^{\threename}\G_1}=\{2,3\} \not\subseteq  \Mod{\K\circ^{\threename}\G_2}=\{3\}  = \Mod{(\K\circ^{\threename}\G_2)\Cup \G_1}  \).

To demonstrate a violation of \postulate{G5}, let \( \G_1 = \{ [{\geq} 0] \} \) and \( \G_2 = \{ [{\geq} 1] \} \).
Then by definition we have \( \G_1 \Cup \G_2=\{ [{\geq} 1 ] \} \).
Now for the base \( \K = \{ [{\geq} 4] \} \) observe that \( \K \circ^{\threename} \G_1 = {\{[{\geq} 0]\}} \) and \( \K \circ^{\threename} (\G_1 \Cup \G_2) = {\{[{\geq} 3]\}} \).
Moreover, we obtain that \( (\K \circ^{\threename} \G_1)\Cup \G_2 = \{ [{\geq} 1 ] \} \).
In summary we obtain \( \Mod{\{ [{\geq} 1 ] \}} = \{ 1,2,3 \} \not\subseteq \{ 3 \} = \Mod{\{ [{\geq} 3 ] \}} \), which shows that \( \circ^{\threename} \) does not satisfy \postulate{G5}.\qedhere

\end{proof}

\section{Related Work}\label{sec:related_works}
In settings beyond propositional logic, we are aware of 
three closely related approaches that propose model-based frameworks for revision of belief bases (or sets) without fixing a particular logic or the internal structure of interpretations, and characterize revision operators via minimal models \`{a} la K\&M with some additional assumptions. In the following, we discuss these results and their relationship to our approach. \Cref{table:approaches} summarizes the three approaches and compares them with K\&M's result and our approach, which comes in \sauerwald{six} variants.

\definecolor{mybggray}{rgb}{0.95,0.95,0.95}
\setlength\dashlinedash{0.2pt}
\setlength\dashlinegap{1.5pt}
\setlength\arrayrulewidth{0.3pt}
\renewcommand\theadfont{\normalsize}
\newcommand\centeredtopline[4]{
\rule[-1.5ex]{0ex}{1ex}\\
\hline
\rowcolor{mybggray}
\multicolumn{6}{c}{{#4}\rule[-#2]{0ex}{#1}}\\[-0.4ex]
\multicolumn{6}{c}{} \\[#3]			
}
\definecolor{mybggray}{rgb}{0.95,0.95,0.95}
\newcommand{\mycenteredtopline}[1]{\centeredtopline
{4.0ex}%
{1.4ex}%
{-1.5ex}%
{#1}%
}

\begin{table}[p]
	\centering
	\footnotesize
	\resizebox{\textwidth}{!}{%
		\begin{tabular}[t]{@{}p{2.68cm}p{0.8cm}p{1.6cm}p{1.2cm}p{1.0cm}p{3.6cm}@{}}
			\toprule
			\makecell[l]{logic\\ setting} & \makecell[l]{belief\\ bases}  & \makecell[l]{assignment\\$\left.\right.$ }
			& \makecell[l]{%
				postulates\\$\left.\right.$} & \makecell[l]{relation\\ encoding} & \makecell[l]{notes\\$\left.\right.$ }
\mycenteredtopline{{Katsuno and Mendelzon} \cite{kat_1991}}
			propositional logic over finite signature & $\MC{P}(\MC{L})$, $\MC{P}_\mathrm{fin}(\MC{L})$, $\mathcal{L}$ & preorder,\quad\quad faithful & \texttt{\postulate{G1}-\postulate{G6}} & Eq.
			\eqref{eq:km_encoding} & logic natively free of critical loops; natively min-friendly; two-way representation theorem
\mycenteredtopline{{Grove} \cite{Grove1988}}
			Boolean-closed compact logics; %
			non-standard model theory based on maximal consistent sets of sentences& $\MC{P}(\MC{L})$ & preorder,\quad\quad\quad\quad faithful,\quad\quad\quad\quad\quad min-complete  & \texttt{\postulate{G1}-\postulate{G6}} & Eq.
			\eqref{eq:km_encoding} & 
			all such logics natively free of critical loops; any assignment min-ex\-pres\-sible; two-way representation theorem
\mycenteredtopline{{Delgrande et al.} \cite{KS_DelgrandePeppasWoltran2018}}
			Tarskian logics with finite $\Omega$, any $\omega, \omega'$ distinguishable by some sentence & $\MC{P}(\MC{L})$ 
			& preorder,\quad\quad faithful,\quad\quad\quad min-expressible & \texttt{\postulate{G1}-\postulate{G6},} \qquad\qquad 
			(\texttt{Acyc}) & Eq. \eqref{eq:dpw_encoding} & 
			extra postulate (Acyc) rules out ``non-preorder operators''; two-way representation theorem
\mycenteredtopline{{Aiguier et al.} \cite{aiguier_2018}}
			Tarskian logics  & $\MC{P}(\MC{L})$  
			& %
			quasi-faithful, min-complete &
			\texttt{\postulate{G1}-\postulate{G3},} \texttt{\postulate{G5},\postulate{G6}}  & Eq. \eqref{eq:aabh_encoding}  &
			non-standard notion of inconsistency; additional ad-hoc constraint on compatibility; one-way representation theorem 
\mycenteredtopline{our approach}
			Tarskian logics & \hspace{1em}\multirow{6}{*}[1em]{\begin{sideways}{arbitrary, closed under abstract union (e.g., $\MC{P}(\MC{L})$, $\MC{P}_\mathrm{fin}(\MC{L})$, $\mathcal{L}$, belief sets) \ \ \ \ \ \ }\end{sideways}} &
			faithful, \qquad\qquad min-retractive, min-complete, min-expressible &
			\texttt{\postulate{G1}-\postulate{G6}} & Def. \ref{def:relation_new} & 
			most general (syntax-in\-de\-pendent version); two-way representation theorem (\Cref{thm:rep1})\\
			\rule{0pt}{4ex}Tarskian logics & %
			& %
			quasi-faithful,\qquad\qquad min-retractive, min-complete,
			min-expressible & \texttt{\postulate{G1}-\postulate{G3},} \texttt{\postulate{G5},\postulate{G6}}   & Def. \ref{def:relation_new} &
			most general (syntax-de\-pen\-dent version); two-way representation theorem (\Cref{thm:rep2})\\
			\rule{0pt}{4ex}Tarskian logics & & faithful, \qquad\qquad preorder,\qquad\qquad min-complete,
			min-expressible & \texttt{\postulate{G1}-\postulate{G6}} \quad\quad (\texttt{Acyc}) & Def. \ref{def:relation_tpo_reduction} &
			preorder preference (syn\-tax-in\-de\-pen\-dent ver\-sion); two-way representation theorem (\Cref{thm:rep3acyc})\\
			\rule{0pt}{4ex}Tarskian logics &  & quasi-faithful, \qquad preorder, \qquad min-complete, 
			min-expressible & \texttt{\postulate{G1}-\postulate{G3},} \texttt{\postulate{G5},\postulate{G6}} \quad\quad (\texttt{Acyc})& Def. \ref{def:relation_tpo_reduction} &
			preorder preference (syn\-tax-de\-pen\-dent version); two-way representation theorem {(\Cref{thm:rep4acyc})}\\
			\rule{0pt}{4ex}\sauerwald{loop-free} Tarskian logics \ $\left.\right.$ (e.g., with disjunction) & & faithful, \qquad\qquad preorder,\qquad\qquad min-complete,
			min-expressible & \texttt{\postulate{G1}-\postulate{G6}}  & Def. \ref{def:relation_tpo_reduction} &
			preorder preference (syn\-tax-in\-de\-pen\-dent ver\-sion); two-way representation theorem (\Cref{thm:rep3})\\
			\rule{0pt}{4ex}\sauerwald{loop-free} Tarskian logics \ $\left.\right.$ (e.g., with disjunction) &  & quasi-faithful, \qquad preorder, \qquad min-complete, 
			min-expressible & \texttt{\postulate{G1}-\postulate{G3},} \texttt{\postulate{G5},\postulate{G6}} & Def. \ref{def:relation_tpo_reduction} &
			preorder preference (syn\-tax-de\-pen\-dent version); two-way representation theorem (\Cref{thm:rep4})\\
			\bottomrule
		\end{tabular}
	}
	\captionsetup{singlelinecheck=off}
	\caption{Overview of our characterization results and comparison with related work. 
	}
	\label{table:approaches}
\end{table}
\paragraph{{Grove} \cite{Grove1988}.} 
One semantic-based approach related to the one of K\&M was proposed by Grove \cite{Grove1988} for belief set revision in Boolean-closed compact logics. Strictly speaking, Grove's approach is not genuinely model-theoretic, as the space of mathematical objects on which he considers the preference relations consists of the inclusion-maximal consistent subsets of $\MC{L}$ rather than of proper worlds. However, under Grove's assumptions, one could define a ``alternative model theory'' where $\Omega$ is simply defined to be the inclusion-maximal consistent subsets of $\MC{L}$ and the satisfaction relation $\models$ is just set membership $\in$. Notably, this alternative model theory produces the same logical entailments as the original one, while it also exhibits the favorable property that any interpretation set $\Omega' \subseteq \Omega$ is expressible by some belief set $\K \subseteq \MC{L}$ (which is obtained by intersecting all elements of $\Omega'$).  

Instead of starting from a preorder relation $\releqK$, Grove characterizes AGM revision operators via \emph{systems of spheres}, collections \textbf{S} of subsets of $\Omega$ satisfying certain conditions. 
Yet, the notion of a system of spheres is closely related to that of a faithful preorder relation $\releqK$ as one can be generated from the other \cite{KS_GaerdenforsRott1995}. 
Given a faithful preorder $\releqK$, for each $\I\in\Omega$, one can define $S_\I$ as a set of interpretations $\I'$ such that $\I'\releqK\I$.
Notably, Grove's condition \emph{(S4)} for a system of spheres directly corresponds to the min-completeness requirement for the corresponding preference relation.

Grove's result constitutes a special case of our representation theorem: from the assumption of Boolean-closedness, it follows that any logic under consideration must be disjunctive and therefore free of critical loops (cf. \Cref{thm:when_tranisitive} and \Cref{cor:disjunctive_tpo}). As discussed above, the definition of $\Omega$ immediately implies min-expressibility for all relations. In the light of these observations, Grove's result turns out to be a special case of the fifth variant of our result in \Cref{table:approaches}. Also, note that the requirement of compactness in Grove's approach serves to ensure the synchronization between the original and alternative model theories. In contrast, our approach operates on genuine interpretations right away, whence compactness is not required.

\paragraph{{Delgrande et al.} \cite{KS_DelgrandePeppasWoltran2018}.}
The representation result of Delgrande et al. \cite{KS_DelgrandePeppasWoltran2018} confines the considered logics to those where the set $\Omega$ of interpretations (or possible worlds) is finite\footnote{Note that this precondition excludes more complex logics such as  first-order or modal logics and most of their fragments, but also propositional logic with infinite signature. On the positive side, this choice guarantees min-completeness of any preorder.} and where any two different interpretations $\omega,\omega' \in \Omega$ can be distinguished by some sentence $\varphi \in \mathcal{L}$, i.e., $\omega \in \Mod{\varphi}$ and $\omega' \not\in \Mod{\varphi}$. Moreover, following prior work by \citeauthor{KS_DelgrandePeppas2015} \cite{KS_DelgrandePeppas2015}, they extend the AGM postulates by the additional postulate \postulate{Acyc}, which reads as follows in their formulation:
\begin{quote}
	For any base $\K$ and all $\G_1,\ldots,\G_n \in \MC{P}(\MC{L})$ with $\Mod{\G_i \cup \K\circ \G_{i+1}}\neq\emptyset$ for each $1\leq i < n$ as well as $\Mod{\G_n \cup \K \circ \G_1}\neq\emptyset$  holds $\Mod{\G_1 \cup \K\circ \G_n}\neq\emptyset$.
\end{quote} 
With these ingredients in place, {Delgrande et al.}~\cite{KS_DelgrandePeppasWoltran2018}~establish that, for the logics they consider, there is a two-way correspondence between those AGM revision operators satisfying \postulate{Acyc} and min-expressible faithful preorder assignments.
Instead of the term ``min-expressible'', they use the term \emph{regular}. As laid out in \Cref{sec:acycresult} and \Cref{sec:extendedtheorems}, Delgrande et al.'s result generalizes way beyond the setting considered by them, if one adds the requirement of min-completeness and adjusts the formulation of \postulate{Acyc} to base logics in the straightforward way. 

In addition, we investigated a line of thought that is complementary of Delgrande et al. \cite{KS_DelgrandePeppasWoltran2018}. As we saw, in \sauerwald{(base) logics that are not loop-free}, one cannot hope for a characterization of \textbf{all} AGM revision operators by means of assignments producing preorders. Our alternative proposal was to relinquish the requirement of using preorders, giving up transitivity and merely retaining min-retractivity. 

\paragraph{{Aiguier et al.} \cite{aiguier_2018}.}
The approach of
Aiguier et al. \cite{aiguier_2018} considers AGM belief base revision in logics with a possibly infinite set $\Omega$ of interpretations. 
Notably, they propose to consider certain bases, that actually \textbf{do} have models, as inconsistent (and thus in need of revision). While, in our view, this is at odds with the foundational assumptions of belief revision (revision should be union/conjunction unless facing unsatisfiability), this appears to be a design choice immaterial to the established results. To avoid confusion, we will ignore it in our presentation.
As far as the postulates are concerned, Aiguier and colleagues decide to rule out \postulate{KM4}/\postulate{G4}, arguing in favor of syntax-dependence. Consequently, they re-define the notion of faithfulness of assignments, eliminating (F3), and arriving at the notion that what we call quasi-faithfulness.
Like us, Aiguier et al. propose to drop the requirement that assignments have to yield preorders. Also, their representation result (Theorem~1) features a condition that corresponds to min-completeness (second bullet point). In addition to the standard notion of compatibiliy, their result hinges on the following additional correspondence between the assignment and the preorder (third bullet point), for every $\K,\G_1,\G_2 \in \MC{P}(\MC{L})$:
\begin{quote}
If $(\K \circ \G_1) \cup \G_2$ is consistent, then $\min(\Mod{\G_1},\preceq_\K) \cap \Mod{\G_2}=\min(\Mod{\G_1 \cup \G_2},\preceq_\K)$.
\end{quote} 
A closer inspection of this extra condition shows that it is essentially a translation of a combination of the postulates \postulate{G5} and \postulate{G6} into the language of assignments and minima. It remains somewhat unclear to us what the intuitive justification of this (arguably rather technical and unwieldy) extra condition should be, beyond providing the missing ingredient to make the result work. Possibly, this is the reason why the presented result is just one-way: it does not provide a characterization of exactly those assignments for which a compatible AGM revision operator exists. Rather it pre-assumes existence of a revision operator under consideration.    

\bigskip

We think that our approach provides improvements regarding ways to construct an appropriate assignment from a given belief revision operator.
For comparison,
Delgrande et al. \cite{KS_DelgrandePeppasWoltran2018} solve this problem by simultaneously revising with all sentences satisfying $ \I_1 $ and $ \I_2 $, in order to ``simulate'' the revision by the desired formula $ form(\I_1,\I_2) $: for every base $\K$,
\begin{equation}\label{eq:dpw_encoding}
	\begin{aligned}
		\preceq_\K \text{ is the transitive closure of } \{ (\I_1,\I_2) \mid \I_1 \models \K\circ (t(\{\I_1\})\cap t(\{\I_2\})) \}
	\end{aligned}
\end{equation}
where $t(\{\omega\})$ is the set of all sentences satisfied by $\omega$.
Aiguier et al. \cite{aiguier_2018} use a similar approach by revising with all sentences at once: they let, for every base $\K$,
\begin{equation}\label{eq:aabh_encoding}
	\I_1 \preceq_\K \I_2 \text{ if } \I_1 \models \K \text{ or } \I_1 \models \K\circ \{\I_1,\I_2 \}^*
\end{equation}
where $\{\omega_1, \omega_2\}^*$ is the set of all sentences satisfied by both $\omega_1$ and $\omega_2$. 
In summary, both Aiguier et al. and Delgrande et al. use an encoding approach in the spirit of Katsuno and Mendelzon, attempting to establish a relation between two interpretations, whenever the revision operator provides evidence. 
As discussed in \Cref{sec:enc_operators}, we take a somewhat dual approach: we establish a relation between any two worlds, unless the considered revision operator provides evidence to the contrary.

\section{Summary and Conclusion}\label{sec:conclusion}
The central objective of our treatise was to provide an exact model-theoretic characterization of AGM belief revision in the most general reasonable sense, i.e., one that uniformly applies  

\begin{itemize}
\item to every logic with a classical model theory (i.e., every Tarskian logic),
\item to any notion of bases that allows for taking some kind of ``unions'' (including the cases of belief sets, sets of sentences, finite sets of sentences, and single sentences), and
\item to all base change operators adhering to \sauerwald{the AGM postulates} (without imposing further restrictions through additional postulates).
\end{itemize} 

To this end, we followed the well-established approach of using \emph{assignments}: functions that map every base $\K$ to a binary relation $\releqK$ over the interpretations $\Omega$ of the considered logic, where $\I_1 \releqK \I_2$ intuitively expresses a preference, i.e., that $\I_1$ is ``closer'' than (or at least as close as) $\I_2$ to being a model of $\K$ (even if it is not a proper model). With this notion in place, the \emph{compatibility} between a revision operator and an assignment then requires that the result of revising a base $\K$ by some base $\G$ yields a base whose models are the $\releqK$-minimal interpretations among the models of~$\Gamma$.

The original result of K\&M for signature-finite propositional logic, establishing a two-way-correspondence between AGM revision operators and faithful assignments that yield total preorders. We showed that in the general case considered by us, this original result fails in many ways and needs substantial adaptations. In particular, aside from delivering total relations and being faithful, the assignment now needs to satisfy

\begin{itemize}
\item \emph{min-expressibility}, guaranteeing existence of a describing base for any model set obtained by taking minimal interpretations among some base's models,
\item \emph{min-completeness}, ensuring that minimal interpretations exist in every base's model set, and
\item \emph{min-retractivity} instead of transitivity, making sure that minimality is inherited to more preferable elements. 
\end{itemize}

While the first two adjustments have been recognized and described in prior work, the notion of min-retractivity (and the decision to replace transitivity by this weaker notion and thus give up on the requirement that preferences be preorders) seems to be novel. Yet, it turns out to be the missing piece for establishing the desired two-way compatibility-correspondency between AGM revision operators and preference assignments of the described kind (cf. \Cref{thm:rep1}).

In view of the fact that the requirement of syntax-independence -- as expressed in Postulate \postulate{G4} -- may be (and has been) put into question, we also established a syntax-dependent version of our characterization (cf. \Cref{thm:rep2}). Crucial to this result is the observation that \postulate{G4} is exactly mirrored by the third faithfulness condition on the semantic side; thus removing it (going from faithfulness to \emph{quasi-faithfulness}) yields the class of assignments compatibility-corresponding to revision operators satisfying the postulates \postulate{G1}--\postulate{G3}, \postulate{G5}, and \postulate{G6}.   

Conceding that transitivity is a rather natural choice for preferences and preorder assignments might be held dear by members of the belief revision community, we went on to investigate for which operators, respectively assignments, respectively logics, our general result holds even if assignments are required to yield preorders. 
We showed that a property introduced by Delgrande et al. \cite{KS_DelgrandePeppasWoltran2018}, called \postulate{Acyc}, could be generalized to our general setting such that total-preorder-representable base change operators are characterized. 
We managed to pinpoint a specific logical phenomenon (called \emph{critical loop}), the absence of which in a logic is necessary and sufficient for \emph{total-preorder-representability}. 
While the criterion by itself maybe somewhat technical and unwieldy, it can be shown to subsume all logics featuring disjunction and therefore all classical logics. 
This justifies to formulate these findings in two theorems: a syntax independent version (\Cref{thm:rep3}) and a syntax-dependent one (\Cref{thm:rep4}).      

As one of the avenues for future work, we will consider the largely untreated problem of iterated base revision. To this end, our aim is to combine the general setting of base logics presented here with work in the line of research by Darwiche and Pearl \cite{KS_DarwichePearl1997}.
In particular, we will explore the specific role of a non-transitive, yet min-retractive, relation, in the process of iteration.

Finally, we will also be working on concrete realisations of the approach presented here in popular KR formalisms such as ontology languages, including Description Logics (DLs), both under standard and non-standard semantics. With our results in place, we will be able to scrutinize the underlying logics for the existence of critical loops and find uniform ways of describing AGM revision by means of appropriate assignments. We also plan to arrive at more fine-grained representation theorems that take computability and complexity aspects into account.

\begin{acks}
\sauerwald{We would like to thank Christoph Beierle, James Delgrande, Eduardo Fermé, Peter Gärdenfors, Gabriele Kern-Isberner, and the many other colleagues whose critical and constructive comments helped to shape and improve this article.

Faiq Miftakhul Falakh was supported by the Indonesia Endowment Fund for Education (LPDP) Scholarship No. 0000156/SC/D/2/lpdp2016 and by the Federal Ministry of Education and Research, Germany (BMBF) in the Center for Scalable Data Analytics and Artificial Intelligence (ScaDS. AI).

Sebastian Rudolph was supported by the European Research Council through his ERC Consolidator Grant ``A Grand Unified Theory of Decidability in Logic-Based Knowledge Representation'' (DeciGUT) under Grant  No. 771779.  
    
Kai Sauerwald worked on this article mainly when he was a member of
the Knowledge-Based System Group at the FernUniversität in Hagen.
He was supported by the Grants BE 1700/9-1 and BE 1700/10-1 of the
German Research Foundation (DFG) awarded to Christoph Beierle within
the DFG Priority Research Program “Intentional Forgetting in Organizations” (SPP 1921).}
\end{acks}

\bibliographystyle{ACM-Reference-Format}
\bibliography{jair,sauerwald}

\newpage
{\noindent\Large{\section*{APPENDICES}}}
\appendix

\section{Tarskian Logics and Model Theory}\label{sec:app_Tarskian}
In the following, we show that the class of logics defined model-theoretically (as laid out in \Cref{sec:modeltheory}) and the class of Tarskian logics coincide. We start by providing the definition of Tarskian logics.

\begin{definition}
	Let \( \MC{L} \) be a set.  A function \( Cn: \MC{P}(\MC{L}) \to \MC{P}(\MC{L}) \) is a called a \emph{Tarskian} consequence operator (on \( \MC{L} \)) if it is a closure operator, i.e., it satisfies the following properties for all subsets \( \MC{K},\MC{K}_1,\MC{K}_2\subseteq \MC{L} \):
	\begin{align*} 
		&	\MC{K}\subseteq Cn(\MC{K})
		\tag{extensive}\\
		&	\text{if } \MC{K}_1 \subseteq \MC{K}_2  \text{, then }  Cn(\MC{K}_1) \subseteq Cn(\MC{K}_2)
		\tag{monotone}\\
		&	 Cn(\MC{K}) = Cn(Cn(\MC{K}))
		\tag{idempotent}
	\end{align*}
Any Tarskian consequence operator  \( Cn: \MC{P}(\MC{L}) \to \MC{P}(\MC{L}) \) gives rise to a \emph{Tarskian consequence relation} \( {\VDash} \subseteq \MC{P}(\MC{L}) \times \MC{L} \) defined by \( \K \VDash \varphi \) if \( \varphi \in Cn(\MC{K}) \). 
Each \( (\MC{L},\VDash) \) obtained from a Tarskian consequence operator  \( Cn: \MC{P}(\MC{L}) \to \MC{P}(\MC{L}) \) will be called a \emph{Tarskian logic} here.
\end{definition}

We proceed to show that the existence of a model-theoretically defined semantics is sufficient and necessary for a logic being Tarskian. 

\begin{proposition}
	For every model theory \( (\MC{L},\Omega,\models) \) there exists a Tarskian logic \( (\MC{L},\VDash) \) with \( \MC{K} \VDash \varphi \) if and only if \( \MC{K} \models \varphi \) for all $\varphi \in \MC{L}$ and \( \MC{K}\in\MC{P}(\MC{L}) \).
\end{proposition}
\begin{proof}    
    Given \( (\MC{L},\Omega,\models) \), let \( Cn: \MC{P}(\MC{L}) \to \MC{P}(\MC{L}) \) be defined by 
    $\MC{K} \mapsto \{\varphi \in \MC{L} \mid \Mod{\MC{K}} \subseteq \Mod{\varphi}\}$.
    We will show that $Cn$ is a Tarskian consequence operator. 
    
    For extensivity, consider some arbitrary $\psi \in \K$. Then we obtain $\Mod{\K} = \bigcap_{\varphi \in \K}\Mod{\varphi} \subseteq \Mod{\psi}$ and hence $\psi \in Cn(\K)$. Hence, since $\psi$ was chosen arbitrarily, we obtain $\K \subseteq Cn(\K)$.
    
    For monotonicity, suppose $\MC{K}_1 \subseteq \MC{K}_2$. Then $\Mod{K_2} =  \bigcap_{\varphi \in \K_2}\Mod{\varphi} \subseteq \bigcap_{\varphi \in \K_1}\Mod{\varphi} = \Mod{K_1}$.
    Therefore, we obtain $Cn(\K_1) = \{\varphi \in \MC{L} \mid \Mod{\K_1} \subseteq \Mod{\varphi}\} \subseteq \{\varphi \in \MC{L} \mid \Mod{\K_1} \subseteq \Mod{\varphi}\} = Cn(\K_2)$.
    
    For idempotency, we show bidirectional inclusion. $Cn(\MC{K}) \subseteq Cn(Cn(\MC{K}))$ is an immediate consequence of extensivity already shown. For the other direction, consider an arbitrary $\psi \in Cn(Cn(\K))$. Then, we obtain $\Mod{Cn(\K)} \subseteq \Mod{\psi}$. On the other hand, we have
    $$\Mod{Cn(\K)} = \bigcap_{{\varphi \in \MC{L}}\atop{\Mod{\MC{K}} \subseteq \Mod{\varphi}}}\Mod{\varphi} = \bigcap_{\varphi \in \K} \Mod{\varphi} = \Mod{\K},$$
    and therefore, we obtain $\Mod{\K} \subseteq \Mod{\psi}$ and finally $\psi \in Cn(\K)$. Hence, since $\psi$ was chosen arbitrarily, we obtain $Cn(Cn(\K)) \subseteq Cn(\K)$.
    
    Let now $\VDash$ denote the Tarskian consequence relation induced by $Cn$. Then we obtain for all $\K \subseteq \MC{L}$ and $\varphi \in \MC{L}$ the following:
\begin{equation*}
        \K \VDash \varphi \Longleftrightarrow \varphi \in Cn(\K) \Longleftrightarrow  \Mod{\MC{K}} \subseteq \Mod{\varphi} \Longleftrightarrow \K \models \varphi. \qedhere
\end{equation*}
\end{proof}

As last step, we show that for each Tarskian logic there is a canonical model-theoretic semantics for this Tarskian logic.

\newpage
\begin{proposition}
	For every Tarskian logic \( (\MC{L},\VDash) \) there exists a model theory \( (\MC{L},\Omega,\models) \) such that \( \MC{K} \VDash \varphi \) if and only if \( \MC{K} \models \varphi \) holds for all for all $\varphi \in \MC{L}$ and \( \MC{K}\in\MC{P}(\MC{L}) \).
\end{proposition}
\begin{proof} 
	Let \( (\MC{L},\VDash) \) be a Tarskian logic and let \( Cn: \MC{P}(\MC{L}) \to \MC{P}(\MC{L}) \) be the corresponding Tarskian consequence operator. 
	We now define an appropriate \( (\MC{L},\Omega,\models) \) as follows: 
	Let \( \Omega = \{ Cn(T) \mid T \subseteq \MC{L}\}\).
	Define the models relation \( {\models} \subseteq \Omega  \times \MC{L}  \) such that some $Cn(T) \in \Omega$ is a model of some $\varphi \in \MC{L}$ whenever \( \varphi \in Cn(T) \). 
	
Then we obtain for all $\K \subseteq \MC{L}$ and $\varphi \in \MC{L}$ the following:
$$
\begin{array}{rll}
\K \models \varphi 
& \Longleftrightarrow & \Mod{\K} \subseteq \Mod{\varphi} \\ 
& \Longleftrightarrow & \bigcap_{\kappa \in \K}\Mod{\kappa} \subseteq \Mod{\varphi} \\
& \Longleftrightarrow & \{Cn(T) \mid T \subseteq \MC{L},\ \K\subseteq Cn(T)\} \subseteq \{Cn(T) \mid T \subseteq \MC{L},\ \varphi\in Cn(T)\} \\
& \Longleftrightarrow & \forall T \subseteq \MC{L}: \K \subseteq Cn(T) \Rightarrow \varphi \in Cn(T) \hspace{5.7cm} (*)\\ 
\end{array}
$$
Moreover, we obtain
$$
\begin{array}{rlll}
(*) & \Longrightarrow & \K \subseteq Cn(\K) \Rightarrow \varphi \in Cn(\K) \hspace{3cm}\ ~ & \text{instantiate } T = \K\ \ \ \ \\ 
& \Longrightarrow & \varphi \in Cn(\K) & \text{extensivity of } Cn\\ 
& \Longrightarrow & \K \VDash \varphi, \\ 
\end{array}
$$	
and on the other hand:
$$
\begin{array}{rlll}
\K \VDash \varphi & \Longrightarrow & \varphi \in Cn(\K) \hspace{1cm}~ & \\ 
& \Longrightarrow & \forall S\subseteq \MC{L}: Cn(\K) \subseteq S \Rightarrow \varphi \in S & \\ 
& \Longrightarrow & \forall T\subseteq \MC{L}: Cn(\K) \subseteq Cn(T) \Rightarrow \varphi \in Cn(T) & \text{restriction to closed sets}\\ 
& \Longrightarrow & \forall T\subseteq \MC{L}: Cn(\K) \subseteq Cn(Cn(T)) \Rightarrow \varphi \in Cn(T) & \text{idempotency of } Cn\\
& \Longrightarrow & (*) & \text{monotonicity of } Cn\\
\end{array}
$$
Concluding, we have established that for all $\K \subseteq \MC{L}$ and $\varphi \in \MC{L}$ the following holds:
\begin{equation*}
    \K \models \varphi \Longleftrightarrow (*) \Longleftrightarrow \K \VDash \varphi.   \qedhere
\end{equation*}
\end{proof}
 
\clearpage
\section{Proof for Example \ref{ex:infsigfail}}\label{sec:app2}
We take up again \Cref{ex:infsigfail} and show that \( \releqK^{\cupwedge} \) is a faithful preorder assignment that is compatible with \( \circ^{\cupwedge} \). This shows that \Cref{thm:km1991} by K\&M does not straightforwardly generalize to \( \mathbb{PL}_\infty  \), i.e., propositional logic with countably infinite many atoms.
\begin{proposition}\label{prop:appendix_plinfty}
    The relation \( \releqK^{\cupwedge} \) is a faithful preorder assignment and is compatible
    with the \basechange\ operator \( \circ^{\cupwedge} \) for \( \mathbb{PL}_\infty  \), yet \( \circ^{\cupwedge} \) does not satisfy \postulate{G3}.
\end{proposition}
	\begin{proof}
	
    We show that \( \releqK^{\cupwedge} \) is a preorder assignment.
    
    \medskip
		\noindent\emph{(Totality)} For totality, assume the contrary, i.e. there are two interpretations $\I_1, \I_2$ with $\I_1\not\releqK^{\cupwedge}\I_2$ and $\I_2\not\releqK^{\cupwedge}\I_1$.
		From the definition of $\releqK^{\cupwedge}$, we have $\I_1, \I_2 \not\models \K$ where both $\I_1$ and $\I_2$ are finite with $|\I_1^\mathbf{true}| \not\geq |\I_2^\mathbf{true}|$ and $|\I_2^\mathbf{true}| \not\geq |\I_1^\mathbf{true}|$. Since $\geq$ is total over integers, this is a contradiction. 
		Reflexivity follows from totality.

		\medskip
        \noindent\emph{(Transitivity)} For transitivity, suppose $\I_1\releqK^{\cupwedge}\I_2$ and $\I_2\releqK^{\cupwedge}\I_3$. We make a case distinction by  $\I_1\releqK^{\cupwedge}\I_2$ and the definition of \( \releqK^{\cupwedge} \):
		\begin{itemize}\setlength{\itemsep}{0pt}
            \item[]\hspace{-4.4ex}\emph{(1)} The case of $\I_1\models\K$. Then $\I_1\releqK^{\cupwedge}\I_3$ follows immediately.
	 \item[]\hspace{-4.4ex}\emph{(2)} The case of $\I_2\not\models\K$ and $\I_2^\mathbf{true}$ is infinite. 
			As $\I_2\releqK^{\cupwedge}\I_3$, we consider three subcases:
			\begin{itemize}
				\item[] (2.1) $\I_2\models\K$. This contradicts the prior assumption, and hence this case is not possible.
				\item[]  (2.2) $\I_3\not\models\K$ with infinite $\I_3^\mathbf{true}$. Then $\I_1\releqK^{\cupwedge}\I_3$ follows.
				\item[] (2.3) $\I_2^\mathbf{true}$ and $\I_3^\mathbf{true}$ are finite. This is also impossible due to immediate contradiction.
			\end{itemize}
			 \item[]\hspace{-4.4ex}\emph{(3)} The case of  $\I_1,\I_2\not\models\K$, both $\I_1^\mathbf{true}$ and $\I_2^\mathbf{true}$ are finite and $|\I_1^\mathbf{true}| \geq |\I_2^\mathbf{true}|$.
			From $\I_2\releqK^{\cupwedge}\I_3$ we consider three subcases:
			\begin{itemize}
				\item[] (3.1) $\I_2\models\K$. This is not possible, immediate contradiction.
				\item[] (3.2) $\I_3\not\models\K$ with infinite $\I_3^\mathbf{true}$. This implies $\I_1\releqK^{\cupwedge}\I_3$.
				\item[] (3.3) $\I_2, \I_3 \not\models\K$, both $\I_2^\mathbf{true}$ and $\I_3^\mathbf{true}$ are finite with $|\I_2^\mathbf{true}| \geq |\I_3^\mathbf{true}|$.
				Since $|\I_1^\mathbf{true}| \geq |\I_2^\mathbf{true}|$ and $|\I_2^\mathbf{true}| \geq |\I_3^\mathbf{true}|$, from transitivity of $\geq$ over integers, we have $|\I_1^\mathbf{true}| \geq |\I_3^\mathbf{true}|$ and finally $\I_1\releqK^{\cupwedge}\I_3$.
			\end{itemize}
		\end{itemize}
		
        We show that \( \releqK^{\cupwedge} \) is faithful and that \( \releq{\abst}^{\cupwedge} \) is compatible with \( \circ^{\cupwedge} \).
        
        \medskip
        \noindent\emph{(Faithfulness)}
		The first condition of faithfulness, the Condition \postulate{F1}, follows from the assumption $\I_1, \I_2 \models \K$ and case (1) of the definition of \( \releqK^{\cupwedge} \), given in \Cref{ex:infsigfail}.
		
		For \postulate{F2}, let $\I_1\models\K$ and $\I_2\not\models\K$.
		From the case (1) of the definition, $\I_1\releqK^{\cupwedge}\I_2$ holds. 
		Now for contradiction assume that $\I_2\releqK^{\cupwedge}\I_1$.
		Following the definition of \( \releqK^{\cupwedge} \), we consider three cases.
		(Case 1) $\I_2\models\K$ contradicts our assumption. The (Case 2) and (Case 3) are not applicable because they require $\I_1\not\models\K$. Hence, $\I_2\not\releqK^{\cupwedge}\I_1$ and therefore $\I_1\relK^{\cupwedge}\I_2$ holds.
		
		For \postulate{F3}, assume $\K\equiv\K'$ (i.e. $\Mod{\K} = \Mod{\K'}$) and let $\I_1\releqK^{\cupwedge}\I_2$.
		We consider three cases.
		(Case 1) $\I_1\models\K$. Then it also holds $\I_1\models\K'$, and hence $\I_1\releq{\K'}\I_2$. 
		(Case 2) $\I_2\not\models\K$ and $\I_2^\mathbf{true}$ is infinite. Then $\I_2\not\models\K'$ and hence $\I_1\releq{\K'}\I_2$.
		(Case 3) where $\I_1, \I_2\not\models\K$ we also have $\I_1,\I_2\not\models\K'$ and consequently $\I_1\releq{\K'}\I_2$. Therefore, we have $\releqK^{\cupwedge}=\releq{\K'}$ (i.e. $\I_1\releqK^{\cupwedge}\I_2$ if and only if $\I_1\releq{\K'}\I_2$).
		
        \medskip
        \noindent\emph{(Compatibility with \( \circ \))}
		For the compatibility with $\circ^{\cupwedge}$, we show that $\Mod{\K\circ^{\cupwedge}\Gamma} = \min(\Mod{\Gamma}, \releqK^{\cupwedge})$.
		For any inconsistent $\Gamma$, we have $\Mod{\K\circ^{\cupwedge}\Gamma} = \emptyset = \min(\Mod{\Gamma}, \releqK^{\cupwedge})$.
        If \( \K\cup\G \) is consistent, then we have \( \Mod{\K\circ\G}=\Mod{\K\cup \G} \).
            Because \( \releqK^{\cupwedge} \) is faithful, we directly obtain \( \Mod{\K\circ\G} = \min(\Mod{\Gamma}, \releqK^{\cupwedge})\). Thus, for the remaining steps of the proof, we assume that \( \K\cup\G \) is inconsistent and $\Gamma$ is consistent.
            
            We show in the following that $\min(\Mod{\Gamma}, \releqK^{\cupwedge}) = \emptyset$ holds by an indirect argument. Thus suppose %
            there existed some $\I_1\in\min(\Mod{\Gamma}, \releqK^{\cupwedge})$.
            This means, that
            $\I_1\in\Mod{\G}$ and there is no other $\I_2\in\Mod{\G}$ such that $\I_2\relK^{\cupwedge}\I_1$.
			Note that from the definition of $\circ^{\cupwedge}$ and our case assumption, we have 
			$ \Mod{\K\circ^{\cupwedge}\Gamma} = \Mod{\K\cup\G} = \emptyset $, and hence  $\I_1,\I_2\not\models\K$.
			Let $\Sigma_\G\subseteq\Sigma$ be the set of atomic symbols occurring in $\G$.
            Clearly, because \( \G \) contains just one (finite) sentence, we have that \( \Sigma_\G \) must be finite.
			We now consider two cases: $\I_1^\mathbf{true}$ can be finite or infinite.
            
             \begin{itemize}\setlength{\itemsep}{0pt}
                \item[]\hspace{-4.4ex}($\I_1^\mathbf{true}$ is finite)~~Then,  there exists an atomic symbol $q$ such that $q\in\Sigma\setminus(\I_1^\mathbf{true}\cup\Sigma_\G)$ (as both $\I_1^\mathbf{true}$ and $\Sigma_\G$ are finite and $\Sigma$ is infinite).
                Then we could define another interpretation $\I_2$ such that $\I_2(q) = \textbf{true}$ and $\I_2(p_i)=\I_1(p_i)$ for all $p_i\in\Sigma\setminus \{q\}$.
                Since $q$ does not occur in $\G$, we have $\I_2\in\Mod{\G}$ and $|\I_2^\mathbf{true}| = |\I_1^\mathbf{true}| + 1$. Hence, $\I_2\relK^{\cupwedge}\I_1$, a contradiction  to the minimality of \( \I_1 \).
                \item[]\hspace{-4.4ex}($\I_1^\mathbf{true}$ is infinite)~~We define another interpretation $\I_2$ such that for all $p_i\in\Sigma$ we set $\I_2(p_i) = \textbf{true}$ if $p_i \in (\Sigma_\G\cap\I_1^\mathbf{true})$ and $\I_2(p_i) = \textbf{false}$ otherwise.
                As $\I_1$ and $\I_2$ coincide on $\Sigma_\G$, we obtain $\I_2\in\Mod{\G}$.
                Since $\I_2^\mathbf{true}$ is finite while $\I_1^\mathbf{true}$ is infinite, we have $\I_2\relK^{\cupwedge}\I_1$, which again is a contradiction to the minimality of \( \I_1 \). \qedhere
            \end{itemize}

	\end{proof}

\clearpage
\section{Proof for Theorem \ref{thm:when_tranisitive}}\label{sec:app_loops_new}

In this section, we present the full proof for \Cref{thm:when_tranisitive} from \Cref{sec:logics_tpo_representable}. We start by repeating the theorem.

\medskip
\noindent \textsc{\Cref{thm:when_tranisitive}.}
	\textit{For all base logics $\mathbb{B}$, the following statements hold:
	\begin{itemize}
		\item[\rm{(I)}] If $\mathbb{B}$ loop-free, then every {\basechange} operator for $\mathbb{B}$ that satisfies \postulate{G1}--\postulate{G3}, \postulate{G5}, and \postulate{G6} is total-preorder-representable. 
		\item[\rm{(II)}] If $\mathbb{B}$ is not loop-free, then there exists a {\basechange} operator for $\mathbb{B}$ that satisfies \postulate{G1}--\postulate{G6} and is not total-preorder-representable.
	\end{itemize}}

\medskip
We dedicate the following Section \ref{sec:logic_capture_hidden_cycles:CL2TOP} to the statement (I) of \Cref{thm:when_tranisitive} while the statement (II) of \Cref{thm:when_tranisitive} is shown in Section \ref{sec:logic_capture_hidden_cycles:TPO2CL}.

\subsection{Total-Preorder-Representability Implies Absence of Critical Loops}\label{sec:logic_capture_hidden_cycles:CL2TOP}
\newcommand{\B}{\mathcal{C}}
\newcommand{\A}{\mathcal{A}}
\newcommand{\critloop}{\mathfrak{C}}

We show (by contraposition) that the absence of critical loops is necessary for total-preorder-representability of all \sauerwald{base change operators that satisfy \postulate{G1}--\postulate{G3}, \postulate{G5}, and \postulate{G6}}. To this end, we will provide a construction which, given a critical loop $\critloop$ in some base logic $\mathbb{B}$, yields a \sauerwald{base change operator $\circ_\critloop$ for $\mathbb{B}$ that is demonstrably not total-preorder-representable, yet satisfies \postulate{G1}--\postulate{G6}}. 

\begin{definition}\label{def:circ_critloop}
	Let $\mathbb{B}  = (\MC{L},{\Omega},\models, \Bases, \Cup)$ be a base logic with  a {\criticalloop} $\critloop = (\G_{0,1},\G_{1,2},\allowbreak\ldots,\G_{n,0})$ and let $ \G_0,\ldots,\G_{n}\in \Bases$ and $ \K $ 
	as in \Cref{def:critical_loop}.
	
	Let $ \B $ denote the set of all $ \coverbase' $ guaranteed by Condition~(3) from \Cref{def:critical_loop}, i.e. $ \coverbase'\in \B $ if there is some $ \coverbase $ with \( \emptyset \neq \Mod{\coverbase'} \subseteq \Mod{\coverbase} \setminus \left( \Mod{\G_{0,1}}\cup\ldots\cup\Mod{\G_{n,0}} \right) \) and $\coverbase$ is consistent with three (or more) bases from \( \{ \G_{0}, \ldots, \G_{n} \} \).
	Now let $ \B'=\{ \coverbase'\in \B \mid \Mod{\coverbase' \Cup \K}=\emptyset \} $, i.e., all belief bases from $ \B $ that are inconsistent with $ \K $.
	Let $ \leqslant_{\B'}$ be an arbitrary linear order on $ \B' $ with respect to which every non-empty subset of $ \B' $ has a minimum.\footnote{Such a $ \leqslant_{\B'}$ exists due to the well-ordering theorem, by courtesy of the \emph{axiom of choice} \cite{KS_Zermelo1904}.} 
	
	We now define $\circ_\critloop$ as follows: for every $\K' \not\equiv \K$ and any $\G$, let $ \K'\circ_\critloop\G=\K'\Cup\G $ if $ \K'\Cup\G $ is consistent, otherwise  $ \K'\circ_\critloop\G=\G $. 
	For $\K'\equiv\K$, we define:
	\begin{equation*}
		\K'\circ_\critloop\G = \begin{cases}
			\G  \Cup   \K'        & \text{if } \Mod{\K' \Cup \G} \neq \emptyset\text{, }                                                                                                  \\
			\G \Cup \G_{\min}^{\B'} & \text{if } \Mod{\K' \Cup \G} = \emptyset, \text{ and } \Mod{\G \Cup \coverbase'} \neq \emptyset \text{ for some } \coverbase'\in \B',                                           \\
			\G \Cup \G_i           & \text{if none of the above applies, } \Mod{\G_{i} \Cup \G} \neq \emptyset, \text{ and } \displaystyle \hspace{-10ex} \bigcup_{\qquad\qquad j\in\{0,\ldots,n\} \setminus \{i,i\oplus 1\}}\hspace{-10ex}\Mod{\G_{j} \Cup \G} = \emptyset, \\[-2.5ex]
			\G                      & \text{if none of the cases above apply,}
		\end{cases}
	\end{equation*}
	where $ \G_{\min}^{\B'}=\min(\{\coverbase' \in \B' \mid \Mod{\coverbase' \Cup \G} \neq \emptyset\}, \leqslant_{\B'}) $.
\end{definition}

In the following, we show that $\circ_\critloop$ from \Cref{def:circ_critloop} is indeed a \sauerwald{base change operators that satisfies \postulate{G1}--\postulate{G6}, but is} not total-preorder-representable.

\begin{proposition}\label{prop:when_tranisitiveIF}
	For a base logic $\mathbb{B}$ with a critical loop $\critloop$, the operator $\circ_\critloop$ for $\mathbb{B}$ satisfies \postulate{G1}--\postulate{G6} and is not total-preorder-representable.
\end{proposition}

\begin{proof}%
	
	We will first show that $\circ_\critloop$ satisfies \postulate{G1}--\postulate{G6}.
	For $ \K'\not\equiv\K $ we obtain a trivial revision  which satisfies \postulate{G1}--\postulate{G6} (cf. \Cref{ex:full-meet}).
	Consider the remaining case of $ \K $ (and any equivalent base):
	
	\emph{Postulates \postulate{G1}--\postulate{G4}}. The satisfaction of \postulate{G1}--\postulate{G3} follows directly from the construction of $ \circ_\critloop $.
	For \postulate{G4} observe that, when computing $ \K\circ_\critloop\G $, the case distinction above only considers the  model sets of the participating bases rather than their syntax. Thus, for \( \K\equiv\K' \) and $ \G_1^*\equiv\G_2^* $ we always obtain $ \K\circ_\critloop\G_1^*\equiv\K'\circ_\critloop\G_2^* $.
	
	\emph{Postulate \postulate{G5} and \postulate{G6}}. 
	Consider two belief bases $ \G^*_1 $ and $ \G^*_2 $.
	If $ \G^*_2 $ is inconsistent  with $ \K\circ_\critloop\G^*_1 $, then we obtain satisfaction of \postulate{G5} immediately.
	For the remaining case of \postulate{G5} and for \postulate{G6} we assume  $ \K\circ_\critloop\G^*_1 $ to be consistent with $ \G^*_2 $, i.e., $\Mod{(\K\circ_\critloop\G^*_1 )\Cup\G^*_2}\neq\emptyset$.
	Consequently, there exists some interpretation $\I$ such that $\I\in\Mod{\K\circ_\critloop\G^*_1}$ and $\I\in\Mod{\G^*_2}$.
	The postulate \postulate{G1} implies that $\I\in\Mod{\G^*_1}$ and hence $ \G^*_1\Cup \G^*_2 $ is consistent.
	We now inspect all different cases from the definition of \( \circ_\critloop \) above that may apply when revising \( \K \) by \( \G^*_1 \):
	\begin{itemize}\setlength{\itemsep}{0pt}
		\item[]If $ \G^*_1 $ is consistent with $ \K $, then we obtain from $\Mod{(\K\circ_\critloop\G^*_1)\Cup\G^*_2}\neq\emptyset$ and \postulate{G2} that $\K$ is consistent with $ \G^*_1 \Cup \G^*_2$.
		This implies $(\K\circ_\critloop\G^*_1)\Cup \G^*_2\equiv (\K \Cup \G^*_1) \Cup \G^*_2 \equiv \K \Cup (\G^*_1 \Cup \G^*_2)  \equiv \K\circ_\critloop(\G^*_1\Cup\G^*_2) $; yielding satisfaction of \postulate{G5} and \postulate{G6}.
		\item[]Next, consider the second case of the definition, where $ \G^*_1 $ is inconsistent with $\K$, but consistent with some $ \coverbase'\in \B' $ and assume $\coverbase'$ is the $ \leqslant_{\B'}$-minimal such base, i.e., $\coverbase' = (\G^*_1)_{\min}^{\B'}$. Then,
		from the construction of $\circ_\critloop$ and the consistency of $(\K\circ_\critloop\G^*_1)\Cup\G^*_2$ we obtain $\Mod{(\K\circ_\critloop\G^*_1)\Cup\G^*_2} = \Mod{\G^*_1\Cup\coverbase'\Cup\G^*_2}\neq\emptyset$.
		Consequently,
		the set $ \G^*_1\Cup\G^*_2 $ is also consistent with $ \coverbase'$, which, together with $ \coverbase'= (\G^*_1)_{\min}^{\B'}$, implies $\coverbase'= (\G^*_1 \Cup \G^*_2)_{\min}^{\B'}$.
		For determining $\K \circ_\critloop (\G^*_1 \Cup \G^*_2)$, note that from $\K$ being inconsistent with $\G^*_1$, it follows that $\K$ must also be inconsistent with $\G^*_1 \Cup \G^*_2$, therefore, due to the existence of $\coverbase'$, the second line of the definition of $\circ_\critloop$ must apply.
		We obtain  
		$ (\K\circ_\critloop\G^*_1)\Cup \G^*_2 \equiv ((\G^*_1)_{\min}^{\B'} \Cup \G^*_1) \Cup \G^*_2 \equiv \coverbase' \Cup \G^*_1 \Cup \G^*_2 \equiv (\G^*_1 \Cup \G^*_2)_{\min}^{\B'} \Cup (\G^*_1\Cup\G^*_2)  \equiv \K\circ_\critloop(\G^*_1\Cup\G^*_2) $%
		; establishing \postulate{G5} and \postulate{G6} for this case.
		\item[]We now inspect the third case from the definition, i.e., we consider some $ \G^*_1 $ that is inconsistent with $\K$ and with all elements from $\B'$.
		If $ \G^*_1 $ is consistent with $ \G_i $ and inconsistent with all $ \G_{j} $, where $j\in\{0,\ldots,n\} \setminus \{ i, i\oplus 1 \}$,
		then by the construction of $\circ_\critloop$ and the consistency of $(\K\circ_\critloop\G^*_1)\Cup\G^*_2$ we have $\Mod{(\K\circ_\critloop\G^*_1)\Cup\G^*_2} = \Mod{\G^*_1\Cup\G_i\Cup\G^*_2}\neq\emptyset$.
		Then, likewise $ \G^*_1\Cup \G^*_2 $ is consistent with $ \G_i $ and inconsistent with all $ \G_{j}$ with $j\in\{0,\ldots,n\} \setminus \{ i, i\oplus 1 \}$. Moreover, if $ \G^*_1 $ is inconsistent with $\K$ and with all elements from $\B'$, then so is $ \G^*_1 \Cup\G^*_2$, i.e., when determining $ \K\circ_\critloop(\G^*_1\Cup\G^*_2)$, the third case of the definition applies.
		Hence, by the definition of \( \circ_\critloop \) we obtain $(\K\circ_\critloop\G^*_1)\Cup \G^*_2 \equiv \G^*_1\Cup\G^*_2\Cup \G_i  \equiv \K\circ_\critloop(\G^*_1\Cup\G^*_2) $. 
		\item[]
		If none of the conditions above applies to $ \G^*_1 $, then they also do not apply to $ \G^*_1\Cup  \G^*_2 $.
		From the construction of $ \circ_\critloop $ we obtain $ \K\circ_\critloop(\G^*_1\Cup\G^*_2)\equiv(\K\circ_\critloop\G^*_1)\Cup \G^*_2 $ $\equiv \G^*_1\Cup\G^*_2$.
	\end{itemize}
	In summary, we obtain that \( \circ_\critloop \) satisfies \postulate{G5} and \postulate{G6} in all cases.
	It remains to show that $\circ_\critloop$ is not total-preorder-representable. 
	Towards a contradiction suppose the contrary, i.e., there is a min-complete faithful preorder assignment $ \releq{\abst} $, such that $\circ_\critloop$ is compatible with~$\releq{\abst}$.
	Transitivity and min-completeness imply that \( \releq{\abst} \) is min-friendly.
	As all $\G_0,\ldots,\G_n$ are consistent,
	there are $ \I_i \in \Mod{\G_{i}} $  for all $i\in\{0,\ldots,n\}$.
	By construction of $\circ_\critloop$ and Condition (2) of \Cref{def:critical_loop},
	we have $\K\circ_\critloop\G_{i,i\oplus 1}=\G_{i,i\oplus 1}\Cup\G_i\equiv\G_i$, and consequently
	$ \I_i\models\K\circ_\critloop\G_{i,i\oplus 1} $ and $ \I_{i\oplus 1}\not\models\K\circ_\critloop\G_{i,i\oplus 1}  $ for each $ i\in\{0,\ldots,n\} $. %
	As $ \circ_\critloop $ is compatible with $ \releq{\abst} $, we obtain $ \Mod{\K\circ_\critloop\G_{i,i\oplus 1} }=\min(\Mod{\G_{i,i\oplus 1} },\releqK) $.
	In particular, the definition of $ \circ_\critloop $ yields $ \I_i\in {\min(\Mod{\G_{i,i \oplus 1}},\releqK)} $ %
	and $ \I_i,\I_{i \oplus 1}\models \G_{i,i\oplus 1} $ and $ \I_{i \oplus 1}\notin {\min(\Mod{\G_{i,{i \oplus 1}}},\releqK)} $.
	We obtain thereof the strict relationship $ \I_i \relK \I_{i \oplus 1} $. %
	In summary, we get $ \I_0 \relK \I_1 \relK  \ldots \relK \I_n \relK \I_0 $, which contradicts the presumed transitivity of \( \releqK \).
\end{proof}

This establishes that the absence of critical loops is a necessary condition for universal total-preorder-representability in any Tarskian logic, because \Cref{thm:when_tranisitive} (II) an immediate consequence of \Cref{prop:when_tranisitiveIF}.

\subsection{Absence of Critical Loops Implies Total-Preorder-Representability}\label{sec:logic_capture_hidden_cycles:TPO2CL}

We will now show that the identified criterion of {\criticalloop} (\Cref{def:critical_loop}) is also sufficient, even in the more general, syntax-dependent setting. That is, we will demonstrate in the following that \Cref{thm:when_tranisitive} (I) holds.
To this end, we need to argue that any base change operator $\circ$ that satisfies \postulate{G1}--\postulate{G3}, \postulate{G5}, and \postulate{G6} for any critical-loop-free $\mathbb{B}$ gives rise to a compatible min-complete quasi-faithful preorder assignment.
We will use the assignment $\rreleqc{\abst}$ introduced in \Cref{sec:transformation2tpo}, where each \( \rreleqcK \) is obtained by a stepwise transformation of $\releqcK $ from \Cref{def:relation_new}.
The general outline of the proof is similar to the proof of \Cref{prop:encodedown} but will require much more careful argumentation for some of the steps. Because of this, we present here the proof in much detail.

Recall that the transformation from $\releqc{(.)}$ to $\rreleqc{(.)}$ consists of three steps.  
For the start, recall that $\releqc{\abst}$ is a min-complete quasi-faithful assignment compatible with $\circ$ by \Cref{prop:agm_withoutg4}. 
This means that \( \releqcK \) is a total relation for each \( \K \), whence transitivity is the only condition that \( \releqcK \) fails to meet to qualify as a total preorder.

For the first step, we will remove  all detached pairs $\setAllDetached$ from \( \releqcK \), resulting in \( \relationstepdashK \). 
The relation \( \relationstepdashK \) will be a non-transitive and non-total relation, but minima of models of bases will be preserved. 
We will then extend \( \relationstepdashK \) to a transitive relation \( \relationstepdashdashK \) in the second step, by taking the transitive closure. 
We will show that only elements from $\setAllDetached$ can be added back by the transitive closure, which guarantees that, again, minima of models of bases are preserved.
In a last step, we obtain the final result $\rreleqc{(.)}$ by ``linearizing`` \( \relationstepdashK \) to a total preorder in a way that minima of models of bases are again preserved.

\paragraph{\textbf{Step I: Removing detached pairs.}}
Let \( \circ \)  be a {\basechange} operator that satisfies \postulate{G1}--\postulate{G3}, \postulate{G5}, and \postulate{G6}. Then, for any two bases \( \K,\G\in\Bases \), all quasi-faithful assignments \( \releq{\abst} \) compatible with \( \circ \) yield  the same set of minimal interpretations of \( \Mod{\G} \) with respect to \( \releq{\K} \).
This property already stipulates much of \( \releq{\K} \) for each \( \K \) (for some base logics \( \releq{\K} \) is even complete determined by that property).
Still, in the general case, when forming a compatible assignment, there is certain freedom on relating those interpretations for which the given \basechange\ operator gives no hint about how to order them.
Such pairs are describes by the notion of detached pairs (see \Cref{def:detached}).

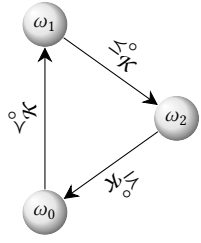
\begin{figure}
\centering
\begin{tikzpicture}
\usetikzlibrary{arrows}
\pgfarrowsdeclarecombine{twotriang}{twotriang}{triangle 90}{triangle 90}{triangle 90}{triangle 90}
\usetikzlibrary{arrows}

\node [inner sep=0em,balls, scale=0.75] (u) at (0,0) {\Large$\I_{0}$};
\node [inner sep=0em,balls, scale=0.75] (l) at (0,2.5) {\Large$\I_{1}$};
\node [inner sep=0em,balls, scale=0.75] (l+1) at (1.75,1.25) {\Large$\I_{2}$};

\draw (u) edge [-Stealthnew] node[above,sloped, scale=1] {$\relcK $} (l);
\draw (l) edge [-Stealthnew] node[above,sloped, scale=1] {$\releqcK $} (l+1);

\draw (u) edge [,Stealthnew-] node[below,sloped, scale=1] {\rotatebox{180}{$\releqcK $}} (l+1) ;

\end{tikzpicture}
\caption{Illustration of a \criticalloop-situation of length 3 on the semantic side.
If \( \mathbb{B} \) does not exhibit a critical loop, then this situation is, due to \Cref{lem:loop_detach}, only possible when ($\I_1$,$\I_2$) or ($\I_2$,$\I_0$) is a detached pair.}\label{fig:circle_sem_3} 
\end{figure}

In the following, we show that every violation of transitivity in $ \releqcK $ involves a detached pair (as illustrated in \Cref{fig:circle_sem_3}).

\begin{lemma}\label{lem:loop_detach}
\sauerwald{Let $\mathbb{B}$ be a loop-free base logic and let  $ \circ $ be} a \basechange\ operator for \( \mathbb{B} \) which satisfies \postulate{G1}--\postulate{G3}, \postulate{G5}, and \postulate{G6}.
If $ \I_0 \releqcK \I_1 $ and $ \I_1 \releqcK \I_2 $ with $ \I_0 \not\releqcK \I_2 $, then $ (\I_0,\I_1) $ or $ (\I_1,\I_2) $ is detached from $ \circ $ in $ \K $.
\end{lemma}
\begin{proof}
	Let \( \I_0,\I_1,\I_2 \) such that a violation of transitivity is obtained as given above, i.e. $ \I_0 \releqcK \I_1 $ and $ \I_1 \releqcK \I_2 $ with $ \I_0 \not\releqcK \I_2 $.
	By \Cref{def:relation_new}, we have that $ \I_0 \not\releqcK \I_2 $ is only possible if \( \I_0\not\models\K \).
	From \Cref{def:relation_new} and $ \I_0 \releqcK \I_1 $, we obtain $\I_1\not\models\K$. By an analogue argument we obtain $\I_2\not\models\K$.
	Thus, for the rest of the proof we have \( \I_0,\I_1,\I_2 \not\models\K \).
	
	Towards a contradiction, assume that $ (\I_0,\I_1) $ and $ (\I_1,\I_2) $ are both not detached from $ \circ $ in $ \K $.	
	By \Cref{lem:totality} the relation $ \releqcK $ is total, and thus we have that $ \I_2 \relcK \I_0 $.	
	As $\I_2 \not\models\K$ and $ \I_0 \not\releqcK \I_2 $, due to \Cref{lem:help}(a), there is a base $ \G_{2,0} \in\Bases $ with $ \I_0,\I_2\models\G_{2,0}  $ such that $ \I_2 \models\K\circ\G_{2,0}  $ and $ \I_0\not\models\K\circ\G_{2,0}  $.
	By \( \I_0,\I_1,\I_2 \not\models\K \) and \Cref{def:relation_new} we obtain $ \I_0 \sqreleqcK \I_1 $ and $ \I_1 \sqreleqcK \I_2 $ (cf. \Cref{def:relation_new_first_relation}).
	Because $ (\I_0,\I_1) $ is not detached, there is some \( \G_{0,1} \in\Bases \) with \( \I_0,\I_1 \models \G_{0,1} \) such that \( \I_0 \models \K\circ\G_{0,1} \) or \( \I_1 \models \K\circ\G_{0,1} \).
	By \Cref{def:relation_new_first_relation} and $ \I_0 \sqreleqcK \I_1 $ we obtain that \( \I_0 \models \K\circ\G_{0,1} \).
	Using an analogue argumentation, there exist $ \G_{1,2}\in\Bases$ satisfying $ \I_1,\I_2\models\G_{1,2} $ and $ \I_1 \models\K\circ\G_{1,2}$.
	
	Recall that \( \releqc{\abst} \) is compatible, min-retractive and quasi-faithful by \Cref{lem:compatibility} and by the proof of \Cref{lem:quasi-faithfulness}.
	Let \( \G_{i} = (\K\circ \G_{i,i \oplus 1}) \Cup \G_{i\oplus 2,i}   \) for each \( i\in\{0,1,2\} \).
	Note that each \( \G_{i} \) is a consistent base, since we have \( \I_i \in \Mod{\G_{i}} \).
	We now show that Conditions (1) and Condition (2) 
	from \Cref{def:critical_loop} are satisfied:
	\begin{itemize}\setlength{\itemsep}{0pt}
		\item[]\hspace{-4.4ex}(1)~~Towards a contradiction, assume that $ \K $ is consistent with some $ \G_{i,i\oplus 1} $.
		From \postulate{G2} we obtain $\Mod{\K\circ\G_{i,i\oplus 1}} = \Mod{\K\Cup\G_{i,i\oplus 1}}$ for some \( i\in\{0,1,2\} \).
		Since \( \I_i \in \Mod{\G_{i}} \), by the definition of $\G_{i}$ we have $\I_i\in \Mod{(\K\circ \G_{i,i \oplus 1}) \Cup \G_{i\oplus 2,i} } = \Mod{(\K\Cup \G_{i,i \oplus 1}) \Cup \G_{i\oplus 2,i} }$ and obtain $\I_i\in\Mod{\K}$ for some \( i\in\{0,1,2\} \), which contradicts \( \I_0,\I_1,\I_2 \not\models\K \).
		
		\item[]\hspace{-4.4ex}(2)~~By the postulate  \postulate{G1} we have \( \Mod{\K\circ \G_{i,i \oplus 1}} \subseteq \Mod{\G_{i,i\oplus 1}} \) for each \( i\in\{0,1,2\} \).
		The definition of \( \G_i \) yields \( \Mod{\G_i} \subseteq \Mod{\G_{i,i\oplus 1}\Cup\G_{i\oplus 2,i}} \) for each \( i\in\{0,1,2\} \).
		Substituting \( i \) by \( i\oplus 1 \) yields \(  \Mod{\G_{i\oplus 1}} \subseteq \Mod{\G_{i\oplus 1,i\oplus 2}\Cup\G_{i,i\oplus 1}} \); showing that \( \Mod{\G_{i}} \cup \Mod{\G_{i \oplus 1}} \subseteq \Mod{\G_{i,i\oplus 1}} \) holds for each \( i\in\{0,1,2\} \).

		We show that each $ \G_i\Cup \G_j $ is inconsistent, by assuming the contrary, i.e., there are some \( i,j\in\{0,1,2\} \) such that \( i \neq j \) and  $ \G_i\Cup \G_j $ is consistent, 
		i.e. there exists some $\I^*\in \Mod{\G_i}\cap\Mod{\G_j}$.
		From the definition of $\G_i$ and the definition of $\G_j$, we obtain $\I^*\in \Mod{\K\circ\G_{i,i\oplus 1}} \cap \Mod{\G_{i\oplus 2,i}} \cap \Mod{\K\circ\G_{j,j\oplus 1}} \cap  \Mod{\G_{j\oplus 2,j}}$.
		Hence, we obtain \( \I^*\in \Mod{ \G_{i\oplus 2,i}\Cup  \G_{j\oplus 2,j}} \) and from compatibility of \( \releqc{\abst} \) with \( \circ \), we obtain \( \I^* \in \min(\Mod{\G_{i,i\oplus 1}},\releqcK) \) and \( \I^* \in \min(\Mod{\G_{j,j\oplus 1}},\releqcK) \).
		Now observe that \( \I_{0},\I_{1},\I_{2} \in \Mod{\G_{i,i\oplus 1}} \cup \Mod{\G_{j,j\oplus 1}} \) holds; this is because we have $\Mod{\G_k}\subseteq\Mod{\G_{k,k\oplus 1}}\cup\Mod{\G_{k\oplus n, k}}$ for each $k\in\{0,1,2\}$.
		Hence, independent of the specific \( i \) and \( j \), we obtain  \( \I^* \releqcK \I_{k} \) from $\I^*\in \Mod{\K\circ\G_{i,i\oplus 1}}$ and \Cref{lem:help}(b) for each \( k\in\{0,1,2\} \). 
		Together, $\I_i\in\Mod{\K\circ\G_{i,i\oplus 1}}$, $\I_j\in\Mod{\K\circ\G_{j,j\oplus 1}}$,  and  compatibility, imply \( \I_{i} \releqcK \I^* \) and \( \I_{j} \releqcK \I^* \).
		Because of  \( \Mod{\G_{i}} \cup \Mod{\G_{i \oplus 1}} \subseteq \Mod{\G_{i,i\oplus 1}} \), we have that $\I_i,\I_j,\I^*\in\Mod{\G_{{i},{i\oplus1}}}$ or $\I_i,\I_j,\I^*\in\Mod{\G_{{j},{j\oplus1}}}$ holds. 
		For the case $\I_i,\I_j,\I^*\in\Mod{\G_{{i},{i\oplus1}}}$, since $\I_j\releqcK\I^*$ and $\I^*\in\min(\Mod{\G_{i,i\oplus 1}},\releqcK)$, from min-retractivity we obtain $\I_j\in\min(\Mod{\G_{i,i\oplus 1}},\releqcK)$. As $\I_i\in\min(\Mod{\G_{i,i\oplus 1}},\releqcK)$, we obtain \( \I_i\releqcK \I_j \) and \( \I_j\releqcK \I_i \). 
		By an analogue argumentation, we obtain for the case of $\I_i,\I_j,\I^*\in\Mod{\G_{{j},{j\oplus1}}}$ the same conclusion, i.e., \( \I_i\releqcK \I_j \) and \( \I_j\releqcK \I_i \).
		This shows that \( \I_i\releqcK \I_j \) and \( \I_j\releqcK \I_i \) must hold in general.
		
		We consider in the following all possible choices for \( i \) and \( j \).
		For the case of \( i=0 \) and \( j=2 \), we obtain a contradiction to     \( \I_2 \relcK \I_0 \).
		We next consider the case of \( i=1 \) and \( j=2 \).
		Because of \( \Mod{\G_0} = \Mod{\K\circ\G_{0,1}} \cap \Mod{\G_{2,0}} = {\min(\Mod{\G_{0,1}},\releqcK)} \cap \Mod{\G_{2,0}} \), we have that \( \I_{0},\I_{2},\I^* \in \Mod{\G_{2,0}} \) holds.
		As $\I_0\releqcK\I^*$ and $\I^*\in\min(\Mod{\G_{2,0}},\releqcK)$ holds, min-retractivity of \( \releqcK \) yields $\I_0\in\min(\Mod{\G_{2,0}},\releqcK)$. 
		Consequently, we obtain that $\I_0\releqcK\I_2$ holds, which is a contradiction to \( \I_2 \relcK \I_0 \).
		The proof for the case of \( i=2 \) and \( j=1 \) is analogous to the case of \( i=1 \) and \( j=2 \).
		We obtain that Condition (2) from \Cref{def:critical_loop} is satisfied.
	\end{itemize}

	Recall that by assumption, the base logic \( \mathbb{B} \) does not exhibit a {\criticalloop}.
	Yet $ \G_{0,1},\allowbreak\G_{1,2},\G_{2,0}  $ satisfy Con\-ditions (1) and Condition (2) of a {\criticalloop}, hence Condition (3) of \Cref{def:critical_loop} must be violated.
	This means that there exists some $ \coverbase\in\Bases $ such that 
	$ \Mod{\G_i\Cup\coverbase}\neq \emptyset $ 
	for every $ i\in\{0,1,2\} $, but no required base \( \coverbase'\in\Bases \) such that Condition (3) is satisfied.
	Consequently, for all $ \G\in\Bases $ holds
	\begin{equation}\tag{\( \star 1\)}\label{eq:loop_detach:e0}
		\Mod{\G}\neq\emptyset \text{ implies }  \Mod{\G} \not\subseteq \Mod{\coverbase}\setminus (\Mod{\G_{0,1}}\cup\Mod{\G_{1,2}}\cup\Mod{\G_{2,0} }).
	\end{equation}
	For the remaining parts of the proof,  let \(  \omegafromCoverbase{i} \in \Omega  \) be an interpretation with $ \omegafromCoverbase{i} \in \Mod{\G_i}\cap\Mod{\coverbase} $ for each \( i\in \{ 0,1,2\} \).
	Because \( \circ \) satisfies \postulate{G1} and \postulate{G3}, we obtain $ \Mod{\K\circ\coverbase} \subseteq \Mod{\coverbase} $ and consistency of $ \K\circ\coverbase $.
	Together with \eqref{eq:loop_detach:e0} we obtain that there exists $ k\in\{0,1,2\} $ with $ \Mod{\K\circ\coverbase} \cap \Mod{\G_{k,k\oplus 1}} \neq \emptyset $.
	We consider each of the two cases \( \Mod{\K\circ\coverbase} \cap \Mod{\K\circ\G_{k,k\oplus 1}} \neq \emptyset \) and \( \Mod{\K\circ\coverbase} \cap \Mod{\K\circ\G_{k,k\oplus 1}} = \emptyset \) independently.
	
	\medskip
	\noindent
	\emph{The case of \( \Mod{\K\circ\coverbase} \cap \Mod{\K\circ\G_{k,k\oplus 1}} \neq \emptyset \).} 
	As first step, we show that %
	\begin{align*}
		\omegafromCoverbase{0} \releqcK \omegafromCoverbase{2}   \text{ and }   \omegafromCoverbase{2} \releqcK \omegafromCoverbase{1}   \text{ and }  \omegafromCoverbase{1} \releqcK \omegafromCoverbase{0}  \tag{\( \star2 \)}\label{eq:loop_detach:e1new}
	\end{align*}
	holds  for this case.
	Clearly,  \( \Mod{\K\circ\coverbase} \cap \Mod{\K\circ\G_{k,k\oplus 1}} \neq \emptyset \) implies that there exists some \( \omegaMinFromCoverbase\in\Omega \) such that $\omegaMinFromCoverbase\in\Mod{\K\circ\coverbase}$ and $\omegaMinFromCoverbase\in\Mod{\K\circ\G_{k,k\oplus 1}}$.
	From the compatibility of \( \circ \) with \( \releqc{\abst} \), we obtain $\omegafromCoverbase{k}\in\min(\Mod{\G_{{k},{k\oplus1}}},\releqcK)$, implying that $\omegafromCoverbase{k}\releqcK\omegaMinFromCoverbase$ holds.
	Remember that $\omegafromCoverbase{k},\omegaMinFromCoverbase \in \Mod{\coverbase}$ and $\omegaMinFromCoverbase\in\min(\Mod{\coverbase},\releqcK)$, by min-retractivity we obtain $\omegafromCoverbase{k}\in {\min(\Mod{\coverbase},\releqcK)} $.
	From this  last observation and from \( \omegafromCoverbase{k\oplus 1}, \omegafromCoverbase{k\oplus 2} \in \Mod{\coverbase} \) we obtain that  \( \omegafromCoverbase{k} \releqcK \omegafromCoverbase{k\oplus 1} \)   and   \( \omegafromCoverbase{k} \releqcK \omegafromCoverbase{k\oplus 2} \)
	holds.
	Remember that by Condition (2) we have \( \omegafromCoverbase{k},\omegafromCoverbase{k\oplus 2} \in \Mod{\G_{k\oplus 2,k}} \) and by compatibility we obtain  \( \omegafromCoverbase{k\oplus 2} \in {\min(\Mod{\G_{k\oplus 2,k}},\releqcK)} \).
	This last observation, together with \( \omegafromCoverbase{k} \releqcK \omegafromCoverbase{k\oplus 2} \), \( \I_{k} \models \K\circ\G_{} \) and min-retractivity, yields \( \omegafromCoverbase{k} \in {\min(\Mod{\G_{k\oplus 2,k}},\releqcK)} \).
	Thus, we have \( \omegafromCoverbase{k\oplus 2}  \releqcK  \omegafromCoverbase{k} \).
	By a symmetric argument, we have \( \omegafromCoverbase{k\oplus 1},\omegafromCoverbase{k\oplus 2} \in \Mod{\G_{{k\oplus 1},{k\oplus 2}}} \) and \( \omegafromCoverbase{k\oplus 1} \in \Mod{\K \circ \G_{{k\oplus 1},{k\oplus 2}}} \). Thus, we obtain \( \omegafromCoverbase{k\oplus 1} \releqcK \omegafromCoverbase{k\oplus 2} \) from \Cref{lem:help}(b). 
	By combination of these observations with \( \omegafromCoverbase{k\oplus 1}, \omegafromCoverbase{k\oplus 2} \in \Mod{\coverbase} \) and $\omegafromCoverbase{k}\in\min(\Mod{\coverbase},\releqcK)$, we obtain \( \omegafromCoverbase{k}, \omegafromCoverbase{k\oplus 1}, \omegafromCoverbase{k\oplus 2} \in\min(\Mod{\coverbase},\releqcK) \) from min-retractivity.
	As direct consequence, we obtain that \eqref{eq:loop_detach:e1new} holds.
	
	We will now show that a contradiction with \( \I_2 \relcK \I_0 \) is unavoidable.
	Recall that \( \I_{0},\I_{2} \in \Mod{\G_{2,0}}  \) and \( \I_{2} \models \K\circ\G_{2,0} \), but \( \I_0 \not \models \K\circ\G_{2,0} \).
	The last observation together with the compatibility of \( \releqc{\abst} \) with \( \circ \) implies that \( \omegafromCoverbase{2} \in \min(\Mod{\G_{2,0}},\releqcK) \) holds.
	Because \eqref{eq:loop_detach:e1new} holds, we obtain \( \omegafromCoverbase{0} \in \min(\Mod{\G_{2,0}},\releqcK) \) from min-retractivity of \( \releqc{\abst} \).
	Similarly, we obtain \( \I_{0},\omegafromCoverbase{0} \in {\min(\Mod{\G_{0,1}},\releqcK)} \) from compatibility and \( \I_{0},\omegafromCoverbase{0} \in \Mod{\K\circ\G_{0,1}}  \); showing that \( \I_{0} \releqcK \omegafromCoverbase{0} \) holds.
	Because of \( \omegafromCoverbase{0},\I_{0},\I_{2} \in \Mod{\G_{2,0}} \), we obtain  \( \I_{0} \in \min(\Mod{\G_{2,0}},\releqcK) \) from \( \omegafromCoverbase{0} \in \min(\Mod{\G_{2,0}},\releqcK) \) and min-retractivity,  and consequently, we obtain the contradiction \( \I_0 \releqcK \I_2 \).
	
	\medskip
	\noindent
	\emph{The case of \( \Mod{\K\circ\coverbase} \cap \Mod{\K\circ\G_{k,k\oplus 1}} = \emptyset \).}
	Using $ \Mod{\K\circ\coverbase} \cap \Mod{\G_{k,k\oplus 1}} \neq \emptyset $ yields that there exist some \( \I^* \in \Mod{\K\circ\coverbase} \cap \Mod{\G_{k,k\oplus 1}}\).
	From \Cref{lem:help}(c) and \(  \I^*,\omegafromCoverbase{k} \in \Mod{\G_{k,k\oplus 1}} \) and \( \omegafromCoverbase{k}\in \Mod{\K\circ\G_{k,k\oplus 1}} \) and \( \I^* \notin \Mod{\K\circ\G_{k,k\oplus 1}} \) we obtain \( \omegafromCoverbase{k} \relcK \I^* \).
	Because \postulate{G1} is satisfied by \( \circ \), we have that \( \I^* \in \Mod{\K\circ\coverbase} \) implies \( \I^* \in \Mod{\coverbase} \).
	We obtain the contradiction \( \I^* \releqcK \omegafromCoverbase{k}   \) from \(  \I^*,\omegafromCoverbase{k} \in \Mod{\coverbase} \) and \( \I^*\in \Mod{\K\circ\coverbase} \) by using \Cref{lem:help}(b).

	\medskip
	In summary, this shows that Conditions (1)--(3) from \Cref{def:critical_loop} are satisfied, i.e., \( \G_{0,1},\G_{1,2},\G_{2,0}  \) form a {\criticalloop}. 
	This contradiction the assumption that \( \mathbb{B} \) does not exhibit a {\criticalloop} and consequently, $ (\I_0,\I_1) $ or $ (\I_1,\I_2) $ is detached from $ \circ $ in $ \K $.
\end{proof}

\Cref{lem:loop_detach} provides the rationale for the first transformation step:
For every $\K \in \Bases$, we obtain ${\relationstepdashK}$ by removing all non-reflexive detached pairs from $\releqcK$, that is, 
$ {\relationstepdashK} = {\releqcK} \setminus \setAllDetached $.    
The resulting ${\relationstepdashK}$ is not guaranteed to be total anymore, and it is not necessarily transitive. But we will show that ${\relationstepdashK}$ inherits other important properties from  $\releqcK$.

\begin{lemma}\label{lem:gettransitive_step1}
Let $\mathbb{B}=(\MC{L},{\Omega},\models,\Bases,\Cup)$ \sauerwald{be a loop-free base logic}, let \( \circ \) be a {\basechange} operator satisfying \postulate{G1}--\postulate{G3}, \postulate{G5}, and \postulate{G6} and let \( \releqcK \) be a quasi-faithful min-friendly assignment compatible with \( \circ \).
For each \( \K,\G\in\Bases \) holds ${\min(\Mod{\G},\relationstepdashK)}={\min(\Mod{\G},\releqcK)}$ and $\relationstepdashK$ is min-complete and reflexive.
\end{lemma}	    
\begin{proof}
By definition of \( {\relationstepdashK} \) we have \( \I \relationstepdashK \I' \) if and only if \( \I \releqcK  \I' \) for all \( (\I,\I') \in \Omega\times\Omega \) which are not detached pairs.
Because for every \( \I,\I' \in \Mod{\G} \) with \( \I \in \min(\Mod{\G},\releqcK) \) we have \( \I \models \K\circ\G \) by compatibility of \( \releqc{\abst} \) with \( \circ \). Consequently, the pair \( (\I,\I') \) is not detached and thus $\min(\Mod{\G},\relationstepdashK)=\min(\Mod{\G},\releqcK)$.
The latter implies that min-completeness of \( \releqcK \) carries over to \( \relationstepdashK\).
Reflexivity of \( \relationstepdashK\) is obtained by construction, the reflexivity of \( \releqcK \) and by the definition of \( \setAllDetached \). 
\end{proof}

\paragraph{\textbf{Step II: Taking the transitive closure.}}    
In this step, for every $\K \in \Bases$, we obtain $\relationstepdashdashK$ by taking the transitive closure of $\relationstepdashK$, i.e., we have \( {\relationstepdashdashK}=TC({\relationstepdashK})=TC({{\releqcK}\setminus\setAllDetached}) \).
The resulting $\relationstepdashdashK$ is still not guaranteed to be total, but it is reflexive and transitive by construction, and it inherits further important properties from  $\relationstepdashK$.
It will turn out that the transitive closure will only add pairs to \( {\relationstepdashK} \) that are detached pairs. This means that \( {\relationstepdashdashK} \) contains only elements from \( {\relationstepdashK} \) and from \( \setAllDetached \). Because adding detached pairs does not influence minimal sets of models of a base \( \G \) with respect to \( \releqK \), we will obtain that these sets are preserved when taking the transitive closure.

\begin{figure}
	\centering
	\begin{tikzpicture}
		
		\node [inner sep=0.1em,balls, scale=0.75] (u) at (0,0) {\Large$\I_{0}$};
		\node [inner sep=0.1em,balls, scale=0.75] (l) at (0,1.5) {\Large$\I_{1}$};
		\node [inner sep=0.1em,balls, scale=0.75] (l+1) at (1,2.5) {\Large$\I_{2}$};
		\node [inner sep=0.1em,balls, scale=0.75] (l+2) at (2.5,2.5) {\Large$\I_{3}$};
		
		\node [inner sep=0.1em,balls, scale=0.75] (w) at (3.75,0.75) {\Large$\I_{i}$};
		\node [inner sep=0.1em,balls, scale=0.75] (w') at (3.25,-0.5) {\Large$\!\I_{\!i^{\!{\ \atop+}\!\!}1\,}\!\!$};
		
		\node [inner sep=0.1em,balls, scale=0.75] (u-1) at (1,-1) {\Large$\I_{n}$};
		
		\draw (u) edge [-Stealthnew] node[above,sloped, scale=1] {$\relcK $} (l);
		\draw (l) edge [-Stealthnew] node[above,sloped, scale=1] {$\releqcK $} (l+1);
		\draw (l+1) edge [-Stealthnew] node[above,sloped, scale=1] {$\releqcK $} (l+2);
		
		\draw (u) edge [double,dotted] (l+1) ;
		\draw (u) edge [double,dotted] (l+2) ;
		\draw (u) edge [double,dotted] (w) ;
		\draw (u) edge [double,dotted] (w') ;
		
		\draw (l) edge [double,dotted] (u-1) ;
		\draw (l) edge [double,dotted] (l+2) ;
		\draw (l) edge [double,dotted] (w) ;
		\draw (l) edge [double,dotted] (w') ;

		\draw (l+1) edge [double,dotted] (u-1) ;
		\draw (l+1) edge [double,dotted] (w) ;
		\draw (l+1) edge [double,dotted] (w') ;
		
		\draw (l+2) edge [double,dotted] (u-1) ;
		\draw (l+2) edge [double,dotted] (w') ;
		
		\draw (w) edge [double,dotted] (u-1) ;
		
		\draw (l+2) edge [dashed,-Stealthnew,bend left,in=155,out=30] (w);
		\draw (w) edge [-Stealthnew]  node[below,sloped, scale=1] {\rotatebox{180}{$\releqcK $}} (w');
		\draw (w') edge [dashed,-Stealthnew,bend left,in=160,out=20] (u-1);
		
		\draw (u-1) edge [-Stealthnew] node[below,sloped, scale=1] {\rotatebox{180}{$\releqcK $}} (u);

		\begin{scope}[xshift=5.5cm, yshift=-1.25cm]
			
			\draw [->] (0.25,0.75) to    (0.75,0.75) node [right] {$\releqcK$ (not detached)};
			\draw [double,dotted, thick] (0.25,0.25) -- (0.75,0.25) node[right] {detached};
		\end{scope}    
	\end{tikzpicture}
	\caption{Illustration of a \criticalloop-situation of length \( n \) on the semantic side. This situation is due to \Cref{lem:loop_type_two} impossible for \( \releqcK \) if \( \mathbb{B} \) does not exhibits a \criticalloop. 
		If \( \mathbb{B} \) does not exhibit a critical loop, then this situation is due to \Cref{lem:loop_type_two} only possible when there is some \( i\in\{1,\ldots,n\} \) such that \( (\I_i,\I_{i\oplus1}) \) is a detached pair.}
	\label{fig:circle_sem_n}
\end{figure}
If the transitive closure would (hypothetically) add non-detached pairs to \( {\relationstepdashK} \), then the relation \( {\releqcK} \) would contain a cycle of interpretations consisting only of non-detached pairs (such as the cycle illustrated in \Cref{fig:circle_sem_n}).
	To make such situations more easy to handle, we introduce the following notion which make implicit use of \( \releqc{\abst} \), defined in \Cref{def:relation_new}.
	\begin{definition}\label{def:strictcircle}
		Let $\mathbb{B}=(\MC{L},{\Omega},\models,\Bases,\Cup)$ a base logic, let \( \K \in \Bases  \) be a base, and let \( \circ \) be a \basechange\ operator for \( \mathbb{B} \) that satisfies \postulate{G1}--\postulate{G3}, \postulate{G5}, and \postulate{G6}. 
		A sequence of interpretations \( {\circlearrowright} = \I_0, \ldots ,\I_n,\I_0 \) from  \( \Omega \) is said to form a \emph{strict cycle of length \( n+1 \) (with respect to \( \circ \) and \( \K \))} if 
		\begin{itemize}
			\item \( \I_0, \ldots ,\I_n \) are satisfying the conditions
\begin{itemize}\setlength{\itemsep}{0pt}
    \item[]\hspace{-4.4ex}(a)~~\( \I_0 \relcK \I_1 \),
    \item[]\hspace{-4.4ex}(b)~~\(  \I_i \releqcK \I_{i\oplus 1} \) for all \( i\in \{ 1,\ldots,n \} \), where \( \oplus \) is addition \( {\mathrm{mod} (n+1)} \), and
\end{itemize}
			\item \( (\I_{i},\I_{i+ 1}) \) is not a detached pair for each \( i\in \{0,\ldots,n\} \), where \( \oplus \) is addition \( {\mathrm{mod} (n+1)} \).
		\end{itemize}
	\end{definition}
	We will also substitute elements in a strict cycle \( \circlearrowright \) and use therefore the following notion.
	For a substitution \( \sigma=\{ \I_{i_1} \mapsto x_1,\ \I_{i_2} \mapsto x_2,\ \ldots \} \), we denote by \( {\circlearrowright}[\sigma] \) the simultaneously replacement of \( \I_{i_j} \) by \( x_j \) in \( {\circlearrowright} \) for all \(  {\I_{i_j} \mapsto  x_j} \in \sigma  \).
	
	\begin{figure}
		\begin{center}
			\begin{tikzpicture}
				\usetikzlibrary{arrows.meta,decorations.markings}
				\draw[ dashed,
				decoration={markings, mark=at position 0.105 with {\arrow{latex[reversed]}}},
				decoration={markings, mark=at position 0.355 with {\arrow{latex[reversed]}}},
				decoration={markings, mark=at position 0.575 with {\arrow{latex[reversed]}}},
				decoration={markings, mark=at position 0.77 with {\arrow{latex[reversed]}}},
				postaction={decorate}
				] (0,0) circle [ radius=1.75] ;
				\node [balls, scale=0.75] (w*)  {\Large$\omega^*$};

				\node [balls, scale=0.75] (wL)  at ++(270:1.75cm) {\Large$\omega_{\lambda}$};
				\node [balls, scale=0.75] (wT)  at ++(90:1.75cm) {\Large$\omega_{b}$};
				\node [balls, scale=0.75] (wP)  at ++(0:1.75cm) {\Large$\omega_{c}$};
				\node [balls, scale=0.75] (wQ)  at ++(180:1.75cm) {\Large$\omega_{a}$};
				
				\draw (w*) edge [-Stealthnew] (wT) ; 
				\draw (w*) edge [Stealthnew-] (wL) ; 
				\draw (w*) edge [-Stealthnew] (wP) ; 
				\draw (w*) edge [-Stealthnew] (wQ) ; 
			\end{tikzpicture}
		\end{center}
		\caption{Exemplary situation of \Cref{lem:loop_comb}. Four interpretations lying on a strict cycle, connected by another interpretation \( \I^* \).}\label{fig:loopcomb}
	\end{figure}
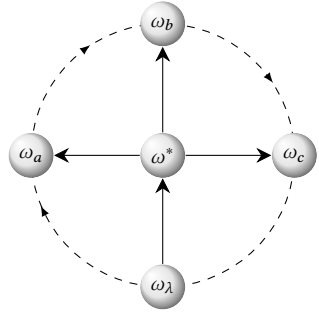
	
	The following lemma will be useful, and describes situations like in \Cref{fig:loopcomb}.
	\begin{lemma}[cross lemma]\label{lem:loop_comb}
		Let $\mathbb{B}=(\MC{L},{\Omega},\models,\Bases,\Cup)$ \sauerwald{be a loop-free base logic}, let \( \K \in \Bases  \) be a base, and let \( \circ \) be a \basechange\ operator for \( \mathbb{B} \) that satisfies \postulate{G1}--\postulate{G3}, \postulate{G5}, and \postulate{G6}.
		If there are \( \I_0, \ldots ,\I_n \in \Omega  \), with \( n> 3 \), and pairwise distinct \( \lambda,a,b,c \in \{ 0,\ldots,n \} \), such that
		\begin{itemize}\setlength{\itemsep}{0pt}
			\item[]\hspace{-4.4ex}(a)~~\( \I_{0},\I_{1},\ldots,\I_{n},\I_{0} \) is a strict cycle of length \( n+1 \),
			\item[]\hspace{-4.4ex}(b)~~there exists an interpretation \( \I^* \) such that
			\begin{align*}
				\I^* & \releqcK  \I_{a} &  \I^*      & \releqcK  \I_{b} &
				\I^* & \releqcK  \I_{c} & \I_{\lambda} & \releqcK   \I^*    \text{ , and}   %
			\end{align*}
			\item[]\hspace{-4.4ex}(c)~~every pair of \( \releqcK \) considered in (b) is not detached from \( \circ \) in \( \K \),
		\end{itemize}
		then there is a strict cycle of length \( m \) with \( 3\leq m \leq n \).
	\end{lemma}
	\begin{proof}
		We assume \( a<b<c \), and we assume that the path \( \I_c,\ldots,\I_\lambda \) does not contain \( \I_a \) and \( \I_b \) (when seeing \( \releqcK \) as graph).
		All other cases will follow by symmetry. We continue by consider several cases:

		\smallskip
		\noindent
		\emph{The case of \(   \I_{\lambda} \relcK \I^* \).} We obtain \( \I_{\lambda} \relcK \I^*  \releqcK  \I_{c} \releqcK \ldots \releqcK \I_{\lambda}  \), which yields that \( {\circlearrowright_{\lambda c}} =  \I_{\lambda},\I^*,\I_{c},\ldots,\I_\lambda \) is a  strict cycle.  Note that because \( {\circlearrowright_{\lambda c}} \) contains \( \I^* \)  and in addition only elements of \( \{\I_0, \ldots ,\I_n\} \setminus \{\I_{a} , \I_{b}\} \), we have that \( {\circlearrowright_{\lambda c}} \) has a length of at most  \( n\). 
		
		\smallskip
		\noindent
		\emph{The case of \(   \I^* \relcK \I_{c}  \) and no prior case applies.}
		If \(   \I^* \relcK \I_{c}  \), then we obtain \( \I^* \relcK \I_{c}   \releqcK \ldots \releqcK \I_{\lambda}   \releqcK  \I^* \), yielding that \( {\circlearrowright_{c\lambda}} =  \I^*,\I_{c},\ldots,\I_\lambda, \I^* \) is a  strict cycle. 
		Note that because \( {\circlearrowright_{c\lambda}} \) contains \( \I^* \)  and in addition only elements of \( \{\I_0, \ldots ,\I_n\} \setminus \{\I_{a} , \I_{b}\} \), we have that \( {\circlearrowright_{c\lambda}} \) has a length of at most  \( n\). 
		
		\smallskip
		\noindent
		\emph{The case of \(   \I^* \relcK \I_{b}  \) and no prior case applies.}
		In this case we have \( \I_{c} \releqcK \I^* \). We obtain \( \I^* \relcK \I_{b}   \releqcK \ldots \releqcK \I_{c}   \releqcK  \I^* \), which yields that \( {\circlearrowright_{bc}} =  \I^*,\I_{b},\ldots,\I_{c}, \I^* \) is a  strict cycle. 
		Note that because \( {\circlearrowright_{bc}} \) contains, beside of \( \I^* \), only elements of \( \{\I_0, \ldots ,\I_n\} \setminus \{\I_{a} , \I_\lambda\} \), we have that \( {\circlearrowright_{bc}} \) has a length of at most  \( n\). 
		
		\smallskip
		\noindent
		\emph{The case of \(   \I^* \relcK \I_{a}  \) and no prior case applies.}
		In this case we have \( \I_{b} \releqcK \I^* \). We obtain \( \I^* \relcK \I_{a}   \releqcK \ldots \releqcK \I_{b}   \releqcK  \I^* \), which yields that \( {\circlearrowright_{ab}} =  \I^*,\I_{a},\ldots,\I_{b}, \I^* \) is a  strict cycle. 
		Note that because \( {\circlearrowright_{ab}} \) contains, beside of \( \I^* \), only elements of \( \{\I_0, \ldots ,\I_n\} \setminus \{\I_{c} , \I_\lambda\} \), we have that \( {\circlearrowright_{ab}} \) has a length of at most  \( n\). 
		
		\smallskip
		If none of the cases above applies, then we have that \(   \I^* \releqcK \I_{\lambda} \) and \(  \I_{a} \releqcK  \I^* \) and \( \I_{b} \releqcK \I^* \) and \( \I_{c} \releqcK \I^* \) holds. 
		For the following line of arguments, recall that \( a<b<c \) holds.
		We consider the case of \( 0<\lambda<a \); for all other cases (where \( 0<\lambda<a \) does not hold), the line of arguments is symmetric to the proof we present here  in the following for the case of \( 0<\lambda<a \).
		Because \( \I_{0}\allowbreak,\I_{1},\allowbreak\ldots,\allowbreak\I_{n},\I_{0} \) is a strict cycle of length \( n+1 \), we obtain that \( \I_{0} \relcK \I_{1} \releqcK \ldots \releqcK  \I_{\lambda} \releqcK \I^* \releqcK \I_{c}\releqcK \ldots  \releqcK \I_{0} \).
		This show that \( {\circlearrowright_{0\lambda c}} = \I_{0},\I_{1} , \ldots ,  \I_{\lambda},\I^*, \I_{c},\ldots,\I_{0} \) is a strict cycle.
		Because \( {\circlearrowright_{0\lambda c}} \) contains \( \I^* \) and additionally only elements from \( \I_0,\ldots,\I_n \), but not \( \I_{a} \) and \( \I_{b} \), we obtain that \( {\circlearrowright_{0\lambda c}} \) has a length of at most \( n \).
		
		In summary, we obtain a strict cycle of length \( m \) with \( 3\leq m \leq n \) for each case.
	\end{proof}	
	Note that it is not necessary to assume that \( \I^* \) is distinct from \( \I_{0},\ldots,\I_n \) in \Cref{lem:loop_comb}.

The following \Cref{lem:loop_type_two} is a central insight of this step, which shows that for \sauerwald{loop-free base logics} no strict cycles exist in \( {\releqcK} \).
\begin{lemma}\label{lem:loop_type_two}
Let $\mathbb{B}=(\MC{L},{\Omega},\models,\Bases,\Cup)$ \sauerwald{be a loop-free base logic}, let \( \K \in \Bases  \) be a base, and let \( \circ \) be a \basechange\ operator for \( \mathbb{B} \) that satisfies \postulate{G1}--\postulate{G3}, \postulate{G5}, and \postulate{G6}.
If there are three or more interpretations \( \I_0, \ldots ,\I_n \in \Omega  \), i.e. \( n \geq 2 \), such that
\begin{itemize}\setlength{\itemsep}{0pt}
\item[]\hspace{-4.4ex}(a)~~\( \I_0 \relcK \I_1 \),
\item[]\hspace{-4.4ex}(b)~~\(  \I_i \releqcK \I_{i\oplus 1} \) for all \( i\in \{ 1,\ldots,n \} \), where \( \oplus \) is addition \( {\mathrm{mod} (n+1)} \), 
\end{itemize}
then there is some \( i\in\{1,\ldots,n\} \) such that \( (\I_i,\I_{i\oplus1}) \) is a detached pair.
\end{lemma}
\begin{proof}
	Let \( \I_0, \ldots ,\I_n \in \Omega  \) such that Condition (a) and Condition (b) of \Cref{lem:loop_type_two} are satisfied.
	With \( \oplus \) we  denote addition \( {\mathrm{mod} (n+1)} \).
	The proof will be by induction. Note that for \( n=2 \) we obtain the result by \Cref{lem:loop_detach}.
	We proceed the proof for the case of \( n>2 \) and assume that \Cref{lem:loop_type_two} already holds for all \( m \) with \( 2 \leq m <n \).
	A consequence of the induction hypothesis is that there is no strict cycle of length \( c \) for \( 3\leq c\leq n \).
	
	We are striving for a contradiction. 
	Therefore, we assume \( {\circlearrowright_{0n}} = \I_0, \ldots ,\I_n ,\I_0 \) is a strict cycle of length \( n+1 \), which is, due to Condition (a) and Condition (b) from \Cref{lem:loop_type_two}, equivalent to assume that \( (\I_i,\I_{i\oplus1}) \) is not a detached pair for each \( i\in\{1,\ldots,n\} \). 
	The remaining parts of the proof show that the existence of the strict cycle \( {\circlearrowright_{0n}} \) implies existence of a critical loop.
	
	As first step, we show that \( \I_0, \ldots ,\I_n \notin \Mod{\K} \) holds.
	If \( \I_{1} \in \Mod{\K} \), then due \Cref{def:relation_new}, we obtain \( \I_{1} \releqcK \I_{0} \), which contradicts Condition (a). 
	If \( \I_{i} \in \Mod{\K} \) for some \( i\in\{0,2,3,\ldots,n \} \), then, because of Condition (b), there is some \( j \) with \( \I_{j} \releqcK \I_{j\oplus 1} \) and \( \I_{j}\notin\Mod{\K}  \) and \( \I_{j\oplus 1}\in\Mod{\K}  \); which is again impossible due to \Cref{def:relation_new}.
	Thus, we have \( \I_{0},\ldots,\I_{n} \notin \Mod{\K} \) for the remaining parts of the proof.
	
	We continue by showing the existence of several bases, which will form a critical loop.
	\Cref{def:relation_new}  and \Cref{def:detached} together implies that for each \( i\in\{1,\ldots,n\} \) exists a base \( \G_{{i},{i\oplus1}} \in \Bases \) such that
	\begin{equation}\label{eq:loop_type_two_e2}
		\I_{i},\I_{i\oplus1} \models \G_{{i},{i\oplus1}} \text{ and } \I_{i} \models \K \circ \G_{{i},{i\oplus1}} \tag{\#1}
	\end{equation}
	holds.
	Moreover, by \( \I_0 \relcK \I_1 \) from Condition (a) and \( \I_1 \not\models\K \) and \Cref{lem:help}(a), there exists a base \( \G_{0,1} \in \Bases \) such that the following holds:
	\begin{equation}\label{eq:loop_type_two_e1}
		\I_{0},\I_{1}  \models \G_{{0},{1}} \text{ and } \I_0  \models \K \circ \G_{{0},{1}} \text{ and } \I_1  \not\models \K \circ \G_{{0},{1}} . \tag{\#2}
	\end{equation}    
	
	We show that \( \G_{0,1},\G_{1,2},\ldots,\G_{n,0} \) is forming a \criticalloop. 
	To this end we are setting \( \G_{i}=(\K\circ\G_{i,i\oplus 1}) \Cup \G_{i\oplus n,i} \) for each \( i\in\{0,\ldots,n\} \). 
	By \eqref{eq:loop_type_two_e2} and \eqref{eq:loop_type_two_e1} each \( \G_{i} \) is a consistent base with \( \I_i \in \Mod{\G_{i}} \).
	We continue by verifying that Conditions (1)--(3) from \Cref{def:critical_loop} are satisfied.
	\begin{itemize}\setlength{\itemsep}{0pt}
		\item[]\hspace{-4.4ex}(1)~~If \( \K \) is inconsistent, then Condition (1) is immediately satisfied. 
		We consider the case where \( \K \) is consistent and \( \Mod{\K\Cup\G_{i,i\oplus1}}\neq\emptyset \) for some \( i\in \{ 0,\ldots,n \}\). 
		From \postulate{G2} we obtain \( \Mod{\K\circ\G_{i,i\oplus1}} = \Mod{\K\Cup\G_{i,i\oplus1}} \). 
		From \( \I_i \in \Mod{\G_{i}} \) and the definition of $\G_{i}$, we obtain $\I_i\in\Mod{\K\circ\G_{i,i\oplus1}}\cap\Mod{\G_{i\oplus n,i}}$. 
		As \( \Mod{\K\circ\G_{i,i\oplus1}} = \Mod{\K\Cup\G_{i,i\oplus1}} \), we obtain $\I_i\in\Mod{\K\Cup\G_{i,i\oplus1}}\cap\Mod{\G_{i\oplus n,i}}$.
		Consequently, we there exists some \( i\in \{ 0,\ldots,n \}\) such that $\I_i\in\Mod{\K}$, yielding a contradiction to \( \I_{0},\ldots,\I_{n} \notin \Mod{\K} \).
		
		\item[]\hspace{-4.4ex}(2)~~By the postulate  \postulate{G1} we have \( \Mod{\K\circ \G_{i,i \oplus 1}} \subseteq \Mod{\G_{i,i\oplus 1}} \) for each \( i\in\{0,\ldots,n\} \).
		The definition of \( \G_i \) yields \( \Mod{\G_i} \subseteq \Mod{\G_{i,i\oplus 1}\Cup\G_{i\oplus n,i}} \subseteq \Mod{\G_{i,i\oplus 1}} \) for each \( i\in\{0,1,2\} \).
		Substitution of \( i \) by \( i\oplus 1 \) yields \(  \Mod{\G_{i\oplus 1}} \subseteq \Mod{\G_{i\oplus 1,i\oplus 2}\Cup\G_{i,i\oplus 1}} \subseteq \Mod{\G_{i,i\oplus 1}} \); showing that \( \Mod{\G_{i}} \cup \Mod{\G_{i \oplus 1}} \subseteq \Mod{\G_{i,i\oplus 1}} \) holds for each \( i\in\{0,\ldots,n\} \).

		We show that each $ \G_i\Cup \G_j $ is inconsistent, by assuming the contrary, i.e. there are some \( i,j\in\{0,\ldots,n\} \) such that \( i \neq j \) and  $ \G_i\Cup \G_j $ is consistent.
		Because of the commutativity of \( \Cup \), we assume \( i < j \) without loss of generality.    
		By compatibility and definition of \( \G_{i} \) and by definition of \( \G_{j} \), there exists some \( \I^*\in \Mod{ \G_{i\oplus n,i}\Cup  \G_{j\oplus n,j}} \) with \( \I^* \in \min(\Mod{\G_{i,i\oplus 1}},\releqcK) \) and \( \I^* \in \min(\Mod{\G_{j,j\oplus 1}},\releqcK) \).
		Recall that \( \I_{i},\I_{i \oplus 1} \in \Mod{\G_{i,i\oplus 1}} \) and  \( \I_{j},\I_{j \oplus 1}\in \Mod{\G_{j,j\oplus 1}} \).
		Consequently, for all \( k\in\{i,i \oplus 1,j,j\oplus 1\} \) holds  \( \I^* \releqcK \I_{k} \).
		Moreover, because of \eqref{eq:loop_type_two_e2} and \eqref{eq:loop_type_two_e1} we obtain \( \I_{i} \releqcK \I^* \) 
		and \( \I_{j} \releqcK \I^* \)  from \Cref{lem:help}(b).
		From \( \I^*\in \Mod{ \G_{i\oplus n,i}\Cup  \G_{j\oplus n,j}} \) we obtain, by an analogue argumentation, that \( \I_{i\oplus n} \releqcK \I^* \) and \( \I_{j\oplus n} \releqcK \I^* \) holds.
		In summary, we have:
		\begin{equation*}
			\begin{aligned}
				\I^*   & \releqcK \I_{i} & \I^*   & \releqcK \I_{j} & \I^* & \releqcK \I_{i \oplus 1} & \I_{i\oplus n} & \releqcK \I^* \\
				\I_{i} & \releqcK \I^*   & \I_{j} & \releqcK \I^*   & \I^* & \releqcK \I_{j \oplus 1} & \I_{j\oplus n} & \releqcK \I^*
			\end{aligned}
			\tag{\( \boxast 1 \)}\label{eq:loop_type_two_boxast_one}
		\end{equation*}    
		Note that all pairs \( (\I^* , \I_{\xi})\),  \( (\I^*, \I_{\xi\oplus 1} ) \) and \( (\I_{\xi\oplus n} , \I^*) \) with \( \xi \in \{ i,j \} \) are not detached.
		
		We are now striving for a contradiction by showing the existence of a strict cycle with a length  of at most \( n \). Recall that \( {\circlearrowright_{0n}} = \I_0, \ldots ,\I_n ,\I_0 \) is a strict cycle of length \( n+1 \).
		At first, we consider two particular cases:
		\begin{itemize}			
			\item[]\hspace{-4.4ex}(\( \I_i=\I_{j\oplus 1} \)) We obtain a strict cycle of length  of at most \( n \) from \Cref{lem:loop_comb}  by using \( {\circlearrowright_{0n}} \) and setting \( \lambda=i \), \( a=j \), \( b=i\oplus 1 \), and \( c= j\oplus n \).
			Note that \( \lambda,a,b,c \) are pairwise distinct indices.
			
			\item[]\hspace{-4.4ex}(\( \I_j=\I_{i\oplus 1} \)) We obtain a strict cycle of length  of at most \( n \) from \Cref{lem:loop_comb}  by using \( {\circlearrowright_{0n}} \) and setting \( \lambda=i \), \( a=j \), \( b=j\oplus 1 \), and \( c = i\oplus n \).
			Note that \( \lambda,a,b,c \) are pairwise distinct indices.
		\end{itemize}
		For all situations not covered by the cases above, we obtain that \( i,i\oplus 1,j,j\oplus 1 \) are pairwise distinct. Because of \eqref{eq:loop_type_two_boxast_one}, we can apply \Cref{lem:loop_comb}  by using \( {\circlearrowright_{0n}} \) and setting \( \lambda=i \), \( a=i \oplus 1 \), \( b=j \), and \(c =  j\oplus n \). This yields a strict cycle with a length  of at most \( n \).
		
		In summary, for every possible case we obtain a contradiction, which shows that Condition (2) of {\criticalloop}s (cf. \Cref{def:critical_loop}) is satisfied.
		
		\item[]\hspace{-4.4ex}(3)~~We show Condition (3) from \Cref{def:critical_loop} by contradiction.
		Therefore, assume there is a base \( \coverbase \in \Bases \) 
		such that for
		\begin{equation}\label{eq:loop_type_two_f1} \tag{\( \star 1 \)}
			B=\{ \G_{i} \mid \Mod{\coverbase\Cup \G} \neq \emptyset \}  \subseteq \{ \G_{0},\ldots,\G_{n} \} 
		\end{equation}  
		holds \( |B|\geq 3 \) and there exists no base \( \coverbase' \in \Bases \) as required by Condition (3).
		Consequently, for each base \(  \G\in \Bases  \) we have
		\begin{equation*}\label{eq:loop_type_two_f2}
			\Mod{\G} \neq \emptyset \text{ implies }  \Mod{\G} \not \subseteq \Mod{\coverbase} \setminus \left( \Mod{\G_{0,1}}\cup\ldots\cup\Mod{\G_{n,0}} \right) . \tag{\( \star2 \)}
		\end{equation*} 
		
		From \eqref{eq:loop_type_two_f1} we obtain that \( \coverbase \) is consistent, and thus, by \postulate{G3}, that \( \K\circ\coverbase \) is consistent and from satisfaction of \postulate{G1} we obtain \( \Mod{\K\circ\coverbase}\subseteq \Mod{\coverbase}  \). 
		Consequently, because of \eqref{eq:loop_type_two_f2}, we have \( \Mod{\K\circ\coverbase } \not\subseteq \Mod{\coverbase} \setminus \left( \Mod{\G_{0,1}}\cup\ldots\cup\Mod{\G_{n,0}} \right) \).
		This implies that \( \Mod{\K\circ\coverbase} \cap \Mod{\G_{k,k\oplus 1}} \neq \emptyset \) for some \( k\in \{ 0,\ldots, n \} \).  Let \( \omegaMinFromCoverbase \)  be an interpretation with \( \omegaMinFromCoverbase \in \Mod{\K\circ\coverbase} \cap \Mod{\G_{k,k\oplus 1}} \).
		From \postulate{G5} and \postulate{G6},   \( \omegaMinFromCoverbase \in \Mod{\K\circ\coverbase} \cap \Mod{\G_{k,k\oplus 1}} \) and commutativity of \( \Cup \) we obtain
		\begin{equation*}\label{eq:loop_type_two_f4}\tag{\( \star3 \)}
			\omegaMinFromCoverbase \in \min(\Mod{\coverbase},\releqcK)  \cap \Mod{\G_{k,k\oplus 1}} = \min(\Mod{\G_{k,k\oplus 1}\Cup \coverbase},\releqcK) \text{ and }  \I_{k} \releqcK  \omegaMinFromCoverbase, 
		\end{equation*}
		whereby the latter is a direct consequence of \( \omegaMinFromCoverbase\in\Mod{\G_{k,k\oplus 1}} \) and \( \I_{k} \in {\min(\Mod{\G_{k,k\oplus 1}},\releqcK)} \).
		Furthermore, let \( \omegafromCoverbase{i} \) be some interpretation with \( \omegafromCoverbase{i} \in \Mod{\coverbase\Cup \G_{i}} \) for each \( \G_{i}\in B \).
		We show as next for each \( \G_\xi \in B \) and for each \( \I_{\xi\oplus n}^* \in \Mod{\G_{\xi\oplus n}} \) and for each \( \I_{\xi\oplus 1}^* \in \Mod{\G_{\xi\oplus 1}} \) that we have
		\begin{equation*}
			\label{eq:loop_type_two_zeta1} \tag{\( \star 4 \)}
			\begin{aligned}
				\omegaMinFromCoverbase & \releqcK \omegafromCoverbase{\xi} & \I_{\xi\oplus n}^* & \releqcK \omegafromCoverbase{\xi}            & \I_\xi  & \releqcK \omegafromCoverbase{\xi}     \\
				&  & \omegafromCoverbase{\xi}            & \releqcK \I_{\xi\oplus 1}^* & \omegafromCoverbase{\xi} & \releqcK \I_\xi  \ .
			\end{aligned}
		\end{equation*}
		We obtain \( \omegaMinFromCoverbase \releqcK \omegafromCoverbase{\xi} \) from \Cref{lem:help}(b), because \( \omegaMinFromCoverbase,\omegafromCoverbase{\xi} \in \Mod{\coverbase} \) and   \(\omegaMinFromCoverbase\in \Mod{\K\circ\coverbase} \) holds.
		From \( \I_\xi , \omegafromCoverbase{\xi} \in {\min(\Mod{\G_{{\xi},{\xi\oplus 1}}},\releqcK)} \) we directly obtain \( \I_\xi  \releqcK \omegafromCoverbase{\xi} \) and \( \omegafromCoverbase{\xi}  \releqcK \I_\xi \).
		Compatibility of \( \releqc{\abst} \) with \( \circ \), together with the definitions of \( \G_\xi \) and Condition (2), yields the remaining statements of \eqref{eq:loop_type_two_zeta1}.
		
		Moreover, as next step, we show for each \( \G_\xi \in B \) and for each \( \I_{\xi\oplus n}^* \in \Mod{\G_{\xi\oplus n}} \) and for each \( \I_{\xi\oplus 1}^* \in \Mod{\G_{\xi\oplus 1}} \) the following holds:
		\begin{equation*}
			\label{eq:loop_type_two_zeta2} \tag{\( \star 5 \)}
			\begin{aligned}
				\I_{\xi\oplus n}^* & \relcK \I_{\xi}  & & \text{ if and only if } & \I_{\xi\oplus n}^* & \relcK \omegafromCoverbase{\xi} \\
				\I_{\xi} & \relcK \I_{\xi\oplus 1}^*  & & \text{ if and only if } & \omegafromCoverbase{\xi} & \relcK \I_{\xi\oplus 1}^* 
			\end{aligned}
		\end{equation*}
		Observe that \eqref{eq:loop_type_two_zeta2} holds,  otherwise, we would obtain a strict cycle of length 3.
		These strict cycles are directly obtainable from  \eqref{eq:loop_type_two_zeta1}:
		if \( \I_{\xi\oplus n}^* \relcK \I_{\xi} \) and  \(   \omegafromCoverbase{\xi} \releqcK \I_{\xi\oplus n}^*\), obtain the strict cycle \( \I_{\xi\oplus n}^* \relcK \I_{\xi} \releqcK  \omegafromCoverbase{\xi} \releqcK  \I_{\xi\oplus n}^* \) with a length of \( 3 \).
		For all other cases, we obtain analogously a strict cycle of length 3.
		
		Now let \( \ell_{\min}, \ell_{\mathrm{med}}, \ell_{\max} \) be integers with \( 0 \leq \ell_{\min} < \ell_{\mathrm{med}} < \ell_{\max} \leq n\) such that \( \G_{k\oplus \ell_{\min}},\allowbreak \G_{k\oplus \ell_{\mathrm{med}}} , \G_{k\oplus\max}  \in B \) and \( \ell_{\min} \) is the smallest number from \( \{0,\ldots,n\} \)  with \( \G_{k\oplus \ell_{\min}} \in B \), and \( \ell_{\max} \) is the greatest number from \( \{0,\ldots,n\} \)  with \( \G_{k\oplus\max} \in B \).
		For convenience, we will sometimes write \( \ell_{x} \) and \( \I_{x} \),  instead of \( \ell_{k\oplus \ell_x} \) and \( \I_{k\oplus \ell_x} \), respectively, for any \( x\in\{{\min},\mathrm{med},{\max}\} \).
		
		We now establish that replacing \( \omega_i \) in \( {\circlearrowright_{0n}} \) by \( \omegafromCoverbase{i} \) for some \( \G_i\in B \) yields again a strict cycle.
		Remember that each pair given in \eqref{eq:loop_type_two_zeta1} and \eqref{eq:loop_type_two_zeta2} is a non-detached pair.
		Because of this and because \( {\circlearrowright_{0n}} \) is a strict cycle of length \( n+1 \), we obtain from \eqref{eq:loop_type_two_zeta1} and \eqref{eq:loop_type_two_zeta2} that
		\( {\circlearrowright_{0n}}[\sigma] \) is also a strict cycle of length \( n+1 \) for each substitution \( \sigma \) with 
		\begin{equation*}
			\sigma \subseteq \{ \I_{\min} \mapsto \omegafromCoverbase{\min},\   \I_{\mathrm{med}} \mapsto \omegafromCoverbase{\mathrm{med}} ,\  \I_{\max} \mapsto \omegafromCoverbase{\max} \}  \ ,
		\end{equation*} i.e.,
		substituting each \( \I_{x} \) by \( \omegafromCoverbase{x} \) in \( \I_0,\I_1,\ldots,\I_n \), for some of \( x\in\{{\min}, \mathrm{med}, {\max}\} \), yields a strict cycle of length \( n+1 \).  
		
		We consider two cases, the case where \( \G_{k}\Cup\coverbase \) is inconsistent and the case where \( \G_{k}\Cup\coverbase \) is consistent.
		
		\medskip
		\noindent\emph{The case of \( \G_{k}\Cup\coverbase \) is inconsistent.} For this case we have that \( \G_k\notin B \) holds.
		Remember that by \eqref{eq:loop_type_two_f4} and \eqref{eq:loop_type_two_zeta1} the following holds:
		\begin{align*}
			\I_{k} & \releqcK  \omegaMinFromCoverbase               & \omegaMinFromCoverbase & \releqcK \omegafromCoverbase{\mathrm{med}} \\
			\omegaMinFromCoverbase & \releqcK \omegafromCoverbase{\min} & \omegaMinFromCoverbase & \releqcK \omegafromCoverbase{\max}
		\end{align*}
		We obtain that there exists a strict cycle with a length of at most \( n \) by using \Cref{lem:loop_comb} when setting \( \lambda=k \), \( a=k\oplus\ell_{\min} \), \( b=k\oplus\ell_{\mathrm{med}} \) and \( c=k\oplus\ell_{\max} \), using the strict cycle \( {\circlearrowright_{0n}}[\I_a \mapsto \omegafromCoverbase{p},\I_b \mapsto \omegafromCoverbase{\tau},\I_c \mapsto \omegafromCoverbase{q}] \).

		\medskip
		\noindent\emph{The case of \( \G_{k}\Cup\coverbase \) is consistent.} This case is equivalent to having \( \ell_{\min}=0 \), i.e., \( \G_{\min} =\G_k \in B \).
		Consequently, we have that \( \omegafromCoverbase{k} \in \Mod{\G_{k}\Cup\coverbase}\).
		From the definition of \( \G_k \), and from $\omegafromCoverbase{k},\omegaMinFromCoverbase\in\Mod{\G_{{k},{k\oplus1}}}$ with $\omegafromCoverbase{k}\in\Mod{\K\circ\G_{{k},{k\oplus1}}}$, and from compatibility and min-retractivity we also obtain $\omegaMinFromCoverbase,\omegafromCoverbase{k}\in\Mod{\K\circ\G_{{k},{k\oplus1}}}\cap\Mod{\coverbase} = {\min(\Mod{\G_{{k},{k\oplus 1}}},\releqcK) \cap\Mod{\coverbase}}$. 
		Consequently, all observations for \( \omegaMinFromCoverbase \) do also hold for \( \omegafromCoverbase{k} \); in particular, this applies to \eqref{eq:loop_type_two_f4}--\eqref{eq:loop_type_two_zeta2}.
		Thus, we assume \( \omegaMinFromCoverbase=\omegafromCoverbase{k} \) in the following.
		
		Together with \eqref{eq:loop_type_two_zeta1} and \eqref{eq:loop_type_two_zeta2} we can summarize as follows:
		\begin{equation*}\label{eq:loop_type_two_ellmin0} \tag{\( \star 6 \)}
			\begin{aligned}
				{\omegaMinFromCoverbase} & \releqcK  \omegafromCoverbase{\max}          & {\omegaMinFromCoverbase}           & \releqcK  \I_{k\oplus 1} & {\omegaMinFromCoverbase}   & \releqcK  \I_{k} \\
				{\omegaMinFromCoverbase} & \releqcK  \omegafromCoverbase{\mathrm{med}} & \I_{k\oplus n} & \releqcK   {\omegaMinFromCoverbase}          & \I_{k} & \releqcK   {\omegaMinFromCoverbase}
			\end{aligned}
		\end{equation*}
		
		We are striving for a contradiction by showing the existence of a strict cycle of length that is strictly smaller than \( n+1 \). 
		Therefore, we will make use of \Cref{lem:loop_comb}, whenever that is possible.
		We consider three cases in the following, depending on the values of \( \ell_{\mathrm{med}} \) and \( \ell_{\max} \). Recall that \( 1\leq \ell_{\mathrm{med}} < \ell_{\max} \leq n \) holds.        
		\begin{itemize}
			\item[]\hspace{-4.4ex}(\(  \ell_{\mathrm{med}} \neq 1 \)) Because of \eqref{eq:loop_type_two_ellmin0}, we can directly apply \Cref{lem:loop_comb} by setting \( \lambda=k \), \( a=k\oplus 1 \), \( b=k\oplus\ell_{\mathrm{med}} \) and \( c=k\oplus\ell_{\max} \), and by using the strict cycle \( {\circlearrowright_{0n}}[\I_\tau \mapsto \omegafromCoverbase{\tau},\allowbreak \I_q \mapsto \omegafromCoverbase{q}] \) for \Cref{lem:loop_comb}, which yields a strict cycle with a length of at most \( n \).
			
			\item[]\hspace{-4.4ex}(\(  \ell_{\max}\neq n \)) 
			We apply \Cref{lem:loop_comb} by setting \( \lambda=k \), \( a=k\oplus\ell_{\mathrm{med}} \), \( b=k\oplus\ell_{\max} \) and \( c=k\oplus n \), and by using the strict cycle \( {\circlearrowright_{0n}}[\I_p \mapsto \omegafromCoverbase{p},\I_\tau \mapsto \omegafromCoverbase{\tau}] \), which yields again  a strict cycle with a length of at most \( n \).
			
			\item[]\hspace{-4.4ex}(\(  \ell_{\mathrm{med}} = 1 \) and \(  \ell_{\max} = n \)) Because of \(  \ell_{\max} = n \), we have that \( \omegafromCoverbase{\max} \in \Mod{\G_{k\oplus n}} \).
			Together with \( {\omegaMinFromCoverbase} \in \Mod{\G_k} \), we obtain from \eqref{eq:loop_type_two_zeta1} that \( \omegafromCoverbase{\max} \releqcK  {\omegaMinFromCoverbase} \).
			From the min-retractivity of \( \releqcK \) and \( \G_{k\oplus n}\in B \), we obtain \( \omegafromCoverbase{\max} \in \min(\Mod{\coverbase},\releqcK) \), which implies \( \omegafromCoverbase{\max} \releqcK \omegafromCoverbase{\mathrm{med}} \).
			
			If \( \omegafromCoverbase{\max} \relcK \omegafromCoverbase{\mathrm{med}} \), then we obtain the strict cycle \( \omegafromCoverbase{\max} ,\allowbreak\omegafromCoverbase{\mathrm{med}},\allowbreak \I_{k \oplus (\ell_{\mathrm{med}}+1)},\allowbreak\ldots,\allowbreak\I_{k \oplus(\ell_{\max}+n)}, \omegafromCoverbase{\max}   \) which has a length of at most \( n \). 
			Due to the induction hypothesis there is no such strict cycle, and thus, \( {\omegafromCoverbase{\max} \relcK \omegafromCoverbase{\mathrm{med}}} \) is impossible.
			From totality of \( \releqcK \) we obtain that \(  \omegafromCoverbase{\mathrm{med}} \releqcK \omegafromCoverbase{\max} \) holds.                 
			If \( k\oplus \ell_{\max} =0 \), then we obtain the strict cycle  \( \omegafromCoverbase{\max} ,\allowbreak \I_{k}, \omegafromCoverbase{\mathrm{med}},\allowbreak \omegafromCoverbase{\max}\) of length 3.         
			If \( k =0 \), then we obtain the strict cycle  \( \I_{k}, \omegafromCoverbase{\mathrm{med}},\allowbreak \omegafromCoverbase{\max},\I_{k}\) of length 3.         
			If none of the prior cases  applies, then \( {\circlearrowright} =  \I_0,\I_1,\allowbreak\ldots,\omegafromCoverbase{\max},\allowbreak\omegafromCoverbase{\mathrm{med}},\allowbreak\ldots,\I_0 \) is a strict cycle.
			Note that \( \I_k \) is not part of the strict cycle \( {\circlearrowright} \), and consequently the length of \( {\circlearrowright} \) is bounded by \( n \).
		\end{itemize}
		We obtain a contradiction in every case, which shows that Condition (3) of \Cref{def:critical_loop} is satisfied.
	\end{itemize}
	In summary, assuming that each \( (\I_i,\I_{i\oplus1}) \) is not a detached pair leads to formation of a \criticalloop\ by \( \G_{0,1},\G_{1,2},\ldots,\G_{n,0} \);
	contradicting the \criticalloop-freeness of \( \mathbb{B} \). Consequently, at least one \( (\I_i,\I_{i\oplus1}) \) has to be a detached pair. This completes the proof of \Cref{lem:loop_type_two}.
\end{proof}

By employing \Cref{lem:loop_type_two} we show now that transformation of \( \relationstepdashK\) to \( \relationstepdashdashK \) by taking the transitive closure only adds detached pairs.

\begin{lemma}\label{lem:extendsnonondetached}
Let $\mathbb{B}=(\MC{L},{\Omega},\models,\Bases,\Cup)$ \sauerwald{be a loop-free base logic}, let \( \K \in \Bases  \) be a base, and let \( \circ \) be a {\basechange} operator for \( \mathbb{B} \) that satisfies \postulate{G1}--\postulate{G3}, \postulate{G5}, and \postulate{G6}. The following holds:
\begin{equation*}
{\relationstepdashK}\subseteq  {\relationstepdashdashK} \subseteq {\releqcK} 
\end{equation*}
\end{lemma}
\begin{proof}
By construction of \( {\relationstepdashdashK} \), we have \( {\relationstepdashK} \subseteq  {\relationstepdashdashK} \), and by construction of \( {\relationstepdashK} \) we have \( {\relationstepdashK} \subseteq {\releqcK} \). 
To show that \( {\relationstepdashdashK} \subseteq {\releqcK}  \) holds, we assume the contrary, i.e., there exists a pair \( (\I_{1},\I_{0}) \in  {\relationstepdashdashK} \) such that \( \I_1 \not\releqcK \I_0  \).
From \( {\relationstepdashK} \subseteq {\releqcK} \) we obtain \( \I_1 \not\relationstepdashK \I_0  \) and
because \( {\releqcK} \) is a total relation, we have that \( \I_0 \relcK \I_1  \) holds.
By the definition of transitive closure (cf. \Cref{sec:relations}), 
there exists \( \I_{2},\ldots,\I_{n} \in \Omega \), for some \( n\in\mathbb{N} \), such that \( \I_1 \relationstepdashK\I_{2} \) and \( \I_{n} \relationstepdashK\I_{0} \) and \( \I_{i} \relationstepdashK\I_{i+1} \) for each \( i\in\{2,\ldots,n-1\} \). From \( {\relationstepdashK} \subseteq {\releqcK} \), we obtain \( \I_0 \relcK \I_1 \releqcK \I_2 \ldots\releqcK \I_n \releqcK \I_0 \).
We obtain a contradiction, because \( {\relationstepdashK} \) does not contain any detached pairs, but due to \Cref{lem:loop_type_two} there is some \( i\in\{2,\ldots,n-1\} \) such that \( (\I_{i},\I_{i+1}) \) is a detached pair. Consequently, we obtain \( {\relationstepdashdashK} \subseteq {\releqcK}  \).  \qedhere

\end{proof}

Combining \Cref{lem:gettransitive_step1} and \Cref{lem:extendsnonondetached} yields that \( \relationstepdashdashK \) is a (possibly non-total) preorder with useful properties. 
In particular, the sets of minimal models for every base \( \G\in\Bases \) coincide for \( \relationstepdashK \) and \( \relationstepdashdashK \).
\begin{lemma}\label{lem:gettransitive_step2}
Let $\mathbb{B}=(\MC{L},{\Omega},\models,\Bases,\Cup)$ be a base logic which does not admit a {\criticalloop},  let \( \K\in\Bases \) and let \( \circ \) be a base change operator for \( \mathbb{B} \) which satisfies \postulate{G1}--\postulate{G3}, \postulate{G5}, and \postulate{G6}.
Then \( {\relationstepdashdashK} \) is a min-complete preorder and for any $\G \in \Bases$ holds $\min(\Mod{\G},\relationstepdashdashK)=\min(\Mod{\G},\relationstepdashK)$.
\end{lemma}	
\begin{proof}
Because of \Cref{lem:extendsnonondetached} we have $\min(\Mod{\G},\relationstepdashdashK)=\min(\Mod{\G},\relationstepdashK)$  for any $\G \in \Bases$, since \( {\relationstepdashdashK} \setminus {\setAllDetached} = {\releqcK} \setminus {\setAllDetached} \).
Recall that by \Cref{lem:gettransitive_step1} we have that \( \relationstepdashK\) is min-complete and reflexive.
Consequently, the transitive closure \( \relationstepdashdashK \) of \( \relationstepdashK\)  is a preorder.
Moreover, as in the proof of \Cref{lem:gettransitive_step1}, from $\min(\Mod{\G},\relationstepdashdashK)=\min(\Mod{\G},\relationstepdashK)$ we obtain that min-completeness carries over from \( \relationstepdashK\) to \( \relationstepdashdashK \).
\end{proof}

\paragraph{\textbf{Step III: Linearizing.}}
As last step, we extend \( \relationstepdashdashK \) to a total relation without losing transitivity.
In order to obtain totality, we make use of the following result from Hansson (see also \Cref{sec:enc_operators}). Note that this theorem requires the \emph{axiom of choice}.

\medskip
\noindent \textsc{\Cref{thm:tpo_extension} [{{Hansson} \cite[Lemma 3]{KS_Hansson1968}}].}
\textit{
    Assume the \emph{axiom of choice}.
    For every preorder \( {\leq} \) on a set \( X \) there exists a total preorder \( {\leq^\mathrm{lin}} \) on \( X \) such that
\begin{itemize}
\item if \( x \leq y  \), then \( x \leq^\mathrm{lin} y \), and
\item if \( x \leq y  \) and \( y \not\leq x  \), then \( x \leq^\mathrm{lin} y  \) and \( y \not\leq^\mathrm{lin} x  \).
\end{itemize}}      
\pagebreak[3]

As stated in \Cref{lem:gettransitive_step2}, the relation $\relationstepdashdashK$ is a preorder.
Thus, we can safely apply \Cref{thm:tpo_extension} to obtain $\rreleqcK$ from $\relationstepdashdashK$ through extension, i.e., \( {\rreleqcK}=(TC({\releqcK}\setminus\setAllDetached))^\mathrm{lin} \). The resulting relation $\rreleqcK$ is then a total preorder, while it still coincides with $\relationstepdashdashK$ regarding the relevant properties. 
Combining \Cref{thm:tpo_extension}, \Cref{lem:gettransitive_step1} and \Cref{lem:gettransitive_step2} we obtain the desired result.

\begin{proposition}\label{prop:gettransitive}
If	$\mathbb{B}$ \sauerwald{is a loop-free base logic}, then, for any given \basechange\ operator $ \circ $ for $\mathbb{B}$ satisfying \postulate{G1}--\postulate{G3}, \postulate{G5}, and \postulate{G6},
the mapping $\rreleqc{\abst}: \K \mapsto {\rreleqcK}$ is a min-friendly quasi-faithful preorder assignment compatible with $\circ$.
\end{proposition}
\begin{proof}
From \Cref{lem:gettransitive_step2} we obtain that $\relationstepdashdash{\abst}$ is a min-complete preorder assignment.
Application of \Cref{thm:tpo_extension} yields a total preorder $\rreleqcK$.
Observe that linearization by \Cref{thm:tpo_extension} retains strict relations, i.e., if \( \I_1  \relationstepdashdashK \I_2 \) and \( \I_2  \not\relationstepdashdashK \I_1 \), then \( \I_1 \rreleqcK \I_2  \) and \( \I_2 \not\rreleqcK \I_1 \). 
Thus, we have \( \I_1 \in \min(\Mod{\G},\relationstepdashdashK) \) and only if \( \I_1 \in \min(\Mod{\G},\rreleqcK) \), which yields \( \min(\Mod{\G},\relationstepdashdashK)=\min(\Mod{\G},\rreleqcK) \) for each base \( \G \in \mathbb{B} \).
Consequently, min-completeness carries over from \( \relationstepdashdashK \) to \( \rreleqcK \).
Moreover, by \Cref{lem:gettransitive_step1} and \Cref{lem:gettransitive_step2} we obtain \( \min(\Mod{\G},\rreleqcK)=\min(\Mod{\G},\releqcK) \) for each base \( \G \) of \( \mathbb{B} \).
As every \( \rreleqcK \) is transitive and total, we obtain that \( \rreleqcK \) is min-retractive and thus,  $\rreleqc{\abst}$ is a min-friendly assignment.
Because \( \releqc{\abst} \) is a quasi-faithful assignment which is compatible with \( \circ \) and we have \( \min(\Mod{\G},\rreleqcK)=\min(\Mod{\G},\releqcK) \) for each \( \K\in\Bases \), we also obtain that \( \rreleqc{\abst} \) is a quasi-faithful assignment which is compatible with \( \circ \).
\end{proof}

This completes the argument regarding the correspondence between the absence of \criticalloop s and total-preorder-representability, by establishing that the former is also sufficient for the latter. Obviously, \Cref{thm:when_tranisitive} (I) is a direct consequence of \Cref{prop:gettransitive}.

\clearpage
\section{Proof for Theorem \ref{thm:acyctpo}}\label{sec:app_acyc}

\sauerwald{In this section, we present a  proof for \Cref{thm:acyctpo} from \Cref{sec:acycresult}.
We start with repeating the theorem.

\medskip
\noindent \textsc{\Cref{thm:acyctpo}.}
    \textit{Let $\mathbb{B} = (\MC{L},\Omega,\models,\Bases,\Cup)$ be a base logic and let $ \circ $ be a change operator for $ \mathbb{B} $ that satisfies \postulate{G1}--\postulate{G3}, \postulate{G5}, and \postulate{G6}.
    Then $ \circ $ is acyclic if and only if $ \circ $ is total-preorder-representable.}

\medskip
We will show the ``if''-part of \Cref{thm:acyctpo} in \Cref{sec:tpo2acyc}, that is acyclicity of {\basechange} operators, i.e., {\basechange} operators that satisfy \postulate{Acyc}, implies total-preorder-representability. 
In \Cref{sec:acyc2tpo} we show the ``only if''-part of \Cref{thm:acyctpo}, that is  total-preorder-representability implies  acyclicity of {\basechange} operators.}

\subsection{From Total-Preorder-Representability to Acyclicity}\label{sec:tpo2acyc}

\sauerwald{We show that total-preorder-representability implies satisfaction of \postulate{Acyc}.}

\begin{proposition}\label{prop:acyc_tpo}
	Let $\mathbb{B} = (\MC{L},\Omega,\models,\Bases,\Cup)$ be a base logic and let $ \circ $ be a change operator for $ \mathbb{B} $ that satisfies \postulate{G1}--\postulate{G3}, \postulate{G5}, and \postulate{G6}.
	If $ \circ $ is total-preorder-representable, then $ \circ $ is \emph{acyclic}.
\end{proposition}
\begin{proof} 
		Let $\K\in \Bases$ and let $\G_0,\ldots,\G_n \in \Bases$ with $\Mod{\G_i \Cup (\K\circ \G_{i\oplus 1})}\neq\emptyset$ for each $0\leq i < n$. %
        We will show that $\Mod{\G_0 \Cup (\K\circ \G_n)}\neq\emptyset$ holds.
		\Cref{thm:rep2} yields that there exists a min-complete quasi-faithful preorder assignment $\releq{\abst}$ compatible with~$ \circ $.
        Recall that this implies that \( \releqK \) is transitive  and that we have \( \Mod{\K \circ \G} = \min(\Mod{\G},\releqK) \) for all \( \G\in\Bases \) .
		From the latter and from $\Mod{\G_n \Cup (\K \circ \G_{0})}\neq\emptyset$, we obtain that $ \min(\Mod{\G_{0}},\releqK) \cap \Mod{\G_n}\neq \emptyset $ holds. %
		Analogously, for each $ 0 \leq i < n $, we obtain that $ \min(\Mod{\G_i},\releqK) \cap \Mod{\G_{i\oplus 1}} \neq \emptyset $ holds.
        In the following, we denote with $ \omega_{i\oplus 1,i} $ an element of $ \min(\Mod{\G_i},\releqK)\cap \Mod{\G_{i\oplus 1}} $, for each $ 0 \leq i \leq n $. 
		Note that $ \omega_{i\oplus 1,i} $ and $ \omega_{i\oplus 2,i\oplus 1} $ are elements of $ \Mod{\G_{i\oplus 1}} $. 
		Since $ \omega_{i\oplus 2,i\oplus 1} \in \min(\Mod{\G_{i\oplus 1}},\releqK) $ holds, we obtain $ \omega_{i\oplus 2,i\oplus 1} \releqK \omega_{i\oplus 1,i} $.
		By subsequent application of transitivity of $ \releqK $, we obtain that $ \omega_{i\oplus 1,i} \releqK \omega_{j\oplus 1,j} $ holds for every $ 0\leq i,j \leq n $.
		Consequently, we obtain $ \omega_{0,n} \releqK \omega_{n,n-1}  $. 
		From this,  $ \omega_{0,n} , \omega_{n,n-1} \in \Mod{\G_n} $ and $ \omega_{n,n-1} \in \min(\Mod{\G_n},\releqcK) $, we obtain that $  \omega_{0,n} \in \min(\Mod{\G_n},\releqcK) $ holds.
		Because we have $ \omega_{0,n}\in\Mod{\G_{0}} $, we obtain $ \omega_{0,n} \in \min(\Mod{\G_n},\releqcK) \cap \Mod{\G_{0}}  $.
		From the compatibility of $ \releq{\abst} $ with $ \circ $, we finally obtain $\Mod{\G_{0} \Cup (\K\circ \G_n)}\neq\emptyset$.\qedhere
\end{proof}

\subsection{From Acyclicity to Total-Preorder-Representability}\label{sec:acyc2tpo}

We will now show that acyclicity is also sufficient. To this end, we need to argue that any {\basechange} operator $\circ$ that satisfies \postulate{Acyc} and  \postulate{G1}--\postulate{G3}, \postulate{G5}, and \postulate{G6} gives rise to a compatible min-complete quasi-faithful preorder assignment $\rreleqc{(.)}$.
We will obtain $\rreleqc{(.)}$ via a step-wise transformation of the assignment $\releqc{(.)}$ from \Cref{def:relation_new} (defined in \Cref{sec:transformation2tpo}).
The transformation steps are the same three steps as in the proof\footnote{As the overall proof of \Cref{thm:when_tranisitive} is more involved, we decided to have a complete and coherent presentation proof of \Cref{thm:when_tranisitive} in \Cref{sec:app_loops_new}. Consequently, to avoid duplication, some concepts and notions used here are defined and motivated in \Cref{sec:app_loops_new}.} of \Cref{thm:when_tranisitive}, respectively \Cref{prop:encodedown}, fully given in \Cref{sec:app_loops_new}, but the proof of correctness is different. 
We will perform the following transformation steps:
\begin{align*}
     {\relationstepdashK} & = {\releqcK\setminus\setAllDetached} \tag{Step I: remove detached pairs}\\
    {\relationstepdashdashK} & = TC({\relationstepdashK}) =TC({{\releqcK}\setminus\setAllDetached}) \tag{Step II: transitive closure} \\
   {\rreleqcK} & =   (\relationstepdashdashK)^\mathrm{lin}   =   (TC({\relationstepdashK}))^\mathrm{lin} =  (TC({\releqcK}\setminus\setAllDetached))^\mathrm{lin} \tag{Step III: linearizing}
\end{align*}
The transformation starts with a min-complete quasi-faithful assignment compatible with $\circ$ by \Cref{prop:agm_withoutg4}, whereby \( \circ \)  satisfies \postulate{Acyc} and  \postulate{G1}--\postulate{G3}, \postulate{G5}, and \postulate{G6}.
In the first step all detached pairs from \( \releqcK \) are removed (cf. \Cref{def:detached}), resulting in \( {\relationstepdashK} = {\releqcK\setminus\setAllDetached} \).
For the second step, we take the transitive closure of \( \relationstepdashK \), yielding \( {\relationstepdashdashK}=TC({\relationstepdashK})=TC({{\releqcK}\setminus\setAllDetached}) \).
Finally,  we are obtaining $\rreleqc{(.)}$ by ``linearizing`` \( \relationstepdashK \) to a total preorder, i.e., for each \( \K \) we have \( {\rreleqcK}=(TC({\releqcK}\setminus\setAllDetached))^\mathrm{lin} \).

In the following, we show that if \( \circ \) is an acyclic {\basechange} operator, then the transformation above yields an min-complete quasi-faithful preorder assignment compatible with \( \circ \).
We start with the central observation that for acyclic {\basechange} operators, all situations in which \( \releqcK \) is not transitive or form a cycle involve detached pairs.
\pagebreak[3]
\begin{lemma}\label{lem:loop_detach_acyc_n}
    Let $\mathbb{B}=(\MC{L},{\Omega},\models,\Bases,\Cup)$ a base logic, let \( \K \in \Bases  \) be a base, and let \( \circ \) be an acyclic \basechange\ operator for \( \mathbb{B} \) that satisfies \postulate{G1}--\postulate{G3}, \postulate{G5}, and \postulate{G6}.
    If there are three or more interpretations \( \I_0, \ldots ,\I_n \in \Omega  \), i.e., \( n \geq 2 \), such that
    \begin{itemize}\setlength{\itemsep}{0pt}
        \item[]\hspace{-4.4ex}(a)~~\( \I_0 \relcK \I_1 \),
        \item[]\hspace{-4.4ex}(b)~~\(  \I_i \releqcK \I_{i\oplus 1} \) for all \( i\in \{ 1,\ldots,n \} \), where \( \oplus \) is addition \( {\mathrm{mod} (n+1)} \), 
    \end{itemize}
    then there is some \( i\in\{1,\ldots,n\} \) such that \( (\I_i,\I_{i\oplus1}) \) is a detached pair.
\end{lemma}
\begin{proof}
    The proof will be by contradiction.      With \( \oplus \) we  denote addition \( {\mathrm{mod} (n+1)} \).
    Let \( \I_0, \ldots ,\I_n \in \Omega  \) such that Condition (a) and Condition (b) of \Cref{lem:loop_detach_acyc_n} are satisfied.
    Towards a contradiction, we assume that for each \( i\in\{1,\ldots,n\} \) the pair \( (\I_i,\I_{i\oplus1}) \) is not detached from $ \circ $ in $ \K $. 
    By employing these assumptions, Condition (a) and \Cref{def:relation_new} together imply that for each \( i\in\{1,\ldots,n\} \) exists a base \( \G_{{i},{i\oplus1}} \in \Bases \) such that
    \begin{equation}\label{eq:loop_type_two_e2acyc}
        \I_{i},\I_{i\oplus1} \models \G_{{i},{i\oplus1}} \text{ and } \I_{i} \models \K \circ \G_{{i},{i\oplus1}} \tag{\#1}
    \end{equation}
    holds.
    Moreover, by \( \I_0 \relcK \I_1 \) from Condition (a) and \( \I_1 \not\models\K \) and \Cref{lem:help}(a), there exists a base \( \G_{0,1} \in \Bases \) such that the following holds:
    \begin{equation}\label{eq:loop_type_two_e1acyc}
        \I_{0},\I_{1}  \models \G_{{0},{1}} \text{ and } \I_0  \models \K \circ \G_{{0},{1}} \text{ and } \I_1  \not\models \K \circ \G_{{0},{1}} . \tag{\#2}
    \end{equation}    
    In summary, there are bases \( \G_{0}=\G_{{1},{2}} \), ..., \( \G_{n}=\G_{{0},{1}} \) such that \( \Mod{\G_{i}\Cup(\K\circ\G_{i \oplus 1})} \neq \emptyset \) holds for \( 0 \leq i \leq n \). %
    From this last observation and because \( \circ \) is an acyclic \basechange\ operator for \( \mathbb{B} \) we obtain \( \Mod{\G_{0}\Cup(\K\circ\G_{n})}  = \Mod{\G_{{1},{2}}\Cup(\K\circ\G_{{0},{1}})} \neq \emptyset \).
        In the following let \( \I_x  \) denote an interpretation from \( \Mod{\G_{{1},{2}} \Cup (\K\circ\G_{{0},{1}})} \). 
        By using \Cref{thm:rep1}  (and the proof thereof), we obtain \( \I_x \in \min(\Mod{\G_{{0},{1}}},\releqcK) = \Mod{\K\circ \G_{{0},{1}}} \) from \( \I_x \in \Mod{\K\circ\G_{{0},{1}}} \).
        Moreover, by employing \Cref{thm:rep1} again, we obtain \( \I_{1} \in {\min(\Mod{\G_{{1},{2}}},\releqcK)} = \Mod{\K\circ \G_{{1},{2}}} \) from \eqref{eq:loop_type_two_e2acyc}. 
        This together with \( \I_{x},\I_{1} \models \G_{{1},{2}} \) yields \( \I_{1} \releqcK \I_{x} \).
        Recall that \( \releqcK \) is a min-retractive relation and consequently, we obtain \( \I_{1} \in \min(\Mod{\G_{{0},{1}}},\releqcK) = \Mod{\K\circ \G_{{0},{1}}} \) from \( \I_x \in \min(\Mod{\G_{{0},{1}}},\releqcK) \) and \( \I_{1} \releqcK \I_{x} \) and \( \I_{x},\I_{1} \models \G_{{0},{1}} \).
        This last observation yields a contraction to \( \I_1 \not\models \K\circ\G_{{0},{1}}  \) from \eqref{eq:loop_type_two_e1acyc}.
\end{proof}

We show that acyclicity implies total-preorder-representability.

\begin{proposition}\label{prop:tpo_acyc}
    Let $\mathbb{B} = (\MC{L},\Omega,\models,\Bases,\Cup)$ be a base logic and let $ \circ $ be a change operator for $ \mathbb{B} $ that satisfies \postulate{G1}--\postulate{G3}, \postulate{G5}, and \postulate{G6}.
    If $ \circ $ is acyclic, then $ \circ $ is total-preorder-representable.
\end{proposition}
\begin{proof}
    As shown (cf. \Cref{prop:agm_withoutg4}) we have that \( \releqc{\abst} \) is a min-expressible min-friendly quasi-faithful assignment that is compatible with \( \circ \).
    In the following will show that \( \rreleqc{\abst} \) is a min-expressible min-friendly quasi-faithful preorder assignment that is compatible with \( \circ \).
    
\paragraph{\textbf{Step I: Removing detached pairs.}} As shown in \Cref{lem:gettransitive_step1}, we have that $\relationstepdashK$ is min-complete and reflexive. Moreover, \Cref{lem:gettransitive_step1} shows that the transformation of \( \releqcK \) to \( \relationstepdashK \) preserves minima with respect to bases, i.e., we have ${\min(\Mod{\G},\relationstepdashK)}={\min(\Mod{\G},\releqcK)}$ for each \( \G\in\Bases \).

\paragraph{\textbf{Step II: Taking the transitive closure.}}
We show that  \( {\relationstepdashdashK} = TC({\relationstepdashK}) \) is a min-complete preorder and for any $\G \in \Bases$ holds $\min(\Mod{\G},\relationstepdashdashK)=\min(\Mod{\G},\relationstepdashK)$.
 As intermediate result, we show that the following holds:
\begin{equation}
    {\relationstepdashK}\subseteq  {\relationstepdashdashK} \subseteq {\releqcK} \label{tag:acycstep2include} \tag{\#1}
\end{equation}
By construction of \( {\relationstepdashdashK} \), we have \( {\relationstepdashK} \subseteq  {\relationstepdashdashK} \), and by construction of \( {\relationstepdashK} \) we have \( {\relationstepdashK} \subseteq {\releqcK} \). 
To show that \( {\relationstepdashdashK} \subseteq {\releqcK}  \) holds we use nearly same proof as in \Cref{lem:extendsnonondetached}, but for the sake of completeness we give the full proof. Towards a contradiction assume the contrary, i.e., there exists a pair \( (\I_{1},\I_{0}) \in  {\relationstepdashdashK} \) such that \( \I_1 \not\releqcK \I_0  \).
From \( {\relationstepdashK} \subseteq {\releqcK} \) we obtain \( \I_1 \not\relationstepdashK \I_0  \) and
because \( {\releqcK} \) is a total relation, we have that \( \I_0 \relcK \I_1  \) holds.
By the definition of transitive closure (cf. \Cref{sec:relations}), 
there exists \( \I_{2},\ldots,\I_{n} \in \Omega \), for some \( n\in\mathbb{N} \), such that \( \I_1 \relationstepdashK\I_{2} \) and \( \I_{n} \relationstepdashK\I_{0} \) and \( \I_{i} \relationstepdashK\I_{i+1} \) for each \( i\in\{2,\ldots,n-1\} \). From \( {\relationstepdashK} \subseteq {\releqcK} \), we obtain \( \I_0 \relcK \I_1 \releqcK \I_2 \ldots\releqcK \I_n \releqcK \I_0 \).
We obtain a contradiction, because \( {\relationstepdashK} \) does not contain any detached pairs, but due to \Cref{lem:loop_detach_acyc_n} there is some \( i\in\{2,\ldots,n-1\} \) such that \( (\I_{i},\I_{i+1}) \) is a detached pair. Consequently, we obtain \( {\relationstepdashdashK} \subseteq {\releqcK}  \).  This completes the proof of claim \eqref{tag:acycstep2include} .

Now observe that \eqref{tag:acycstep2include} implies $\min(\Mod{\G},\relationstepdashdashK)=\min(\Mod{\G},\relationstepdashK)$ for any $\G \in \Bases$ holds, since \( {\relationstepdashdashK} \setminus {\setAllDetached} = {\releqcK} \setminus {\setAllDetached} \).
Consequently, as minima are preserved, we obtain that min-completeness carries over from \( \relationstepdashK\) to \( \relationstepdashdashK \).
Moreover, recall that we have shown in \textbf{Step I} that \( \relationstepdashK\) is a  min-complete and reflexive relation. 
Thus, the transitive closure \( \relationstepdashdashK \) of \( \relationstepdashK\)  is a preorder.

\paragraph{\textbf{Step III: Linearizing.}} 
Given \textbf{Step I} and \textbf{Step II}, proving that $\rreleqc{\abst} $ is a min-friendly quasi-faithful preorder assignment compatible with $\circ$ is analogue to proving \Cref{prop:gettransitive}.
We provide the proof for the sake of completeness.

As shown in \textbf{Step II}, $\relationstepdashdash{\abst}$ is a min-complete preorder assignment and thus, application of \Cref{thm:tpo_extension} yields a total preorder $\rreleqcK$.
Observe that linearization by \Cref{thm:tpo_extension} retains strict relations, i.e., if \( \I_1  \relationstepdashdashK \I_2 \) and \( \I_2  \not\relationstepdashdashK \I_1 \), then \( \I_1 \rreleqcK \I_2  \) and \( \I_2 \not\rreleqcK \I_1 \). 
Thus, we have \( \I_1 \in \min(\Mod{\G},\relationstepdashdashK) \) and only if \( \I_1 \in \min(\Mod{\G},\rreleqcK) \), which yields \( \min(\Mod{\G},\relationstepdashdashK)=\min(\Mod{\G},\rreleqcK) \) for each base \( \G \in \mathbb{B} \).
Consequently, min-completeness carries over from \( \relationstepdashdashK \) to \( \rreleqcK \).
Moreover, from \textbf{Step I} and \textbf{Step II} we obtain \( \min(\Mod{\G},\rreleqcK)=\min(\Mod{\G},\releqcK) \) for each base \( \G \in \Bases \).
As every \( \rreleqcK \) is transitive and total, we obtain that \( \rreleqcK \) is min-retractive and thus,  $\rreleqc{\abst}$ is a min-friendly assignment.
Because \( \releqc{\abst} \) is a quasi-faithful assignment which is compatible with \( \circ \) and we have \( \min(\Mod{\G},\rreleqcK)=\min(\Mod{\G},\releqcK) \) for each \( \K\in\Bases \), we also obtain that \( \rreleqc{\abst} \) is a quasi-faithful assignment which is compatible with \( \circ \).
\end{proof}
This completes the argument regarding the correspondence between acyclicity and total-preorder-representability, by establishing that the former is also sufficient for the latter. 
\Cref{thm:acyctpo} is a direct consequence of \Cref{prop:acyc_tpo} and \Cref{prop:tpo_acyc}.

\end{document}